\documentclass[authordate,onecolumn,noenvironments]{miri-tech-article}

\usepackage[T1]{fontenc}                
\usepackage{lmodern}                    
\usepackage{miritools}                  
\usepackage{stmaryrd}                   
\SetSymbolFont{stmry}{bold}{U}{stmry}{m}{n}  
\usepackage{mleftright}                 
\usepackage{dsfont}                     
\usepackage{comment}                    
\usepackage{float}                      
\usepackage{thmtools}                   
\usepackage{thm-restate}                
\usepackage{bm}                         
\usepackage{booktabs}                   
\usepackage[framemethod=tikz]{mdframed} 

\floatstyle{plaintop}
\restylefloat{table}
\newtheorem{theorem}{Theorem}[subsection]
\newtheorem{definition}[theorem]{Definition}

\newenvironment{key}{\begin{mdframed}[innertopmargin=-.3\baselineskip]}{\end{mdframed}}
\newmdtheoremenv[innertopmargin=-.3\baselineskip, nobreak=true]{keydef}[theorem]{Definition}

\numberwithin{equation}{subsection}
\newcounter{Hequation}

\makeatletter
\g@addto@macro\equation{\stepcounter{Hequation}}
\makeatother

\newtheorem{convention}[theorem]{Convention}
\newtheorem{desideratum}{Desideratum}
\newtheorem{lemma}[theorem]{Lemma}
\newtheorem{proposition}[theorem]{Proposition}
\newtheorem{corollary}[theorem]{Corollary}
\newtheorem{defprop}[theorem]{Definition/Proposition}
\newtheorem{defalg}[theorem]{Definition/Algorithm}
\newenvironment{sketch}[1][]{\noindent\emph{Proof sketch.}\proofin{#1}\begin{quote}}{\end{quote}}

\newcommand{\Des}[1]{Desideratum~\ref{des:#1}}

\renewcommand{\implies}{\rightarrow}
\renewcommand{\iff}{\leftrightarrow}
\newcommand{\Implies}{\Rightarrow}

\newcommand\pfun{\mathrel{\ooalign{\hfil$\mapstochar\mkern5mu$\hfil\cr$\to$\cr}}}

\newcommand{\fin}{\operatorname{Fin}}
\newcommand{\quot}[1]{``#1"}
\newcommand{\enc}[1]{{\underline{#1}}}
\newcommand{\seq}[1]{{\overline{#1}}}

\newcommand{\poly}{\operatorname{poly}}
\newcommand{\any}{-}
\newcommand{\affconst}{c}
\newcommand{\fconst}{1}

\newcommand{\Ack}{\operatorname{Ack}}

\newcommand{\exf}{\xi}

\newcommand{\rank}{\operatorname{rank}}
\newcommand{\cash}{c}
\newcommand{\runtime}{r}
\newcommand{\LUV}[1]{#1}
\newcommand{\LUVs}{\mathcal{U}}

\newcommand{\prob}{p}
\newcommand{\pseudo}{p}
\newcommand{\price}{p}
\newcommand{\mm}{m}
\newcommand{\nn}{n}
\newcommand{\nN}{N}
\newcommand{\fuz}{w}
\newcommand{\deff}{f}

\newcommand{\phis}{\seq{\phi}}
\newcommand{\psis}{\seq{\psi}}

\newcommand{\deltas}{\seq{\delta}}
\newcommand{\varepsilons}{\seq{{\varepsilon}}}
\newcommand{\alphas}{\seq{{\alpha}}}

\newcommand{\probs}{\seq{{\prob}}}

\newcommand{\ec}[1][\ ]{e.c.#1}
\newcommand{\pgenable}{\ensuremath{\MP}\text{-generable}\xspace}
\newcommand{\representscomputations}{can represent computable functions}

\newcommand{\li}{logical inductor}
\newcommand{\LItitle}{Logical Inductor}
\newcommand{\lic}{logical induction criterion}

\newcommand{\LICtitle}{Logical Induction Criterion}

\newcommand{\rparenthetical}[1]{\hfill\upshape{{\small (#1)}}} 
\newcommand{\proofin}[1]{\rparenthetical{Proof in:~#1.}}
\newcommand{\tinysection}[1]{\medskip\noindent\textbf{#1.}}
\newcommand{\tinysectionend}[1]{\medskip\noindent}
\newcommand{\MainTextOnly}[1]{#1}

\newcommand{\afffeat}{\mathcal{A\-F}}

\newcommand{\feature}{\alpha} 
\newcommand{\features}{\mathcal{F}} 
\newcommand{\exfeatures}{\mathcal{E\!F}} 

\newcommand{\aff}{A}
\newcommand{\affs}{\seq\aff}
\newcommand{\BCS}[1][\MP]{\mathcal{B\-C\-S}(#1)} 
\newcommand{\affluv}{B} 
\newcommand{\BLCS}[1][\MP]{\mathcal{B\-L\-C\-S}(#1)} 
\newcommand{\luvval}{\mathbb{U}}
\newcommand{\trade}{T}%

\newcommand{\Trader}{\seq{\trade}} 
\newcommand{\strat}{T} 
\newcommand{\hold}{H}
\newcommand{\EF}{\exfeatures}
\newcommand{\alta}{\mathrm{Ex}} %

\newcommand{\gens}{\dagger}
\newcommand{\gen}[1]{{#1^\gens}} 
  \newcommand{\magn}[1]{{\|#1\|_1}} 
  \newcommand{\magnit}[1]{{\|#1\|_{\rm mg}}} 

\newcommand{\append}{,}

\newcommand{\supertrade}{F}

\newcommand{\adj}{{{\operatorname{fix}}}}
\newcommand{\Valuations}{\mathcal{V}}

\newcommand{\LIA}{\text{\upshape{\texttt{LIA}}}\xspace}
\newcommand{\marketmaker}{\text{\upshape{\texttt{MarketMaker}}}\xspace}
\newcommand{\budgeter}{\text{\upshape{\texttt{Budgeter}}}\xspace}
\newcommand{\tradingfirm}{\text{\upshape{\texttt{TradingFirm}}}\xspace}

\newcommand{\Bayesian}{\mathrm{Pr}}
\newcommand{\BelState}{\mathbb{P}}
\newcommand{\BelSeq}{\seq{\BelState}}

\newcommand{\Valuation}{\mathbb{V}}
\renewcommand{\UU}{\mathbb{U}} 
\newcommand{\Market}{\seq{\Pricing}}
\newcommand{\Pricing}{\mathbb{P}}
\newcommand{\pf}[2][\nn]{{#2}^{*#1}}
\newcommand{\pt}{\BelState}
\newcommand{\MP}{\seq{\pt}}
\newcommand{\ftn}{(\MP)} 

\newcommand{\dt}{D}
\newcommand{\DP}{\seq{\dt}}
\newcommand{\pcworlds}{\mathcal{P\-C}}
\newcommand{\cworlds}{\mathcal{P\-C}} 

\newcommand{\World}{\mathbb{W}}
\newcommand{\Worlds}{\mathcal{W}}
\newcommand{\Theory}{\Gamma}
\newcommand{\Lang}{\mathcal{L}}
\newcommand{\Sentences}{\mathcal{S}}

\newcommand{\piasev}{\mathit{\operatorname{\pi Aeq7}}}

\newcommand{\zfseq}{\mathit{\operatorname{ZFCaxioms}}}
\newcommand{\clusters}{\mathit{\operatorname{clusters}}}
\newcommand{\prgzot}{\mathit{\operatorname{prg012}}}
\newcommand{\prgz}{\mathit{\operatorname{prg0}}}
\newcommand{\prgo}{\mathit{\operatorname{prg1}}}
\newcommand{\prgt}{\mathit{\operatorname{prg2}}}

\newcommand{\consen}{\operatorname{Con}}

\newcommand{\returntrade}{\textrm{return }}
\newcommand{\underpriced}{{\rm Under}}
\newcommand{\overpriced}{{\rm Over}}
\newcommand{\localportfoliovalue}{b}
\newcommand{\itwoi}[1]{h(#1)}
\newcommand{\amountbought}{{\rm Bought}}

\newcommand{\lowprice}{{\rm Low}}

\DeclareMathOperator{\OneOperator}{\mathds{1}} 
\newcommand{\ctsind}[1]{\operatorname{Ind}_{\text{\small{${#1}$}}}} 
\newcommand{\thmind}[1][\Theory]{\operatorname{Thm}_{#1}\-} 
\newcommand{\thmval}[1][\Theory]{\operatorname{Val}_{#1}\-} 
\newcommand{\sentind}[1]{\OneOperator\mleft(#1\mright)}

\title{Logical Induction}
\author{
  Scott Garrabrant, Tsvi Benson-Tilsen, Andrew Critch, Nate Soares, \and Jessica Taylor\\
  \{scott,tsvi,critch,nate,jessica\}@intelligence.org\\
  Machine Intelligence Research Institute
}

\begin{document}

\publishingnote{See \href{https://intelligence.org/files/LogicalInductionAbridged.pdf}{https://intelligence.org/files/LogicalInductionAbridged.pdf} for an abridged version of this paper.}

\maketitle

\begin{abstract}
  We present a computable algorithm that assigns probabilities to every logical statement in a given formal language, and refines those probabilities over time. For instance, if the language is Peano arithmetic, it assigns probabilities to all arithmetical statements, including claims about the twin prime conjecture, the outputs of long-running computations, and its own probabilities. We show that our algorithm, an instance of what we call a \emph{logical inductor}, satisfies a number of intuitive desiderata, including: (1)~it learns to predict patterns of truth and falsehood in logical statements, often long before having the resources to evaluate the statements, so long as the patterns can be written down in polynomial time; (2)~it learns to use appropriate statistical summaries to predict sequences of statements whose truth values appear pseudorandom; and (3)~it learns to have accurate beliefs about its own current beliefs, in a manner that avoids the standard paradoxes of self-reference. For example, if a given computer program only ever produces outputs in a certain range, a logical inductor learns this fact in a timely manner; and if late digits in the decimal expansion of $\pi$ are difficult to predict, then a logical inductor learns to assign $\approx 10\%$ probability to ``the $n$th digit of $\pi$ is a 7'' for large $n$. Logical inductors also learn to trust their future beliefs more than their current beliefs, and their beliefs are coherent in the limit (whenever $\phi \implies \psi$, $\pt_\infty(\phi) \le \pt_\infty(\psi)$, and so on); and logical inductors strictly dominate the universal semimeasure in the limit.

  These properties and many others all follow from a single \emph{logical induction criterion}, which is motivated by a series of stock trading analogies. Roughly speaking, each logical sentence $\phi$ is associated with a stock that is worth \$1 per share if $\phi$ is true and nothing otherwise, and we interpret the belief-state of a logically uncertain reasoner as a set of market prices, where $\pt_\nn(\phi)=50\%$ means that on day $\nn$, shares of $\phi$ may be bought or sold from the reasoner for 50\textcent{}. The logical induction criterion says (very roughly) that there should not be any polynomial-time computable trading strategy with finite risk tolerance that earns unbounded profits in that market over time. This criterion bears strong resemblance to the ``no Dutch book'' criteria that support both expected utility theory \citep{Von-Neumann:1944} and Bayesian probability theory \citep{Ramsey:1931,DeFinetti:1937:foresight}.
\end{abstract}

\tableofcontents
\newpage

%
%
%
%
%
%
%
%
%
%
%
%
%
%
%
%
%
\section{Introduction}\label{sec:intro}

Every student of mathematics has experienced uncertainty about conjectures for which there is ``quite a bit of evidence'', such as the Riemann hypothesis or the twin prime conjecture. Indeed, when \citet{zhang2014bounded} proved a bound on the gap between primes, we were tempted to increase our credence in the twin prime conjecture. But how much evidence does this bound provide for the twin prime conjecture? Can we quantify the degree to which it should increase our confidence?

The natural impulse is to appeal to probability theory in general and Bayes' theorem in particular. Bayes' theorem gives rules for how to use observations to update empirical uncertainty about unknown events in the physical world. However, probability theory lacks the tools to manage uncertainty about logical facts.

Consider encountering a computer connected to an input wire and an output wire. If we know what algorithm the computer implements, then there are two distinct ways to be uncertain about the output. We could be uncertain about the input---maybe it's determined by a coin toss we didn't see. Alternatively, we could be uncertain because we haven't had the time to reason out what the program does---perhaps it computes the parity of the 87,653rd digit in the decimal expansion of $\pi$, and we don't personally know whether it's even or odd.

The first type of uncertainty is about \emph{empirical} facts. No amount of thinking in isolation will tell us whether the coin came up heads. To resolve empirical uncertainty we must observe the coin, and then Bayes' theorem gives a principled account of how to update our beliefs.

The second type of uncertainty is about a \emph{logical} fact, about what a known computation will output when evaluated. In this case, reasoning in isolation can and should change our beliefs: we can reduce our uncertainty by thinking more about $\pi$, without making any new observations of the external world.  

In any given practical scenario, reasoners usually experience a mix of both empirical uncertainty (about how the world is) and logical uncertainty (about what that implies). In this paper, we focus entirely on the problem of managing logical uncertainty. Probability theory does not address this problem, because probability-theoretic reasoners cannot possess uncertainty about logical facts. For example, let $\phi$ stand for the claim that the 87,653rd digit of $\pi$ is a 7. If this claim is true, then $(1+1=2) \Implies \phi$. But the laws of probability theory say that if $A \Implies B$ then $\Bayesian(A) \le \Bayesian(B)$. Thus, a perfect Bayesian must be at least as sure of $\phi$ as they are that $1+1=2$! Recognition of this problem dates at least back to \citet{Good:1950:weighing}.

Many have proposed methods for relaxing the criterion $\Bayesian(A) \le \Bayesian(B)$ until such a time as the implication has been proven \mbp{see, e.g, the work of \citet{Hacking:1967,Christiano:2014:omniscience}}. But this leaves open the question of how probabilities should be assigned before the implication is proven, and this brings us back to the search for a principled method for managing uncertainty about logical facts when relationships between them are suspected but unproven.

We propose a partial solution, which we call \emph{logical induction}. Very roughly, our setup works as follows. We consider reasoners that assign probabilities to sentences written in some formal language and refine those probabilities over time. Assuming the language is sufficiently expressive, these sentences can say things like ``Goldbach's conjecture is true'' or ``the computation \texttt{prg} on input \texttt{i} produces the output \texttt{prg(i)=0}''. The reasoner is given access to a slow deductive process that emits theorems over time, and tasked with assigning probabilities in a manner that outpaces deduction, e.g., by assigning high probabilities to sentences that are eventually proven, and low probabilities to sentences that are eventually refuted, well before they can be verified deductively. Logical inductors carry out this task in a way that satisfies many desirable properties, including:
\begin{enumerate*}
  \item Their beliefs are logically consistent in the limit as time approaches infinity.
  \item They learn to make their probabilities respect many different patterns in logic, at a rate that outpaces deduction.
  \item They learn to know what they know, and trust their future beliefs, while avoiding paradoxes of self-reference.
\end{enumerate*}
These claims (and many others) will be made precise in \Sec{properties}.

A logical inductor is any sequence of probabilities that satisfies our \emph{logical induction criterion}, which works roughly as follows. We interpret a reasoner's probabilities as prices in a stock market, where the probability of $\phi$ is interpreted as the price of a share that is worth \$1 if $\phi$ is true, and \$0 otherwise \mbp{similar to \citet{beygelzimer2012learning}}. We consider a collection of stock traders who buy and sell shares at the market prices, and define a sense in which traders can exploit markets that have irrational beliefs. The logical induction criterion then says that it should not be possible to exploit the market prices using any trading strategy that can be generated in polynomial-time.

Our main finding is a computable algorithm which satisfies the logical induction criterion, plus proofs that a variety of different desiderata follow from this criterion.

The logical induction criterion can be seen as a weakening of the ``no Dutch book'' criterion that \citet{Ramsey:1931,DeFinetti:1937:foresight} used to support standard probability theory, which is analogous to the ``no Dutch book'' criterion that \citet{Von-Neumann:1944} used to support expected utility theory. Under this interpretation, our criterion says (roughly) that a rational deductively limited reasoner should have beliefs that can't be exploited by any Dutch book strategy constructed by an efficient (polynomial-time) algorithm. Because of the analogy, and the variety of desirable properties that follow immediately from this one criterion, we believe that the logical induction criterion captures a portion of what it means to do good reasoning about logical facts in the face of deductive limitations. That said, there are clear drawbacks to our algorithm: it does not use its resources efficiently; it is not a decision-making algorithm (i.e., it does not ``think about what to think about''); and the properties above hold either asymptotically (with poor convergence bounds) or in the limit. In other words, our algorithm gives a theoretically interesting but ultimately impractical account of how to manage logical uncertainty.

\subsection{Desiderata for Reasoning under Logical Uncertainty}\label{sec:desiderata}
For historical context, we now review a number of desiderata that have been proposed in the literature as desirable features of ``good reasoning'' in the face of logical uncertainty. A major obstacle in the study of logical uncertainty is that it's not clear what would count as a satisfactory solution. In lieu of a solution, a common tactic is to list desiderata that intuition says a good reasoner should meet. One can then examine them for patterns, relationships, and incompatibilities. A multitude of desiderata have been proposed throughout the years; below, we have collected a variety of them. Each is stated in its colloquial form; many will be stated formally and studied thoroughly later in this paper.

\begin{desideratum}[Computable Approximability]\label{des:computable}\label{des:first}
  The method for assigning probabilities to logical claims (and refining them over time) should be computable.
  \par\rparenthetical{See \Sec{construct} for our algorithm.}
\end{desideratum}

\noindent A good method for refining beliefs about logic can never be entirely finished, because a reasoner can always learn additional logical facts by thinking for longer. Nevertheless, if the algorithm refining beliefs is going to have any hope of practicality, it should at least be computable. This idea dates back at least to \citet{Good:1950:weighing}, and has been discussed in depth by \citet{Hacking:1967,Eells:1990:OldEvidence}, among others.

\Des{computable} may seem obvious, but it is not without its teeth. It rules out certain proposals, such as that of \citet{Hutter:2013}, which has no computable approximation \citep{Sawin:2013:pi1pi2}.

\begin{desideratum}[Coherence in the Limit]\label{des:coherent}\label{des:second} 
  The belief state that the reasoner is approximating better and better over time should be logically consistent.
\par\rparenthetical{Discussed in \Sec{limitprops}.}
\end{desideratum}

\noindent First formalized by \citet{Gaifman:1964}, the idea of \Des{coherent} is that the belief state that the reasoner is approximating---the beliefs they would have if they had infinite time to think---should be internally consistent. This means that, in the limit of reasoning, a reasoner should assign $\Bayesian(\phi) \leq \Bayesian(\psi)$ whenever $\phi \Implies \psi$, and they should assign probability~1 to all theorems and~0 to all contradictions, and so on.

\begin{desideratum}[Approximate Coherence]\label{des:ic}
  The belief state of the reasoner should be approximately coherent. For example, if the reasoner knows that two statements are mutually exclusive, then it should assign probabilities to those sentences that sum to no more than 1, even if it cannot yet prove either sentence.
  \par\rparenthetical{Discussed in sections~\ref{sec:timelylearning} and~\ref{sec:logicpatterns}.}
\end{desideratum}

\noindent Being coherent in the limit is desirable, but good deductively limited reasoning requires approximate coherence at finite times. Consider two claims about a particular computation \texttt{prg}, which takes a number \texttt{n} as input and produces a number \texttt{prg(n)} as output. Assume the first claim says \texttt{prg(7)=0}, and the second says \texttt{prg(7)=1}. Clearly, these claims are mutually exclusive, and once a reasoner realizes this fact, they should assign probabilities to the two claims that sum to at most 1, even before they can evaluate \texttt{prg(7)}. Limit coherence does not guarantee this: a reasoner could assign bad probabilities (say, 100\% to both claims) right up until they can evaluate \texttt{prg(7)}, at which point they start assigning the correct probabilities. Intuitively, a good reasoner should be able to recognize the mutual exclusivity \emph{before} they've proven either claim. In other words, a good reasoner's beliefs should be approximately coherent.

\Des{ic} dates back to at least \citet{Good:1950:weighing}, who proposes a weakening of the condition of coherence that could apply to the belief states of limited reasoners. \citet{Hacking:1967} proposes an alternative weakening, as do \citet{Garrabrant:2016:ic}.

\begin{desideratum}[Learning of Statistical Patterns]\label{des:stats}
  In lieu of knowledge that bears on a logical fact, a good reasoner should assign probabilities to that fact in accordance with the rate at which similar claims are true.\rparenthetical{Discussed in \Sec{statpatterns}.}
\end{desideratum}

\noindent For example, a good reasoner should assign probability $\approx 10\%$ to the claim ``the $n$th digit of $\pi$ is a 7'' for large $n$ (assuming there is no efficient way for a reasoner to guess the digits of $\pi$ for large $n$). This desideratum dates at least back to \citet{Savage:1967:personal}, and seems clearly desirable. If a reasoner thought the $10^{100}$th digit of $\pi$ was almost surely a 9, but had no reason for believing this, we would be suspicious of their reasoning methods. \Des{stats} is difficult to state formally; for two attempts, refer to \citet{Garrabrant:2015:alu,Garrabrant:2016:ac}.

\begin{desideratum}[Calibration]\label{des:calibration}
  Good reasoners should be well-calibrated. That is, among events that a reasoner says should occur with probability $p$, they should in fact occur about $p$ proportion of the time.
  \rparenthetical{Discussed in \Sec{calibration}.}
\end{desideratum}

\noindent Calibration as a desirable property dates back to Pascal, and perhaps farther. If things that a reasoner says should happen 30\% of the time actually wind up happening 80\% of the time, then they aren't particularly reliable.

\begin{desideratum}[Non-Dogmatism]\label{des:nondogmatism}
  A good reasoner should not have extreme beliefs about mathematical facts, unless those beliefs have a basis in proof.
  \par\rparenthetical{Discussed in \Sec{conservatism}.}
\end{desideratum}

\noindent It would be worrying to see a mathematical reasoner place extreme confidence in a mathematical proposition, without any proof to back up their belief. The virtue of skepticism is particularly apparent in probability theory, where Bayes' theorem says that a probabilistic reasoner can never update away from ``extreme'' (0 or~1) probabilities. Accordingly, Cromwell's law \mkbibparens{so named by the statistician \citet{Lindley:1991:MakingDecisions}} says that a reasonable person should avoid extreme probabilities except when applied to statements that are logically true or false. We are dealing with logical uncertainty, so it is natural to extend Cromwell's law to say that extreme probabilities should also be avoided on logical statements, except in cases where the statements have been \emph{proven} true or false. In settings where reasoners are able to update away from~0 or~1 probabilities, this means that a good reasoner's beliefs shouldn't be ``stuck'' at probability~1 or~0 on statements that lack proofs or disproofs.

In the domain of logical uncertainty, \Des{nondogmatism} can be traced back to \citet[Sec. 53]{Carnap:1962:LogicalProbability}, and has been demanded by many, including \citet{Gaifman:1982:RichProbabilities,Hutter:2013}.

\begin{desideratum}[Uniform Non-Dogmatism]\label{des:pa}
  A good reasoner should assign a non-zero probability to any computably enumerable consistent theory (viewed as a limit of finite conjunctions).
  \rparenthetical{Discussed in \Sec{conservatism}.}
\end{desideratum}

\noindent For example the axioms of Peano arithmetic are computably enumerable, and if we construct an ever-growing conjunction of these axioms, we can ask that the limit of a reasoner's credence in these conjunctions converge to a value bounded above 0, even though there are infinitely many conjuncts.  The first formal statement of \Des{pa} that we know of is given by \citet{Demski:2012a}, though it is implicitly assumed whenever asking for a set of beliefs that can reason accurately about arbitrary arithmetical claims \mkbibparens{as is done by, e.g., \citet{Savage:1967:personal,Hacking:1967}}.

\begin{desideratum}[Universal Inductivity]\label{des:solomonoff}
  Given enough time to think, the beliefs of a good reasoner should dominate the universal semimeasure.
  \par\rparenthetical{Discussed in \Sec{conservatism}.}
\end{desideratum}

\noindent Good reasoning in general has been studied for quite some time, and reveals some lessons that are useful for the study of good reasoning under deductive limitation. \citet{Solomonoff:1964,Solomonoff:1964a,zvonkin1970complexity,Li:1993} have given a compelling formal treatment of good reasoning assuming logical omniscience in the domain of sequence prediction, by describing an inductive process (known as a universal semimeasure) with a number of nice properties, including (1) it assigns non-zero prior probability to every computable sequence of observations; (2) it assigns higher prior probability to simpler hypotheses; and (3) it predicts as well or better than any computable predictor, modulo a constant amount of error. Alas, universal semimeasures are uncomputable; nevertheless, they provide a formal model of what it means to predict sequences well, and we can ask logically uncertain reasoners to copy those successes. For example, we can ask that they would perform as well as a universal semimeasure if given enough time to think.

\begin{desideratum}[Approximate Bayesianism]\label{des:bayes}
  The reasoner's beliefs should admit of some notion of conditional probabilities, which approximately satisfy both Bayes' theorem and the other desiderata listed here.
  \rparenthetical{Discussed in \Sec{conditionals}.}
\end{desideratum}

\noindent Bayes' rule gives a fairly satisfying account of how to manage empirical uncertainty in principle \mbp{as argued extensively by \citet{Jaynes:2003}}, where beliefs are updated by conditioning a probability distribution. As discussed by \citet{Good:1950:weighing,Glymour:1980:OldEvidence}, creating a distribution that satisfies both coherence and Bayes' theorem requires logical omniscience. Still, we can ask that the approximation schemes used by a limited agent be approximately Bayesian in some fashion, while retaining whatever good properties the unconditional probabilities have.

\begin{desideratum}[Introspection]\label{des:introspection}
  If a good reasoner knows something, she should also know that she knows it.
  \rparenthetical{Discussed in \Sec{introspection}.}
\end{desideratum}

\noindent Proposed by \citet{Hintikka:1962:knowledge}, this desideratum is popular among epistemic logicians. It is not completely clear that this is a desirable property. For instance, reasoners should perhaps be allowed to have ``implicit knowledge'' (which they know without knowing that they know it), and it's not clear where the recursion should stop (do you know that you know that you know that you know that ${1 = 1}$?). This desideratum has been formalized in many different ways; see \citet{Christiano:2013:definability,Campbell:2015:SelfReference} for a sample.

\begin{desideratum}[Self-Trust]\label{des:lob}\label{des:penult}
  A good reasoner thinking about a hard problem should expect that, in the future, her beliefs about the problem will be more accurate than her current beliefs.
  \rparenthetical{Discussed in \Sec{selftrust}.}
\end{desideratum}

\noindent Stronger than self-knowledge is self-\emph{trust}---a desideratum that dates at least back to \citet{Hilbert:1902}, when mathematicians searched for logics that placed confidence in their own machinery. While \citet{Godel:1934} showed that strong forms of self-trust are impossible in a formal proof setting, experience demonstrates that human mathematicians are capable of trusting their future reasoning, relatively well, most of the time. A method for managing logical uncertainty that achieves this type of self-trust would be highly desirable.

\begin{desideratum}[Approximate Inexploitability]\label{des:inexp}\label{des:last}
  It should not be possible to run a Dutch book against a good reasoner in practice.
  \rparenthetical{See \Sec{criterion} for our proposal.}
\end{desideratum}
  
\noindent Expected utility theory and probability theory are both supported in part by ``Dutch book'' arguments which say that an agent is rational if (and only if) there is no way for a clever bookie to design a ``Dutch book'' which extracts arbitrary amounts of money from the reasoner \citep{Von-Neumann:1944,DeFinetti:1937:foresight}. As noted by \citet{Eells:1990:OldEvidence}, these constraints are implausibly strong: all it takes to run a Dutch book according to de Finetti's formulation is for the bookie to know a logical fact that the reasoner does not know. Thus, to avoid being Dutch booked by de Finetti's formulation, a reasoner must be logically omniscient.

\citet{Hacking:1967,Eells:1990:OldEvidence} call for weakenings of the Dutch book constraints, in the hopes that reasoners that are approximately inexploitable would do good approximate reasoning. This idea is the cornerstone of our framework---in particular, we consider reasoners that cannot be exploited in polynomial time, using a formalism defined below. See \Def{lic} for details.

\begin{desideratum}[Gaifman Inductivity]\label{des:gaifman}
  Given a $\Pi_1$ statement $\phi$ (i.e., a universal generalization of the form ``for every $x$, $\psi$''), as the set of examples the reasoner has seen goes to ``all examples'', the reasoner's belief in $\phi$ should approach certainty.
  \par\rparenthetical{Discussed below.}
\end{desideratum}

\noindent Proposed by \citet{Gaifman:1964}, \Des{gaifman} states that a reasoner should ``generalize well'', in the sense that as they see more instances of a universal claim (such as ``for every $x$, $\psi(x)$ is true'') they should eventually believe the universal with probability 1. \Des{gaifman} has been advocated by \citet{Hutter:2013}.

\begin{restatable}[Efficiency]{desideratum}{desefficient}\label{des:efficient}
  The algorithm for assigning probabilities to logical claims should run efficiently, and be usable in practice.
  \rparenthetical{Discussed in \Sec{applications}.}
\end{restatable}

\noindent One goal of understanding ``good reasoning'' in the face of logical uncertainty is to design algorithms for reasoning using limited computational resources. For that, the algorithm for assigning probabilities to logical claims needs to be not only computable, but efficient. \citet{Aaronson:2013:PhilosophersComplexity} gives a compelling argument that solutions to logical uncertainty require understanding complexity theory, and this idea is closely related to the study of bounded rationality \citep{simon1982models} and efficient meta-reasoning \citep{russell1991principles}.

\begin{restatable}[Decision Rationality]{desideratum}{desdes}\label{des:br}
  The algorithm for assigning probabilities to logical claims should be able to target specific, decision-relevant claims, and it should reason about those claims as efficiently as possible given the computing resources available.
  \rparenthetical{Discussed in \Sec{openquestions}.}
\end{restatable}

\noindent This desideratum dates at least back to \citet{Savage:1967:personal}, who asks for an extension to probability theory that takes into account the costs of thinking. For a method of reasoning under logical uncertainty to aid in the understanding of good bounded reasoning, it must be possible for an agent to use the reasoning system to reason efficiently about specific decision-relevant logical claims, using only enough resources to refine the probabilities well enough for the right decision to become clear. This desideratum blurs the line between decision-making and logical reasoning; see \citet{russell1991right,Hay:2012:Selecting} for a discussion.

\begin{restatable}[Answers Counterpossible Questions]{desideratum}{descounterp}\label{des:counterpossibilities}
  When asked questions about contradictory states of affairs, a good reasoner should give reasonable answers.
  \par\rparenthetical{Discussed in \Sec{openquestions}.}
\end{restatable}

\noindent In logic, the principle of explosion says that from a contradiction, anything follows. By contrast, when human mathematicians are asked counterpossible questions, such as ``what would follow from Fermat's last theorem being false?'', they often give reasonable answers, such as ``then there would exist non-modular elliptic curves'', rather than just saying ``anything follows from a contradiction''. \citet{Soares:2015:toward} point out that some deterministic decision-making algorithms reason about counterpossible questions (``what would happen if my deterministic algorithm had the output $a$ vs $b$ vs $c$?''). The topic of counterpossibilities has been studied by philosophers including \citet{Cohen:1990,vander2004counterpossibles,brogaard2007counterpossibles,krakauer2012counterpossibles,bjerring2014counterpossibles}, and it is reasonable to hope that a good logically uncertain reasoner would give reasonable answers to counterpossible questions.

\begin{restatable}[Use of Old Evidence]{desideratum}{desoldev}\label{des:oldevidence}
  When a bounded reasoner comes up with a new theory that neatly describes anomalies in the old theory, that old evidence should count as evidence in favor of the new theory.
  \rparenthetical{Discussed in \Sec{openquestions}.}
\end{restatable}

\noindent The problem of old evidence is a longstanding problem in probability theory \citep{Glymour:1980:OldEvidence}. Roughly, the problem is that a perfect Bayesian reasoner always uses all available evidence, and keeps score for all possible hypotheses at all times, so no hypothesis ever gets a ``boost'' from old evidence. Human reasoners, by contrast, have trouble thinking up good hypotheses, and when they do, those new hypotheses often get a large boost by retrodicting old evidence. For example, the precession of the perihelion of Mercury was known for quite some time before the development of the theory of General Relativity, and could not be explained by Newtonian mechanics, so it was counted as strong evidence in favor of Einstein's theory. \citet{Garber:1983:OldEvidence,Jeffrey:1983:OldEvidence} have speculated that a solution to the problem of logical omniscience would shed light on solutions to the problem of old evidence.\par\bigskip

\noindent Our solution does not achieve all these desiderata. Doing so would be impossible; Desiderata~\ref{des:computable},~\ref{des:coherent}, and~\ref{des:gaifman} cannot be satisfied simultaneously. Further, \citet{Sawin:2013:pi1pi2} have shown that Desiderata~\ref{des:computable},~\ref{des:nondogmatism},~\ref{des:gaifman}, and a very weak form of~\ref{des:coherent} are incompatible; an ideal belief state that is non-dogmatic, Gaifman inductive, and coherent in a weak sense has no computable approximation. Our algorithm is computably approximable, approximately coherent, and non-dogmatic, so it cannot satisfy~\ref{des:gaifman}. Our algorithm also fails to meet~\ref{des:efficient} and~\ref{des:br}, because while our algorithm is computable, it is purely inductive, and so it does not touch upon the decision problem of thinking about what to think about and how to think about it with minimal resource usage. As for~\ref{des:counterpossibilities} and~\ref{des:oldevidence}, the case is interesting but unclear; we give these topics some treatment in \Sec{discussion}.

Our algorithm does satisfy desiderata~\ref{des:first} through~\ref{des:last}. In fact, our algorithm is designed to meet only~\ref{des:computable} and~\ref{des:inexp}, from which~\ref{des:second}-\ref{des:penult} will all be shown to  follow. This is evidence that our logical induction criterion captures a portion of what it means to manage uncertainty about logical claims, analogous to how Bayesian probability theory is supported in part by the fact that a host of good properties follow from a single criterion (``don't be exploitable by a Dutch book''). That said, there is ample room to disagree about how well our algorithm achieves certain desiderata, e.g.\ when the desiderata is met only in the asymptote, or with error terms that vanish only slowly.

\subsection{Related Work}\label{sec:relatedwork}

The study of logical uncertainty is an old topic. It can be traced all the way back to Bernoulli, who laid the foundations of statistics, and later \citet{boole1854investigation}, who was interested in the unification of logic with probability from the start. Refer to \citet{hailperin1996sentential} for a historical account. Our algorithm assigns probabilities to sentences of logic directly; this thread can be traced back through \citet{Los:1955} and later \citet{Gaifman:1964}, who developed the notion of coherence that we use in this paper. More recently, that thread has been followed by \citet{Demski:2012a}, whose framework we use, and \citet{Hutter:2013}, who define a probability distribution on logical sentences that is quite desirable, but which admits of no computable approximation \citep{Sawin:2013:pi1pi2}.

The objective of our algorithm is to manage uncertainty about logical facts (such as facts about mathematical conjectures or long-running computer programs). When it comes to the problem of developing formal tools for manipulating uncertainty, our methods are heavily inspired by Bayesian probability theory, and so can be traced back to Pascal, who was followed by Bayes, Laplace, \citet{kolmogorov1950foundations,savage1954foundations,Carnap:1962:LogicalProbability,Jaynes:2003}, and many others. \citet{polya1990mathematics} was among the first in the literature to explicitly study the way that mathematicians engage in plausible reasoning, which is tightly related to the object of our study.

We are interested in the subject of what it means to do ``good reasoning'' under logical uncertainty. In this, our approach is quite similar to the approach of
\citet{Ramsey:1931,DeFinetti:1937:foresight,Von-Neumann:1944,teller1973conditionalization,lewis1999papers,Joyce:1999}, who each developed axiomatizations of rational behavior and produced arguments supporting those axioms. In particular, they each supported their proposals with Dutch book arguments, and those Dutch book arguments were a key inspiration for our logical induction criterion.

The fact that using a coherent probability distribution requires logical omniscience (and is therefore unsatisfactory when it comes to managing logical uncertainty) dates at least back to \citet{Good:1950:weighing}. \citet{Savage:1967:personal} also recognized the problem, and stated a number of formal desiderata that our solution in fact meets. \citet{Hacking:1967} addressed the problem by discussing notions of approximate coherence and weakenings of the Dutch book criteria. While his methods are ultimately unsatisfactory, our approach is quite similar to his in spirit.

The flaw in Bayesian probability theory was also highlighted by \citet{Glymour:1980:OldEvidence}, and dubbed the ``problem of old evidence'' by \citet{Garber:1983:OldEvidence} in response to Glymor's criticism. \citet{Eells:1990:OldEvidence} gave a lucid discussion of the problem, revealed flaws in Garber's arguments and in Hacking's solution, and named a number of other desiderata which our algorithm manages to satisfy. Refer to \citet{zynda1995old} and \citet{sprenger2015novel} for relevant philosophical discussion in the wake of Eells. Of note is the treatment of \citet{adams1996primer}, who uses logical deduction to reason about an unknown probability distribution that satisfies certain logical axioms. Our approach works in precisely the opposite direction: we use probabilistic methods to create an approximate distribution where logical facts are the subject.

Straddling the boundary between philosophy and computer science, \citet{Aaronson:2013:PhilosophersComplexity} has made a compelling case that computational complexity must play a role in answering questions about logical uncertainty. These arguments also provided some inspiration for our approach, and roughly speaking, we weaken the Dutch book criterion of standard probability theory by considering only exploitation strategies that can be constructed by a polynomial-time machine. The study of logical uncertainty is also tightly related to the study of bounded rationality \citep{simon1982models,russell1991right,rubinstein1998modeling,russell2016rationality}.

\citet{fagin1987belief} also straddled the boundary between philosophy and computer science with early discussions of algorithms that manage uncertainty in the face of resource limitations. \mbp{See also their discussions of uncertainty and knowledge \citep{Fagin:1995:knowledge,Halpern:2003}.} This is a central topic in the field of artificial intelligence (AI), where scientists and engineers have pursued many different paths of research. The related work in this field is extensive, including (but not limited to) work on probabilistic programming \citep{vajda2014probabilistic,mccallum2009factorie,wood2014new,de2015probabilistic}; probabilistic inductive logic programming \citep{muggleton2015latest,de2008probabilistic,de2008logical,kersting2007bayesian}; and meta-reasoning \citep{russell1991principles,zilberstein2008metareasoning,Hay:2012:Selecting}. The work most closely related to our own is perhaps the work of \citet{thimm2009measuring} and others on reasoning using inconsistent knowledge bases, a task which is analogous to constructing an approximately coherent probability distribution. \mbp{See also \citet{muino2011measuring,thimm2013inconsistency,potyka2015probabilistic,potyka2015solving}.} Our framework also bears some resemblance to the Markov logic network framework of \citet{Richardson:2006}, in that both algorithms are coherent in the limit. Where Markov logic networks are specialized to individual restricted domains of discourse, our algorithm reasons about all logical sentences. \mbp{See also \citet{kok2005learning,singla2005discriminative,tran2008event,lowd2007efficient,mihalkova2007mapping,wang2008hybrid,khot2015markov}.}

In that regard, our algorithm draws significant inspiration from Solomonoff's theory of inductive inference \citep{Solomonoff:1964,Solomonoff:1964a} and the developments on that theory made by \citet{zvonkin1970complexity,Li:1993}. Indeed, we view our algorithm as a Solomonoff-style approach to the problem of reasoning under logical uncertainty, and as a result, our algorithm bears a strong resemblance to many algorithms that are popular methods for practical statistics and machine learning; refer to \citet{opitz1999popular,dietterich2000ensemble} for reviews of popular and successful ensemble methods. Our approach is also similar in spirit to the probabilistic numerics approach of \citet{briol2015probabilistic}, but where probabilistic numerics is concerned with algorithms that give probabilistic answers to individual particular numerical questions, we are concerned with algorithms that assign probabilities to all queries in a given formal language. \mbp{See also \citep{briol2015frank,hennig2015probabilistic}.}

Finally, our method of interpreting beliefs as prices and using prediction markets to generate reasonable beliefs bears heavy resemblance to the work of \citet{beygelzimer2012learning}
who use similar mechanisms to design a learning algorithm that bets on events. Our results can be seen as an extension of that idea to the case where the events are every sentence written in some formal language, in a way that learns inductively to predict logical facts while avoiding the standard paradoxes of self-reference.

The work sampled here is only a small sample of the related work, and it neglects contributions from many other fields, including but not limited to epistemic logic \citep{gardenfors1988knowledge,meyer2004epistemic,schlesinger1985range,sowa1999knowledge,guarino1998formal}, game theory \citep{rantala1979urn,hintikka1979impossible,bacharach1994epistemic,lipman1991decide,battigalli1999recent,binmore1992foundations}, paraconsistent logic \citep{
blair1989paraconsistent,priest2002paraconsistent,mortensen2013inconsistent,fuhrmann2013relevant,akama2016paraconsistent} and fuzzy logic \citep{klir1995fuzzy,yen1999fuzzy,gerla2013fuzzy}. The full history is too long and rich for us to do it justice here.

\subsection{Overview}\label{sec:overview}
Our main result is a formalization of \Des{inexp} above, which we call the \emph{\lic{}}, along with a computable algorithm that meets the criterion, plus proofs that formal versions of Desiderata~\ref{des:second}-\ref{des:penult} all follow from the criterion.

In \Sec{notation} we define some notation. In \Sec{criterion} we state the logical induction criterion and our main theorem, which says that there exists a computable logical inductor. The \lic{} is motivated by a series of stock trading analogies, which are also introduced in \Sec{criterion}.

In \Sec{properties} we discuss a number of properties that follow from this criterion, including properties that hold in the limit, properties that relate to pattern-recognition, calibration properties, and properties that relate to self-knowledge and self-trust.

A computable \li{} is described in \Sec{construct}. Very roughly, the idea is that given any trader, it's possible to construct market prices at which they make no trades (because they think the prices are right); and given an enumeration of traders, it's possible to aggregate their trades into one ``supertrader'' (which takes more and more traders into account each day); and thus it is possible to construct a series of prices which is not exploitable by any trader in the enumeration.

In \Sec{selectedproofs} we give a few selected proofs. In \Sec{discussion} we conclude with a discussion of applications of logical inductors, variations on the logical induction framework, speculation about what makes logical inductors tick, and directions for future research. The remaining proofs can be found in the appendix.

\section{Notation}\label{sec:notation}

This section defines notation used throughout the paper. The reader is invited to skim it, or perhaps skip it entirely and use it only as a reference when needed.

\tinysection{Common sets and functions} The set of positive natural numbers is denoted by ${\NN^+}$, where the superscript makes it clear that 0 is not included. We work with $\NN^+$ instead of $\NN^{\ge 0}$ because we regularly consider initial segments of infinite sequences up to and including the element at index $\nn$, and it will be convenient for those lists to have length $\nn$.  Sums written $\sum_{i\leq \nn}(\any)$ are understood to start at $i=1$.  We use $\RR$ to denote the set of real numbers, and $\QQ$ to denote the set of rational numbers. When considering continuous functions with range in $\QQ$, we use the subspace topology on $\QQ$ inherited from $\RR$. We use $\BB$ to denote the set $\{0, 1\}$ interpreted as Boolean values. In particular, Boolean operations like $\land$,~$\lor$,~$\lnot$,~$\implies$ and~$\iff$ are defined on $\BB$, for example, $(1\land 1) = 1$, $\lnot 1 = 0$, and so on.

We write $\fin(X)$ for the set of all finite subsets of $X$, and $\smash{X^{\NN^+}}$ for all infinite sequences with elements in $X$. In general, we use $B^A$ to denote the set of functions with domain $A$ and codomain $B$. We treat the expression $f : A \to B$ as equivalent to $f\in B^A$, i.e., both state that $f$ is a function that takes inputs from the set $A$ and produces an output in the set $B$. We write $f : A\pfun B$ to indicate that $f$ is a partial function from $A$ to $B$. We denote equivalence of expressions that represent functions by $\equiv$, e.g., $(x-1)^2 \equiv x^2-2x+1$. We write $\|\any\|_1$ for the $\ell_1$ norm. When $A$ is an affine combination, $\|A\|_1$ includes the trailing coefficient.

\tinysection{Logical sentences} We generally use the symbols $\phi, \psi, \chi$ to denote well-formed formulas in some language of propositional logic $\Lang$ (such as a theory of first order logic; see below), which includes the basic logical connectives $\lnot$, $\land$, $\lor$, $\implies$, $\iff$, and uses modus ponens as its rule of inference. We assume that $\Lang$ has been chosen so that its sentences can be interpreted as claims about some class of mathematical objects, such as natural numbers or computer programs. We commonly write $\Sentences$ for the set of all sentences in $\Lang$, and $\Theory$ for a set of axioms from which to write proofs in the language. We write $\Theory\vdash\phi$ when $\phi$ can be proven from $\Theory$ via modus ponens.

We will write logical formulas inside quotes $\quot{\!\any\!}$, such as $\phi := \quot{x = 3}$. The exception is after $\vdash$, where we do not write quotes, in keeping with standard conventions. We sometimes define sentences such as ${\phi := \quot{\text{Goldbach's conjecture}}}$, in which case it is understood that the English text could be expanded into a precise arithmetical claim.

We use underlines to indicate when a symbol in a formula should be replaced by the expression it stands for.  For example, if $n:=3$, then $\phi:= \quot{x > \enc\nn}$ means $\phi=\quot{x>3}$, and $\psi := \quot{\enc{\smash{\phi}} \to (x = \enc\nn + 1)}$ means $\psi = \quot{x > 3 \to (x = 3+1)}$. If $\phi$ and $\psi$ denote formulas, then $\lnot \phi$ denotes $\quot{\lnot(\enc{\smash{\phi}})}$ and $\phi \land \psi$ denotes $\quot{(\enc{\smash{\phi}})\land(\enc{\smash{\psi}})}$ and so on. For instance, if $\phi := \quot{x > 3}$ then $\lnot\phi$ denotes $\quot{\lnot(x > 3)}$.

\tinysection{First order theories and prime sentences} 
We consider any theory in first order logic (such as Peano Arithmetic, $\PA$) as a set of axioms that includes the axioms of first order logic, so that modus ponens is the only rule of inference needed for proofs.
As such, we view any first order theory as specified in a propositional calculus \mbp{following \citet{enderton2001mathematical}} whose atoms are the so-called ``prime'' sentences of first order logic, i.e., quantified sentences like $\quot{\exists x \colon \cdots}$, and atomic sentences like $\quot{t_1=t_2}$ and $\quot{R(t_1,\ldots,t_n)}$ where the $t_i$ are closed terms.  Thus, every first-order sentence can be viewed as a Boolean combination of prime sentences with logical connectives (viewing $\quot{\forall x \colon \cdots}$ as shorthand for $\quot{\lnot\exists x \colon \lnot \cdots}$).  For example, the sentence
\[
  \phi := \quot{((1+1=2)\wedge(\forall x \colon x>0))\implies(\exists y \colon \forall z \colon (7>1+1)\implies(y+z>2))}
\]
is decomposed into $\quot{1+1=2}$, $\quot{\exists x \colon \lnot(x>0)}$ and $\quot{\exists y \colon \forall z \colon (7>1+1)\implies(y+z>2)}$, where the leading $\quot{\lnot}$ in front of the second statement is factored out as a Boolean operator.   
In particular, note that while $(7>1+1)$ is a prime sentence, it \emph{does not} occur in the Boolean decomposition of $\phi$ into primes, since it occurs within a quantifier. We choose this view because we will not always assume that the theories we manipulate include the quantifier axioms of first-order logic.

\tinysection{Defining values by formulas} We often view a formula that is free in one variable as a way of defining a particular number that satisfies that formula.  For example, given the formula $X(\nu) = \quot{\nu^2=9 \; \wedge \; \nu>0}$, we would like to think of $X$ as representing the unique value ``3'', in such a way that that we can then have $\quot{5X+1}$ refer to the number $16$.

To formalize this, we use the following notational convention.  Let $\LUV{X}$ be a formula free in one variable.  We write $\LUV{X}(x)$ for the formula resulting from substituting $x$ for the free variable of $\LUV{X}$. If
\[
  \Theory \vdash{\exists x \forall y \colon \LUV{X}(y) \to y=x},
\]
then we say that $\LUV{X}$ defines a unique value (via $\Theory$), and we refer to that value as ``the value'' of $X$. We will be careful in distinguishing between what $\Theory$ can prove about $X(\nu)$ on the one hand, and the values of $X(\nu)$ in different models of $\Theory$ on the other.

If $\LUV{X}_1, \ldots, \LUV{X}_k$ are all formulas free in one variable that define a unique value (via $\Theory$), then for any $k$-place relationship $R$, we write
$
\quot{R(\LUV{X}_1,\LUV{X}_2,\ldots,\LUV{X}_k)}
$
as an abbreviation for
\[
  \quot{\forall x_1x_2\ldots x_k \colon \LUV{X}_1(x_1) \land \LUV{X}_2(x_2) \land \ldots \land \LUV{X}_k(x_k) \to
  R(x_1, x_2, \ldots, x_k)}.
\]
For example,
$\quot{\LUV{Z}=2\LUV{X}+Y}$ 
is shorthand for
\[
  \quot{\forall xyz \colon \LUV{X}(x) \wedge \LUV{Y}(y) \wedge \LUV{Z}(z) \implies z = 2x+y}.
\]
This convention allows us to write concise expressions that describe relationships between well-defined values, even when those values may be difficult or impossible to determine via computation.  

\tinysection{Representing computations} When we say a theory $\Theory$ in first order logic ``\representscomputations'', we mean that its language is used to refer to computer programs in such a way that $\Theory$ satisfies the representability theorem for computable functions.  This means that for every (total) computable function $f : \NN^+ \to \NN^+$, there exists a $\Theory$-formula $\gamma_f$ with two free variables such that for all $\nn,y\in\NN^+$, 
\[
  y=f(\nn) \text{~ if and only if ~} \Theory \vdash \forall \nu \colon \gamma_f(\enc \nn,\nu) \iff \nu = \enc{y},
\]
where ``$\gamma_f(\enc{\nn},\nu)$'' stands, in the usual way, for the formula resulting from substituting an encoding of $\nn$ and the symbol $\nu$ for its free variables.  In particular, note that this condition requires $\Theory$ to be consistent.

When $\Theory$ \representscomputations, we use $\quot{\enc{f}(\enc{n})}$ as shorthand for the formula $\quot{\gamma_f(\enc{n}, \nu)}$.   In particular, since $\quot{\gamma_f(\enc{n}, \nu)}$ is free in a single variable $\nu$ and defines a unique value,  we use $\quot{\enc{f}(\enc{n})}$ by the above convention to write, e.g.,
\[
{\quot{\enc{f}(3) < \enc{g}(3)}}
\]
as shorthand for
\[
  \quot{\forall xy \colon \gamma_f(3, x) \land \gamma_g(3, y) \to x < y}.
\]
In particular, note that writing down a sentence like $\quot{\enc{f}(3) > 4}$ does not involve computing the value $f(3)$; it merely requires writing out the definition of $\gamma_f$. This distinction is important when $f$ has a very slow runtime.

\tinysection{Sequences} We denote infinite sequences using overlines, like $\seq{x} := (x_1, x_2, \ldots)$, where it is understood that $x_i$ denotes the $i$th element of $\seq{x}$, for $i\in {\NN^+}$. To define sequences of sentences compactly, we use parenthetical expressions such as $\phis := (\quot{\enc{n} > 7})_{\nn\in {\NN^+}}$, which defines the sequence
\[
  (\quot{1 > 7}, \quot{2 > 7}, \quot{3 > 7}, \ldots).
\]
We define ${x_{\le n} := (x_1, \ldots, x_n)}$. Given another element $y$, we abuse notation in the usual way and define $(x_{\le \nn}\append y) = (x_1,\ldots,x_n, y)$ to be the list $x_{\le n}$ with $y$ appended at the end. We write $()$ for the empty sequence.

A sequence $\seq{x}$ is called \emph{computable} if there is a computable function $f$ such that $f(\nn)=x_\nn$ for all $\nn\in {\NN^+}$, in which case we say $f$ computes $\seq{x}$.  

\tinysection{Asymptotics} Given any sequences $\seq x$ and $\seq y$, we write 
\begin{align*}
  x_\nn \eqsim_\nn y_\nn & \quad\text{for}\quad \lim_{n \to\infty} x_\nn - y_\nn = 0,\\
  x_\nn \gtrsim_\nn y_\nn & \quad\text{for}\quad \liminf_{n \to\infty} x_\nn - y_\nn \ge 0,\text{~and}\\
  x_\nn \lesssim_\nn y_\nn & \quad\text{for}\quad \limsup_{n \to\infty} x_\nn - y_\nn \le 0.
\end{align*}

%
%
%
%
%
%
%
%
%
%
%
%
%
%
%
%
%

\section{The \LICtitle}\label{sec:framework}\label{sec:criterion}
In this section, we will develop a framework in which we can state the logical induction criterion and a number of properties possessed by logical inductors. The framework will culminate in the following definition, and a theorem saying that computable logical inductors exist for every deductive process.

\begin{key}
\begin{restatable}[The Logical Induction Criterion]{definition}{criterion}\label{def:lic}
  A market $\MP$ is said to satisfy the \textbf{\lic{}} relative to a deductive process $\DP$ if there is no efficiently computable trader $\Trader$ that exploits $\MP$ relative to $\DP$.  A market $\MP$ meeting this criterion is called a \textbf{\li{} over $\bm{\DP}$}.
\end{restatable}
\end{key}

\noindent We will now define markets, deductive processes, efficient computability, traders, and exploitation.

\subsection{Markets}\label{sec:valuations}

We will be concerned with methods for assigning values in the interval $[0, 1]$ to sentences of logic. We will variously interpret those values as prices, probabilities, and truth values, depending on the context. Let $\Lang$ be a language of propositional logic, and let $\Sentences$ be the set of all sentences written in $\Lang$. We then define:

\begin{definition}[Valuation]\label{def:market}
  A \textbf{valuation} is any function $\Valuation : \Sentences \to [0, 1]$.  We refer to $\Valuation(\phi)$ as the value of $\phi$ according to $\Valuation$. A valuation is called rational if its image is in $\mathbb{Q}$.
\end{definition}

\noindent First let us treat the case where we interpret the values as prices.

\begin{definition}[Pricing]\label{def:pricing}
  A \textbf{pricing} $\Pricing : \Sentences \to \QQ \cap [0, 1]$ is any computable rational valuation. If $\Pricing(\phi)=\price$ we say that the price of a $\phi$-share according to $\Pricing$ is $\price$, where the intended interpretation is that a $\phi$-share is worth \$1 if $\phi$ is true.
\end{definition}

\begin{keydef}[Market]\label{def:marketprocess}
  A \textbf{market} $\Market=(\Pricing_1,\Pricing_2,\ldots)$ is a computable sequence of pricings $\Pricing_i : \Sentences \to \QQ \cap [0,1]$.
\end{keydef}

\noindent We can visualize a market as a series of pricings that may change day by day.  The properties proven in \Sec{properties} will apply to any market that satisfies the \lic{}. \Theorem{lc} will show that the prices of a logical inductor can reasonably be interpreted as probabilities, so we will often speak as if the prices in a market represent the beliefs of a reasoner, where $\pt_\nn(\phi)=0.75$ is interpreted as saying that on day $\nn$, the reasoner assigns 75\% probability to $\phi$.

In fact, the logical inductor that we construct in \Sec{construct} has the additional property of being finite at every timestep, which means we can visualize it as a series of finite belief states that a reasoner of interest writes down each day.

\begin{definition}[Belief State]\label{def:belstate}
  A \textbf{belief state} $\BelState : \Sentences \to \QQ \cap [0, 1]$ is a computable rational valuation with finite support, where $\BelState(\phi)$ is interpreted as the probability of~$\phi$ (which is 0 for all but finitely many $\phi$).
\end{definition}

\noindent We can visualize  a belief state as a finite list of $(\phi, \prob)$ pairs, where the $\phi$ are unique sentences and the $\prob$ are rational-number probabilities, and $\BelState(\phi)$ is defined to be $\prob$ if $(\phi, \prob)$ occurs in the list, and $0$ otherwise.

\begin{definition}[Computable Belief Sequence]\label{def:belseq}
  A \textbf{computable belief sequence} $\BelSeq = (\BelState_1,\BelState_2,\ldots)$ is a computable sequence of belief states, interpreted as a reasoner's explicit beliefs about logic as they are refined over time.
\end{definition}
\noindent We can visualize a computable belief sequence as a large spreadsheet where each column is a belief state, and the rows are labeled by an enumeration of all logical sentences. We can then imagine a reasoner of interest working on this spreadsheet, by working on one column per day.

Philosophically, the reason for this setup is as follows. Most people know that the sentence $\quot{1+1\text{\ is even}}$ is true, and that the sentence $\quot{1+1+1+1\text{\ is even}}$ is true. But consider, is the following sentence true?
\[
  \quot{1+1+1+1+1+1+1+1+1+1+1+1+1\text{\ is even}}
\]
To answer, we must pause and count the ones. Since we wish to separate the question of what a reasoner already knows from what they could infer using further computing resources, we require that the reasoner write out their beliefs about logic explicitly, and refine them day by day.

In this framework, we can visualize a reasoner as a person who computes the belief sequence by filling in a large spreadsheet, always working on the $\nn$th column on the $\nn$th day, by refining and extending her previous work as she learns new facts and takes more sentences into account, while perhaps making use of computer assistance. For example, a reasoner who has noticed that $\quot{1 + \cdots + 1\text{\ is even}}$ is true iff the sentence has an even number of ones, might program her computer to write~1 into as many of the true $\quot{1+\cdots+1\text{\ is even}}$ cells per day as it can before resources run out. As another example, a reasoner who finds a bound on the prime gap might go back and update her probability on the twin prime conjecture. In our algorithm, the reasoner will have more and more computing power each day, with which to construct her next belief state.

\subsection{Deductive Processes}\label{sec:worlds}
We are interested in the question of what it means for reasoners to assign ``reasonable probabilities'' to statements of logic. Roughly speaking, we will imagine reasoners that have access to some formal deductive process, such as a community of mathematicians who submit machine-checked proofs to an official curated database. We will study reasoners that ``outpace'' this deductive process, e.g., by assigning high probabilities to conjectures that will eventually be proven, and low probabilities to conjectures that will eventually be disproven, well before the relevant proofs are actually found.

\begin{keydef}[Deductive Process]\label{def:dedproc}
  A \textbf{deductive process} $\DP : \NN^+ \to \fin(\Sentences)$ is a computable nested sequence $\dt_1 \subseteq \dt_2 \subseteq \dt_3 \ldots$ of finite sets of sentences.  We write $\dt_\infty$ for the union $\bigcup_\nn \dt_\nn$.
\end{keydef}

This is a rather barren notion of ``deduction''. We will consider cases where we fix some theory $\Theory$, and $\dt_\nn$ is interpreted as the theorems proven up to and including day $\nn$. In this case, $\DP$ can be visualized as a slow process that reveals the knowledge of $\Theory$ over time. Roughly speaking, we will mainly concern ourselves with the case where $\DP$ eventually rules out all and only the worlds that are inconsistent with $\Theory$.

\begin{definition}[World]\label{def:world}
  A \textbf{world} is any truth assignment $\World:\Sentences\to\BB$. If $\World(\phi)=1$ we say that $\phi$ is \textbf{true in $\World$}. If $\World(\phi)=0$ we say that $\phi$ is \textbf{false in $\World$}. We write $\Worlds$ for the set of all worlds.
\end{definition}

\noindent Observe that worlds are valuations, and that they are not necessarily consistent. This terminology is nonstandard; the term ``world'' is usually reserved for consistent truth assignments. Logically uncertain reasoners cannot immediately tell which truth assignments are inconsistent, because revealing inconsistencies requires time and effort. We use the following notion of consistency:

\begin{definition}[Propositional Consistency]
  A world $\World$ is called \textbf{propositionally consistent}, abbreviated \textbf{p.c.}, if for all $\phi\in\Sentences$, $\World(\phi)$ is determined by Boolean algebra from the truth values that $\World$ assigns to the prime sentences of $\phi$. In other words, $\World$ is p.c.\ if $\World(\phi\land\psi) = \World(\phi) \land \World(\psi)$, $\World(\phi\lor\psi) = \World(\phi)\lor\World(\psi)$, and so on.

  Given a set of sentences $\dt$, we define $\pcworlds(\dt)$ to be the set of all p.c.\ worlds where $\World(\phi)=1$ for all $\phi\in\dt$. We refer to $\pcworlds(\dt)$ as the set of worlds \textbf{propositionally consistent with $\bm{D}$}.
  
  Given a set of sentences $\Theory$ interpreted as a theory, we will refer to $\pcworlds(\Theory)$ as the set of worlds \textbf{consistent with $\bm{\Theory}$}, because in this case $\pcworlds(\Theory)$ is equal to the set of all worlds $\World$ such that
  \[
    \Theory \cup \{\phi\mid \World(\phi)=1\} \cup \{\neg\phi \mid \World(\phi)=0\} \nvdash \bot.
  \]
\end{definition}

\noindent Note that a limited reasoner won't be able to tell whether a given world $\World$ is in $\cworlds(\Theory)$. A reasoner can computably check whether a restriction of $\World$ to a finite domain is propositionally consistent with a finite set of sentences, but that's about it. Roughly speaking, the definition of exploitation (below) will say that a good reasoner should perform well when measured on day $\nn$ by worlds propositionally consistent with $\dt_\nn$, and we ourselves will be interested in deductive processes that pin down a particular theory $\Theory$ by propositional consistency:

\begin{definition}[$\Theory$-Complete]
  Given a theory $\Theory$, we say that a deductive process $\DP$ is \textbf{$\bm{\Theory}$-complete} if
  \[
    \pcworlds(\dt_\infty) = \cworlds(\Theory).
  \]
\end{definition}

As a canonical example, let $\dt_\nn$ be the set of all theorems of $\PA$ provable in at most $\nn$ characters.\footnote{Because $\PA$ is a first-order theory, and the only assumption we made about $\Lang$ is that it is a propositional logic, note that the axioms of first-order logic---namely, specialization and distribution---must be included as theorems in $\DP$.} Then $\DP$ is $\PA$-complete, and a reasoner with access to $\DP$ can be interpreted as someone who on day $\nn$ knows all $\PA$-theorems provable in $\leq \nn$ characters, who must manage her uncertainty about other mathematical facts.

\subsection{Efficient Computability}\label{sec:efc}

We use the following notion of efficiency throughout the paper:

\begin{keydef}[Efficiently Computable]\label{def:ec}
  An infinite sequence $\seq{x}$ is called \textbf{efficiently computable}, abbreviated \textbf{\ec[]}, if there is a computable function $f$ that outputs $x_\nn$ on input $\nn$, with runtime polynomial in $\nn$ (i.e. in the length of $n$ written in unary).
\end{keydef}

\noindent Our framework is not wedded to this definition; stricter notions of efficiency (e.g., sequences that can be computed in $\Oo(n^2)$ time) would yield ``dumber'' inductors with better runtimes, and vice versa. We use the set of polynomial-time computable functions because it has some closure properties that are convenient for our purposes.

\subsection{Traders}\label{sec:traders}

Roughly speaking, traders are functions that see the day $\nn$ and the history of market prices up to and including day $\nn$, and then produce a series of buy and sell orders, by executing a strategy that is continuous as a function of the market history.

A linear combination of sentences can be interpreted as a ``market order'', where $3\phi - 2\psi$ says to buy 3 shares of $\phi$ and sell 2 shares of $\psi$. Very roughly, a trading strategy for day $\nn$ will be a method for producing market orders where the coefficients are not numbers but \emph{functions} which depend (continuously) on the market prices up to and including day $\nn$.

\begin{definition}[Valuation Feature]\label{def:valfeature}
A valuation \textbf{feature} $\feature : [0, 1]^{\Sentences \times \NN^+} \to \RR$ is a continuous function from valuation sequences to real numbers such that $\feature(\seq\Valuation)$ depends only on the initial sequence $\Valuation_{\le \nn}$ for some~$\nn\in\NN^+$ called the \emph{rank} of the feature, $\rank(\feature)$.  For any $\mm\geq\nn$, we define $\feature(\Valuation_{\leq m})$ in the natural way.   We will often deal with features that have range in $[0, 1]$; we call these $[0,1]$-features.

We write $\features$ for the set of all features, $\features_\nn$ for the set of valuation features of rank $\le\nn$, and define an
\textbf{$\bm\features$-progression} $\seq\feature$ to be a sequence of features such that $\feature_\nn \in \features_\nn$.
\end{definition}

\noindent The following valuation features find the price of a sentence on a particular day:

\begin{definition}[Price Feature]
  For each $\phi\in\Sentences$ and $\nn\in\NN^+$, we define a \textbf{price feature} $\phi^{*\nn} \in \features_\nn$ by the formula
  \[
    \pf{\phi}(\seq\Valuation) := \Valuation_\nn(\phi).
  \]
  We call these ``price features'' because they will almost always be applied to a market $\MP$, in which case $\pf{\phi}$ gives the price $\Pricing_n(\phi)$ of $\phi$ on day $\nn$ as a function of $\MP$.
\end{definition}

\noindent Very roughly, trading strategies will be linear combinations of sentences where the coefficients are valuation features. The set of all valuation features is not computably enumerable, so we define an expressible subset:

\begin{restatable}[Expressible Feature]{definition}{expressiblefeatures}\label{def:tf}
  An \textbf{expressible feature} $\exf \in \features$ is a valuation feature expressible by an algebraic expression built from price features $\pf{\phi}$ for each $\nn\in\NN^+$ and $\phi\in\Sentences$, rational 
  numbers, addition, multiplication, $\max(\any, \any)$, and a ``safe reciprocation'' function $\max(1,\any)^{-1}$.  \MainTextOnly{See Appendix~\ref{app:expressiblefeatures} for more details and examples.  \footnote{In particular, expressible features are a generalization of arithmetic circuits.  The specific definition is somewhat arbitrary; what matters is that expressible features be (1) continuous; (2) compactly specifiable in polynomial time; 
and (3) expressive enough to identify a variety of inefficiencies in a market.}}
  
  We write $\exfeatures$ for the set of all expressible features, $\exfeatures_\nn$ for the set of expressible features of rank $\leq \nn$, and define an \textbf{$\bm\exfeatures$-progression} to be a sequence $\seq\exf$ such that $\exf_\nn\in\exfeatures_\nn$.
\end{restatable}

For those familiar with abstract algebra, note that for each~$\nn$, $\exfeatures_\nn$ is a commutative ring.  We will write $2 - \pf[6]{\phi}$ for the function $\seq{\Valuation} \mapsto 2 - \pf[6]{\phi}(\seq{\Valuation})$ and so on, in the usual way. For example, the feature
\[
  \exf := \max(0, \pf[6]{\phi} - \pf[7]{\psi})
\]
checks whether the value of $\phi$ on day 6 is higher than the value of $\psi$ on day 7. If so, it returns the difference; otherwise, it returns 0. If $\exf$ is applied to a market $\MP$, and $\pt_6(\phi)=0.5$ and $\pt_7(\psi)=0.2$, then $\exf(\MP)=0.3$. Observe that $\rank(\exf)=7$, and that $\exf$ is continuous.

The reason for the continuity constraint on valuation features is as follows. Traders will be allowed to use valuation features (which depend on the price history) to decide how many shares of different sentences to buy and sell. This creates a delicate situation, because we'll be constructing a market that has prices which depend on the behavior of certain traders, creating a circular dependency where the prices depend on trades that depend on the prices.

This circularity is related to classic paradoxes of self-trust. What should be the price on a paradoxical sentence $\chi$ that says ``I am true iff my price is less than 50 cents in this market''? If the price is less than 50\textcent, then $\chi$ pays out \$1, and traders can make a fortune buying $\chi$. If the price is 50\textcent{} or higher, then $\chi$ pays out \$0, and traders can make a fortune selling $\chi$. If traders are allowed to have a discontinuous trading strategy---buy $\chi$ if $\pt(\chi)<0.5$, sell $\chi$ otherwise---then there is no way to find prices that clear the market.

Continuity breaks the circularity, by ensuring that if there's a price where a trader buys~$\chi$ and a price where they sell~$\chi$ then there's a price in between where they neither buy nor sell. In \Sec{construct} we will see that this is sufficient to allow stable prices to be found, and in \Sec{introspection} we will see that it is sufficient to subvert the standard paradoxes of self-reference. The continuity constraint can be interpreted as saying that the trader has only finite-precision access to the market prices---they can see the prices, but there is some $\varepsilon > 0$ such that their behavior is insensitive to an $\varepsilon$ shift in prices.

We are almost ready to define trading strategies as a linear combination of sentences with expressible features as coefficients. However, there is one more complication. It will be convenient to record not only the amount of shares bought and sold, but also the amount of cash spent or received. For example, consider again the market order $3\phi - 2\psi$. If it is executed on day 7 in a market $\MP$, and $\pt_7(\phi)=0.4$ and $\pt_7(\psi)=0.3$, then the cost is $3\cdot 40\text{\textcent} - 2\cdot 30\text{\textcent} = 60\text{\textcent}$. We can record the whole trade as an affine combination $-0.6 + 3\phi - 2\psi$, which can be read as ``the trader spent 60 cents to buy 3 shares of $\phi$ and sell 2 shares of $\psi$''. Extending this idea to the case where the coefficients are expressible features, we get the following notion:

\begin{definition}[Trading Strategy]\label{def:tradestrat}
  A \textbf{trading strategy for day $\nn$}, also called an \textbf{$\bm \nn$-strategy}, is an affine combination of the form
  \[
    \trade = \cash + \exf_1\phi_1 + \cdots + \exf_k\phi_k,
  \]
  where $\phi_1,\ldots,\phi_k$ are sentences, $\exf_1,\ldots,\exf_k$ are expressible features of rank $\leq \nn$, and
  \[
    \cash = -\sum_i \exf_i \pf{\phi_i}
  \]
  is a ``cash term'' recording the net cash flow when executing a transaction that buys $\exf_i$ shares of $\phi_i$ for each $i$ at the prevailing market price. (Buying negative shares is called ``selling''.)
  We define $\trade[1]$ to be $\cash$, and $\trade[\phi]$ to be the coefficient of $\phi$ in $\trade$, which is $0$ if $\phi \not\in (\phi_1, \ldots, \phi_k)$.  
  
  An $\nn$-strategy $T$ can be encoded by the tuples $ (\phi_1,\ldots\phi_k)$ and $(\exf_1,\ldots\exf_k)$ because the $\cash$ term is determined by them.  Explicitly, by linearity we have
  \[
    \trade = \exf_1 \cdot (\phi_1 - \pf{\phi_1}) + \cdots + \exf_k \cdot (\phi_k - \pf{\phi_k}),
  \]
  which means any $\nn$-strategy can be written as a linear combination of $(\phi_i - \pf{\phi_i})$ terms, each of which means ``buy one share of $\phi_i$ at the prevailing price''.
\end{definition}

As an example, consider the following trading strategy for day 5:
\[
  \left[
    \pf[5]{(\lnot\lnot\phi)} - \pf[5]{\phi}
  \right]
  \cdot
  \left(
    \phi - \pf[5]{\phi}
  \right)
  +
  \left[
    \pf[5]{\phi} - \pf[5]{(\lnot\lnot\phi)}
  \right]
  \cdot
  \left(
    \lnot\lnot\phi - \pf[5]{(\lnot\lnot\phi)}
  \right).
\]
This strategy compares the price of $\phi$ on day 5 to the price of $\lnot\lnot\phi$ on day 5. If the former is less expensive by $\delta$, it purchases $\delta$ shares of $\phi$ at the prevailing prices, and sells $\delta$ shares of $\lnot\lnot\phi$ at the prevailing prices. Otherwise, it does the opposite. In short, this strategy arbitrages $\phi$ against $\lnot\lnot\phi$, by buying the cheaper one and selling the more expensive one.

We can now state the key definition of this section:

\begin{keydef}[Trader]\label{def:trader}
  A \textbf{trader} $\Trader$ is a sequence $(\trade_1, \trade_2, \ldots)$ where each $\trade_\nn$ is a trading strategy for day $\nn$.
\end{keydef}

\noindent We can visualize a trader as a person who gets to see the day $\nn$, think for a while, and then produce a trading strategy for day $\nn$, which will observe the history of market prices up to and including day $\nn$ and execute a market order to buy and sell different sentences at the prevailing market prices.

We will often consider the set of efficiently computable traders, which have to produce their trading strategy in a time polynomial in $\nn$. We can visualize \ec traders as traders who are computationally limited: each day they get to think for longer and longer---we can imagine them writing computer programs each morning that assist them in their analysis of the market prices---but their total runtime may only grow polynomially in $\nn$.

If $s := \trade_\nn[\phi] > 0$, we say that $\Trader$ buys $s$ shares of $\phi$ on day $\nn$, and if $s < 0$, we say that $\Trader$ sells $|s|$ shares of $\phi$ on day $\nn$.  Similarly, if $d := \trade_\nn[\fconst] > 0$, we say that $\Trader$ receives $d$ dollars  on day $\nn$, and if $d < 0$, we say that $\Trader$ pays out $|d|$ dollars on day $\nn$.

Each trade $\trade_\nn$ has value zero according to $\pt_\nn$, regardless of what market $\MP$ it is executed in. Clever traders are the ones who make trades that are later revealed by a deductive process $\DP$ to have a high worth (e.g., by purchasing shares of provable sentences when the price is low). As an example, a trader $\Trader$ with a basic grasp of arithmetic and skepticism about some of the market $\MP$'s confident conjectures might execute the following trade orders on day $\nn$:
\begin{table}[H]
\centering
\caption{Visualizing markets and trades}
\label{tab:bets}
\begin{tabular}{lll}
  Sentence & Market prices & Trade \\\toprule
  $\phi :\leftrightarrow 1+1=2$ & $\pt_\nn(\phi)=90$\textcent & $\trade_\nn[\phi] = 4$ shares \\\midrule
  $\psi :\leftrightarrow 1+1\neq 2$ & $\pt_\nn(\psi)=5$\textcent & $\trade_\nn[\psi] = -3$ shares \\\midrule
  $\chi :\leftrightarrow  \quot{\text{Goldbach's conjecture}}$ & $\pt_\nn(\chi) = 98$\textcent & $\trade_\nn[\chi] = -1$ share
\end{tabular}
\end{table}
\noindent The net value of the shares bought and sold at these prices would be \[4 \cdot 90\text{\textcent} - 3 \cdot 5\text{\textcent} - 1 \cdot 98\text{\textcent} = \$2.47,\] so if those three sentences were the only sentences bought and sold by $\trade_\nn$, $\trade_\nn[1]$ would be $-2.47$.

Trade strategies are a special case of affine combinations of sentences:

\begin{definition}[Affine Combination]\label{def:affcomsen}
  An \textbf{$\bm\features$-combination} $\aff : \Sentences \cup \{1\} \to \features$ is an affine expression of the form
  \[
    \aff := \affconst + \feature_1\phi_1 + \cdots + \feature_k\phi_k,
  \]
  where $(\phi_1,\ldots,\phi_k)$ are sentences and $(\affconst, \feature_1, \ldots, \feature_k)$ are in $\features$. We define \textbf{$\RR$-combinations}, \textbf{$\QQ$-combinations}, and \textbf{$\bm{\exfeatures}$-combinations} analogously.
 
  We write $\aff[\fconst]$ for the trailing coefficient $\affconst$, and $\aff[\phi]$ for the coefficient of $\phi$, which is $0$ if $\phi \not\in (\phi_1, \ldots, \phi_k)$. The \textbf{rank} of $\aff$ is defined to be the maximum rank among all its coefficients.  Given any valuation $\Valuation$, we abuse notation in the usual way and define the \textbf{value} of $\aff$ (according to $\Valuation$) linearly by:
  \[
    \Valuation(\aff) := \affconst + \feature_1\Valuation(\phi_1) + \cdots + \feature_k\Valuation(\phi_k).
  \]
  An \textbf{$\bm\features$-combination progression} is a sequence $\seq\aff$ of affine combinations where $\aff_\nn$ has rank $\le n$. An \textbf{$\bm\exfeatures$-combination progression} is defined similarly.
\end{definition}

\noindent Note that a trade $\trade$ is an $\features$-combination, and the holdings $\trade \ftn$ from $\trade$ against $\MP$ is a $\QQ$-combination. We will use affine combinations to encode the net holdings $\sum_{i\leq n}\trade_i(\MP)$ of a trader after interacting with a market $\MP$, and later to encode linear inequalities that hold between the truth values of different sentences.

\subsection{Exploitation}\label{sec:exploitation}

We will now define exploitation, beginning with an example. Let $\Lang$ be the language of $\PA$, and $\DP$ be a $\PA$-complete deductive process. Consider a market $\MP$ that assigns $\pt_\nn(\quot{1+1=2})=0.5$ for all $\nn$, and a trader who buys one share of $\quot{1+1=2}$ each day. Imagine a reasoner behind the market obligated to buy and sell shares at the listed prices, who is also obligated to pay out \$1 to holders of $\phi$-shares if and when $\DP$ says $\phi$. Let~$t$ be the first day when $\quot{1+1=2}\in\dt_t$. On each day, the reasoner receives 50\textcent{} from $\Trader$, but after day~$t$, the reasoner must pay \$1 every day thereafter. They lose~50\textcent{} each day, and $\Trader$ gains~50\textcent{} each day, despite the fact that $\Trader$ never risked more than \$$t/2$. In cases like these, we say that $\Trader$ exploits $\MP$.

With this example in mind, we define exploitation as follows:
\begin{keydef}[Exploitation]\label{def:exploitation}
  A trader $\Trader$ is said to \textbf{exploit} a valuation sequence $\seq\Valuation$ relative to a deductive process $\DP$ if the set of values
  \[
    \left\{ \World\mleft({\textstyle \sum_{i \leq n} \trade_i\mleft(\seq\Valuation\mright)}\mright) \,\middle|\, \nn\in\NN^+, \World\in\pcworlds(\dt_\nn) \right\}
  \]
  is bounded below, but not bounded above.
\end{keydef}

Given a world $\World$, the number $\World(\sum_{i \leq \nn}\trade_i(\MP))$ is the value of the trader's net holdings after interacting with the market $\MP$, where a share of $\phi$ is valued at \$1 if $\phi$ is true in $\World$ and \$0 otherwise. The set $\{ \World(\sum_{i \leq n} \trade_i(\MP)) \mid \nn\in\NN^+, \World\in\pcworlds(\dt_\nn) \}$ is the set of all assessments of $\Trader$'s net worth, across all time, according to worlds that were propositionally consistent with $\DP$ at the time. We informally call these \emph{plausible assessments} of the trader's net worth. Using this terminology, \Def{exploitation} says that a trader exploits the market if their plausible net worth is bounded below, but not above.

Roughly speaking, we can imagine that there is a person behind the market who acts as a market maker, obligated to buy and sell shares at the listed prices. We can imagine that anyone who sold a $\phi$-share is obligated to pay \$1 if and when $\DP$ says $\phi$. Then, very roughly, a trader exploits the market if they are able to make unbounded returns off of a finite investment.

This analogy is illustrative but incomplete---traders can exploit the market even if they never purchase a sentence that appears in $\DP$. For example, let $\phi$ and $\psi$ be two sentences such that $(\phi \lor \psi)$ is provable in $\PA$, but such that neither $\phi$ nor $\psi$ is provable in $\PA$. Consider a trader that bought 10 $\phi$-shares at a price of 20\textcent{} each, and 10 $\psi$-shares at a price of 30\textcent{} each. Once $\DP$ says $(\phi\lor\psi)$, all remaining p.c.\ worlds will agree that the portfolio $-5 + 10\phi + 10\psi$ has a value of at least +5, despite the fact that neither $\phi$ nor $\psi$ is ever proven. If the trader is allowed to keep buying $\phi$ and $\psi$ shares at those prices, they would exploit the market, despite the fact that they never buy decidable sentences. In other words, our notion of exploitation rewards traders for arbitrage, even if they arbitrage between sentences that never ``pay out''.

\subsection{Main Result}\label{sec:mainthm}

Recall the \lic{}:

\criterion*

\noindent We may now state our main result:

\begin{key}
\begin{restatable}{theorem}{logindcri}\label{thm:li}
  For any deductive process $\DP$, there exists a computable belief sequence $\MP$ satisfying the \lic{} relative to $\DP$.
\end{restatable}
\end{key}

\begin{proof}
  In \Sec{construct}, we show how to take an arbitrary deductive process $\DP$ and construct a computable belief sequence $\seq\LIA$. \Thm{lia} shows that $\seq\LIA$ is a logical inductor relative to the given $\DP$.
\end{proof}

\begin{definition}[Logical Inductor over $\Theory$]
  Given a theory $\Theory$, a \li{} over a $\Theory$-complete deductive process $\DP$ is called a \textbf{\li{} over $\boldsymbol\Theory$}.
\end{definition}

\begin{corollary}
  For any recursively axiomatizable theory $\Theory$, there exists a computable belief sequence that is a logical inductor over $\Theory$.
\end{corollary}

%
%
%
%
%
%
%
%
%
%
%
%
%
%
%
%
%

\section{Properties of \LItitle{}s}\label{sec:properties}
Here is an intuitive argument that logical inductors perform good reasoning under logical uncertainty:

\begin{quote}
  Consider any polynomial-time method for efficiently identifying patterns in logic. If the market prices don't learn to reflect that pattern, a clever trader can use that pattern to exploit the market. Thus, a logical inductor must learn to identify those patterns.
\end{quote}

\noindent In this section, we will provide evidence supporting this intuitive argument, by demonstrating a number of desirable properties possessed by logical inductors. The properties that we demonstrate are broken into twelve categories:

\begin{enumerate}
  \item \textbf{Convergence and Coherence:} In the limit, the prices of a logical inductor describe a belief state which is fully logically consistent, and represents a probability distribution over all consistent worlds.
  \item \textbf{Timely Learning:} For any efficiently computable sequence of theorems, a logical inductor learns to assign them high probability in a timely manner, regardless of how difficult they are to prove. (And similarly for assigning low probabilities to refutable statements.)
  \item \textbf{Calibration and Unbiasedness:} Logical inductors are well-calibrated and, given good feedback, unbiased.
  \item \textbf{Learning Statistical Patterns:} If a sequence of sentences appears pseudorandom to all reasoners with the same runtime as the logical inductor, it learns the appropriate statistical summary (assigning, e.g., 10\% probability to the claim ``the $n$th digit of $\pi$ is a 7'' for large $n$, if digits of $\pi$ are actually hard to predict).
  \item \textbf{Learning Logical Relationships:} Logical inductors inductively learn to respect logical constraints that hold between different types of claims, such as by ensuring that mutually exclusive sentences have probabilities summing to at most~1.
  \item \textbf{Non-Dogmatism:} The probability that a logical inductor assigns to an independent sentence $\phi$ is bounded away from 0 and 1 in the limit, by an amount dependent on the complexity of $\phi$. In fact, logical inductors strictly dominate the universal semimeasure in the limit. This means that we can condition logical inductors on independent sentences, and when we do, they perform empirical induction.
  \item \textbf{Conditionals:} Given a logical inductor $\MP$, the market given by the conditional probabilities $\MP(\any\mid\psi)$ is a logical inductor over $\DP$ extended to include $\psi$. Thus, when we condition logical inductors on new axioms, they continue to perform logical induction.
  \item \textbf{Expectations:} Logical inductors give rise to a well-behaved notion of the expected value of a logically uncertain variable.
  \item \textbf{Trust in Consistency:} If the theory $\Theory$ underlying a logical inductor's deductive process is expressive enough to talk about itself, then the logical inductor learns inductively to trust $\Theory$.
  \item \textbf{Reasoning about Halting:} If there's an efficient method for generating programs that halt, a logical inductor will learn in a timely manner that those programs halt (often long before having the resources to evaluate them). If there's an efficient method for generating programs that don't halt, a logical inductor will at least learn not to expect them to halt for a very long time.
  \item \textbf{Introspection:} Logical inductors ``know what they know'', in that their beliefs about their current probabilities and expectations are accurate.
  \item \textbf{Self-Trust:} Logical inductors trust their future beliefs.
\end{enumerate}

\noindent For the sake of brevity, proofs are deferred to \Sec{selectedproofs} and the appendix. Some example proofs are sketched in this section, by outlining discontinuous traders that would exploit any market that lacked the desired property. The deferred proofs define polynomial-time continuous traders that approximate those discontinuous strategies.

In what follows, let $\Lang$ be a language of propositional logic; let $\Sentences$ be the set of sentences written in $\Lang$; let $\Theory \subset \Sentences$ be a computably enumerable set of propositional formulas written in $\Lang$ (such as \PA, where the propositional variables are prime sentences in first-order logic, as discussed in \Sec{notation}); and let $\MP$ be a computable logical inductor over $\Theory$, i.e., a market satisfying the \lic{} relative to some  $\Theory$-complete deductive process $\DP$. We assume in this section that $\Theory$ is consistent.
 
Note that while the computable belief sequence $\seq\LIA$ that we define has finite support on each day, in this section we assume only that $\MP$ is a market. We do this because our results below hold in this more general case, and can be applied to $\seq\LIA$ as a special case.

In sections~\ref{sec:expectations}-\ref{sec:selftrust} we will assume that $\Theory$ \representscomputations.  This assumption is not necessary until \Sec{expectations}.

\subsection{Convergence and Coherence}\label{sec:limitprops}

Firstly, the market prices of a \li{} converge:

\begin{restatable}[Convergence]{theorem}{convergence}\label{thm:con}\label{thm:first}
  The limit ${\pt_\infty:\Sentences\rightarrow[0,1]}$ defined by \[\pt_\infty(\phi) := \lim_{\nn\rightarrow\infty} \pt_\nn(\phi)\] exists for all $\phi$.
\end{restatable}
\begin{restatable}[\ref{sec:convergence} or~\ref{app:con}]{sketch}{sketchcon}
  Roughly speaking, if $\MP$ never makes up its mind about $\phi$, then it can be exploited by a trader arbitraging shares of $\phi$ across different days. More precisely, suppose by way of contradiction that the limit $\pt_\infty(\phi)$ does not exist. Then for some $p\in [0, 1]$ and $\varepsilon > 0$, we have $\pt_\nn(\phi) < p-\varepsilon$ infinitely often and also $\pt_\nn(\phi) > p+\varepsilon$  infinitely often. A trader can wait until $\pt_\nn(\phi) < p-\varepsilon$ and then buy a share in $\phi$ at the low market price of $\pt_\nn(\phi)$. Then the trader waits until some later $m$ such that $\pt_m(\phi) > p+\varepsilon$, and sells back the share in $\phi$ at the higher price. This trader makes a total profit of $2\varepsilon$ every time $\pt_\nn(\phi)$ oscillates in this way, at no risk, and therefore exploits $\MP$. Since $\MP$ implements a logical inductor, this is not possible; therefore the limit $\pt_\infty(\phi)$ must in fact exist.
\end{restatable}

\noindent This sketch showcases the main intuition for the convergence of $\MP$, but elides a number of crucial details. In particular, the trader we have sketched makes use of discontinuous trading functions, and so is not a well-formed trader. These details are treated in \Sec{convergence}.

Next, the limiting beliefs of a \li{} represent a coherent probability distribution:

\begin{restatable}[Limit Coherence]{theorem}{limitcoherence}\label{thm:lc}
  $\pt_\infty$ is coherent, i.e., it gives rise to an internally consistent probability measure $\Bayesian$ on the set $\cworlds(\Theory)$ of all worlds consistent with $\Theory$, defined by the formula \[\Bayesian(\World(\phi)=1):=\pt_\infty(\phi).\]
  In particular, if $\Theory$ contains the axioms of first-order logic, then $\pt_\infty$ defines a probability measure on the set of first-order completions of~$\Theory$.
\end{restatable}
\begin{sketch}[\ref{sec:limitcoherence} or~\ref{app:lc}]
  The limit $\pt_\infty(\phi)$ exists by the convergence theorem, so $\Bayesian$ is well-defined. \citet{Gaifman:1964} shows that $\Bayesian$ defines a probability measure over $\pcworlds(\dt_\infty)$ so long as the following three implications hold for all sentences $\phi$ and $\psi$:
  \begin{itemize}
    \item If $\Theory \vdash \phi$, then $\pt_\infty(\phi) = 1$,
    \item If $\Theory \vdash \lnot \phi$, then $\pt_\infty(\phi) = 0$,
    \item If $\Theory \vdash \lnot(\phi \land \psi)$, then $\pt_\infty(\phi \lor \psi) = \pt_\infty(\phi) + \pt_\infty(\psi)$.
  \end{itemize}
  Let us demonstrate each of these three properties.

  First suppose that $\Theory \vdash \phi$, but $\pt_\infty(\phi) = 1-\varepsilon$ for some $\varepsilon>0$. Then shares of $\phi$ will be underpriced, as they are worth 1 in every consistent world, but only cost $1-\varepsilon$. There is a trader who waits until $\phi$ is propositionally provable from $\dt_\nn$, and until $\pt_\nn(\phi)$ has approximately converged, and then starts buying shares of $\phi$ every day at the price $\pt_\nn(\phi)$. Since $\phi$ has appeared in $\DP$, the shares immediately have a minimum plausible value of \$1. Thus the trader makes $1-\pt_\nn(\phi) \approx \varepsilon$ profit every day, earning an unbounded total value, contradicting the \lic{}. But $\MP$ cannot be exploited, so $\pt_\infty(\phi)$ must be~1.

  Similarly, if $\Theory \vdash \lnot\phi$ but $\pt_\infty(\phi) = \varepsilon>0$, then a trader could exploit $\MP$ by selling off shares in $\phi$ for a profit of $\pt_\nn(\phi)\approx \varepsilon$ each day.

  Finally, suppose that $\Theory \vdash \lnot(\phi \land \psi)$, but for some $\varepsilon > 0$,
  \[
    \pt_\infty(\phi \lor \psi) = \pt_\infty(\phi) + \pt_\infty(\psi) \pm \varepsilon.
  \] 
  Then there is a trader that waits until $\pt_n$ has approximately converged on these sentences, and until  $\lnot(\phi \land \psi)$ is propositionally provable from $\dt_\nn$. At that point it's a good deal to sell (buy) a share in $\phi \lor \psi$, and buy (sell) a share in each of $\phi$ and $\psi$; the stocks will have values that cancel out in every plausible world. Thus this trader makes a profit of $\approx \varepsilon$ from the price differential, and can then repeat the process. Thus, they would exploit $\MP$. But this is impossible, so $\pt_\infty$ must be coherent.
\end{sketch}

\Thm{lc} says that if $\MP$ were allowed to run forever, and we interpreted its prices as probabilities, then we would find its beliefs to be perfectly consistent. In the limit, $\MP$ assigns probability~1 to every theorem and~0 to every contradiction. On independent sentences, its beliefs obey the constraints of probability theory; if $\phi$ provably implies $\psi$, then the probability of $\psi$ converges to a point no lower than the limiting probability of $\phi$, regardless of whether they are decidable. The resulting probabilities correspond to a probability distribution over all possible ways that $\Theory$ could be completed.

This justifies interpreting the market prices of a \li{} as probabilities. Logical inductors are not the first computable procedure for assigning probabilities to sentences in a manner that is coherent in the limit; the algorithm of \citet{Demski:2012a} also has this property. The main appeal of logical induction is that their beliefs become reasonable in a timely manner, outpacing the underlying deductive process.

\subsection{Timely Learning}\label{sec:timelylearning}

It is not too difficult to define a reasoner that assigns probability~1 to all (and only) the provable sentences, in the limit: simply assign probability 0 to all sentences, and then enumerate all logical proofs, and assign probability~1 to the proven sentences. The real trick is to recognize patterns in a timely manner, well before the sentences can be proven by slow deduction.

Logical inductors learn to outpace deduction on any efficiently computable sequence of provable statements.\footnote{Recall that a sequence $\seq{x}$ is efficiently computable iff there exists a computable function $\nn \mapsto x_\nn$ with runtime polynomial in $\nn$.} To illustrate, consider our canonical example where $\dt_\nn$ is the set of all theorems of $\PA$ provable in at most $\nn$ characters, and suppose $\phis$ is an \ec sequence of theorems which are easy to generate but difficult to prove. Let $f(\nn)$ be the length of the shortest proof of $\phi_\nn$, and assume that $f$ is some fast-growing function. At any given time $\nn$, the statement $\phi_\nn$ is ever further out beyond $\dt_\nn$---it might take 1 day to prove $\phi_1$, 10 days to prove $\phi_2$, 100 days to prove $\phi_3$, and so on. One might therefore expect that $\phi_\nn$ will also be ``out of reach'' for $\pt_\nn$, and that we have to wait until a much later day close to $f(\nn)$ before expecting $\pt_{f(\nn)}(\phi_\nn)$ to be accurate. However, this is not the case! After some finite time $\nN$, $\smash{\MP}$ will recognize the pattern and begin assigning high probability to $\phis$ in a timely manner.

\begin{restatable}[Provability Induction]{theorem}{restatepi}\label{thm:provind}\label{thm:patfirst}
  Let $\phis$ be an \ec sequence of theorems. Then
  \[
    \pt_\nn(\phi_\nn) \eqsim_\nn 1.
  \]
  Furthermore, let $\psis$ be an \ec sequence of disprovable sentences. Then
  \[
    \pt_\nn(\psi_\nn) \eqsim_\nn 0.
  \]
\end{restatable}
\begin{sketch}[\ref{sec:provind} or~\ref{app:provind}]
  Consider a trader that acts as follows. First wait until the time $a$ when $\pt_a(\phi_a)$ drops below $1 - \varepsilon$ and buy a share of $\phi_a$. Then wait until $\phi_a$ is worth 1 in all worlds plausible at time $f(a)$. Then repeat this process. If $\pt_\nn(\phi_\nn)$ drops below $1 - \varepsilon$ infinitely often, then this trader makes $\varepsilon$ profit infinitely often, off of an initial investment of \$1, and therefore exploits the market. $\MP$ is inexploitable, so $\pt_\nn(\phi_\nn)$ must converge to~1. By a similar argument, $\pt_\nn(\psi_\nn)$ must converge to~0.\footnote{The traders sketched here are optimized for ease of proof, not for efficiency---a clever trader trying to profit from low prices on efficiently computable theorems would be able to exploit the market faster than this.}
\end{sketch}

In other words, $\MP$ will learn to start believing $\phi_\nn$ by day $\nn$ at the latest, despite the fact that $\phi_\nn$ won't be deductively confirmed until day $f(\nn)$, which is potentially much later. In colloquial terms, if $\phis$ is a sequence of facts that can be generated efficiently, then $\MP$ inductively learns the pattern, and its belief in $\phis$ becomes accurate faster than $\DP$ can computationally verify the individual sentences.

For example, imagine that \texttt{prg(n)} is a program with fast-growing runtime, which always outputs either 0, 1, or 2 for all $\nn$, but such that there is no proof of this in the general case. Then
\[
  \quot{\forall x \colon \texttt{prg($x$)}=0 \lor \texttt{prg($x$)}=1 \lor \texttt{prg($x$)}=2}
\]
is \emph{not} provable. Now consider the sequence of statements
\[
  \seq{\prgzot} := \big( \quot{\texttt{prg($\enc{n}$)}=0 \lor \texttt{prg($\enc{n}$)}=1 \lor \texttt{prg($\enc{n}$)}=2} \big)_{\nn\in {\NN^+}}
\]
where each $\prgzot_\nn$ states that \texttt{prg} outputs a 0, 1, or 2 on that $\nn$ in particular. Each individual $\prgzot_\nn$ is provable (it can be proven by running \texttt{prg} on input $n$), and $\seq{\prgzot}$ is efficiently computable (because the sentences themselves can be written down quickly, even if \texttt{prg} is very difficult to evaluate). Thus, provability induction says that any logical inductor will ``learn the pattern'' and start assigning high probabilities to each individual $\prgzot_\nn$ no later than day~$\nn$.

Imagine that $\DP$ won't determine the output of \texttt{prg($\nn$)} until the $f(\nn)$th day, by evaluating \texttt{prg($\nn$)} in full. Provability induction says that $\MP$ will eventually recognize the pattern $\seq{\prgzot}$ and start assigning high probability to $\prgzot_\nn$ no later than the $\nn$th day, $f(\nn)-\nn$ days before the evaluation finishes. This is true regardless of the size of $f(n)$, so if $f$ is fast-growing, $\smash{\MP}$ will outpace $\smash{\DP}$ by an ever-growing margin.

\begin{quote}
  \textbf{Analogy: Ramanujan and Hardy.} Imagine that the statements $\phis$ are being output by an algorithm that uses heuristics to generate mathematical facts without proofs, playing a role similar to the famously brilliant, often-unrigorous mathematician Srinivasa Ramanujan. Then $\MP$ plays the historical role of the beliefs of the rigorous G.H.\ Hardy who tries to verify those results according to a slow deductive process ($\smash{\DP}$). After Hardy ($\MP$) verifies enough of Ramanujan's claims ($\phi_{\le \nn}$), he begins to trust Ramanujan, even if the proofs of Ramanujan's later conjectures are incredibly long, putting them ever-further beyond Hardy's current abilities to rigorously verify them. In this story, Hardy's inductive reasoning (and Ramanujan's also) outpaces his deductive reasoning.
\end{quote}

This idiom of assigning the right probabilities to $\phi_\nn$ no later than day~$\nn$ will be common throughout the paper, so we give it a name.

\begin{definition}[Timely Manner]\label{def:timelymanner}
  Let $\phis$ be an \ec sequence of sentences, and $\probs$ be an \ec sequence of rational numbers. We say that $\MP$ assigns $\probs$ to $\phis$ in a \textbf{timely manner} if for every $\varepsilon > 0$, there exists a time $\nN$ such that for all $\nn > \nN$,
  \[
    |\pt_\nn(\phi_\nn) - \prob_\nn| < \varepsilon.
  \]
  In other words, $\MP$ assigns $\probs$ to $\phis$ in a timely manner if
  \[
    \pt_\nn(\phi_\nn) \eqsim_\nn \prob_\nn.
  \]
\end{definition}

\noindent Note that there are no requirements on how large $\nN$ gets as a function of $\varepsilon$. As such, when we say that $\MP$ assigns probabilities $\probs$ to $\phis$ in a timely manner, it may take a very long time for convergence to occur. (See \Sec{bounds} for a discussion.)

As an example, imagine the reasoner who recognizes that sentences of the form $\quot{1+1+\cdots+1\text{\ is even}}$ are true iff the number of ones is even. Let $\phis$ be the sequence where $\phi_\nn$ is the version of that sentence with $2n$ ones. If the reasoner starts writing a probability near 100\% in the $\phi_\nn$ cell by day $\nn$ at the latest, then intuitively, she has begun incorporating the pattern into her beliefs, and we say that she is assigning high probabilities to $\phis$ in a timely manner.

We can visualize ourselves as taking $\MP$'s belief states, sorting them by $\phis$ on one axis and days on another, and then looking at the main diagonal of cells, to check the probability of each $\phi_\nn$ on day $\nn$. Checking the $\nn$th sentence on the $\nn$th day is a rather arbitrary choice, and we might hope that a good reasoner would assign high probabilities to \ec sequences of theorems at a faster rate than that. It is easy to show that this is the case, by the closure properties of efficient computability. For example, if $\phis$ is an \ec sequence of theorems, then so are $\phis_{2n}$ and $\phis_{2n+1}$, which each enumerate half of $\smash{\phis}$ at twice the speed, so by \Theorem{provind}, $\MP$ will eventually learn to believe $\phis$ at a rate of at least two per day. Similarly, $\MP$ will learn to believe $\phis_{3n}$ and $\phis_{n^2}$ and $\phis_{10n^3 + 3}$ in a timely manner, and so on. Thus, up to polynomial transformations, it doesn't really matter which diagonal we check when checking whether a logical inductor has begun ``noticing a pattern''.

Furthermore, we will show that if $\MP$ assigns the correct probability on the main diagonal, then $\MP$ also learns to keep them there:

\begin{restatable}[Persistence of Knowledge]{theorem}{perkno}\label{thm:perkno}
  Let $\phis$ be an \ec sequence of sentences, and $\probs$ be an \ec sequence of rational-number probabilities. If $\pt_\infty(\phi_\nn)\eqsim_\nn \prob_\nn$, then
  \[
    \sup_{\mm\ge\nn}|\pt_\mm(\phi_\nn)-\prob_\nn|\eqsim_\nn 0.
  \] 
  Furthermore, if $\pt_\infty(\phi_\nn)\lesssim_\nn \prob_\nn$, then
  \[
    \sup_{\mm\geq \nn}\pt_\mm(\phi_\nn)\lesssim_\nn \prob_\nn,
  \]
  and if $\pt_\infty(\phi_\nn)\gtrsim_\nn \prob_\nn$, then
  \[
   \inf_{\mm\geq \nn}\pt_\mm(\phi_\nn)\gtrsim_\nn \prob_\nn.
  \]
  \proofin{\ref{app:perkno}}
\end{restatable}
\noindent In other words, if $\MP$ assigns $\probs$ to $\phis$ in the limit, then $\MP$ learns to assign probability near~$\prob_\nn$ to~$\phi_\nn$ at all times $\mm \ge \nn$. This theorem paired with the closure properties of the set of efficiently computable sequences means that checking the probability of $\phi_\nn$ on the $\nn$th day is a fine way to check whether $\MP$ has begun recognizing a pattern encoded by $\phis$. As such, we invite the reader to be on the lookout for statements of the form $\pt_\nn(\phi_\nn)$ as signs that $\MP$ is recognizing a pattern, often in a way that outpaces the underlying deductive process.

Theorems~\ref{thm:provind} (\nameref{thm:provind}) and~\ref{thm:perkno} (\nameref{thm:perkno}) only apply when the pattern of limiting probabilities is itself efficiently computable. For example, consider the sequence of sentences
\[
  \seq{\piasev} := \big(\quot{\enc{\pi}[\enc{\Ack}(\enc{\nn}, \enc{\nn})] = 7}\big)_{\nn\in {\NN^+}}
\]
where $\pi[i]$ is the $i$th digit in the decimal expansion of $\pi$ and $\Ack$ is the Ackermann function. Each individual sentence is decidable, so the limiting probabilities are~0 for some $\piasev_\nn$ and~1 for others. But that pattern of 1s and 0s is not efficiently computable (assuming there is no efficient way to predict the Ackermann digits of $\pi$), so provability induction has nothing to say on the topic.

In cases where the pattern of limiting probabilities are not \ec[,] we can still show that if $\MP$ is going to make its probabilities follow a certain pattern eventually, then it learns to make its probabilities follow that pattern in a timely manner. For instance, assume that each individual sentence $\piasev_\nn$ (for $n > 4$) is going to spend a long time sitting at 10\% probability before eventually being resolved to either~1 or~0. Then $\MP$ will learn to assign $\pt_\nn(\piasev_\nn) \approx 0.1$ in a timely manner:

\begin{restatable}[Preemptive Learning]{theorem}{restatetbo}\label{thm:tbo}
  Let $\phis$ be an \ec sequence of sentences. Then
  \[
    \liminf_{\nn \to\infty} \pt_\nn(\phi_\nn) = \liminf_{\nn \to\infty} \sup_{\mm \ge \nn} \pt_\mm(\phi_\nn).
  \]
  Furthermore,
  \[
    \limsup_{\nn \to\infty} \pt_\nn(\phi_\nn) = \limsup_{\nn \to\infty}\inf_{\mm \ge \nn} \pt_\mm(\phi_\nn).
  \]
  \proofin{\ref{app:tbo}}
\end{restatable}

Let's unpack \Thm{tbo}. The quantity $\sup_{\mm \ge \nn} \pt_\mm(\phi_\nn)$ is an upper bound on the price $\pt_\mm(\phi_\nn)$ on or after day $\nn$, which we can interpret as the highest price tag that that $\MP$ will ever put on $\phi_\nn$ after we first start checking it on day $\nn$. We can imagine a sequence of these values: On day $\nn$, we start watching $\phi_\nn$. As time goes on, its price travels up and down until eventually settling somewhere. This happens for each $\nn$. The limit infimum of $\sup_{\mm \ge \nn} \pt_\mm(\phi_\nn)$ is the greatest lower bound $p$ past which a generic $\phi_\nn$ (for $\nn$ large) will definitely be pushed after we started watching it. \nameref{thm:tbo} says that if $\MP$ always eventually pushes $\phi_\nn$ up to a probability at least $p$, then it will learn to assign each $\phi_\nn$ a probability at least $p$ in a timely manner (and similarly for least upper bounds).

For example, if each individual $\piasev_\nn$ is \emph{eventually} recognized as a claim about digits of $\pi$ and placed at probability 10\% for a long time before being resolved, then $\MP$ learns to assign it probability 10\% on the main diagonal. In general, if $\MP$ is going to learn a pattern eventually, it learns it in a timely manner.

This leaves open the question of whether a logical inductor $\MP$ is smart enough to recognize that the $\seq{\piasev}$ should each have probability 10\% before they are settled (assuming the Ackermann digits of $\pi$ are hard to predict). We will return to that question in \Sec{statpatterns}, but first, we examine the reverse question.

\subsection{Calibration and Unbiasedness}\label{sec:calibration}

\Theorem{provind} shows that logical inductors are good at detecting patterns in what is provable. Next, we ask: when a logical inductor learns a pattern, when must that pattern be real? In common parlance, a source of probabilistic estimates is called \emph{well calibrated} if among statements where it assigns a probability near~$\prob$, the estimates are correct with frequency roughly~$\prob$.

In the case of reasoning under logical uncertainty, measuring calibration is not easy. Consider the sequence $\seq\clusters$ constructed from correlated clusters of size 1, 10, 100, 1000, \ldots, where the truth value of each cluster is determined by the parity of a late digit of $\pi$:
\begin{align*}
  \clusters_1 :\iff & \quot{\pi[\Ack(1,1)]\text{ is even}} \\
  \clusters_2 :\iff \cdots :\iff \clusters_{11} :\iff & \quot{\pi[\Ack(2,2)]\text{ is even}} \\
  \clusters_{12} :\iff \cdots :\iff \clusters_{111} :\iff & \quot{\pi[\Ack(3,3)]\text{ is even}} \\
  \clusters_{112} :\iff \cdots :\iff \clusters_{1111} :\iff & \quot{\pi[\Ack(4,4)]\text{ is even}}
\end{align*}
and so on. A reasoner who can't predict the parity of the Ackermann digits of $\pi$ should assign 50\% (marginal) probability to any individual $\clusters_\nn$ for $\nn$ large. But consider what happens if the 9th cluster turns out to be true, and the next billion sentences are all true. A reasoner who assigned 50\% to those billion sentences was assigning the \emph{right} probabilities, but their calibration is abysmal: on the billionth day, they have assigned 50\% probability a billion sentences that were overwhelmingly true. And if the 12th cluster comes up false, then on the trillionth day, they have assigned 50\% probability to a \emph{trillion} sentences that were overwhelmingly false! In cases like these, the frequency of truth oscillates eternally, and the good reasoner only appears well-calibrated on the rare days where it crosses 50\%.

The natural way to correct for correlations such as these is to check $\MP$'s conditional probabilities instead of its marginal probabilities. This doesn't work very well in our setting, because given a logical sentence $\phi$, the quantity that we care about will almost always be the marginal probability of $\phi$. The reason we deal with sequences is because that lets us show that $\phi$ has reasonable probabilities relative to various related sentences. For example, if $\phi := \quot{\texttt{prg}(32)=17}$, then we can use our theorems to relate the probability of $\phi$ to the probability of the sequence $(\quot{\texttt{prg}(\enc{\nn})=17})_{\nn \in \NN^+}$, and to the sequence $(\quot{\texttt{prg}(32)=\enc{\nn}})_{\nn \in \NN^+}$, and to the sequence $(\quot{\texttt{prg}(\enc{\nn})>\enc{\nn}})_{\nn \in \NN^+}$, and so on, to show that $\phi$ eventually has reasonable beliefs about \texttt{prg} (hopefully before $\MP$ has the resources to simply evaluate \texttt{prg} on input $32$). But at the end of the day, we'll want to reason about the marginal probability of $\phi$ itself. In this case, approximately-well-calibrated conditional probabilities wouldn't buy us much: there are $2^{n-1}$ possible truth assignments to the first $n-1$ elements of $\phis$, so if we try to compute the marginal probability of $\phi_\nn$ from all the different conditional probabilities, exponentially many small errors would render the answer useless. Furthermore, intuitively, if $\phis$ is utterly unpredictable to $\MP$, then the probabilities of all the different truth assignments to $\phi_{\le \nn-1}$ will go to 0 as $\nn$ gets large, which means the conditional probabilities won't necessarily be reasonable. (In \Sec{statpatterns} will formalize a notion of pseudorandomness.)

Despite these difficulties, we can recover some good calibration properties on the marginal probabilities if we either (a) restrict our consideration to sequences where the average frequency of truth converges; or (b) look at subsequences of $\phis$ where $\MP$ has ``good feedback'' about the truth values of previous elements of the subsequence, in a manner defined below.

To state our first calibration property, we will define two different sorts of indicator functions that will prove useful in many different contexts.

\begin{definition}[Theorem Indicator]\label{def:thmind}
  Given a sentence $\phi$, define $\thmind(\phi)$ to be 1 if $\Theory \vdash \phi$ and 0 otherwise.
\end{definition}

\begin{definition}[Continuous Threshold Indicator]\label{def:ctsind}
  Let $\delta > 0$ be a rational number, and $x$ and $y$ be real numbers. We then define
  \[
    \ctsind{\delta}(x > y) :=
    \begin{dcases}
      0&\text{if }  x\leq y\\
      \frac{x-y}{\delta} &\text{if }\hphantom{x\leq\;} y < x\leq y+\delta\\
      1&\text{if } \hphantom{y\leq\; y < x\leq\;} y+\delta < x.
    \end{dcases}
  \]
  Notice that $\ctsind{\delta}(x > y)$ has no false positives, and that it is linear in the region between $y$ and $y+\delta$. We define $\ctsind{\delta}(x < y)$ analogously, and we define
  \[
    \ctsind{\delta}(a < x < b) := \min( \ctsind{\delta}(x > a), \ctsind{\delta}(x < b) ).
  \]
  Observe that we can generalize this definition to the case where $x$ and $y$ are expressible features, in which case ${\ctsind{\delta}(x > y)}$ is an expressible $[0,1]$-feature.
\end{definition}

Now we can state our calibration theorem.

\begin{restatable}[Recurring Calibration]{theorem}{simcal}\label{thm:simcal}
  Let $\phis$ be an \ec sequence of decidable sentences, $a$ and $b$ be rational numbers, $\deltas$ be an \ec sequence of positive rational numbers, and suppose that $\sum_\nn \left(\ctsind{\delta_i}(a<\pt_i(\phi_i)<b)\right)_{i \in \NN^+} = \infty$. Then, if the sequence
  \[
    \left(
    \frac
      {\sum_{i \leq \nn} \ctsind{\delta_i}(a < \pt_i(\phi_i) < b) \cdot \thmind(\phi_i)}
      {\sum_{i \leq \nn} \ctsind{\delta_i}(a < \pt_i(\phi_i) < b)}
    \right)_{\nn\in\NN^+}
  \]
  converges, it converges to a point in $[a, b]$. Furthermore, if it diverges, it has a limit point in $[a, b]$.
  \proofin{\ref{app:simcal}}
\end{restatable}

\noindent Roughly, this says that if $\pt_\nn(\phi_\nn)\approx 80\%$ infinitely often, then if we look at the subsequence where it's 80\%, the limiting frequency of truth on that subsequence is 80\% (if it converges).

In colloquial terms, on subsequences where $\MP$ says 80\% and it makes sense to talk about the frequency of truth, the frequency of truth is 80\%, i.e., $\MP$ isn't seeing shadows. If the frequency of truth diverges---as in the case with $\seq\clusters$---then $\MP$ is still well-calibrated infinitely often, but its calibration might still appear abysmal at times (if they can't predict the swings).

Note that calibration alone is not a very strong property: a reasoner can always cheat to improve their calibration (i.e., by assigning probability 80\% to things that they're sure are true, in order to bring up the average truth of their ``80\%'' predictions). What we really want is some notion of ``unbiasedness'', which says that there is no efficient method for detecting a predictable bias in a logical inductor's beliefs. This is something we can get on sequences where the limiting frequency of truth converges, though again, if the limiting frequency of truth diverges, all we can guarantee is a limit point.

\begin{definition}[Divergent Weighting]\label{def:fuz}
  A \textbf{divergent weighting} $\seq\fuz \in [0, 1]^{\NN^+}$ is an infinite sequence of real numbers in $[0, 1]$, such that $\sum_\nn \fuz_n = \infty$.
\end{definition}

\noindent Note that divergent weightings have codomain $[0, 1]$ as opposed to $\{0,1\}$, meaning the weightings may single out fuzzy subsets of the sequence. For purposes of intuition, imagine that $\seq\fuz$ is a sequence of 0s and 1s, in which case each $\seq\fuz$ can be interpreted as a subsequence. The constraint that the $\fuz_{\nn}$ sum to $\infty$ ensures that this subsequence is infinite.

\begin{definition}[Generable From $\MP$]\label{def:ece}
  A sequence of rational numbers $\seq q$ is called \textbf{generable from $\bm{\MP}$} if 
  there exists an \ec $\exfeatures$-progression $\seq{\gen q}$ such that $\gen{q_\nn}\ftn = q_\nn$ for all~$\nn$. In this case we say that $\seq q$ is \textbf{\pgenable}. \pgenable $\RR$-sequences, $\QQ$-combination sequences, and $\RR$-combination sequences are defined analogously.
\end{definition}

Divergent weightings generable from $\MP$ are fuzzy subsequences that are allowed to depend continuously (via expressible market features) on the market history. For example, the sequence $(\ctsind{0.01}(\pt_\nn(\phi_\nn) > 0.5))_{\nn\in\NN^+}$ is a \pgenable sequence that singles out all times $\nn$ when $\pt_\nn(\phi_\nn)$ is greater than 50\%. Note that the set of \pgenable divergent weightings is larger than the set of \ec divergent weightings, as the \pgenable weightings are allowed to vary continuously with the market prices.

\begin{restatable}[Recurring Unbiasedness]{theorem}{recurringunbiasedness}\label{thm:recurringunbiasedness}
  Given an \ec sequence of decidable sentences $\phis$ and a \pgenable divergent weighting $\seq\fuz$, the sequence
  \[
      \frac
        {\sum_{i\leq \nn}\fuz_i  \cdot(\pt_i(\phi_i)-\thmind(\phi_i))}
        {\sum_{i\leq \nn}\fuz_i}
  \]
  has $0$ as a limit point. In particular, if it converges, it converges to~$0$.
  \par\proofin{\ref{app:recurringunbiasedness}}
\end{restatable}

\noindent Letting $\seq\fuz=(1,1,\ldots)$, this theorem says that the difference between the average probability $\pt_\nn(\phi_\nn)$ and the average frequency of truth is 0 infinitely often (and 0 always, if the latter converges). Letting each $\fuz_\nn$ be $\ctsind{\delta}(a < \pt_\nn(\phi_\nn) < b)$, we recover \Theorem{simcal}. In general, the fraction in \Thm{recurringunbiasedness} can be interpreted as a measure of the ``bias'' of $\MP$ on the fuzzy subsequence of $\phis$ singled out by $\fuz$. Then this theorem says that $\MP$ is unbiased on all \pgenable subsequences where the frequency of truth converges (and unbiased infinitely often on subsequences where it diverges). Thus, if an \ec sequence of sentences can be decomposed (by any \pgenable weighting) into subsequences where the frequency of truth converges, then $\MP$ learns to assign probabilities such that there is no efficient method for detecting a predictable bias in its beliefs.

However, not every sequence can be broken down into well-behaved subsequences by a \pgenable divergent weighting (if, for example, the truth values move ``pseudorandomly'' in correlated clusters, as in the case of $\seq\clusters$). In these cases, it is natural to wonder whether there are any conditions where $\MP$ will be unbiased anyway. Below, we show that the bias converges to zero whenever the weighting $\seq\fuz$ is sparse enough that $\MP$ can gather sufficient feedback about $\phi_\nn$ in between guesses:

\begin{definition}[Deferral Function]\label{def:deferralfunc}
  A function $\deff : \NN^+ \to \NN^+$ is called a \textbf{deferral function} if
  \begin{enumerate}
    \item $\deff(\nn) > \nn$ for all $\nn$, and
    \item $\deff(\nn)$ can be computed in time polynomial in $\deff(\nn)$, i.e., if there is some algorithm and a polynomial function $h$ such that for all $\nn$, the algorithm computes $\deff(\nn)$ within $h(\deff(\nn))$ steps.
  \end{enumerate}
  If $\deff$ is a deferral function, we say that $\deff$ \textbf{defers} $\nn$ to $\deff(\nn)$.
\end{definition}

\begin{restatable}[Unbiasedness From Feedback]{theorem}{restatewub}\label{thm:wub}
  Let $\seq\phi$ be any \ec sequence of decidable sentences, and $\seq\fuz$ be any \pgenable divergent weighting.
  If there exists a strictly increasing deferral function $\deff$ such that the support of $\seq\fuz$ is contained in the image of $f$ and $\thmind(\phi_{f(\nn)})$ is computable in $\Oo(\deff(n+1))$ time, then 
  \[
    \frac
        {\sum_{i\leq \nn}\fuz_i  \cdot(\pt_i(\phi_i)-\thmind(\phi_i))}
        {\sum_{i\leq \nn}\fuz_i}
    \eqsim_\nn 0.
  \]
  In this case, we say ``$\seq\fuz$ allows good feedback on $\seq\phi$''.
  \proofin{\ref{app:wub}}
\end{restatable}

\noindent In other words, $\MP$ is unbiased on any subsequence of the data where a polynomial-time machine can figure out how the previous elements of the subsequence turned out before $\MP$ is forced to predict the next one. This is perhaps the best we can hope for: On ill-behaved sequences such as $\seq\clusters$, where the frequency of truth diverges and (most likely) no polynomial-time algorithm can predict the jumps, the $\pt_\nn(\phi_\nn)$ might be pure guesswork.

So how well does $\MP$ perform on sequences like $\seq\clusters$? To answer, we turn to the question of how $\MP$ behaves in the face of sequences that it finds utterly unpredictable.

\subsection{Learning Statistical Patterns}\label{sec:statpatterns}

Consider the digits in the decimal expansion of $\pi$. A good reasoner thinking about the $10^{1,000,000}$th digit of $\pi$, in lieu of any efficient method for predicting the digit before they must make their prediction, should assign roughly 10\% probability to that digit being a~7. We will now show that logical inductors learn statistical patterns of this form.

To formalize this claim, we need some way of formalizing the idea that a sequence is ``apparently random'' to a reasoner. Intuitively, this notion must be defined relative to a specific reasoner's computational limitations. After all, the digits of $\pi$ are perfectly deterministic; they only appear random to a reasoner who lacks the resources to compute them. Roughly speaking, we will define a sequence to be pseudorandom (relative to $\MP$) if there is no \ec way to single out any one subsequence that is more likely true than any other subsequence, not even using expressions written in terms of the market prices (by way of expressible features):

\begin{restatable}[Pseudorandom Sequence]{definition}{pseudorandom}\label{def:pseudorandom}
  Given a set $S$ of divergent weightings (\Def{fuz}), a sequence $\phis$ of decidable sentences is called \textbf{pseudorandom with frequency $\bm{\prob}$} over $S$ if, for all weightings $\seq\fuz\in S$,
  \[
    \lim_{\nn \to\infty} \frac{\sum_{i \leq \nn} \fuz_i \cdot \thmind(\phi_i)}{\sum_{i \leq n} \fuz_i}
  \]
  exists and is equal to $\prob$.
\end{restatable}

\noindent Note that if the sequence $\phis$ is \emph{actually} randomly generated (say, by adding $(c_1, c_2, \ldots)$ to the language of $\Theory$, and tossing a coin weighted with probability $\prob$ towards heads for each $i$, to determine whether to add $c_i$ or $\lnot c_i$ as an axiom) then $\smash{\phis}$ is pseudorandom with frequency~$p$ almost surely.\footnote{Note that actually adding randomness to $\Theory$ in this fashion is not allowed, because we assumed that the axioms of $\Theory$ are recursively enumerable. It is possible to construct a logical inductor that has access to a source of randomness, by adding one bit of randomness to the market each day, but that topic is beyond the scope of this paper.} Now:

\begin{restatable}[Learning Pseudorandom Frequencies]{theorem}{benford}\label{thm:benford}
Let $\phis$ be an \ec sequence of decidable sentences. If $\seq\phi$ is pseudorandom with frequency~$p$ over the set of all \pgenable divergent weightings, then 
  \[
    \pt_\nn(\phi_\nn) \eqsim_\nn p.
  \]
  \proofin{\ref{sec:prand} or~\ref{app:benford}}
\end{restatable}

For example, consider again the sequence $\seq{\piasev}$ where the $\nn$th element says that the $\Ack(\nn, \nn)$th decimal digit of $\pi$ is a 7. The individual $\mathit{\piasev}_\nn$ statements are easy to write down (i.e., efficiently computable), but each one is difficult to decide. Assuming there's no good way to predict the Ackermann digits of $\pi$ using a \pgenable divergent weighting, $\MP$ will assign probability 10\% to each $\piasev_\nn$ in a timely manner, while it waits for the resources to determine whether the sentence is true or false. Of course, on each individual $\piasev_\nn$, $\MP$'s probability will go to 0 or 1 eventually, i.e., $\lim_{\mm\to\infty}\pt_\mm(\piasev_\nn) \in \{0,1\}$.

\Thm{benford} still tells us nothing about how $\MP$ handles $\seq\clusters$ (defined above), because the frequency of truth in that sequence diverges, so it does not count as pseudorandom by the above definition. To handle this case we will weaken our notion of pseudorandomness, so that it includes more sequences, yielding a stronger theorem. We will do this by allowing sequences to count as pseudorandom so long as the limiting frequency of truth converges on ``independent subsequences'' where the $n+1$st element of the subsequence doesn't come until after the $n$th element can be decided, as described below. Refer to \citet{Garrabrant:2016:ac} for a discussion of why this is a good way to broaden the set of sequences that count as pseudorandom.

\begin{definition}[$\deff$-Patient Divergent Weighting]
  Let $\deff$ be a deferral function. We say that a divergent weighting $\seq \fuz$ is \textbf{$\bm{\deff}$-patient} if there is some constant $C$ such that, for all $\nn$,
  \[
    \sum_{i=n}^{\deff(n)} \fuz_i(\MP) \leq C
  \]
  In other words, $\seq \fuz$ is $\deff$-patient if the weight it places between days $\nn$ and $\deff(\nn)$ is bounded.
\end{definition}

While we are at it, we will also strengthen \Thm{benford} in three additional ways: we will allow the probabilities on the sentences to vary with time, and with the market prices, and we will generalize $\eqsim_\nn$ to $\gtrsim_\nn$ and $\lesssim_\nn$.

\begin{restatable}[Varied Pseudorandom Sequence]{definition}{seqprand}\label{def:seqprand}
  Given a deferral function $\deff$, a set $S$ of $\deff$-patient divergent weightings, an \ec sequence $\seq\phi$ of $\Theory$-decidable sentences, and a \pgenable sequence $\seq\pseudo$ of rational probabilities,
  $\phis$ is called a \textbf{$\bm{\seq\pseudo}$-varied pseudorandom sequence} (relative to $S$) if, for all $\seq\fuz \in S$,
  \[
    \frac{\sum_{i \leq \nn} \fuz_i \cdot (\pseudo_i- \thmind(\phi_i))}{\sum_{i \leq n} \fuz_i} \eqsim_\nn 0.
  \]
  Furthermore, we can replace $\eqsim_\nn$ with $\gtrsim_\nn$ or $\lesssim_\nn$, in which case we say $\seq\phi$ is \textbf{varied pseudorandom above $\bm{\seq\pseudo}$} or \textbf{varied pseudorandom below $\bm{\seq\pseudo}$}, respectively.
\end{restatable}

\begin{restatable}[Learning Varied Pseudorandom Frequencies]{theorem}{prand}\label{thm:prand}
  Given an \ec sequence $\phis$ of $\Theory$-decidable sentences and a \pgenable sequence $\seq\pseudo$ of rational probabilities, 
  if there exists some $\deff$ such that $\phis$ is $\seq\pseudo$-varied pseudorandom (relative to all $\deff$-patient \pgenable divergent weightings), then
  \[
    \pt_\nn(\phi_\nn) \eqsim_\nn \pseudo_\nn.
  \]
  Furthermore, if $\phis$ is varied pseudorandom above or below $\seq\pseudo$, then the $\eqsim_\nn$ may be replaced with $\gtrsim_\nn$ or $\lesssim_\nn$ (respectively).
  \proofin{\ref{app:seqprand}}
\end{restatable}
\noindent Thus we see that $\MP$ does learn to assign marginal probabilities $\pt_\nn(\clusters_\nn)\approx 0.5$, assuming the Ackermann digits of $\pi$ are actually difficult to predict. Note that while \Thm{prand} requires each $\pseudo_\nn$ to be rational, the fact that the theorem is generalized to varied pseudorandom above/below sequences means that \Thm{prand} is a strict generalization of \Theorem{benford}.

In short, \Thm{prand} shows that logical inductors reliably learn in a timely manner to recognize appropriate statistical patterns, whenever those patterns (which may vary over time and with the market prices) are the best available method for predicting the sequence using \pgenable methods.

\subsection{Learning Logical Relationships}\label{sec:logicpatterns}

Most of the above properties discuss the ability of a logical inductor to recognize patterns in a single sequence---for example, they recognize \ec sequences of theorems in a timely manner, and they fall back on the appropriate statistical summaries in the face of pseudorandomness. We will now examine the ability of logical inductors to learn relationships between sequences.

Let us return to the example of the computer program \texttt{prg} which outputs either 0, 1, or 2 on all inputs, but for which this cannot be proven in general by $\Theory$. \Theorem{provind} says that the pattern
\[
  \seq{\prgzot} := \big( \quot{\texttt{prg($\enc{n}$)}=0 \lor \texttt{prg($\enc{n}$)}=1 \lor \texttt{prg($\enc{n}$)}=2} \big)_{\nn\in {\NN^+}}
\]
will be learned, in the sense that $\MP$ will assign each $\prgzot_\nn$ a probability near~1 in a timely manner. But what about the following three individual sequences?
\begin{align*}
  \seq{\prgz} := \big( \quot{\texttt{prg($\enc{n}$)}=0} \big)_{\nn\in {\NN^+}} \\
  \seq{\prgo} := \big( \quot{\texttt{prg($\enc{n}$)}=1} \big)_{\nn\in {\NN^+}} \\
  \seq{\prgt} := \big( \quot{\texttt{prg($\enc{n}$)}=2} \big)_{\nn\in {\NN^+}}
\end{align*}
None of the three sequences is a sequence of only theorems, so provability induction does not have much to say. If they are utterly pseudorandom relative to $\runtime$, then \Theorem{prand} says that $\MP$ will fall back on the appropriate statistical summary, but that tells us little in cases where there are predictable non-conclusive patterns (e.g., if \texttt{prg(i)} is more likely to output 2 when \texttt{helper(i)} outputs 17). In fact, if $\MP$ is doing good reasoning, the probabilities on the $(\prgz_\nn, \prgo_\nn, \prgt_\nn)$ triplet ought to shift, as $\MP$ gains new knowledge about related facts and updates its beliefs. How could we tell if those intermediate beliefs were reasonable?

One way is to check their sum. If $\MP$ believes that $\texttt{prg(i)}\in\{0, 1, 2\}$ and it knows how disjunction works, then it should be the case that whenever $\pt_\nn(\prgzot_t) \approx 1$, $\pt_\nn(\prgz_t) + \pt_\nn(\prgo_t) + \pt_\nn(\prgt_t) \approx 1$. And this is precisely the case. In fact, logical inductors recognize mutual exclusion between efficiently computable tuples of any size, in a timely manner:

\begin{restatable}[Learning Exclusive-Exhaustive Relationships]{theorem}{restatelex}\label{thm:lex}
  Let $\seq{\phi^1},\ldots,\seq{\phi^k}$ be $k$ \ec sequences of sentences, such that for all $\nn$, $\Theory$ proves that $\phi^1_\nn,\ldots,\phi^k_\nn$ are exclusive and exhaustive (i.e.\ exactly one of them is true). Then
\[ 
 \pt_\nn(\phi^1_\nn)+\cdots+\pt_\nn(\phi^k_\nn)  \eqsim_\nn 1.
\]
\end{restatable}
\begin{sketch}[\ref{app:lex}]
  Consider the trader that acts as follows. On day $\nn$, they check the prices of $\phi_\nn^1\ldots\phi_\nn^k$. If the sum of the prices is higher (lower) than 1 by some fixed threshold $\varepsilon > 0$, they sell (buy) a share of each, wait until the values of the shares are the same in every plausible world, and make a profit of $\varepsilon$. (It is guaranteed that eventually,  in every plausible world exactly one of the shares will be valued at 1.) If the sum goes above $1 + \varepsilon$ (below $1 - \varepsilon$) on the main diagonal infinitely often, this trader exploits $\MP$. Logical inductors are inexploitable, so it must be the case that the sum of the prices goes to~1 along the main diagonal.
\end{sketch}

This theorem suggests that logical inductors are good at learning to assign probabilities that respect logical relationships between related sentences. To show that this is true in full generality, we will generalize \Thm{lex} to any linear inequalities that hold between the actual truth-values of different sentences.

First, we define the following convention:

\begin{convention}[Constraint]
  An $\RR$-combination $\aff$ can be viewed as a \textbf{constraint}, in which case we say that a valuation $\Valuation$ \textbf{satisfies} the constraint if $\Valuation(\aff) \ge 0$.
\end{convention}

\noindent For example, the constraint
\[
  \operatorname{AND} := -2 + \phi + \psi
\]
says that both $\phi$ and $\psi$ are true, and it is satisfied by $\World$ iff $\World(\phi) = \World(\psi) = 1$. As another example, the pair of constraints
\[
  \operatorname{XOR} := (1 - \phi - \psi, \phi + \psi - 1)
\]
say that exactly one of $\phi$ and $\psi$ is true, and are satisfied by $\pt_7$ iff $\pt_7(\phi) + \pt_7(\psi) = 1$.

\begin{definition}[Bounded Combination Sequence]\label{def:bap}
  By $\BCS$ (mnemonic: \textbf{bounded combination sequences}) we denote the set of all \pgenable $\RR$-combination sequences $\seq{\aff}$ that are bounded, in the sense that there exists some bound $b$ such that $\| \aff_\nn \|_1 \le b$ for all $\nn$, where $\|\!\any\!\|_1$ includes the trailing coefficient.
\end{definition}

\begin{restatable}[Affine Provability Induction]{theorem}{affprovind}\label{thm:affprovind}
  Let $\seq{\aff}\in\BCS$ and $b \in \RR$. If, for all consistent worlds $\World\in\cworlds(\Theory)$ and all $\nn \in \NN^+$, it is the case that $\World(\aff_n ) \ge b$, then
  \[
    \pt_\nn(\aff_\nn ) \gtrsim_\nn b,
  \]
  and similarly for $=$ and $\eqsim_\nn$, and for $\leq$ and $\lesssim_\nn$.
  \proofin{\ref{app:affprovind}}
\end{restatable}

For example, consider the constraint sequence
\[
    \seq{A} := \big(1 - \prgz_\nn - \prgo_\nn - \prgt_\nn\big)_{\nn \in \NN^+}
\]
For all $\nn$ and all consistent worlds $\World \in \cworlds(\Theory)$, the value $\World(A_\nn)$ is 0, so applying \Thm{affcoh} to $\seq{A}$, we get that $\pt_\nn(A_\nn) \eqsim_\nn 0$. By linearity, this means
\[
  \pt_\nn(\prgz_\nn) + \pt_\nn(\prgo_\nn) + \pt_\nn(\prgt_\nn) \eqsim_\nn 1,
\]
i.e., $\MP$ learns that the three sequences are mutually exclusive and exhaustive in a timely manner, regardless of how difficult \texttt{prg} is to evaluate. \nameref{thm:affprovind} is a generalization of this idea, where the coefficients may vary (day by day, and with the market prices).

We can push this idea further, as follows:

\begin{restatable}[Affine Coherence]{theorem}{affcoh}\label{thm:affcoh}
  Let $\seq{\aff}\in\BCS$. Then
  \[
    \liminf_{\nn\rightarrow\infty}\inf_{\World\in\cworlds(\Theory)}
      \World(\aff_\nn )
    \le \liminf_{\nn\rightarrow\infty}
      \pt_\infty(\aff_\nn )
    \le \liminf_{\nn\to\infty}
      \pt_\nn(\aff_\nn ),
  \]
  and
  \[
    \limsup_{\nn\to\infty} \pt_\nn(\aff_\nn )
    \le \limsup_{\nn\rightarrow\infty} \pt_\infty(\aff_\nn )
    \le \limsup_{\nn\rightarrow\infty} \sup_{\World\in\cworlds(\Theory)} 
      \World(\aff_\nn ).
   \]
  \proofin{\ref{app:affcoh}}
\end{restatable}

\noindent This theorem ties the ground truth on $\seq{\aff}$, to the value of $\seq{\aff}$ in the limit, to the value of $\seq{\aff}$ on the main diagonal. In words, it says that if all consistent worlds value $\aff_\nn$ in $(a, b)$ for $\nn$ large, then $\pt_\infty$ values $\aff_\nn$ in $(c, d) \subseteq (a, b)$ for $n$ large (because $\pt_\infty$ is a weighted mixture of all consistent worlds), and $\MP$ learns to assign probabilities such that $\pt_\nn(\aff_\nn)\in (c, d)$ in a timely manner. In colloquial terms, $\MP$ learns in a timely manner to respect \emph{all} linear inequalities that actually hold between sentences, so long as those relationships can be enumerated in polynomial time.

For example, if \texttt{helper(i)=err} always implies \texttt{prg(i)=0}, $\MP$ will learn this pattern, and start assigning probabilities to $\pt_\nn(\quot{\text{\texttt{prg(\enc{$\nn$})=0}}})$ which are no lower than those of $\pt_\nn(\quot{\text{\texttt{helper(\enc{n})=err}}})$. In general, if a series of sentences obey some complicated linear inequalities, then so long as those constraints can be \emph{written down} in polynomial time, $\MP$ will learn the pattern, and start assigning probabilities that respect those constraints in a timely manner.

This doesn't mean that $\MP$ will assign the \emph{correct} values (0 or 1) to each sentence in a timely manner; that would be impossible for a deductively limited reasoner. Rather, $\MP$'s probabilities will start \emph{satisfying the constraints} in a timely manner. For example, imagine a set of complex constraints holds between seven sequences, such that exactly three sentences in each septuplet are true, but it's difficult to tell which three. Then $\MP$ will learn this pattern, and start ensuring that its probabilities on each septuplet sum to 3, even if it can't yet assign particularly high probabilities to the correct three.

If we watch an individual septuplet as $\MP$ reasons, other constraints will push the probabilities on those seven sentences up and down. One sentence might be refuted and have its probability go to zero. Another might get a boost when $\MP$ discovers that it's likely implied by a high-probability sentence. Another might take a hit when $\MP$ discovers it likely implies a low-probability sentence. Throughout all this, \Thm{affcoh} says that $\MP$ will ensure that the seven probabilities always sum to $\approx 3$. $\MP$'s beliefs on any given day arise from this interplay of many constraints, inductively learned.

Observe that \nameref{thm:affcoh} is a direct generalization of \Theorem{provind}. One way to interpret this theorem is that it says that $\MP$ is very good at learning inductively to predict long-running computations. Given any \ec sequence of statements about the computation, if they are true then $\MP$ learns to believe them in a timely manner, and if they are false then $\MP$ learns to disbelieve them in a timely manner, and if they are related by logical constraints (such as by exclusivity or implication) to some other \ec sequence of statements, then $\MP$ learns to make its probabilities respect those constraints in a timely manner. This is one of the main reasons why we think this class of algorithms deserves the name of ``logical inductor''.

\nameref{thm:affcoh} can also be interpreted as an approximate coherence condition on the finite belief-states of $\MP$. It says that if a certain relationship among truth values is going to hold in the future, then $\MP$ learns to make that relationship hold approximately in its probabilities, in a timely manner.\footnote{Another notion of approximate coherence goes by the name of ``inductive coherence'' \citep{Garrabrant:2016:ic}. A reasoner is called inductively coherent if (1) $\pt_\nn(\bot) \eqsim_\nn 0$; (2) $\pt_\nn(\phi_\nn)$ converges whenever $\phis$ is efficiently computable and each $\phi_\nn$ provably implies $\phi_{\nn + 1}$; and (3) for all efficiently computable sequences of provably mutually exclusive and exhaustive triplets $(\phi_\nn, \psi_\nn, \chi_\nn)$, $\pt_\nn(\phi_\nn) + \pt_\nn(\psi_\nn) + \pt_\nn(\chi_\nn) \eqsim_\nn 1$. \citeauthor{Garrabrant:2016:ic} show that inductive coherence implies coherence in the limit, and argue that this is a good notion of approximate coherence. Theorems~\ref{thm:lc} (\nameref{thm:lc}) and~\ref{thm:affcoh} (\nameref{thm:affcoh}) imply inductive coherence, and indeed, logical induction is a much stronger notion.}

In fact, we can use this idea to strengthen every theorem in sections~\ref{sec:timelylearning}-\ref{sec:statpatterns}, as below. (Readers without interest in the strengthened theorems are invited to skip to \Sec{conservatism}.)

\subsubsection{Affine Strengthenings}

Observe that \Theorem{affprovind} is a strengthening of \Theorem{provind}.

\begin{restatable}[Persistence of Affine Knowledge]{theorem}{peraffkno}\label{thm:peraffkno}
  Let $\seq{\aff}\in\BCS$. Then
  \[
    \liminf_{\nn\rightarrow\infty}\inf_{\mm\geq \nn}\pt_\mm(\aff_\nn )= \liminf_{\nn\to\infty}\pt_\infty(\aff_\nn )
  \]
  and
    \[
     \limsup_{\nn\rightarrow\infty}\sup_{\mm\geq \nn}\pt_\mm(\aff_\nn )=\limsup_{\nn\to\infty}\pt_\infty(\aff_\nn ).
  \]
  \proofin{\ref{app:peraffkno}}
\end{restatable}

\noindent To see that this is a generalization of \Theorem{perkno}, it might help to first replace $\seq{\aff}$ with a sequence $\probs$ of rational probabilities.

\begin{restatable}[Affine Preemptive Learning]{theorem}{affpolymax}\label{thm:affpolymax}
  Let $\seq{\aff}\in\BCS$. Then
  \[
       \liminf_{\nn\to\infty} \pt_\nn(\aff_\nn )= \liminf_{\nn\rightarrow\infty}\sup_{\mm\geq \nn} \pt_\mm(\aff_\nn ) 
  \]
  and
  \[
         \limsup_{\nn\to\infty} \pt_\nn(\aff_\nn )= \limsup_{\nn\rightarrow\infty}\inf_{\mm\geq \nn} \pt_\mm(\aff_\nn ) \ .
  \]
  \proofin{\ref{app:affpolymax}}
\end{restatable}

\begin{definition}[Determined via $\Theory$]\label{def:thmval}
  We say that a $\RR$-combination $\aff$ is \textbf{determined via $\bm{\Theory}$} if, in all worlds $\World \in \cworlds(\Theory)$, the value $\World(\aff )$ is equal. Let $\thmval(\aff)$ denote this value.
  
  Similarly, a sequence $\seq{\aff}$ of $\RR$-combinations is said to be determined via $\Theory$ if $\aff_\nn$ is determined via $\Theory$ for all $\nn$.
\end{definition}

\begin{restatable}[Affine Recurring Unbiasedness]{theorem}{recunbiasedaff}\label{thm:recunbiasedaff}
  If $\seq{\aff}\in\BCS$ is determined via $\Theory$, and $\seq\fuz$ is a \pgenable divergent weighting,
  \[
      \frac
        {\sum_{i\leq \nn}\fuz_i  \cdot(\pt_i(\aff_i)-\thmval(\aff_i))}
        {\sum_{i\leq \nn}\fuz_i}
  \]
  has $0$ as a limit point. In particular, if it converges, it converges to~$0$.
    \proofin{\ref{app:recunbiasedaff}}
\end{restatable}

\begin{restatable}[Affine Unbiasedness from Feedback]{theorem}{wubaff}\label{thm:wubaff}
  Given $\seq\aff \in \BCS$ that is determined via $\Theory$, a strictly increasing deferral function~$\deff$ such that $\thmval(\aff_n )$ can be computed in time $\Oo(\deff(\nn+1))$, and a \pgenable divergent weighting $\seq\fuz$ such that the support of $\seq\fuz$ is contained in the image of $f$,
  \[
    \frac
        {\sum_{i\leq \nn}\fuz_i  \cdot(\pt_i(\aff_i)-\thmval(\aff_i))}
        {\sum_{i\leq \nn}\fuz_i}
    \eqsim_\nn 0.
  \]
  In this case, we say ``$\seq\fuz$ allows good feedback on $\seq\aff$''.
  \proofin{\ref{app:wubaff}}
\end{restatable}

\begin{restatable}[Learning Pseudorandom Affine Sequences]{theorem}{prandaff}\label{thm:prandaff}
  Given a $\seq\aff \in \BCS$ which is determined via $\Theory$, if there exists deferral function $\deff$ such that for any \pgenable $\deff$-patient divergent weighting $\seq\fuz$,
  \[
    \frac{\sum_{i \leq \nn} \fuz_i  \cdot \thmval(\aff_i )}{\sum_{i \leq n} \fuz_i } \gtrsim_\nn 0,
  \]  
  then
  \[
    \pt_\nn(\aff_\nn ) \gtrsim_\nn 0,
  \]
  and similarly for $\eqsim_\nn$, and $\lesssim_\nn$.
  \proofin{\ref{app:prandaff}}
\end{restatable}

\subsection{Non-Dogmatism}\label{sec:conservatism}

Cromwell's rule says that a reasoner should not assign extreme probabilities (0 or~1) except when applied to statements that are logically true or false. The rule was named by \citet{Lindley:1991:MakingDecisions}, in light of the fact that Bayes' theorem says that a Bayesian reasoner can never update away from probabilities~0 or~1, and in reference to the famous plea:

\begin{quote}
  I beseech you, in the bowels of Christ, think it possible that you may be mistaken.
  \hfill\emph{-- Oliver Cromwell}
\end{quote}

\noindent The obvious generalization of Cromwell's rule to a setting where a reasoner is uncertain about logic is that they also should not assign extreme probabilities to sentences that have not yet been proven or disproven. Logical inductors \emph{do not} satisfy this rule, as evidenced by the following theorem:

\begin{restatable}[Closure under Finite Perturbations]{theorem}{ifp}\label{thm:ifp}
  Let $\MP$  and $\seq{\pt^\prime}$ be markets with $\pt_\nn = \pt^\prime_\nn$ for all but finitely many $\nn$. Then $\MP$ is a logical inductor if and only if $\seq{\pt^\prime}$ is a logical inductor.
  \proofin{\ref{app:ifp}}
\end{restatable}

\noindent This means that we can take a logical inductor, completely ruin its beliefs on the 23rd day (e.g., by setting $\pt_{23}(\phi)=0$ for all $\phi$), and it will still be a logical inductor. Nevertheless, there is still a sense in which logical inductors are non-dogmatic, and can ``think it possible that they may be mistaken'':

\begin{restatable}[Non-Dogmatism]{theorem}{restatenondog}\label{thm:nd}
 If $\Theory \nvdash \phi$ then 
  $\pt_\infty(\phi)<1$, and if $\Theory \nvdash \neg\phi$ then $\pt_\infty(\phi)>0$.
\end{restatable}

\begin{sketch}[\ref{app:nondog}] Consider a trader that watches $\phi$ and buys whenever it gets low, as follows. The trader starts with \$1. They spend their first 50 cents when $\pt_\nn(\phi) < 1/2$, purchasing one share. They spend their next 25 cents when $\pt_\nn(\phi) < 1/4$, purchasing another share. They keep waiting for $\pt_\nn(\phi)$ to drop low enough that they can spend the next half of their initial wealth to buy one more share. Because $\phi$ is independent, there always remains at least one world $\World$ such that $\World(\phi)=1$, so if $\pt_\nn(\phi) \to 0$ as $n \to\infty$ then their maximum plausible profits are \$1 + \$1 + \$1 +\dots which diverges, and they exploit the market. Thus, $\pt_\infty(\phi)$ must be bounded away from zero.
\end{sketch}

In other words, if $\phi$ is independent from $\Theory$, then $\MP$'s beliefs about $\phi$ won't get stuck converging to 0 or 1. By \Theorem{ifp}, $\MP$ may occasionally jump to unwarranted conclusions---believing with ``100\% certainty'', say, that Euclid's fifth postulate follows from the first four---but it always corrects these errors, and eventually develops conservative beliefs about independent sentences.

\Thm{nd} guarantees that $\MP$ will be reasonable about independent sentences, but it doesn't guarantee reasonable beliefs about \emph{theories}, because theories can require infinitely many axioms. For example, let $\Theory$ be a theory of pure first-order logic, and imagine that the language $\Lang$ has a free binary relation symbol $\quot{\!\in\!}$. Now consider the sequence $\seq{\zfseq}$ of first-order axioms of Zermelo-Fraenkel set theory ($\ZFC$) which say to interpret $\quot{\!\in\!}$ in the set-theoretic way, and note that $\seq{\zfseq}$ is infinite. Each individual sentence $\zfseq_\nn$ is consistent with first-order logic, but if $\pt_\infty$'s odds on each axiom were 50:50 and independent, then it would say that the probability of them all being true simultaneously was zero. Fortunately, for any computably enumerable sequence of sentences that are mutually consistent, $\pt_\infty$ assigns positive probability to them all being simultaneously true.

\begin{restatable}[Uniform Non-Dogmatism]{theorem}{restateunondog}\label{thm:obu}
  For any computably enumerable sequence of sentences $\phis$ such that $\Theory\cup\phis$ is consistent, there is a constant $\varepsilon>0$ such that for all $\nn$,
  \[
    \pt_\infty(\phi_\nn)\geq \varepsilon.
  \]
    \proofin{\ref{app:obu}}
\end{restatable}

\noindent If $\phi_\nn$ is the conjunction of the first $\nn$ axioms of $\ZFC$, \Thm{obu} shows that $\pt_\infty$ assigns positive probability to theories in which the symbol $\quot{\!\in\!}$ satisfies all axioms of $\ZFC$ (assuming $\ZFC$ is consistent).

Reasoning about individual sentences again, we can put bounds on how far each sentence $\phi$ is bounded away from 0 and 1, in terms of the prefix complexity $\kappa(\phi)$ of $\phi$, i.e., the length of the shortest prefix that causes a fixed universal Turing machine to output $\phi$.\footnote{We use prefix complexity (the length of the shortest prefix that causes a UTM to output $\phi$) instead of Kolmogorov complexity (the length of the shortest complete program that causes a UTM to output $\phi$) because it makes the proof slightly easier. (And, in the opinion of the authors, prefix complexity is the more natural concept.) Both types of complexity are defined relative to an arbitrary choice of universal Turing machine (UTM), but our theorems hold for every \li{} regardless of the choice of UTM, because changing the UTM only amounts to changing the constant terms by some fixed amount.}

\begin{restatable}[Occam Bounds]{theorem}{restateoccam}\label{thm:ob}
 There exists a fixed positive constant $C$ such that for any sentence $\phi$ with prefix complexity $\kappa(\phi)$, if $\Theory\nvdash\neg\phi$, then \[\pt_\infty(\phi)\geq C2^{-\kappa(\phi)},\] and if $\Theory\nvdash\phi$, then \[\pt_\infty(\phi)\leq 1-C2^{-\kappa(\phi)}.\]
     \proofin{\ref{app:ob}}
\end{restatable}

This means that if we add a sequence of constant symbols $(c_1, c_2, \ldots)$ not mentioned in $\Theory$ to the language $\Lang$, then $\MP$'s beliefs about statements involving those constants will depend on the complexity of the claim. Roughly speaking, if you ask after the probability of a claim like $\quot{c_1 = 10 \land c_2 = 7 \land \ldots \land c_n = -3}$ then the answer will be no lower than the probability that a simplicity prior assigns to the shortest program that outputs $(10, 7, \ldots, -3)$.

In fact, the probability may be a fair bit higher, if the claim is part of a particularly simple sequence of sentences. In other words, logical inductors can be used to reason about \emph{empirical} uncertainty as well as logical uncertainty, by using $\pt_\infty$ as a full-fledged sequence predictor:

\begin{restatable}[Domination of the Universal Semimeasure]{theorem}{restatedus}\label{thm:dus}
  Let $(b_1, b_2, \ldots)$ be a sequence of zero-arity predicate symbols in $\Lang$ not mentioned in $\Theory$, and let $\sigma_{\le n}=(\sigma_1,\ldots,\sigma_\nn)$ be any finite bitstring. Define
  \[
    \pt_\infty(\sigma_{\le n}) := \pt_\infty(\quot{(b_1 \iff \enc{\sigma_1}=1) \land (b_2 \iff \enc{\sigma_2}=1) \land \ldots \land (b_n \iff \enc{\sigma_n}=1)}),
  \]
  such that, for example, $\pt_\infty(01101) = \pt_\infty(\quot{\lnot b_1 \land b_2 \land b_3 \land \lnot b_4 \land b_5})$.
  Let $M$ be a universal continuous semimeasure. Then there is some positive constant $C$ such that for any finite bitstring $\sigma_{\le n}$,
  \[
    \pt_\infty(\sigma_{\le n}) \ge C \cdot M(\sigma_{\le n}).
  \]
  \proofin{\ref{app:dus}}
\end{restatable}
\noindent In other words, logical inductors can be viewed as a computable approximation to a normalized probability distribution that dominates the universal semimeasure. In fact, this dominance is strict:

\begin{restatable}[Strict Domination of the Universal Semimeasure]{theorem}{restatestrict}\label{thm:strict}
  The universal continuous semimeasure does not dominate $\pt_\infty$; that is, for any positive constant $C$ there is some finite bitstring $\sigma_{\le n}$ such that 
  \[
    \pt_\infty(\sigma_{\le n}) > C \cdot M(\sigma_{\le n}).
  \]
  \proofin{\ref{app:strict}}
\end{restatable}

\noindent In particular, by \Theorem{obu}, logical inductors assign positive probability to the set of all completions of theories like $\PA$ and $\ZFC$, whereas universal semimeasures do not. This is why we can't construct approximately coherent beliefs about logic by fixing an enumeration of logical sentences and conditioning a universal semimeasure on more axioms of Peano arithmetic each day: the probabilities that the semimeasure assigns to those conjunctions must go to zero, so the conditional probabilities may misbehave. (If this were not the case, it would be possible to sample a complete extension of Peano arithmetic with positive probability, because universal semimeasures are approximable from below; but this is impossible. See the proof of \Thm{strict} for details.) While $\pt_\infty$ is limit-computable, it is not approximable from below, so it can and does outperform the universal semimeasure when reasoning about arithmetical claims.

\subsection{Conditionals}\label{sec:conditionals}

One way to interpret \Theorem{dus} is that when we condition $\pt_\infty$ on independent sentences about which it knows nothing, it performs empirical (scientific) induction. We will now show that when we condition $\MP$, it also performs logical induction.

In probability theory, it is common to discuss conditional probabilities such as $\Bayesian(A \mid B) := \Bayesian(A \land B)/\Bayesian(B)$ (for any $B$ with $\Bayesian(B)>0$), where $\Bayesian(A \mid B)$ is interpreted as the probability of $A$ restricted to worlds where $B$ is true. In the domain of logical uncertainty, we can define conditional probabilities in the analogous way:

\begin{definition}[Conditional Probability]\label{def:condp}
  Let $\phi$ and $\psi$ be sentences, and let $\Valuation$ be a valuation with $\Valuation(\psi) > 0$. Then we define
  \[
    \Valuation(\phi \mid \psi) :=
    \begin{cases}
      {\Valuation(\phi\wedge\psi)}/{\Valuation(\psi)} & 
      \text{if }\Valuation(\phi\wedge\psi) < \Valuation(\psi)\\
      1 & 
      \mbox{otherwise.}
    \end{cases}
  \]
  Given a valuation sequence $\seq{\Valuation}$, we define
  \[
    \seq{\Valuation}(\any\mid\psi) := (\Valuation_1(\any\mid\psi), \Valuation_2(\any\mid\psi), \ldots).
  \]
\end{definition}

\noindent Defining $\Valuation(\phi\mid\psi)$ to be~1 if $\Valuation(\psi)=0$ is nonstandard, but convenient for our theorem statements and proofs. The reader is welcome to ignore the conditional probabilities in cases where $\Valuation(\psi)=0$, or to justify our definition from the principle of explosion (which says that from a contradiction, anything follows). This definition also caps $\Valuation(\phi\mid\psi)$ at~1, which is necessary because there's no guarantee that $\Valuation$ knows that $\phi \land \psi$ should have a lower probability than $\psi$. For example, if it takes $\MP$ more than 17 days to learn how $\quot{\!\land\!}$ interacts with $\phi$ and $\psi$, then it might be the case that $\pt_{17}(\phi\land\psi)=0.12$ and $\pt_{17}(\psi)=0.01$, in which case the uncapped ``conditional probability'' of $\phi\land\psi$ given $\psi$ according to $\pt_{17}$ would be twelve hundred percent.

This fact doesn't exactly induce confidence in ${\MP(\any\mid\psi)}$. Nevertheless, we have the following theorem:
\begin{restatable}[Closure Under Conditioning]{theorem}{restatescon}\label{thm:scon}
  The sequence $\MP(\any\mid\psi)$ is a logical inductor over $\Theory \cup \{\psi\}$. Furthermore, given any efficiently computable sequence $\seq{\psi}$ of sentences, the sequence
  \[
    \left(\pt_1(\any \mid \psi_1), \pt_2(\any \mid \psi_1 \land \psi_2), \pt_3(\any \mid \psi_1 \land \psi_2 \land \psi_3), \ldots\right),
  \]
  where the $\nn$th pricing is conditioned on the first $\nn$ sentences in $\seq{\psi}$, is a logical inductor over $\Theory \cup \{\psi_i \mid i \in \NN^+\}$.
  \proofin{\ref{app:scon}}
\end{restatable}

In other words, if we condition logical inductors on logical sentences, the result is still a logical inductor, and so the conditional probabilities of a logical inductor continues to satisfy all the desirable properties satisfied by all logical inductors. This also means that one can obtain a logical inductor for Peano arithmetic by starting with a logical inductor over an empty theory, and conditioning it on $\PA$.

With that idea in mind, we will now begin examining questions about logical inductors that assume $\Theory$ \representscomputations, such as questions about $\MP$'s beliefs about $\Theory$, computer programs, and itself.

\subsection{Expectations}\label{sec:expectations}

In probability theory, it is common to ask the expected (average) value of a variable that takes on different values in different possible worlds. Emboldened by our success with conditional probabilities, we will now define a notion of the expected values of \emph{logical} variables, and show that these are also fairly well-behaved. This machinery will be useful later when we ask logical inductors for their beliefs about themselves.

We begin by defining a notion of logically uncertain variables, which play a role analogous to the role of random variables in probability theory. For the sake of brevity, we will restrict our attention to logically uncertain variables with their value in $[0, 1]$; it is easy enough to extend this notion to a notion of arbitrary bounded real-valued logically uncertain variables. (It does, however, require carrying a variable's bounds around everywhere, which makes the notation cumbersome.)

To define logically uncertain variables, we will need to assume that $\Theory$ is capable of representing rational numbers and proving things about them. Later, we will use expected values to construct sentences that talk about things like the expected outputs of a computer program. Thus, in this section and in the remainder of \Sec{properties}, we will assume that $\Theory$ \representscomputations.

\begin{definition}[Logically Uncertain Variable]\label{def:luv}
  A \textbf{logically uncertain variable}, abbreviated \textbf{LUV}, is any formula $X$ free in one variable that defines a unique value via $\Theory$, in the sense that
  \[
    \Theory \vdash{\exists x \colon  \left( X(x) \land \forall x' \colon X(x') \to x'=x \right)}.
  \]
  We refer to that value as the \textbf{value} of $X$. If $\Theory$ proves that the value of $X$ is in $[0, 1]$, we call $X$ a \textbf{$\bm{[0,1]}$-LUV}.

  Given a $[0,1]$-LUV $X$ and a consistent world $\World \in \cworlds(\Theory)$, the \textbf{value of $\bm X$ in $\World$} is defined to be
  \[
    \World(X) := \sup \left\{x \in [0, 1] \mid \World(\quot{X \ge \enc{x}})=1 \right\}.
  \]
  In other words, $\World(X)$ is the supremum of values that do not exceed $X$ according to $\World$. (This rather roundabout definition is necessary in cases where $\World$ assigns $X$ a non-standard value.)

  We write $\LUVs$ for the set of all $[0, 1]$-LUVs. When manipulating logically uncertain variables, we use shorthand like $\quot{X < 0.5}$ for $\quot{\forall x \colon X(x) \to x < 0.5}$. See \Sec{notation} for details.
\end{definition}

As an example, $\mathit{Half} := \quot{\nu = 0.5}$ is a LUV, where the unique real number that makes $\mathit{Half}$ true is rather obvious. A more complicated LUV is 
\[
  \mathit{TwinPrime} := \quot{\text{1 if the twin prime conjecture is true, 0 otherwise}};
\]
this is a deterministic quantity (assuming $\Theory$ actually proves the twin prime conjecture one way or the other), but it's reasonable for a limited reasoner to be uncertain about the value of that quantity. In general, if $f : \NN^+ \to [0, 1]$ is a computable function then $\quot{\enc{f}(7)}$ is a LUV, because $\quot{\enc{f}(7)}$ is shorthand for the formula $\quot{\gamma_f(7, \nu)}$, where $\gamma_f$ is the predicate of $\Theory$ representing $f$.

With LUVs in hand, we can define a notion of $\MP$'s expected value for a LUV $X$ on day $\nn$ with precision $k$. The obvious idea is to take the sum
\[
  \lim_{k \to\infty} \sum_{i=0}^{k-1} \frac{i}{k} \pt_\nn\mleft(
    \quot{\enc{i}/\enc{k} < \enc{X} \le (\enc{i}+1)/\enc{k}}
  \mright).
\]
However, if $\pt_\nn$ hasn't yet figured out that $X$ pins down a unique value, then it might put high probability on $X$ being in multiple different intervals, and the simple integral of a $[0,1]$-valued LUV could fall outside the $[0,1]$ interval. This is a nuisance when we want to treat the expectations of $[0,1]$-LUVs as other $[0,1]$-LUVs, so instead, we will define expectations using an analog of a cumulative distribution function. In probability theory, the expectation of a $[0,1]$-valued random variable $V$ with density function $\rho_V$ is given by $\EE(V) =\int_0^1 x \cdot \rho_V(x) dx$. We can rewrite this using integration by parts as
\[
  \EE(V) =\int_0^1 \Bayesian(V>x)dx.
\]
This motivates the following definition of expectations for LUVs:

\begin{definition}[Expectation]\label{def:e}
  For a given valuation $\Valuation$, we define the \textbf{approximate expectation operator} $\EE_k^{\Valuation}$ for $\Valuation$ with precision $k$ by
  \[
    \EE_k^{\Valuation}(X) := \sum_{i=0}^{k-1}\frac{1}{k} \Valuation\mleft(\quot{\enc{X} > \enc{i}/\enc{k}} \mright).
  \]
  where $X$ is a $[0,1]$-LUV\@.
\end{definition}
\noindent This has the desirable property that $\EE_k^{\Valuation}(X)\in [0, 1]$, because $\Valuation(\any)\in [0, 1]$.

We will often want to take a limit of $\EE_k^{\pt_\nn}(X)$ as both $k$ and $n$ approach $\infty$. We hereby make the fairly arbitrary choice to focus on the case $k=\nn$ for simplicity, adopting the shorthand  
\[
  \EE_\nn := \EE_\nn^{\pt_\nn} .
\]
In other words, when we examine how a logical inductor's expectations change on a sequence of sentences over time, we will (arbitrarily) consider approximate expectations that gain in precision at a rate of one unit per day.

We will now show that the expectation operator $\EE_\nn$ possesses properties that make it worthy of that name.

\begin{restatable}[Expectations Converge]{theorem}{restateec}\label{thm:ec}
  The limit ${\EE_\infty:\Sentences\rightarrow[0,1]}$  defined by
  \[
    \EE_\infty(X) := \lim_{\nn\rightarrow\infty} \EE_\nn(X)
  \] exists for all $X \in \LUVs$.
  \proofin{\ref{app:ec}}
\end{restatable}

Note that $\EE_\infty(X)$ might not be rational.

Because $\pt_\infty$ defines a probability measure over $\cworlds(\Theory)$, $\EE_\infty(X)$ is the average value of $\World(X)$ across all consistent worlds (weighted by $\pt_\infty$). In other words, every LUV $X$ can be seen as a random variable with respect to the measure $\pt_\infty$, and $\EE_\infty$ acts as the standard expectation operator on $\pt_\infty$. Furthermore,

\begin{restatable}[Linearity of Expectation]{theorem}{restateloe}\label{thm:loe}
  Let $\seq a, \seq b$ be bounded \pgenable sequences of rational numbers, and let $\seq{X}, \seq{Y}$, and $\seq{Z}$ be \ec sequences of $[0,1]$-LUVs. If we have $\Theory \vdash{Z_\nn = a_\nn X_\nn + b_\nn Y_\nn}$ for all $\nn$, then
  \[
    a_\nn \EE_\nn(X_\nn) + b_\nn \EE_\nn(Y_\nn) \eqsim_\nn \EE_\nn(Z_\nn).
  \]
  \proofin{\ref{app:loe}}
\end{restatable}

For our next result, we want a LUV which can be proven to take value~1 if $\phi$ is true and~0 otherwise.

\begin{definition}[Indicator LUV]
  For any sentence $\phi$, we define its \textbf{indicator LUV} by the formula
  \[
    \OneOperator(\phi) := \quot{(\enc{\phi} \wedge (\nu = 1)) \vee (\lnot\enc{\phi} \wedge (\nu = 0))}.
  \]
\end{definition}

\noindent Observe that $\OneOperator(\phi)(1)$ is equivalent to $\phi$, and $\OneOperator(\phi)(0)$ is equivalent to $\lnot\phi$.

\begin{restatable}[Expectations of Indicators]{theorem}{restateei}\label{thm:ei}
  Let $\phis$ be an \ec sequence of sentences. Then
  \[
    \EE_\nn(\OneOperator(\phi_\nn)) \eqsim_\nn \pt_\nn(\phi_\nn).
  \]
  \proofin{\ref{app:ei}}
\end{restatable}

\noindent In colloquial terms, \Thm{ei} says that a logical inductor learns that asking for the expected value of $\OneOperator(\phi)$ is the same as asking for the probability of $\phi$.

To further demonstrate that expectations work as expected, we will show that they satisfy generalized versions of all theorems proven in sections~\ref{sec:timelylearning}-\ref{sec:logicpatterns}. (Readers without interest in the versions of those theorems for expectations are invited to skip to \Sec{metamath}.)

\subsubsection{Collected Theorems for Expectations}

\begin{definition}[LUV Valuation]
  A LUV valuation is any function $\luvval: \LUVs\to[0, 1]$.  Note that 
  $\EE_n^\Valuation$ and $\EE_\infty^\Valuation$ are LUV valuations for any valuation $\Valuation$ and $\nn \in \NN^+$, and that every world $\World \in \cworlds(\Theory)$ is a LUV valuation.
\end{definition}

\begin{definition}[LUV Combination]
  An \textbf{$\bm\features$-LUV-combination} $\affluv : \LUVs \cup \{1\} \to \features$ is an affine expression of the form
  \[
    \affluv := \affconst + \alpha_1 X_1 + \cdots + \alpha_k X_k,
  \]
  where $(X_1,\ldots,X_k)$ are $[0, 1]$-LUVs and $(\affconst, \alpha_1, \ldots, \alpha_k)$ are in $\features$.  An \textbf{$\bm\exfeatures$-LUV-combination}, an \textbf{$\RR$-LUV-combination}, and a \textbf{$\QQ$-LUV-combination} are defined similarly.
  
  The following concepts are all defined analogously to how they are defined for sentence combinations: $\affluv[\fconst]$, $\affluv[X]$, $\rank(\affluv)$, $\luvval(\affluv)$ for any LUV valuation $\luvval$,
  \textbf{$\bm\features$-LUV-combination progressions}, \textbf{$\bm\exfeatures$-LUV-combination progressions}, and \pgenable LUV-combination sequences. (See definitions~\ref{def:affcomsen} and~\ref{def:ece} for details.)
\end{definition}

\begin{definition}[Bounded LUV-Combination Sequence]\label{def:blcp}
  By $\BLCS$ (mnemonic: \textbf{bounded LUV-combination sequences}) we denote the set of all \pgenable $\RR$-LUV-combination sequences $\seq{\affluv}$ that are bounded, in the sense that there exists some bound $b$ such that $\| \affluv_\nn \|_1 \le b$ for all $\nn$, where $\|\!\any\!\|_1$ includes the trailing coefficient.
\end{definition}

\begin{restatable}[Expectation Provability Induction]{theorem}{expprovind}\label{thm:expprovind}
  Let $\seq{\affluv}\in\BLCS$ and $b \in \RR$. If, for all consistent worlds $\World\in\cworlds(\Theory)$ and all $\nn \in \NN^+$, it is the case that $\World(\affluv_n ) \ge b$, then
  \[
    \EE_\nn(\affluv_\nn ) \gtrsim_\nn b,
  \]
  and similarly for $=$ and $\eqsim_\nn$, and for $\leq$ and $\lesssim_\nn$.
  \proofin{\ref{app:expprovind}}
\end{restatable}

\begin{restatable}[Expectation Coherence]{theorem}{expcoh}\label{thm:expcoh}
  Let $\seq{\affluv}\in\BLCS$. Then
  \[
    \liminf_{\nn\rightarrow\infty}\inf_{\World\in\cworlds(\Theory)}
      \World(\affluv_\nn )
    \le \liminf_{\nn\rightarrow\infty}
      \EE_\infty(\affluv_\nn )
    \le \liminf_{\nn\to\infty}
      \EE_\nn(\affluv_\nn ),
  \]
  and
  \[
    \limsup_{\nn\to\infty} \EE_\nn(\affluv_\nn )
    \le \limsup_{\nn\rightarrow\infty} \EE_\infty(\affluv_\nn )
    \le \limsup_{\nn\rightarrow\infty} \sup_{\World\in\cworlds(\Theory)} 
      \World(\affluv_\nn ).
   \]
  \proofin{\ref{app:expcoh}}
\end{restatable}

\begin{restatable}[Persistence of Expectation Knowledge]{theorem}{perexpkno}\label{thm:perexpkno}
  Let $\seq{\affluv}\in\BLCS$. Then
  \[
    \liminf_{\nn\rightarrow\infty}\inf_{\mm\geq \nn}\EE_\mm(\affluv_\nn )= \liminf_{\nn\to\infty}\EE_\infty(\affluv_\nn )
  \]
  and
    \[
     \limsup_{\nn\rightarrow\infty}\sup_{\mm\geq \nn}\EE_\mm(\affluv_\nn )=\limsup_{\nn\to\infty}\EE_\infty(\affluv_\nn ).
  \]
  \proofin{\ref{app:perexpkno}}
\end{restatable}

\begin{restatable}[Expectation Preemptive Learning]{theorem}{exppolymax}\label{thm:exppolymax}
  Let $\seq{\affluv}\in\BLCS$. Then
  \[
       \liminf_{\nn\to\infty} \EE_\nn(\affluv_\nn )= \liminf_{\nn\rightarrow\infty}\sup_{\mm\geq \nn} \EE_\mm(\affluv_\nn ) 
  \]
  and
  \[
         \limsup_{\nn\to\infty} \EE_\nn(\affluv_\nn )= \limsup_{\nn\rightarrow\infty}\inf_{\mm\geq \nn} \EE_\mm(\affluv_\nn ) \ .
  \]
  \proofin{\ref{app:exppolymax}}
\end{restatable}

\begin{definition}[Determined via $\Theory$ (for LUV-Combinations)]\label{def:affthmval}
  We say that a $\RR$-LUV-combination $\affluv$ is \textbf{determined via $\bm{\Theory}$} if, in all worlds $\World \in \cworlds(\Theory)$, the value $\World(\affluv)$ is equal. Let $\thmval(\affluv)$ denote this value.
  
  Similarly, a sequence $\seq{\affluv}$ of $\RR$-LUV-combinations is said to be determined via $\Theory$ if $\affluv_\nn$ is determined via $\Theory$ for all $\nn$.
\end{definition}

\begin{restatable}[Expectation Recurring Unbiasedness]{theorem}{recurringunbiasednessexp}\label{thm:recurringunbiasednessexp}
  If $\seq{\affluv}\in\BLCS$ is determined via $\Theory$, and $\seq\fuz$ is a \pgenable divergent weighting weighting such that the support of $\seq\fuz$ is contained in the image of $f$,
  \[
      \frac
        {\sum_{i\leq \nn}\fuz_i  \cdot(\EE_i(\affluv_i)-\thmval(\affluv_i))}
        {\sum_{i\leq \nn}\fuz_i}
  \]
  has $0$ as a limit point. In particular, if it converges, it converges to~$0$.
\end{restatable}

\begin{restatable}[Expectation Unbiasedness From Feedback]{theorem}{wubexp}\label{thm:wubexp}
  Given $\seq\affluv \in \BLCS$ that is determined via $\Theory$, a strictly increasing deferral function~$\deff$ such that $\thmval(\aff_n )$ can be computed in time $\Oo(\deff(\nn+1))$, and a \pgenable divergent weighting $\fuz$,
  \[
    \frac
        {\sum_{i\leq \nn}\fuz_i  \cdot(\EE_i(\affluv_i)-\thmval(\affluv_i))}
        {\sum_{i\leq \nn}\fuz_i}
    \eqsim_\nn 0.
  \]
  In this case, we say ``$\seq\fuz$ allows good feedback on $\seq\affluv$''.
  \proofin{\ref{app:wubexp}}
\end{restatable}

\begin{restatable}[Learning Pseudorandom LUV Sequences]{theorem}{prandexp}\label{thm:prandexp}
  Given a $\seq\affluv \in \BLCS$ which is determined via $\Theory$, if there exists a deferral function $\deff$ such that for any \pgenable $\deff$-patient divergent weighting $\seq\fuz$,
  \[
    \frac{\sum_{i \leq \nn} \fuz_i  \cdot \thmval(\affluv_i )}{\sum_{i \leq n} \fuz_i } \gtrsim_\nn 0,
  \]
  then
  \[
    \EE_\nn(\affluv_\nn ) \gtrsim_\nn 0.
  \]
  \proofin{\ref{app:prandexp}}
\end{restatable}

\subsection{Trust in Consistency}\label{sec:metamath}

The theorems above all support the hypothesis that logical inductors develop reasonable beliefs about logic. One might then wonder what a logical inductor has to say about some of the classic questions in meta-mathematics. For example, what does a logical inductor over $\PA$ say about the consistency of Peano arithmetic?

\begin{definition}[Consistency Statement]
  Given a recursively axiomatizable theory $\Theory^\prime$, define the \textbf{$\bm{\nn}$-consistency statement} of $\Theory^\prime$ to be the formula with one free variable $\nu$ such that
  \[
    \consen(\Theory^\prime)(\nu) := \quot{\textnormal{There is no proof of $\bot$ from $\enc{\Theory^\prime}$ with $\nu$ or fewer symbols}},
  \]
  written in $\Lang$ using a G\"odel encoding. For instance, $\consen(\PA)(\quot{\-\Ack(10,10)})$ says that any proof of $\bot$ from $\PA$ requires at least $\Ack(10, 10)$ symbols.

  We further define $\quot{\enc{\Theory^\prime}\textnormal{\ is consistent}}$ to be the universal generalization
  \[
    \quot{\forall n \colon \textnormal{there is no proof of $\bot$ from $\enc{\Theory^\prime}$ in $n$ or fewer symbols}},
  \]
  and $\quot{\enc{\Theory^\prime}\textnormal{\ is inconsistent}}$ for its negation.
\end{definition}

\begin{restatable}[Belief in Finitistic Consistency]{theorem}{restatepac}\label{thm:pac}
  Let $f$ be any computable function. Then
  \[
    \pt_\nn(\consen(\Theory)(\quot{\enc{f}(\enc{\nn})})) \eqsim_\nn 1.
  \]
  \proofin{\ref{app:pac}}
\end{restatable}

\noindent In other words, if $\Theory$ is in fact consistent, then $\MP$ learns to trust it for arbitrary finite amounts of time. For any fast-growing function $f$ you can name, $\MP$ eventually learns to believe $\Theory$ is consistent for proofs of length at most $f(n)$, by day $\nn$ at the latest. In colloquial terms, if we take a logical inductor over $\PA$ and show it a computable function $f$ that, on each input $\nn$, tries a new method for finding an inconsistency in $\PA$, then the logical inductor will stare at the function for a while and eventually conclude that it's not going to succeed (by learning to assign low probability to $f(\nn)$ proving $\bot$ from $\PA$ by day $\nn$ at the latest, regardless of how long $f$ runs). That is to say, a logical inductor over $\PA$ learns to trust Peano arithmetic \emph{inductively}.

By the same mechanism, a logical inductor over $\Theory$ can learn inductively to trust the consistency of \emph{any} consistent theory, including consistent theories that are stronger than $\Theory$ (in the sense that they can prove $\Theory$ consistent):

\begin{restatable}[Belief in the Consistency of a Stronger Theory]{theorem}{restatepazfc}\label{thm:pazfc}
  Let $\Theory^\prime$ be any recursively axiomatizable consistent theory. Then
  \[
    \pt_\nn(\consen(\Theory^\prime)(\quot{\enc{f}(\enc{\nn})})) \eqsim_\nn 1.
  \]
  \proofin{\ref{app:pazfc}}
\end{restatable}

\noindent For instance, a logical inductor over $\PA$ can learn inductively to trust the consistency of $\ZFC$ for finite proofs of arbitrary length (assuming $\ZFC$ is in fact consistent).

These two theorems alone are unimpressive. Any algorithm that assumes consistency until proven otherwise can satisfy these theorems, and because every inconsistent theory admits a finite proof of inconsistency, those na\"\i{}ve algorithms will disbelieve any inconsistent theory eventually. But those algorithms will still believe inconsistent theories for quite a long time, whereas logical inductors learn to distrust inconsistent theories in a timely manner:

\begin{restatable}[Disbelief in Inconsistent Theories]{theorem}{restateincons}\label{thm:incons}
  Let $\seq{\Theory^\prime}$ be an \ec sequence of recursively axiomatizable inconsistent theories. Then
  \[
    \pt_\nn(\quot{\enc{\Theory^\prime_\nn}\textnormal{\ is inconsistent}}) \eqsim_\nn 1,
  \]
  so
  \[
    \pt_\nn(\quot{\enc{\Theory^\prime_\nn}\textnormal{\ is consistent}}) \eqsim_\nn 0.
  \]
  \proofin{\ref{app:incons}}
\end{restatable}

\noindent In other words, logical inductors learn in a timely manner to distrust inconsistent theories that can be efficiently named, even if the shortest proofs of inconsistency are very long.

Note that \Theorem{pac} \emph{does not say}
\[
  \pt_\infty(\quot{\enc{\Theory}\textnormal{ is consistent}})
\]
is equal to 1, nor even that it's particularly high. On the contrary, by \Theorem{nd}, the limiting probability on that sentence is bounded away from 0 and 1 (because both that sentence and its negation are consistent with $\Theory$). Intuitively, $\DP$ never reveals evidence against the existence of non-standard numbers, so $\MP$ remains open to the possibility. This is important for \Theorem{scon}, which say that logical inductors can safely be conditioned on any sequence of statements that are consistent with $\Theory$, but it also means that $\MP$ will not give an affirmative answer to the question of whether $\PA$ is consistent in full generality.

In colloquial terms, if you hand a logical inductor any \emph{particular} computation, it will tell you that that computation isn't going to output a proof $\bot$ from the axioms of $\PA$, but if you ask whether $\PA$ is consistent \emph{in general}, it will start waxing philosophical about non-standard numbers and independent sentences---not unlike a human philosopher.

A reasonable objection here is that \Theorem{pac} is not talking about the consistency of the Peano axioms, it's talking about \emph{computations} that search for proofs of contradiction from $\PA$. This is precisely correct, and brings us to our next topic.

\subsection{Reasoning about Halting}\label{sec:halting}

Consider the famous halting problem of \citet{turing1936computable}. Turing proved that there is no general algorithm for determining whether or not an arbitrary computation halts. Let's examine what happens when we confront logical inductors with the halting problem.

\begin{restatable}[Learning of Halting Patterns]{theorem}{restatehalts}\label{thm:halts}
  Let $\seq{m}$ be an \ec sequence of Turing machines, and $\seq{x}$ be an \ec sequence of bitstrings, such that $m_\nn$ halts on input $x_\nn$ for all $\nn$. Then
  \[
    \pt_\nn(\quot{\text{$\enc{m_\nn}$ halts on input $\enc{x_\nn}$}}) \eqsim_\nn 1.
  \]
  \proofin{\ref{app:halts}}
\end{restatable}

Note that the individual Turing machines \emph{do not} need to have fast runtime. All that is required is that the \emph{sequence} $\seq{m}$ be efficiently computable, i.e., it must be possible to write out the source code specifying~$m_\nn$ in time polynomial in~$\nn$. The runtime of an individual~$m_\nn$ is immaterial for our purposes. So long as the $m_\nn$ all halt on the corresponding $x_\nn$, $\MP$ recognizes the pattern and learns to assign high probability to $\quot{\text{$\enc{m_\nn}$ halts on input $\enc{x_\nn}$}}$ no later than the $\nn$th day.

Of course, this is not so hard on its own---a function that assigns probability~1 to everything also satisfies this property. The real trick is separating the halting machines from the non-halting ones. This is harder. It is easy enough to show that $\MP$ learns to recognize \ec sequences of machines that \emph{provably} fail to halt:

\begin{restatable}[Learning of Provable Non-Halting Patterns]{theorem}{restateloops}\label{thm:loops}
  Let $\seq{q}$ be an \ec sequence of Turing machines, and $\seq{y}$ be an \ec sequence of bitstrings, such that $q_\nn$ \emph{provably} fails to halt on input $y_\nn$ for all $\nn$. Then
  \[
    \pt_\nn(\quot{\text{$\enc{q_\nn}$ halts on input $\enc{y_\nn}$}}) \eqsim_\nn 0.
  \]
  \proofin{\ref{app:loops}}
\end{restatable}
\noindent Of course, it's not too difficult to disbelieve that the provably-halting machines will halt; what makes the above theorem non-trivial is that $\MP$ learns \emph{in a timely manner} to expect that those machines won't halt. Together, the two theorems above say that if there is any efficient method for generating computer programs that definitively either halt or don't (according to $\Theory$) then $\MP$ will learn the pattern.

The above two theorems only apply to cases where $\Theory$ can prove that the machine either halts or doesn't. The more interesting case is the one where a Turing machine~$q$ fails to halt on input~$y$, but $\Theory$ is not strong enough to prove this fact. In this case, $\pt_\infty$'s probability of~$q$ halting on input~$y$ is positive, by \Theorem{nd}. Nevertheless, $\MP$ still learns to stop expecting that those machines will halt after any reasonable amount of time:

\begin{restatable}[Learning not to Anticipate Halting]{theorem}{restatedontwait}\label{thm:dontwait}
  Let $\seq{q}$ be an \ec sequence of Turing machines, and let $\seq{y}$ be an \ec sequence of bitstrings, such that $q_\nn$ does not halt on input $y_\nn$ for any $\nn$. Let $f$ be any computable function. Then
  \[
    \pt_\nn(\quot{\text{$\enc{q_\nn}$ halts on input $\enc{y_\nn}$ within $\enc{f}(\enc{\nn})$ steps}}) \eqsim_\nn 0.
  \]
  \proofin{\ref{app:dontwait}}
\end{restatable}

For example, let $\seq{y}$ be an enumeration of all bitstrings, and let $\seq{q}$ be the constant sequence $(q, q, \ldots)$ where $q$ is a Turing machine that does not halt on any input. If $\Theory$ cannot prove this fact, then $\MP$ will never be able to attain certainty about claims that say~$q$ fails to halt, but by \Thm{dontwait}, it still learns to expect that~$q$ will run longer than any computable function you can name. 
In colloquial terms, while $\MP$ won't become certain that non-halting machines don't halt (which is impossible), it \emph{will} put them in the ``don't hold your breath'' category (along with some long-running machines that do halt, of course).

These theorems can be interpreted as justifying the intuitions that many computer scientists have long held towards the halting problem: It is impossible to tell whether or not a Turing machine halts in full generality, but for large classes of well-behaved computer programs (such as \ec sequences of halting programs and provably non-halting programs) it's quite possible to develop reasonable and accurate beliefs. The boundary between machines that compute fast-growing functions and machines that never halt is difficult to distinguish, but even in those cases, it's easy to learn to stop expecting those machines to halt within any reasonable amount of time. \mbp{See also the work of \citet{calude2008most} for other formal results backing up this intuition.} 

One possible objection here is that the crux of the halting problem (and of the $\Theory$-trust problem) are not about making good predictions, they are about handling diagonalization and paradoxes of self-reference. G\"odel's incompleteness theorem constructs a sentence that says ``there is no proof of this sentence from the axioms of $\PA$'', and Turing's proof of the undecidability of the halting problem constructs a machine which halts iff some other machine thinks it loops. $\MP$ learning to trust $\Theory$ is different altogether from $\MP$ learning to trust \emph{itself}. So let us turn to the topic of $\MP$'s beliefs about $\MP$.

\subsection{Introspection}\label{sec:introspection}

Because we're assuming $\Theory$ \representscomputations, we can write sentences describing the beliefs of $\MP$ at different times. What happens when we ask $\MP$ about sentences that refer to itself?

For instance, consider a sentence $\psi := \quot{\enc{\pt}_{\enc{\nn}}(\enc{\smash{\phi}}) > 0.7}$ for some specific $\nn$ and $\phi$, where $\MP$'s beliefs about $\psi$ should depend on what its beliefs about $\phi$ are on the $\nn$th day. Will $\MP$ figure this out and get the probabilities right on day $\nn$? For any particular $\phi$ and $\nn$ it's hard to say, because it depends on whether $\MP$ has learned how $\psi$ relates to $\MP$ and $\phi$ yet. If however we take an \ec \emph{sequence} of $\psis$ which all say ``$\phi$ will have probability greater than 0.7 on day $\nn$'' with $\nn$ varying, then we can guarantee that $\MP$ will learn the pattern, and start having accurate beliefs about its own beliefs:

\begin{restatable}[Introspection]{theorem}{restateref}\label{thm:ref}
  Let $\phis$ be an \ec sequence of sentences, and $\seq{a}$, $\seq{b}$ be \pgenable sequences of probabilities. Then, for any \ec sequence of positive rationals $\deltas \to 0$, there exists a sequence of positive rationals $\varepsilons \to 0$ such that for all $\nn$:
  \begin{enumerate}    
    \item if $\pt_\nn(\phi_\nn)\in(a_\nn+\delta_\nn,b_\nn-\delta_\nn)$, then
    \[
      \pt_\nn(\quot{\enc{a_\nn} < \enc{\pt}_\enc{\nn}(\enc{\phi_\nn}) < \enc{b_\nn}}) > 1-\varepsilon_\nn,
    \] 
    \item if $\pt_\nn(\phi_\nn)\notin(a_\nn-\delta_\nn,b_\nn+\delta_\nn)$, then
    \[
      \pt_\nn(\quot{\enc{a_\nn} < \enc{\pt}_\enc{\nn}(\enc{\phi_\nn}) < \enc{b_\nn}}) < \varepsilon_\nn.
    \] 
  \end{enumerate}
  \proofin{\ref{app:ref}}
\end{restatable}

\noindent In other words, for any pattern in $\MP$'s beliefs that can be efficiently written down (such as ``$\MP$'s probabilities on $\phis$ are between $a$ and $b$ on these days''), $\MP$ learns to believe the pattern if it's true, and to disbelieve it if it's false (with vanishing error).

At a first glance, this sort of self-reflection may seem to make \li{}s vulnerable to paradox. For example, consider the sequence of sentences
\[
  \seq{\chi^{0.5}} := (\quot{{\enc{\pt}_{\enc{\nn}}}(\enc{\chi^{0.5}_\nn}) < 0.5})_{\nn\in {\NN^+}}
\]
such that $\chi^{0.5}_\nn$ is true iff $\MP$ assigns it a probability less than 50\% on day $\nn$. Such a sequence can be defined by G\"odel's diagonal lemma. These sentences are probabilistic versions of the classic ``liar sentence'', which has caused quite a ruckus in the setting of formal logic \citep{grim1991incomplete,mcgee1990truth,glanzberg2001liar,gupta1993revision,eklund2002inconsistent}. Because our setting is probabilistic, it's perhaps most closely related to the ``unexpected hanging'' paradox---$\chi^{0.5}_\nn$ is true iff $\MP$ thinks it is unlikely on day $\nn$. How do logical inductors handle this sort of paradox?

\begin{restatable}[Paradox Resistance]{theorem}{restatelp}\label{thm:lp}
  Fix a rational $\prob\in(0,1)$, and define an \ec sequence of ``paradoxical sentences'' $\seq{\chi^\prob}$ satisfying
  \[
    \Theory \vdash{{\enc{\chi^\prob_\nn}} \iff \left(
      {\enc{\pt}_{\enc{\nn}}}({\enc{\chi^\prob_\nn}}) < \enc{\prob}
    \right)}
  \]
  for all $\nn$. Then
  \[
    \lim_{\nn\to\infty}\pt_\nn(\chi^\prob_\nn)=\prob.
  \]
  \proofin{\ref{app:lp}}
\end{restatable}

\noindent A logical inductor responds to paradoxical sentences $\seq{\chi^\prob}$ by assigning probabilities that converge on $\prob$. For example, if the sentences say ``$\MP$ will assign me a probability less than 80\% on day $\nn$'', then $\pt_\nn$ (once it has learned the pattern) starts assigning probabilities extremely close to 80\%---so close that traders can't tell if it's slightly above or slightly below. By \Theorem{recurringunbiasedness}, the frequency of truth in $\chi^\prob_{\leq \nn}$ will have a limit point at 0.8 as $n \to \infty$, and by the definition of logical induction, there will be no efficiently expressible method for identifying a bias in the price.

Let us spend a bit of time understanding this result. After day $\nn$, $\chi_\nn^{0.8}$ is ``easy'' to get right, at least for someone with enough computing power to compute $\pt_\nn(\chi_\nn^{0.8})$ to the necessary precision (it will wind up \emph{very} close to 0.8 for large $\nn$). Before day $\nn$, we can interpret the probability of $\chi_\nn^{0.8}$ as the price of a share that's going to pay out \$1 if the price on day $\nn$ is less than 80\textcent{}, and \$0 otherwise. What's the value of this share? Insofar as the price on day $\nn$ is going to be low, the value is high; insofar as the price is going to be high, the value is low. So what actually happens on the $\nn$th day? Smart traders buy $\chi_\nn^{0.8}$ if its price is lower than 80\textcent{}, and sell it if its price is higher than 80\textcent{}. By the continuity constraints on the traders, each one has a price at which they stop buying $\chi_\nn^{0.8}$, and \Theorem{lp} tells us that the stable price exists extremely close to 80\textcent{}. Intuitively, it must be so close that traders can't tell which way it's going to go, biased on the low side, so that it looks 80\% likely to be below and 20\% likely to be above to any efficient inspection. For if the probability seemed more than 80\% likely to be below, traders would buy; and if it seemed anymore than 20\% likely to be above, traders would sell.

To visualize this, imagine that your friend owns a high-precision brain-scanner and can read off your beliefs. Imagine they ask you what probability you assign to the claim ``you will assign probability $<$80\% to this claim at precisely 10am tomorrow''. As 10am approaches, what happens to your belief in this claim? If you become extremely confident that it's going to be true, then your confidence should drop. But if you become fairly confident it's going to be false, then your confidence should spike. Thus, your probabilities should oscillate, pushing your belief so close to 80\% that you're not quite sure which way the brain scanner will actually call it. In response to a paradoxical claim, this is exactly how $\MP$ behaves, once it's learned how the paradoxical sentences work.

Thus, logical inductors have reasonable beliefs about their own beliefs even in the face of paradox. We can further show that logical inductors have ``introspective access'' to their own beliefs and expectations, via the medium of logically uncertain variables:

\begin{restatable}[Expectations of Probabilities]{theorem}{restateepr}\label{thm:epr}
  Let $\phis$ be an efficiently computable sequence of sentences. Then
  \[
    \pt_\nn(\phi_\nn)\eqsim_\nn\EE_\nn(\quot{\enc{\pt}_\enc{\nn}(\enc{\phi_\nn})}).
  \]
  \proofin{\ref{app:epr}}
\end{restatable}

\begin{restatable}[Iterated Expectations]{theorem}{restateer}\label{thm:er}
 Suppose $\seq{X}$ is an efficiently computable sequence of LUVs. Then
  \[
    \EE_\nn(X_\nn)\eqsim_\nn\EE_\nn(\quot{\enc{\EE}_\enc{\nn}(\enc{X_\nn})}).
  \]
  \proofin{\ref{app:er}}
\end{restatable}

Next, we turn our attention to the question of what a \li{} believes about its \emph{future} beliefs.

\subsection{Self-Trust}\label{sec:selftrust}

The coherence conditions of classical probability theory guarantee that a probabilistic reasoner trusts their future beliefs, whenever their beliefs change in response to new empirical observations. For example, if a reasoner $\Bayesian(\any)$ knows that tomorrow they'll see some evidence $e$ that will convince them that Miss Scarlet was the murderer, then they already believe that she was the murderer today:
\[
  \Bayesian(\mathrm{Scarlet}) = \Bayesian(\mathrm{Scarlet}\mid e) \Bayesian(e) + \Bayesian(\mathrm{Scarlet}\mid \lnot e) \Bayesian(\lnot e).
\]
In colloquial terms, this says ``my current beliefs are \emph{already} a mixture of my expected future beliefs, weighted by the probability of the evidence that I expect to see.''

Logical inductors obey similar coherence conditions with respect to their future beliefs, with the difference being that a logical inductor updates its belief by gaining more knowledge about \emph{logical} facts, both by observing an ongoing process of deduction and by thinking for longer periods of time. Thus, the self-trust properties of a logical inductor follow a slightly different pattern:

\begin{restatable}[Expected Future Expectations]{theorem}{restatecee}\label{thm:cee}
  Let $\deff$ be a deferral function (as per \Def{deferralfunc}), and let $\seq{X}$ denote an \ec sequence of $[0,1]$-LUVs. Then
  \[
    \EE_\nn(X_\nn) \eqsim_\nn
    \EE_\nn(\quot{\enc{\EE}_{\enc{\deff}(\enc{\nn})}(\enc{X_\nn})}).
  \]
  \proofin{\ref{app:cee}}
\end{restatable}
\noindent Roughly speaking, \Thm{cee} says that a logical inductor's current expectation of $X$ on day $\nn$ is \emph{already} equal to its expected value of $X$ in $\deff(\nn)$ days. In particular, it learns in a timely manner to set its current expectations equal to its future expectations on any LUV\@. In colloquial terms, once a logical inductor has figured out how expectations work, it will never say ``I currently believe that the $\seq{X}$ variables have low values, but tomorrow I'm going to learn that they have high values''. Logical inductors already expect today what they expect to expect tomorrow.

It follows immediately from theorems~\ref{thm:cee} (\nameref{thm:cee}) and~\ref{thm:ei} (\nameref{thm:ei}) that the current beliefs of a logical inductor are set, in a timely manner, to equal their future expected beliefs.
\begin{restatable}[No Expected Net Update]{theorem}{restateceu}\label{thm:ceu}
  Let $\deff$ be a deferral function, and let $\phis$ be an \ec sequence of sentences. Then
  \[
      \pt_\nn(\phi_\nn) \eqsim_\nn 
      \EE_\nn(\quot{\enc{\pt}_{\enc{\deff}(\enc{\nn})}(\enc{\phi_\nn})}).
  \]
  \proofin{\ref{app:ceu}}
\end{restatable}
\noindent In particular, if $\MP$ knows that its future self is going to assign some sequence $\probs$ of probabilities to $\phis$, then it starts assigning $\probs$ to $\phis$ in a timely manner.

\Theorem{cee} can be generalized to cases where the LUV on day~$\nn$ is multiplied by an expressible feature:

\begin{restatable}[No Expected Net Update under Conditionals]{theorem}{restateccee}\label{thm:ccee}
  Let $\deff$ be a deferral function, and let $\seq{X}$ denote an \ec sequence of $[0,1]$-LUVs, and let $\seq w$ denote a \pgenable sequence of real numbers in $[0, 1]$.  Then
  \[
  \EE_\nn(\quot{\enc{X_\nn} \cdot \enc{w}_{\enc{\deff}(\enc{\nn})}}) \eqsim_\nn
  \EE_\nn(\quot{\enc{\EE}_{\enc{\deff}(\enc{\nn})}(\enc{X_\nn}) \cdot \enc{w}_{\enc{\deff}(\enc{\nn})}}).
  \]
  \proofin{\ref{app:ccee}}
\end{restatable}

\noindent To see why \Thm{ccee} is interesting, it helps to imagine the case where $\seq{X}$ is a series of bundles of goods and services, and $w_\nn$ is $\ctsind{\delta_\nn}(\EE_{\deff(\nn)}(X_\nn) > 0.7)$ for some sequence of rational numbers $\deltas \to 0$, as per \Def{ctsind}. This value is 1 if $\MP$ will expect the $\nn$th bundle to be worth more than 70\textcent{} on day $\deff(\nn)$, and 0 otherwise, and intermediate if the case isn't quite clear. Then
\[
  \EE_\nn\left(\quot{\enc{X}_\enc{\nn} \cdot
    \enc{\ctsind{\delta_\nn}}\mleft(
      \enc{\EE}_{\enc{f}(\enc{\nn})}(\enc{X}_\enc{\nn}) > 0.7
    \mright)}
  \right)
\]
can be interpreted as $\MP$'s expected value of the bundle on day $\nn$, in cases where $\MP$ is going to think it's worth at least 70\textcent{} on day $f(\nn)$. Now assume that $\ctsind{\delta_\nn}(\EE_{f(\nn)}(X_\nn)) > 0$ and divide it out of both sides, in which case the theorem roughly says
\[
  \EE_\mathrm{now}(X \mid \EE_\mathrm{later}(X) > 0.7) \eqsim \EE_\mathrm{now}(\EE_\mathrm{later}(X) \mid \EE_\mathrm{later}(X) > 0.7),
\]
which says that $\MP$'s expected value of the bundle now, given that it's going to think the bundle has a value of at least 70\textcent{} later, is equal to whatever it expects to think later, conditioned on thinking later that the bundle is worth at least 70\textcent{}.

Combining this idea with indicator functions, we get the following theorem:
\begin{restatable}[Self-Trust]{theorem}{restatest}\label{thm:st}
  Let $\deff$ be a deferral function, $\phis$ be an \ec sequence of sentences, $\deltas$ be an \ec sequence of positive rational numbers, and $\probs$ be a \pgenable sequence of rational probabilities. Then
  \[
    \EE_\nn\left(\quot{
      \enc{\OneOperator(\phi_\nn)} \cdot
      \enc{\ctsind{\delta_\nn}}\mleft(
        \enc{\pt}_{\enc{\deff}(\enc{\nn})}(\enc{\phi_\nn}) > \enc{\prob_\nn}
      \mright)
    }\right)
    \gtrsim_\nn
    \prob_\nn \cdot
    \EE_\nn\left(\quot{
      \enc{\ctsind{\delta_\nn}}\mleft(
        \enc{\pt}_{\enc{\deff}(\enc{\nn})}(\enc{\phi_\nn}) > \enc{\prob_\nn}
      \mright)
    }\right).
  \]
  \proofin{\ref{app:st}}
\end{restatable}
\noindent Very roughly speaking, if we squint at \Thm{st}, it says something like 
\[
    \EE_\mathrm{now}(\phi \mid P_\mathrm{later}(\phi) > \prob) \gtrsim \prob,
\]
i.e., if we ask $\MP$ what it would believe about $\phi$ now if it learned that it was going to believe $\phi$ with probability at least~$\prob$ in the future, then it will answer with a probability that is at least~$\prob$.

As a matter of fact, \Thm{st} actually says something slightly weaker, which is also more desirable. Let each $\phi_\nn$ be the self-referential sentence $\quot{\enc{\pt}_{\enc{\deff}(\enc{\nn})}(\enc{\phi_\nn}) < 0.5}$ which says that the future $\pt_{\deff(\nn)}$ will assign probability less than 0.5 to $\phi_\nn$. Then, conditional on $\pt_{\deff(\nn)}(\phi_\nn) \ge 0.5$, $\pt_\nn$ should believe that the probability of $\phi_\nn$ is 0. And indeed, this is what a logical inductor will do:
\[
  \pt_\nn\left(
    \quot{\enc{\phi_\nn} \land
    (
      \enc{\pt}_{\enc{\deff}(\enc{\nn})}(
        \enc{\phi_\nn}
      ) \ge 0.5
    )}
  \right) \eqsim_\nn 0,
\]
by \Theorem{perkno}, because each of those conjunctions is disprovable. This is why \Thm{st} uses continuous indicator functions: With discrete conjunctions, the result would be undesirable (not to mention false).

What \Thm{st} says is that $\MP$ attains self-trust of the ``if in the future I will believe $x$ is very likely, then it must be because $x$ is very likely'' variety, while retaining the ability to think it can outperform its future self's beliefs when its future self confronts paradoxes. In colloquial terms, if we ask ``what's your probability on the paradoxical sentence $\phi_\nn$ given that your future self believes it with probability \emph{exactly} 0.5?'' then $\MP$ will answer ``very low'', but if we ask ``what's your probability on the paradoxical sentence $\phi_\nn$ given that your future self believes it with probability \emph{extremely close to} 0.5?'' then $\MP$ will answer ``roughly 0.5.''

Still speaking roughly, this means that logical inductors trust their future beliefs to be accurate and only change for good reasons. \Thm{st} says that if you ask ``what's the probability of $\phi$, given that in the future you're going to believe it's more than $95\%$ likely?'' then you'll get an answer that's no less than $0.95$, even if the logical inductor currently thinks that $\phi$ is unlikely.

\renewcommand{\proofin}[1]{} 


%
%
%
%
%
%
%
%
%
%
\section{Construction}\label{sec:construct}

In this section, we show how given any deductive process $\DP$, we can construct a computable belief sequence, called $\seq\LIA$, that satisfies the \lic{} relative to $\DP$. Roughly speaking, $\seq\LIA$ works by simulating an economy of traders and using Brouwer's fixed point theorem to set market prices such that no trader can exploit the market relative to $\DP$.

We will build $\LIA$ from three subroutines called \marketmaker, \budgeter, and \tradingfirm. Intuitively, \marketmaker will be an algorithm that sets market prices by anticipating what a single trader is about to do, \budgeter will be an algorithm for altering a trader to stay within a certain budget, and \tradingfirm will be an algorithm that uses \budgeter to combine together an infinite sequence of carefully chosen \ec traders (via a sum calculable in finite time) into a single trader that exploits a given market if any \ec trader exploits that market. Then, $\LIA$ will work by using \marketmaker to make a market not exploitable by \tradingfirm and hence not exploitable by any \ec trader, thereby satisfying the \lic{}.

To begin, we will need a few basic data types for our subroutines to pass around:

\begin{definition}[Belief History]
  An \textbf{$\bm \nn$-belief history} $\BelState_{\leq \nn} = (\BelState_1,\ldots,\BelState_\nn)$ is a finite list of belief states of length $\nn$.
\end{definition}

\begin{definition}[Strategy History]
  An \textbf{$\bm \nn$-strategy history} $\trade_{\leq \nn} = (\trade_1,\ldots,\trade_\nn)$ is a finite list of trading strategies of length $n$, where $\trade_i$ is an $i$-strategy.
\end{definition}

\newcommand{\support}{\operatorname{Support}}
\begin{definition}[Support]
For any valuation $\Valuation$ we define
\begin{align*}
\support(\Valuation) &:= \{ \phi\in\Sentences \mid 
\Valuation(\phi)\neq 0 \},\\
\intertext{and for any $\nn$-strategy $\trade_\nn$ we define}
\support(\trade_\nn) &:= \{ \phi\in\Sentences \mid 
\trade_\nn[\phi]\not\equiv 0 \}.
\end{align*}
Observe that for any belief state $\BelState$ and any $\nn$-strategy $\trade_\nn$, $\support(\BelState)$ and $\support(\trade_\nn)$ are computable from the finite lists representing $\BelState$ and $\trade_\nn$.
\end{definition}

\subsection{Constructing \texorpdfstring{\marketmaker}{MarketMaker}}
Here we define the \marketmaker subroutine and establish its key properties.  Intuitively, given any trader $\Trader$ as input, on each day $\nn$, \marketmaker looks at the trading strategy $\trade_\nn$ and the valuations $\pt_{\leq \nn-1}$ output by \marketmaker on previous days. It then uses an approximate fixed point (guaranteed to exist by Brouwer's fixed point theorem) that sets prices $\pt_\nn$ for that day such that when the trader's strategy $ \trade_\nn$ reacts to the prices, the resulting trade $\trade_\nn(\pt_{\leq \nn})$ earns at most a very small positive amount of value in any world. Intuitively, the fixed point finds the trader's ``fair prices'', such that they abstain from betting, except possibly to buy sentences at a price very close to \$1 or sell them at a price very close to \$0, thereby guaranteeing that very little value can be gained from the trade.

\begin{lemma}[Fixed Point Lemma]\label{lem:fpl}
  Let $\trade_\nn$ be any $\nn$-strategy, and let $\pt_{\leq \nn-1}$ be any $(\nn-1)$-belief history. There exists a valuation $\Valuation$ with $\support(\Valuation)\subseteq \support(\trade_\nn)$ such that
  \begin{equation}\label{eq:fplemma}
    \text{for all worlds $\World \in \Worlds$:}\quad\World\left(\trade_\nn(\pt_{\le \nn-1}\append\Valuation)\right) \le  0.
  \end{equation}
\end{lemma}

\begin{proof}
We will use Brouwer's fixed point theorem to find ``prices'' $\Valuation$ such that $\trade_\nn$ only ever buys shares for \$1 or sells them for \$0, so it cannot make a profit in any world. Intuitively, we do this by making a ``price adjustment'' mapping called $\adj$ that moves prices toward 1 or 0 (respectively) as long as $\trade_\nn$ would buy or sell (respectively) any shares at those prices, and finding a fixed point of that mapping.

First, we let $\Sentences' = \support(\trade_\nn)$ and focus on the set
\begin{equation*}
\Valuations' := \{ \Valuation \mid \support(\Valuation)\subseteq \Sentences'\}.
\end{equation*}
Observe that $\Valuations'$ is equal to the natural inclusion of the 
finite-dimensional cube $[0,1]^{\Sentences'}$ in the space of all 
valuations $\Valuations = [0,1]^{\Sentences}$.  We now define our ``price adjustment'' function $\adj:\Valuations'\to\Valuations'$ as follows:
\begin{equation*}
 \adj(\Valuation)(\phi):=
   \max\mleft(0,\; \min\mleft(
     1, \; \Valuation(\phi)+ \trade_\nn (\pt_{\le \nn-1}\append\Valuation)[\phi]
   \mright)
 \mright).
\end{equation*}
This map has the odd property that it adds prices and trade volumes, but it does the trick. Notice that $\adj$ is a function from the compact, convex space $\Valuations'$ to itself, so if it is continuous, it satisfies the antecedent of Brouwer's fixed point theorem. Observe that $\adj$ is in fact continuous, because trade strategies are continuous. Indeed, we required that trade strategies be continuous for precisely this purpose. Thus, by Brouwer's fixed point theorem, $\adj$ has at least one fixed point $\Valuation^\adj$ that satisfies, for all sentences $\phi\in\Sentences'$,
\begin{equation*}
\Valuation^\adj\-(\phi) = \max(0, \; \min( 1, \; \Valuation^\adj\-(\phi) + \trade_\nn (\pt_{\le \nn-1}\append\Valuation^\adj)[\phi]) ).
\end{equation*}
Fix a world $\World$ and observe from this equation that if $\trade_\nn$ buys some shares of $\phi\in\Sentences'$ at these prices, i.e.\ if $\trade_\nn(\pt_{\le \nn-1}\append\Valuation^\adj)[\phi]>0$, then $\Valuation^\adj\-(\phi)=1$, and in particular, 
$\World(\phi) - \Valuation^\adj\-(\phi)\leq 0$.  Similarly, if $\trade_\nn$ sells some shares of $\phi$, i.e.\ if $\trade_\nn(\pt_{\le \nn-1}\append\Valuation^\adj)[\phi]<0$, then $\Valuation^\adj\-(\phi)=0$, so 
$\World(\phi) - \Valuation^\adj\-(\phi) \geq 0$.  In either case, we have
\begin{align*}
0 &\geq (\World(\phi)-\Valuation^\adj\-(\phi))\cdot \trade_\nn(\pt_{\le \nn-1}\append\Valuation^\adj)[\phi]\\
\intertext{since the two factors always have opposite sign (or at least one factor is 0).  Summing over all $\phi$, remembering that $\trade_\nn(\Valuation_{\leq \nn})[\phi]=0$ for $\phi \notin \Sentences'$, gives}
0 &\geq \sum_{\phi\in\Sentences} (\World(\phi)-\Valuation^\adj\-(\phi))\cdot \trade_\nn(\pt_{\le \nn-1}\append\Valuation^\adj)[\phi]
\nonumber\\
&=\World(\trade_\nn(\pt_{\leq n}\append\Valuation^\adj)) - \Valuation^\adj\-(\trade_\nn(\pt_{\le \nn-1}\append\Valuation^\adj))
\end{align*}
since the values of the ``cash'' terms $\World(\trade_\nn(\pt_{\leq n}\append\Valuation^\adj)[1])$ and $\Valuation^\adj\-(\trade_\nn(\pt_{\leq n}\append\Valuation^\adj)[1])$ are by definition both equal to $\trade_\nn(\pt_{\leq n}\append\Valuation^\adj)[1]$ and therefore cancel. But
\[
    \Valuation^\adj\-(\trade_\nn(\pt_{\le \nn-1}\append\Valuation^\adj)) = 0
\]
by definition of a trading strategy, so for any world $\World$, we have
\[
0\ge\World(\trade_\nn(\pt_{\le \nn-1}\append\Valuation^\adj)).
\]
\end{proof}

\begin{defprop}[\marketmaker]\label{def:markemaker}
  There exists a computable function, henceforth named \marketmaker, satisfying the following definition.  Given as input any $\nn\in\NN^+$, any $\nn$-strategy $\trade_\nn$, and any $(\nn-1)$-belief history $\pt_{\le \nn-1}$, $\marketmaker_\nn(\trade_\nn,\pt_{\le \nn-1})$ returns a belief state $\BelState$ with $\support(\BelState)\subseteq\support(\trade_\nn)$ such that
  \begin{equation}\tag{\thetheorem}\label{eq:marketmaker}
    \text{for all worlds $\World \in \Worlds$:}\quad\World\mleft(\trade_\nn(\pt_{\le \nn-1}\append\BelState)\mright) \le  2^{-n}.
  \end{equation}
\end{defprop}
\begin{proof}
 Essentially, we will find a rational approximation $\BelState$ to the fixed point $\Valuation^\adj$ in the previous lemma, by brute force search.  This requires some care, because the set of all worlds is uncountably infinite.

First, given $\trade_\nn$ and $\pt_{\le \nn-1}$, let $\Sentences':=\support(\trade_\nn)$, $\Valuations' := \{\Valuation \mid \support(\Valuation)\subseteq \Sentences'\}$, and take $\Valuation^\adj\in\Valuations'$ satisfying~\eq{fplemma}.  Let 
$\Worlds' := \{\World \mid \support(\World)\subseteq\Sentences'\}$, and for any world $\World$, define $\World'\in\Worlds'$ by
\[
\World'(\phi) := 
\begin{cases}
\World(\phi) &\mbox{if $\phi\in\Sentences'$,}\\
0 &\mbox{otherwise.}
\end{cases}
\]
Observe that for any $\World\in\Worlds$, the function $\Valuations'\to\RR$ given by
\[\Valuation\mapsto\World
(\trade_\nn(\pt_{\le \nn-1},\Valuation)) = \World'
(\trade_\nn(\pt_{\le \nn-1},\Valuation))
\]
is a continuous function of $\Valuation$ that depends only on $\World'$.  Since the set $\Worlds'$ is finite, the function
\[
\Valuation\mapsto\sup_{\World\in\Worlds} \World
(\trade_\nn(\pt_{\le \nn-1}\append\Valuation)) = \max_{\World'\in\Worlds'} \World'
(\trade_\nn(\pt_{\le \nn-1}\append\Valuation))
\]
is the maximum of a finite number of continuous functions, and is therefore continuous.  Hence there is some neighborhood in $\Valuations'$ around $\Valuation^\adj$ with image in $(-\infty,2^{-n})
\subset \RR$. 
By the density of rational points in $\Valuations'$, there is therefore some belief state $\BelState\in\Valuations'\cap \QQ^{\Sentences}$ satisfying~\eq{marketmaker}, as needed.

It remains to show that such a $\BelState$ can in fact be found  by brute force search.  First, recall that a belief state $\BelState$ is a rational-valued finite-support map from $\Sentences$ to  $[0,1]$,
 and so can be represented by a finite list of pairs $(\phi,q)$ with $\phi\in\Sentences$ and $q\in\QQ\cap [0,1]$.  Since $\Sentences$ and $[0,1]\cap \QQ$  are computably enumerable, so is the set of all belief states.
 
Thus, we can computably ``search'' though all possible $\BelState$s, so we need 
only establish that given $n$, $\trade_\nn$, and $\pt_{\le \nn-1}$ we can computably decide whether each $\BelState$ in our search satisfies
\eq{marketmaker} until we find one.  First note that the finite set $\support(\trade_\nn)$ can be computed by searching the expression specifying $\trade_\nn$ for all the sentences $\phi$ that occur within it.  Moreover, \Eqn{marketmaker} need only be be checked for worlds $\World'\in\Worlds'$, since any other $\World$ returns the same value as its corresponding $\World'$.  Now, for any fixed world $\World'\in\Worlds'$ and candidate $\BelState$, we can compute each value in the language of expressible features 
\[
\World'(\trade_\nn(\pt_{\le \nn-1}\append\pt)) = \trade_\nn(\pt_{\le \nn-1}\append\pt)[1] + \sum_{\phi\in\Sentences'} \World'(\phi)\cdot \trade_\nn(\pt_{\le \nn-1}\append\pt)[\phi]
\]
directly by evaluating the expressible features $\trade_\nn[\phi]$ on the given belief history $(\pt_{\le \nn-1}\append\pt)$, as $\phi\in\Sentences'$ varies.  Since $\Worlds'$ is a finite set, we can do this for all $\World'$ with a finite computation.  Thus, checking whether
a belief state $\BelState$ satisfies condition~\eq{marketmaker} is computably decidable, and a solution to~\eq{marketmaker} can therefore be found by enumerating all belief states $\BelState$ and searching through them for the first one that works.
\end{proof}

\begin{lemma}[\marketmaker Inexploitability]\label{lem:mm}
  Let $\Trader$ be any trader.  The sequence of belief states $\seq\pt$ defined recursively by
  \begin{align*}
    \pt_\nn&:=\marketmaker_\nn(\trade_\nn,\pt_{\le \nn-1}),
  \end{align*}
  with base case $\pt_1 = \marketmaker(\trade_1,())$, is not exploited by $\Trader$ relative to any deductive process $\DP$.
\end{lemma}

\begin{proof}
  By the definition of \marketmaker, we have that for every $\nn$, the belief state $\pt=\pt_\nn$ satisfies \Eqn{marketmaker}, i.e.,
  \[
    \text{for all worlds $\World \in \Worlds$ and all $\nn \in \NN^+$:}\quad\World(\trade_\nn(\MP)) \leq 2^{-n}.
  \]
  Hence by linearity of $\World$, for all $n\in\NN^+$ we have:
  \[
    \World\mleft({\textstyle \sum_{i \leq n}} \trade_i(\MP)\mright)
    = \sum_{i \leq n}\World(\trade_i(\MP))
    \leq \sum_{i \leq n} 2^{-i} < 1.
  \]
  Therefore, given any deductive process $\DP$,
  \[
    \sup \left\{ \World\mleft({\textstyle \sum_{i \leq n}} \trade_i(\MP)\mright) \,\middle|\, \nn\in\NN^+, \World\in\pcworlds(\dt_\nn) \right\} \leq 1 <\infty,
  \]
  so $\Trader$ does not exploit $\seq\pt$ relative to $\DP$.
\end{proof}

%
%
%
%
%
%
%
%
%
%
%
%
%
%
\subsection{Constructing \texorpdfstring{\budgeter}{Budgeter}}

Here we introduce a subroutine for turning a trader with potentially infinite losses into a trader that will never have less than $-\$b$ in any world $\World\in\pcworlds(\dt_\nn)$ on any day $\nn$, for some bound $b$, in such a way that does not affect the trader if it wouldn't have fallen below $-\$b$ to begin with.

\begin{defprop}[\budgeter]
  Given any deductive process $\DP$, there exists a computable function, henceforth called $\budgeter^\DP$, satisfying the following definition.  Given inputs $\nn$ and $b\in\NN^+$, an $\nn$-strategy history $\trade_{\leq \nn}$, and an $(\nn-1)$-belief history $\pt_{\le \nn-1}$, $\budgeter^\DP$ returns an $\nn$-strategy $\budgeter^\DP_\nn(b,\trade_{\leq \nn},\pt_{\le \nn-1})$, such that
    \begin{align}
      \text{if:\quad}&
      \World\mleft({\textstyle \sum_{i \leq m} \trade_{i}(\pt_{\leq i})}\mright)
      \le -b\text{~for some $m<\nn$ and $\World\in \pcworlds(\dt_m)$,}\notag\\
      \text{then:\quad}&
      \budgeter^\DP_\nn(b,\trade_{\leq \nn},\pt_{\le \nn-1}) = 0,\notag\\\tag{\thetheorem}\label{eq:budgeter}
      \text{else:\quad}&
      \budgeter^\DP_\nn(b,\trade_{\leq \nn},\pt_{\le \nn-1}) =
      \shortintertext{\[
        \trade_\nn \cdot\inf_{\World \in \pcworlds(\dt_\nn)}
        \left[
          \max\mleft(1, \frac{-\World(\trade_\nn)}
                             {b+\World\mleft({\sum_{i \leq \nn-1} \trade_{i}(\pt_{\leq i})}\mright)}
          \mright)
        \right]^{-1}.
      \]}
    \end{align}
\end{defprop}

\begin{proof}
Let $\Sentences'=\bigcup_{i\le\nn}\support(\trade_i)$, $\Worlds' = \{\World \mid \support(\World)\subseteq \Sentences'\}$, and for any world $\World$, write
\[
  \World'(\phi) := 
  \begin{cases}
    \World(\phi) &\mbox{if $\phi\in\Sentences'$,}\\
    0 &\mbox{otherwise.}
  \end{cases}
\]

Now, observe that we can computably check the ``if'' statement in the function definition.  This is because $\World({\textstyle \sum_{i \leq m} \trade_{i}(\pt_{\leq i})})$ depends only on $\World' \in\Worlds'$, a finite set.  We can check whether $\World'\in\pcworlds(\dt_\mm)$ in finite time by checking whether any assignment of truth values to the finite set of prime sentences occurring in sentences of $\dt_\nn$  yields the assignment $\World'$ on $\support(\World')$. The set of sentences $\dt_\nn$ is computable given $n$, because $\DP$ is computable by definition.

It remains to show that the ``else'' expression can be computed and returns an $\nn$-trading strategy.  First, the infimum can be computed over $\World' \in \Worlds' \cap \pcworlds(\dt_\nn)$, a finite set, since the values in the $\inf$ depend only on $\World'$, and the $\inf$ operator itself can be re-expressed in the language of expressible features using $\max$ and multiplication by $(-1)$.  The values $\World'(\trade_\nn)$ and $\World'(\sum_{i \leq \nn-1} \trade_{i}(\pt_{\leq i}))$ are finite sums, and the denominator ${b+\World({\textstyle \sum_{i \leq \nn-1} \trade_{i}(\pt_{\leq i})} )}$ is a fixed positive rational (so we can safely multiply by its reciprocal).  The remaining operations are all single-step evaluations in the language of expressible valuation features, completing the proof.
\end{proof}

\noindent Let us reflect on the meaning of these operations.  The quantity $b + \World(\sum_{i<\nn}\trade_i(\pt_{\leq i}) )$ is the amount of money the trader has available on day $\nn$ according to $\World$ (assuming they started with a budget of $b$), and $-\World(\trade_\nn)$ is the amount they're going to lose on day $\nn$ according to $\World$ as a function of the upcoming prices, and so the infimum above is the trader's trade on day $\nn$ scaled down such that they can't overspend their budget according to any world propositionally consistent with $\dt_\nn$.

\begin{lemma}[Properties of \budgeter]\label{lem:budgeter}
  Let $\Trader$ be any trader, and $\MP$ be any sequence of belief states. Given $\nn$ and $b$, let $B^b_\nn$ denote $\budgeter^\DP_\nn(b,\trade_{\leq \nn},\pt_{\le \nn-1})$. Then:
  \begin{enumerate}
    \item for all $b,\nn\in\NN^+$, if for all $\mm \le \nn$ and $\World \in \pcworlds(\dt_\mm)$ we have $\World\mleft(\sum_{i\leq m}\trade_i(\MP)\mright)>-b$, then
      \[
        B_\nn^b(\MP)=\trade_\nn(\MP)\text{;}
      \]
    \item for all $b,n\in\NN^+$ and all $\World\in\pcworlds(\dt_\nn)$, we have
      \[
        \World\mleft(\textstyle\sum_{i\leq n}B_i^b(\MP)\mright)\geq -b\text{;}
      \]
    \item If $\seq \trade$ exploits $\MP$ relative to $\DP$, then so does $\seq {B}^b$ for some $b\in\NN^+$.
  \end{enumerate}
\end{lemma}

\tinysection{Part 1}

\begin{proof}
  Suppose that for some time step $n$, for all $m\leq n$ and all worlds $\World\in\pcworlds(\dt_\mm)$ plausible at time $m$ we have
  \[
    \World\left({\textstyle \sum_{i \leq \mm}
    \trade_{i}(\MP)} \right)> -b,
  \]
  so by linearity of $\World(\any)$, we have in particular that
  \[
    b+ \World\left({\textstyle \sum_{i \leq \nn-1 }
    \trade_{i}(\MP)} \right)
    > -\World\left({\textstyle 
    \trade_{\nn}(\MP)} \right).
  \]
  Since $\nn-1\le \nn$, the LHS is positive, so we have
  \[
    1> \frac{-\World\mleft({\textstyle \trade_{\nn}(\MP)}\mright)}
            {b+ \World\mleft({\textstyle \sum_{i \leq \nn-1} \trade_{i}(\MP)}\mright)}.
  \]
  Therefore, by the definition of $\budgeter^\DP$ (and $\trade_i(\MP) = \trade_i(\pt_{\le i})$), since the ``if'' clause doesn't trigger by the assumption on the $ \World\left({\textstyle \sum_{i \leq \mm}
    \trade_{i}(\MP)} \right)$ for $m<n$,
  \begin{align*}
    B_\nn^b(\MP) &\equiv
      \trade_\nn(\MP)
      \cdot
      \inf_{\World\in\pcworlds(\dt_\nn)}
      \left.
      1 \middle/ \max\mleft(1,
        \frac{-\World(\trade_\nn(\MP))}
             {b+\World\mleft({\textstyle \sum_{i \leq \nn-1} \trade_{i}(\MP)}\mright)}
      \mright)
      \right.\\
    &= \trade_\nn(\pt_{\leq \nn}) \cdot \inf_{\World\in\pcworlds(\dt_\nn)} 1/1 \\
    &= \trade_\nn(\MP)
  \end{align*}
  as needed.
\end{proof}

\tinysection{Part 2}
\begin{proof}
  Suppose for a contradiction that for some $n$ and some $\World\in\pcworlds(\dt_\nn)$, 
  \[
      \World\left({\textstyle \sum_{i \leq \nn} B^b_{i}(\MP)} \right) < -b.
  \]
  Assume that $n$ is the least such day, and fix some such $\World\in\pcworlds(\dt_\nn)$. By the minimality of $n$ it must be that $\World(B^b_\nn(\MP))< 0$, or else we would have $\World\mleft({\textstyle \sum_{i \leq \nn-1}
  B^b_{i}(\MP)}\mright) < -b$. Since $B^b_\nn(\MP)$ is a non-negative multiple of $\trade_\nn(\MP)$, we also have $\World(T_\nn(\MP))< 0$.  However, since $B^b_{n}\not\equiv 0$,
  from the definition of $\budgeter^\DP$ we have
  \begin{align*}
    \World\mleft({\textstyle B^b_{n}} \mright) 
    &= 
      \World\mleft({\trade_\nn(\MP)}\mright)
      \cdot
      \left(
        \left.
        \inf_{\World'\in\pcworlds(\dt_\nn)}
        1 \middle/ \max\mleft(1, \frac{-\World'(\trade_\nn(\MP))}
                                     {b+\World'({\textstyle \sum_{i \leq \nn-1} \trade_{i}(\MP)})}
        \mright)
        \right.
      \right) \\
    &\geq 
      \World\mleft({\trade_\nn(\MP)}\mright)
      \cdot 
      \left.
      1 \middle/ \max\mleft(1, \frac{-\World(\trade_\nn(\MP))}
                                   {b+\World({\textstyle \sum_{i \leq \nn-1} \trade_{i}(\MP)})}
      \mright)
      \right.
      \text{\ (since  $\World\mleft({\trade_\nn(\MP)}\mright) <0$)}  \\
    &\geq 
      \World\mleft({\trade_\nn(\MP)}\mright)
      \cdot 
      \frac{b+\World({\textstyle \sum_{i \leq \nn-1} \trade_{i}(\MP)})}
           {-\World(\trade_\nn(\MP))}  \\
           \shortintertext{since $-\World\mleft({\trade_\nn(\MP)}\mright) >0$ and $B^b_{n}\not\equiv 0$ implies $b+\World({\textstyle \sum_{i \leq \nn-1} \trade_{i}(\MP)})>0$. Hence, this}
    &= 
      -b-\World\mleft({\textstyle \sum_{i \le \nn} \trade_{i}(\MP)} \mright).
\end{align*}

Further, since $B^b_{n}\not\equiv 0$, we have
\begin{align*}
  \text{for all\ } j \le n-1\colon\quad& 
  \World\mleft({\textstyle \sum_{i \le j}\trade_{i}(\MP)}\mright) > -b,
  \text{\ which by Part 1 implies that}\\
  \text{for all\ } j \le n-1 \colon\quad& 
  B^b_j(\MP) = \trade_j(\MP), \text{ therefore}\\
  &\World(B^b_\nn)\geq  -b-\World\mleft({\textstyle \sum_{i \le \nn-1} B^b_{i}(\MP)} \mright),\text{ hence}\\
  &\World\mleft({\textstyle \sum_{i \le n} B^b_{i}(\MP)}\mright) \geq -b.\qedhere
\end{align*}
\end{proof}

\tinysection{Part 3}
\begin{proof}
  By definition of exploitation, the set 
  \[
    \left\{ \World\mleft({\textstyle \sum_{i \le \nn} \trade_i(\MP)}\mright) \,\middle|\, \nn\in\NN^+, \World\in\pcworlds(\dt_\nn) \right\}
  \]
  is unbounded above, and is strictly bounded below by some integer $b$. Then by Part 1, for all $n$ we have $\trade_\nn(\MP) = B^b_\nn(\MP)$. Thus, 
  \[
    \left\{ \World\mleft({\textstyle \sum_{i \le \nn} B^b_i(\MP)}\mright) \,\middle|\, \nn\in\NN^+, \World\in\pcworlds(\dt_\nn) \right\}
  \]
  is unbounded above and bounded below, i.e., $\seq{B}^b$ exploits $\MP$ relative to $\DP$.
\end{proof}

%
%
%
%
%
%
%
%
%
%
%
%
%
%
\subsection{Constructing \texorpdfstring{\tradingfirm}{TradingFirm}}

Next we define \tradingfirm, which combines an (enumerable) infinite sequence of \ec traders into a single ``supertrader'' that exploits a given belief sequence $\MP$ relative to $\DP$ if any \ec trader does. It does this by taking each \ec trader, budgeting it, and scaling its trades down so that traders later in the sequence carry less weight to begin with.

To begin, we will need a computable sequence that includes every \ec trader at least once. The following trick is standard, but we include it here for completeness:

\begin{proposition}[Redundant Enumeration of \ec Traders]\label{prop:enumeration}
  There exists a computable sequence $\smash{(\Trader^k)_{k\in\NN^+}}$ of \ec traders such that every \ec trader occurs at least once in the sequence.
\end{proposition}
\newcommand{\ect}{\operatorname{ECT}}
\begin{proof}
\newcommand{\setof}[1]{\{\text{#1}\}}
Fix a computable enumeration of all ordered pairs $(M_k,f_k)$ where $M_k$ is a Turing machine and $f_k$ is a polynomial with coefficients in $\ZZ$.  We define a computable function 
\[\ect: \setof{Turing machines} \times \setof{Integer polynomials} \times (n \in \NN^+) \to \setof{$\nn$-strategies}\]
that runs as follows: $\ect(M,f,n)$ first runs $M(n)$ for up to $f(n)$ time steps, and if in that time $M(n)$ halts and returns a valid $\nn$-strategy $\trade_\nn$, then $\ect(M,f,n)$ returns that strategy, otherwise it returns 0 (as an $\nn$-strategy).  Observe that $\ect(M_k,f_k,\any)$ is always an \ec trader, and that every \ec trader occurs as $\ect(M_k,f_k,\any)$ for some~$k$.
\end{proof}

\begin{defprop}[$\tradingfirm$]
  Given any deductive process $\DP$, there exists a computable function, henceforth called $\tradingfirm^\DP$, satisfying the following definition.  
  By Proposition~\ref{prop:enumeration}, we fix a computable enumeration $\Trader^k$ including every \ec trader at least once, and let 
  \[
    S^k_\nn=\begin{cases}T^k_\nn&\mbox{if }\nn\geq k\\0&\mbox{otherwise}.\end{cases}
  \]
  Given input $n\in\NN^+$ and an $(\nn-1)$-belief history $\pt_{\leq{n-1}}$, $\tradingfirm^\DP$ returns an $\nn$-strategy given by
  \begin{equation}\tag{\thetheorem}\label{eq:tradingfirm}
    \tradingfirm^\DP_\nn(\pt_{\leq \nn-1}) = \sum_{k\in  \NN^+}\sum_{b\in\NN^+}2^{-k-b}\cdot \budgeter^\DP_\nn(b,S^k_{\leq n} ,\pt_{\le \nn-1}).
  \end{equation}
\end{defprop}
\begin{proof}
We need only show that the infinite sum in \Eqn{tradingfirm} is equivalent to a computable finite sum.  
Writing 
\[B^{b,k}_\nn = \budgeter^\DP_\nn(b,S^k_{\leq n} ,\pt_{\le \nn-1}),
\]
(an $n$-strategy), the sum on the RHS of~\eq{tradingfirm} is equivalent to
\[
\sum_{k\in\NN^+}\sum_{b\in\NN^+} 2^{-k-b}\cdot B^{b,k}_\nn .
\]
Since $S^k_\nn = 0$ for $k>\nn$, we also have $B^{b,k}_\nn=0$ for $k>\nn$, so the sum is equivalent to
\[
=\sum_{k\leq n}\sum_{b\in\NN^+} 2^{-k-b}\cdot B^{b,k}_\nn .
\]

Now, assume $C_n$ is a positive integer such that $\sum_{i \leq \nn}\|S^k_i(\seq\Valuation)\|_1<C_n$ for all $k \leq  \nn$ and any valuation sequence $\seq\Valuation$ (we will show below that such a $C_n$ can be computed from $\pt_{\le \nn-1}$). Since the valuations $\World$ and $\MP$ are always $[0,1]$-valued, for any $m \leq n$ the values $ \World\mleft(\sum_{i\leq m}S^k_i(\pt_{\leq m})\mright) $ are bounded below by $-\sum_{i \leq \mm}\|S^k_i(\pt_{\leq m})\|_1>-C_n$.   By property 1 of $\budgeter^\DP$  (Lemma~\ref{lem:budgeter}.1), $B^{b,k}_\nn = S^k_\nn$ when $b > C_n$, so the sum is equivalent to
\begin{align*}
  =& \left(\sum_{k\leq n}\sum_{b \leq C_n}2^{-k-b}\cdot B^{b,k}_\nn\right)+\left(\sum_{k\leq n}\sum_{b> C_n}2^{-k-b}\cdot S^k_\nn\right)\\
  =& \left(\sum_{k\leq n}\sum_{b \leq C_n}2^{-k-b}\cdot B^{b,k}_\nn\right)+\left(\sum_{k\leq n}2^{-k-C_n}\cdot S^k_\nn\right)
\end{align*}
which is a finite sum of trading strategies, and hence is itself a trading strategy. 
Since the $B^{b,k}_\nn$ and the $S^k_\nn$ are computable from $\pt_{\le \nn-1}$, this finite sum is computable.

It remains to justify our assumption that integers $C_n$ can be computed from $\pt_{\le \nn-1}$ with  $C_n>\sum_{i \leq \nn}\|S^k_i(\seq\Valuation)\|_1$ for all $k \leq \nn$ and  $\seq\Valuation$.  To see this, first consider how to bound a single expressible feature $\xi$. We can show by induction on the structure of $\xi$ (see \ref{app:expressiblefeatures}) that, given constant  bounds on the absolute value $|\zeta(\seq\Valuation)|$ of each subexpression $\zeta$ of $\xi$, we can compute a constant bound on $|\xi(\seq\Valuation)|$; for example, the bound on $\zeta \cdot \eta$ is the product of the bound on $\zeta$ and the bound on $\eta$. Thus, given a single trading strategy $S^k_i$ and any $\phi$, we can compute a constant upper bound on $|S^k_i[\phi](\seq\Valuation)|$ for all $\seq\Valuation$. Since $\|S^k_i(\seq\Valuation)\|_1 \leq \sum_{\phi \in \support(S^k_i)} 2 |S^k_i[\phi](\seq\Valuation)|$ and $\support(S^k_i)$ is computable, we can bound each $\|S^k_i(\seq\Valuation)\|_1$, and hence also $\sum_{i \leq \nn}\|S^k_i(\seq\Valuation)\|_1$, as needed.

\end{proof}

\begin{lemma}[Trading Firm Dominance]\label{lem:tfdom}
  Let $\MP$ be any sequence of belief states, and $\DP$ be a deductive process. If there exists any \ec trader $\Trader$ that exploits $\MP$ relative to $\DP$, then the sequence
  \[
    \left(\tradingfirm^\DP_\nn(\pt_{\leq \nn-1})\right)_{\nn\in\NN^+}
  \]
  also exploits $\MP$ (relative to $\DP$).
\end{lemma}

\begin{proof}
  Suppose that some \ec trader exploits $\MP$. That trader occurs as $\seq{\trade}^k$ for some $k$ in the enumeration used by $\tradingfirm^\DP$.  First, we show that $\seq{S}^k$ (from the definition of $\tradingfirm^\DP$) also exploits $\MP$.  It suffices to show that there exist constants $c_1\in\RR^+$ and $c_2\in\RR$ such that for all $n\in\NN^+$ and $\World\in\pcworlds(\dt_\nn)$, 
  \[
    \World\mleft(\textstyle\sum_{i\leq n}S^k_i(\MP)\mright)\geq c_1\cdot \World\mleft(\textstyle\sum_{i\leq n}T^k_i(\MP)\mright)+c_2.
  \]
  Taking $c_1=1$ and $c_2=-\sum_{i<k}\|T^k_i(\MP)\|_1$, where $\|\cdot\|_1$ denotes the $\ell_1$ norm on $\RR$-combinations of sentences, we have 
  \[
    \World\mleft(\textstyle\sum_{i\leq n}S^k_i(\MP)\mright)\geq 1\cdot \World\mleft(\textstyle\sum_{i\leq n}T^k_i(\MP)\mright)- \left(\textstyle\sum_{i<k}\|T^k_i(\MP)\|_1\right),
  \]
  so $\seq{S}^k$ exploits $\MP$.  By Lemma~\ref{lem:budgeter}.3, we thus have that for some $b\in\NN^+$, the trader $\seq{B}^{b,k}$ given by
  \[B^{b,k}_\nn := \budgeter^\DP_\nn(b,S^k_{\leq n},\pt_{\le \nn-1})\]
  also exploits $\MP$.

  Next, we show that the trader $\seq\supertrade$ given by
  \[
    \supertrade_\nn := \tradingfirm^\DP_\nn(\pt_{\le \nn-1})
  \] 
  exploits $\MP$. Again, it suffices to show that there exist constants $c_1\in\RR^+$ and $c_2\in\RR$ such that for all $n\in\NN^+$ and $\World\in\pcworlds(\dt_\nn)$, 
  \[\World\left(\sum_{i\leq n}\supertrade_i\right)\geq c_1\cdot \World\left(\sum_{i\leq n}B^{b,k}_i\right)+c_2.\]
  It will suffice to take $c_1=2^{-k-b}$ and $c_2=-2$, because we have 
  \begin{align*}
    &\World\mleft(\sum_{i\leq n}\supertrade_i\mright)-2^{-k-b}\cdot \World\mleft(\sum_{i\leq n}B^{b,k}_i\mright)\\
    =&\sum_{(k^\prime,b^\prime)\neq(k,b)}2^{-k^\prime-b^\prime}\cdot \World\mleft(\sum_{i\leq n}B^{b^\prime,k^\prime}_i\mright)\\
    \geq& \sum_{(k^\prime,b^\prime)\neq(k,b)}2^{-k^\prime-b^\prime}\cdot (-b^\prime)\geq -2
  \end{align*}
by Lemma~\ref{lem:budgeter}.2,   hence
  \[\World\mleft(\sum_{i\leq n}\supertrade_i\mright)\geq 2^{-k-b}\cdot \World\mleft(\sum_{i\leq n}B^{b,k}_i\mright) -2.\]
  Thus, $\seq\supertrade$ exploits $\MP$.
\end{proof}

%
%
%
%
%
%
%
%
%
%
%
%
%
%
\subsection{Constructing \texorpdfstring{$\seq\LIA$}{LIA}}

We are finally ready to build $\LIA$.  With the subroutines above, the idea is now fairly simple: we pit \marketmaker and \tradingfirm against each other in a recursion, and \marketmaker wins. Imagine that on each day, \tradingfirm outputs an ever-larger mixture of traders, then \marketmaker carefully examines that mixture and outputs a belief state on which that mixture makes at most a tiny amount of money on net.

\begin{defalg}[A Logical Induction Algorithm]\label{alg:li}\label{def:lia}
  Given a deductive process $\DP$, define the computable belief sequence $\seq{\LIA}=(\LIA_1, \LIA_2, \ldots)$ recursively by
  \[
    \LIA_\nn := \marketmaker_\nn(\tradingfirm^\DP_\nn(\LIA_{\le \nn-1}), \LIA_{\le \nn-1}),
  \]
  beginning from the base case $\LIA_{\leq 0}:=()$.
\end{defalg}
\begin{theorem}[\LIA is a Logical Inductor]\label{thm:lia}
 $\seq{\LIA}$ satisfies the \lic{} relative to $\DP$, i.e., LIA is not exploitable by any \ec trader relative to the deductive process $\DP$. 
\end{theorem}
\begin{proof}\label{proof:lia}
  By \Lem{tfdom}, if any \ec trader exploits $\seq{\LIA}$ (relative to $\DP$), then so does the trader $\seq\supertrade := (\tradingfirm^\DP_\nn(\LIA_{\le \nn-1}))_{\nn\in\NN^+}$. By \Lem{mm}, $\seq\supertrade$ does not exploit $\seq{\LIA}$. Therefore no \ec trader exploits $\seq{\LIA}$. 
\end{proof}

\subsection{Questions of Runtime and Convergence Rates}\label{sec:bounds}

In this paper, we have optimized our definitions for the theoretical clarity of results rather than for the efficiency of our algorithms.   This leaves open many interesting questions about the relationship between runtime and convergence rates of logical inductors that have not been addressed here.  Indeed, the runtime of $\LIA$ is underspecified because it depends heavily on the particular enumerations of traders and rational numbers used in the definitions of \tradingfirm and \marketmaker.  

For logical inductors in general, there will be some tradeoff between the runtime of $\pt_\nn$ as a function of $\nn$ and how quickly the values $\pt_\nn(\phi)$ converge to $\pt_\infty(\phi)$ as $\nn$ grows.  Quantifying this tradeoff may be a fruitful source of interesting open problems.  Note, however, the following important constraint on the convergence rate of any logical inductor, regardless of its implementation, which arises from the halting problem:

\begin{proposition}[Uncomputable Convergence Rates]
Let $\MP$ be a logical inductor over a theory $\Theory$ that \representscomputations, and suppose $f:\Sentences\times\QQ^+\to\NN$ is a function such that for every sentence $\phi$, if $\Theory\vdash\phi$ then $\pt_\nn(\phi) > 1-\varepsilon$ for all $\nn>f(\phi,\varepsilon)$.  Then $f$ must be uncomputable.
\end{proposition}
\begin{proof}
  Suppose for contradiction that such a computable $f$ were given.  We will show that $f$ could be used to computably determine whether $\Theory\vdash\phi$ for an arbitrary sentence $\phi$, a task which is known to be impossible for a first-order theory that \representscomputations. (If we assumed further that $\Theory$ were sound as a theory of the natural numbers, this would allow us to solve the halting problem by letting $\phi$ be a sentence of the form ``$M$ halts''.)

  Given a sentence $\phi$, we run two searches in parallel. If we find that $\Theory \vdash \phi$, then we return True. If we find that for some $b,n\in\NN^+$ we have
  \begin{equation}\tag{\thetheorem}\label{eq:nothalting}
    n>f\mleft(\phi,\frac{1}{b}\mright) \text{~and~} \pt_\nn(\phi)\leq 1-\frac{1}{b},
  \end{equation}
  then we return False. Both of these conditions are computably enumerable since $f$, $\pt_n$, and verifying witnesses to $\Theory \vdash \phi$ are computable functions. 

  Suppose first that  $\Theory \vdash \phi$. Then by definition of $f$ we have $\pt_\nn(\phi)>1-\frac{1}{b}$ for all   $n>f\mleft(\phi,\frac{1}{b}\mright)$, and hence we find a witness for $\Theory \vdash \phi$ and return True. Now suppose that  $\Theory \nvdash \phi$. Then by \Theorem{nd} we have that $\pt_\infty(\phi)<1-\varepsilon$ for some $\varepsilon>0$, and hence for some $b$ and all sufficiently large $n$ we have $\pt_n(\phi)<1-1/b$. Therefore \ref{eq:nothalting} holds and we return False. Thus our search always halts and returns a Boolean value that correctly indicates whether $\Theory \vdash \phi$. 
\end{proof}

%
%
%
%
%
%
%
%
%
%
%
%
%
%
%
%
%


\section{Selected Proofs}\label{sec:selectedproofs}

In this section, we exhibit a few selected stand-alone proofs of certain key theorems. These theorems hold for any $\MP$ satisfying the \lic{}, which we recall here:

\criterion*

\noindent Only our notation (\Sec{notation}), framework (\Sec{framework}), and continuous threshold indicator (\Def{ctsind}) are needed to understand the results and proofs in this section. Shorter proofs of these theorems can be found in the appendix, but those rely on significantly more machinery.


\subsection{Convergence}\label{sec:convergence}
Recall Theorem~\ref{thm:con} and the proof sketch given:

\convergence*

\sketchcon*

\noindent We will define a trader $\Trader$ that executes a strategy similar to this one, and hence exploits the market $\MP$ if $\lim_{\nn\to\infty} \pt_\nn(\phi)$ diverges. To do this, there are two technicalities we must deal with. First, the strategy outlined above uses a discontinuous function of the market prices $\pt_\nn(\phi)$, and therefore is not permitted. This is relatively easy to fix using the continuous indicator functions of \Def{ctsind}.

The second technicality is more subtle. Suppose we define our trader to buy $\phi$-shares whenever their price $\pt_\nn(\phi)$ is low, and sell them back whenever their price is high.  Then  it is possible that the trader makes the following trades in sequence against the market $\MP$: buy 10 $\phi$-shares on consecutive days, then sell 10 $\phi$-shares; then buy 100 $\phi$-shares consecutively, and then sell them off; then buy 1000 $\phi$-shares, then sell them off; and so on. Although this trader makes profit on each batch, it always spends more on the next batch, taking larger and larger risks (relative to the remaining plausible worlds). Then the plausible value of this trader's holdings will be unbounded below, and so it does not exploit $\MP$. In short, this trader is not tracking its budget, and so may have unboundedly negative plausible net worth. We will fix this problem by having our trader $\Trader$ track how many net $\phi$-shares it has bought, and not buying too many, thereby maintaining bounded risk. This will be sufficient to prove the theorem.

\begin{proof}[Proof of \Thm{con}]
  Suppose by way of contradiction that the limit $\pt_\infty$ does not exist. Then, for some sentence $\phi$ and some rational numbers
  $p\in [0, 1]$ and $\varepsilon > 0$, we have that $\pt_\nn(\phi) < p-\varepsilon$ infinitely often and $\pt_\nn(\phi) > p+\varepsilon$  infinitely often. We will show that $\MP$ can be exploited by a trader $\Trader$ who buys below and sells above these prices infinitely often,  contrary to the logical induction criterion. 

  \tinysection{Definition of the trader $\Trader$} We will define $\Trader$ recursively along with another sequence of $\EF$-combinations $\seq\hold$ (mnemonic: ``holdings'') which tracks the sum of the trader's previous trades.  Our base cases are 
  \[\strat_1 := \overline{0}\]
  \[\hold_1:=\overline{0} .\] 
  For $n>1$, we define a recurrence whereby $\Trader$ will buy some $\phi$-shares whenever $\pf{\phi} < p-\varepsilon/2$, up to  $(1-\hold_{\nn-1}[\phi])$ shares when $\pf{\phi} < p-\varepsilon$, and sells some $\phi$-shares whenever $\pf{\phi} > p+\varepsilon/2$, up to $\hold_{n-1}$ shares when $\pf{\phi} > p+\varepsilon$:
  \begin{equation}\label{eq:convergencetn}
    \begin{aligned}
      \strat_\nn[\phi] &:= (1-\hold_{\nn-1}[\phi]) \cdot \ctsind{\varepsilon/2}( \pf{\phi}< p-\varepsilon/2)\\
      &\hphantom{\; := (1)  } - \hold_{\nn-1}[\phi] \cdot \ctsind{\varepsilon/2}( \pf{\phi}> p+\varepsilon/2),\\
      \strat_\nn &:= \strat_\nn[\phi] \cdot(\phi - \pf{\phi})\\
      \hold_\nn &:= \hold_{n-1} + \strat_\nn .
    \end{aligned}
\end{equation}
  The trade coefficients $\strat[\phi]$ are chosen so that the number of $\phi$-shares $\hold_\nn[\phi]$ that it owns is always in $[0,1]$ (it never buys more than $1-\hold_{\nn-1}[\phi]$ and never sells more than $\hold_{\nn-1}[\phi]$).
  Observe that each $\strat_\nn$ is a valid trading strategy for day $\nn$ (see \Def{tradestrat}) because it is of the form $\exf\cdot(\phi-\pf\phi)$.  

  To complete the definition, we must argue that $\Trader$ is efficiently computable.  For this, observe that the $3\nn+2$ definition ($:=$) equations defining $\strat_1,\ldots,\strat_\nn$ above can be written down in time polynomial in $\nn$. Thus, a combination of feature expressions defining $\strat_\nn$ from scratch can be written down in $\poly(\nn)$ time (indeed, the expression is just a concatenation of $\nn$ copies of the three ``$:=$'' equations written above, along with the base cases), so $\seq\strat$ is efficiently computable.

  \tinysection{Proof of exploitation} 
  To show $\Trader$ exploits $\MP$ over $\DP$, we must compute upper and lower bounds on the set of plausible values $\World(\hold_\nn(\MP))$ (since $\hold_\nn = \sum_{i\leq n} \strat_\nn$) for worlds $\World\in\pcworlds(\dt_\nn)$. 

  While proving exploitation, we leave the constant argument $\MP$ implicit to reduce clutter, writing, e.g., $\phi^{*i}$ for $\phi^{*i}(\MP) = \pt_i(\phi)$, $\strat_\nn[\phi]$ for $\strat_\nn[\phi](\MP)$, and so on.

  First, since each $\strat_i[1] = -\strat_i[\phi]\cdot\pf[i]\phi$, the trader's ``cash'' held on day $\nn$ is
  \begin{align*}
    \hold_\nn[1] &= \sum_{i\le\nn} \strat_i[1] = - \sum_{i\le\nn} \strat_i[\phi]\cdot\phi^{*i}\\
    \shortintertext{which we can regroup, to compare the prices $\pf[i]\phi$ to $p$, as}
    \hold_\nn[1]&=  \sum_{i\le\nn} \mleft(\strat_i[\phi]\cdot(p-\phi^{*i})\mright) - p \cdot \sum_{i\le\nn} \strat_i[\phi]\\
    &=  \sum_{i\le\nn} \mleft(\strat_i[\phi]\cdot(p-\phi^{*i})\mright) - p \cdot \hold_\nn[\phi] . \\
    \intertext{Now, if $\phi^{*i} < p-\varepsilon/2$ then $\strat_i[\phi]\ge 0$,  if $\phi^{*i} > p+\varepsilon/2$ then $\strat_i[\phi]\le 0$, and if $p-\varepsilon/2 \le \phi^{*i} \le p+\varepsilon/2$ then $\strat_i[\phi] = 0$, so for all $i$ the product $\strat_i[\phi]\cdot(p-\phi^{*i})$ is equal to or greater than $|\strat_i[\phi]|\cdot\varepsilon/2$:} 
    \hold_\nn[1] &  \ge  - p \cdot \hold_\nn[\phi] + \sum_{i \le \nn}|\strat_i[\phi]| \cdot \varepsilon/2 . \\
    \intertext{Moreover, by design, $\hold_\nn[\phi] \in [0,1]$ for all $\nn$, so}
    \hold_\nn[1] &  \ge -p +  \sum_{i \le \nn}|\strat_i[\phi]| \cdot \varepsilon/2.
  \end{align*}
  Now, by assumption, $\pf[i]\phi$ lies above and below $(p-\varepsilon,p+\varepsilon)$ infinitely often, so from \Eqn{convergencetn}, $\hold_i[\phi]=0$ and $\hold_i[\phi]=1$ infinitely often.  Since the sum $\sum_{i \le \nn}|\strat_i[\phi]|$ is the total variation in the sequence $\hold_i[\phi]$, it must diverge (by the triangle inequality) as $\nn\to\infty$, so 
  \[
    \lim_{\nn\to\infty}\hold_\nn[1] = \infty .
  \]
  Moreover, in any world $\World$, the trader's non-cash holdings $\hold_\nn[\phi]\cdot\phi$ have value 
  $\World(\hold_\nn[\phi]\cdot\phi)= \hold_\nn[\phi]\cdot\World(\phi)\geq 0$ (since $\hold_\nn[\phi] > 0$), so its combined holdings $\hold_\nn = \hold_\nn[1] + \hold_\nn[\phi]\cdot\phi$ have value
  \[
    \World(\hold_\nn)= \World\mleft(\hold_\nn[1]+\hold_\nn[\phi]\cdot\phi\mright) = \hold_\nn[1]+\hold_\nn[\phi]\cdot\World(\phi) \ge  \hold_\nn[1]
  \]
  so in \emph{every} world $\World$ we have
  \[
    \lim_{\nn\to\infty}\World(\hold_\nn) = \infty .
  \]
  This contradicts that $\MP$ is a \li{}; therefore, the limit $\pt_\infty(\phi)$ must
  exist.
\end{proof}


\subsection{Limit Coherence}\label{sec:limitcoherence}
Recall Theorem~\ref{thm:lc}:

\limitcoherence*

\begin{proof}[Proof of \Thm{lc}]
  By \Theorem{con}, the limit $\pt_\infty(\phi)$ exists for all sentences $\phi\in\Sentences$. Therefore, $\Bayesian(\World(\phi)=1): =\pt_\infty(\phi)$ is well-defined as a function of basic subsets of the set of all consistent worlds $\pcworlds(\dt_\infty) = \cworlds(\Gamma)$.

  \citet{Gaifman:1964} shows that $\Bayesian$ extends to a probability measure over $\cworlds(\Theory)$ so long as the following three implications hold for all sentences $\phi$ and $\psi$:
  \begin{itemize}
    \item If $\Theory \vdash \phi$, then $\pt_\infty(\phi) = 1$.
    \item If $\Theory \vdash \lnot \phi$, then $\pt_\infty(\phi) = 0$.
    \item If $\Theory \vdash \lnot(\phi \land \psi)$, then $\pt_\infty(\phi \lor \psi) = \pt_\infty(\phi) + \pt_\infty(\psi)$.
  \end{itemize}
  Since the three conditions are quite similar in form, we will prove them simultaneously using four exemplar traders and parallel arguments. 

  \tinysection{Definition of the traders} Suppose that one of the three conditions is violated by a margin of $\varepsilon$, i.e., one of the following four cases holds:
  \begin{alignat*}{2}    
    &(L^1)\;\;\Theory \vdash \phi,  \text{ but } && (I^1)\;\; \pt_\infty(\phi) <  1-\varepsilon;\\
    &(L^2)\;\; \Theory \vdash \lnot \phi, \text{ but } && (I^2)\;\; \pt_\infty(\phi) > \varepsilon;\\
    &(L^3)\;\; \Theory \vdash \lnot(\phi \land \psi),  \text{ but } && (I^3)\;\; \pt_\infty(\phi \lor \psi) < \pt_\infty(\phi) + \pt_\infty(\psi) - \varepsilon;\text{ or}\\
    &(L^4)\;\; \Theory \vdash \lnot(\phi \land \psi), \text{ but \quad } &&(I^4)\;\; \pt_\infty(\phi \lor \psi) > \pt_\infty(\phi) + \pt_\infty(\psi) + \varepsilon.
  \end{alignat*}
  Let $i\in\{1,2,3,4\}$ be the case that holds.  Since the limit $\pt_\infty$ exists, there is some sufficiently large time $s_\varepsilon$ such that for all $n>s_\varepsilon$, the inequality $I^i$ holds with $\nn$ in place of $\infty$.  Furthermore, since $\DP$ is a $\Theory$-complete deductive process, for some sufficiently large $s_\Theory$ and all $n> s_\Theory$, the logical condition $L^i$ holds with $\dt_\nn$ in place of $\Theory$.  Thus, letting $s:= \max( s_\varepsilon, s_\Theory )$, for $n > s$ one of the following cases holds:
  \begin{alignat*}{2}    
    &(L^1_\nn)\;\;\dt_\nn \vdash \phi,  \text{ but } && (I^1_\nn)\;\; \pt_\nn(\phi) <  1-\varepsilon;\\
    &(L^2_\nn)\;\; \dt_\nn \vdash \lnot \phi, \text{ but } && (I^2_\nn)\;\; \pt_\nn(\phi) > \varepsilon;\\
    &(L^3_\nn)\;\; \dt_\nn \vdash \lnot(\phi \land \psi),  \text{ but } && (I^3_\nn)\;\; \pt_\nn(\phi \lor \psi) < \pt_\nn(\phi) + \pt_\nn(\psi) - \varepsilon;\text{ or}\\
    &(L^4_\nn)\;\; \dt_\nn \vdash \lnot(\phi \land \psi), \text{ but \quad } &&(I^4_\nn)\;\; \pt_\nn(\phi \lor \psi) > \pt_\nn(\phi) + \pt_\nn(\psi) + \varepsilon.
  \end{alignat*}
  (When interpreting these, be sure to remember that each $\dt_\nn$ is finite, and $\dt\vdash$ indicates using provability using only propositional calculus, i.e., modus ponens.  In particular, the axioms of first order logic are not assumed to be in $\dt_\nn$.)

  We now define, for each of the above four cases, a trader that will exploit the market $\MP$.  For $n>s$, let
  \begin{align*}
    \strat^1_\nn &:= \phi-\phi^{*\nn}\\
    \strat^2_\nn &:= -(\phi-\phi^{*\nn})\\
    \strat^3_\nn &:=  \left((\phi \lor \psi)-(\phi \lor \psi)^{*\nn}\right)-(\phi-\phi^{*\nn}) -(\psi-\psi^{*\nn})\\
    \strat^4_\nn &:=  (\phi-\phi^{*\nn})+(\psi-\psi^{*\nn}) -\left((\phi \lor \psi)-(\phi \lor \psi)^{*\nn}\right)
  \end{align*}
  and for $\nn\le s$ let $\strat^i_\nn = 0$. Each $\trade^i_\nn$ can be written down in $\Oo(\log(\nn))$ time (the constant $s$ can be hard-coded at a fixed cost), so these $\Trader^i$ are all \ec traders.

  \tinysection{Proof of exploitation} 
We leave the constant argument $\MP$ implicit to reduce clutter, writing, e.g., $\phi^{*i}$ for $\phi^{*i}(\MP) = \pt_i(\phi)$, $\strat_\nn[\phi]$ for $\strat_\nn[\phi](\MP)$, and so on.

  Consider case 1, where $L^1_\nn$ and $I^1_\nn$ hold for $\nn>s$, and look at the trader $\Trader^1$.  For any $n>s$ and any world $\World\in\pcworlds(\dt_\nn)$, by linearity of $\World$ we have
  \begin{align*}
    \World\mleft({\textstyle \sum_{i \leq n} \strat^1_i}\mright) 
    &= \sum_{i \le \nn} \strat^1_i[\phi] \cdot \mleft(\World(\phi) - \phi^{*i} \mright)\\
    \shortintertext{but $\strat^1_i[\phi]\equiv 1$ iff $i>s$, so this sum is}
    &= \sum_{s < i \le \nn} 1  \cdot \mleft( \World(\phi) - \phi^{*i} \mright) .\\
    \shortintertext{Now, by our choice of $s$, $\World(\phi) = 1$, and $i>s$ implies $\phi^{*i} < 1-\varepsilon$, so this is}
    &\geq \sum_{s < i \le \nn}  \mleft( 1 - (1-\varepsilon) \mright) \\
    &= \varepsilon\cdot(\nn -s)\\
    &\to \infty \text{ as } \nn\to\infty.
  \end{align*}
  In particular, $\Trader^1$ exploits $\MP$, i.e., the set of values
  \[
    \left\{ \World\mleft({\textstyle \sum_{i \leq n} \strat_i}\mright)\mleft(\MP \mright) \,\middle|\, \nn\in\NN^+, \World\in\pcworlds(\dt_\nn) \right\}
  \]
  is bounded below but not bounded above.  The analysis for case 2 is identical: if $L^2_\nn$ and $I^2_\nn$ hold for $\nn>s$, then $\Trader^2$ exploits $\MP$.

  Now consider case 3, where  $L^3_\nn$ and $I^3_\nn$ hold for $\nn>s$.
  Then for any time step $n>s$ and any world
  $\World\in\pcworlds(\dt_\nn)$,
  \begin{align*}
    \World\mleft({\textstyle \sum_{i \leq n} \strat^3_i}\mright) 
    &= \sum_{i \le \nn} \left((\World(\lnot(\phi \land \psi))-(\phi \lor \psi)^{* i}) - (\World(\phi)-\phi^{* i}) - (\World(\psi)-\psi^{* i}) \right) \\
    &= \sum_{s < i \le \nn} \mleft( \World(\phi \lor \psi) - \World(\phi) - \World(\psi) \mright) - \mleft((\phi \lor \psi)^{*i}  - \phi^{*i} - \psi^{*i} \mright) \\
    \intertext{but by our choice of $s$, $\World(\phi \lor \psi) - \World(\phi) - \World(\psi)=0$, and $i>s$ implies the inequality $(\phi \lor \psi)^{*i} - \phi^{*i} - \psi^{*i}< -\varepsilon$, so the above sum is}
    &\geq \sum_{s < i \le \nn} \varepsilon \\
    &= \varepsilon\cdot(\nn-s) \to\infty \text{ as }\nn\to\infty .
  \end{align*}
  So $\Trader^3$ exploits $\MP$, contradicting the \lic{}.  The analysis for case 4 is identical. Hence, all four implications must hold for $\MP$ to satisfy the \lic{}.
\end{proof}


\subsection{Non-dogmatism}
Recall \Thm{nd}:

\restatenondog*

\begin{proof}[Proof of \Thm{nd}]
We prove the second implication, since the first implication is similar, with selling in place of buying.  Suppose for a contradiction that $\Theory \nvdash \neg\phi$ but that $\pt_\infty(\phi)=0$.  

\tinysection{Definition of the trader $\Trader$} We define $\Trader$ recursively, along with helper functions $\seq\beta^k$ that will ensure that for every $k$, our trader will buy one share of $\phi$ for a price of at most $2^{-k}$:
\begin{align*}
\text{for $k=1,\ldots,\nn$:\hphantom{$,+1$}}\;\;\;\\
\beta^k_k &:=0\\
\text{for $i=k+1,\ldots,\nn$:}\;\;\;\\
\beta^k_i &:= \ctsind{2^{-k-1}}(\phi^{*i}<2^{-k})\cdot\left(1-\sum_{j=k}^{i-1}\beta^k_j \right)\\
\strat_i[\phi] &:= \sum_{j \le i} \beta^k_j\\
\strat_i &:= \strat_i[\phi]\cdot(\phi-\phi^{*i})
\end{align*}
\noindent Note that all the equations defining $\strat_\nn$ can be written down (from scratch) in $\Oo(\nn^3\log(n))$ time, so $\Trader$ is an \ec trader.

\tinysection{Proof of exploitation} We leave the constant argument $\MP$ implicit to reduce clutter, writing, e.g., $\phi^{*i}$ for $\phi^{*i}(\MP) = \pt_i(\phi)$, $\strat_\nn[\phi]$ for $\strat_\nn[\phi](\MP)$, and so on.

Observe from the recursion above for $\Trader$ that for all $i>0$ and $k>0$,
\[
0\le\sum_{j=k}^i\beta^k_j\le 1
\]
and for any $i$ and any $k\le i$,
\[
\beta^k_i \ge 0.
\]
Next, observe that for any $k>0$, for $i\geq$ some threshold $f(k)$, we will have 
$\phi^{*i} < 2^{-k-1}$, in which case the indicator in the definition of $\beta^k_i$ will equal $1$, at which point 
$\sum_{j=k}^i\beta^k_j = 1$.  Thus, for all $\nn\ge f(k)$,
\[
\sum_{i=k}^\nn\beta^k_i = 1.
\]
Letting $\hold_\nn=\sum_{i\le\nn}\strat_i$, the following shows that our trader will eventually own an arbitrarily large number of $\phi$-shares:
\begin{align}
\nonumber\hold_\nn[\phi]&=\sum_{i\le\nn}\sum_{k\le i}\beta^k_i\\
\nonumber &=\sum_{k\le\nn}\sum_{k\le i\le\nn}\beta^k_i\\
\nonumber &\ge \sum_{\substack{k\le\nn \\ f(k)\le\nn}}\sum_{k\le i\le\nn}\beta^k_i\\
&= \sum_{\substack{k\le\nn \\ f(k)\le\nn}} 1 \quad\to\infty\text{~ as ~}\nn\to\infty \label{eqn:manyshares}
\end{align}
Next we show that our trader never spends more than a total of \$1.
\begin{align*}
\hold_\nn[1] &= -\sum_{i\le\nn}\sum_{k\le i}\beta^k_i\cdot\phi^{*i},\\
\intertext{but the indicator function defining $\beta_i^k$ ensures that $\phi^{*i}\le 2^{-k}$ whenever $\beta^k_i$ is non-zero, so this is}
&\ge -\sum_{i\le\nn}\sum_{k\le i}\beta^k_i \cdot 2^{-k}\\
&= -\sum_{k\le\nn}2^{-k}\cdot \sum_{k\le i\le\nn}\beta^k_i \\
&\ge -\sum_{k\le\nn}2^{-k}\cdot 1\\
\end{align*}

Now, for any world $\World$, since $\hold_\nn[\phi]\ge 0$ for all $\nn$ and $\World(\phi)\geq 0$, we have 
\begin{align*}
\World(\hold_\nn) &= \hold_\nn[1] + \hold_\nn[\phi]\World(\phi)\\
&\geq -1 + 0\cdot 0 \geq -1
\end{align*}
so the values $\World(\hold_\nn)$ are bounded below as $\nn$ varies. Moreover, since $\Theory\nvdash\lnot\phi$, for every $\nn$ there is always some $\World\in\pcworlds(\dt_\nn)$ where $\World(\phi)=1$ (since any consistent truth assignment can be extended to a truth assignment on all sentences), in which case
\[
\World(\hold_\nn) \geq -1 + \hold_\nn[\phi] \cdot 1 
\]
But by equation~\ref{eqn:manyshares}, this $\lim_{\nn\to\infty}\hold_\nn[\phi] = \infty$, so 
$\lim_{\nn\to\infty}\World(\hold_\nn) = \infty$ as well.
Hence, our \ec trader exploits the market, contradicting the \lic{}.  Therefore, if $\pt_\infty(\phi)=0$, we must have $\Theory\vdash\lnot\phi$.
\end{proof}

\subsection{Learning Pseudorandom Frequencies}\label{sec:prand}
Recall \Thm{benford}:

\benford*

\noindent Before beginning the proof, the following intuition may be helpful. If the theorem does not hold, assume without loss of generality that $\MP$ repeatedly underprices the $\phi_\nn$. Then a trader can buy $\phi_\nn$-shares whenever their price goes below $p-\varepsilon$. By the assumption that the truth values of the $\phi_\nn$ are pseudorandom, roughly $p$ proportion of the shares will pay out. Since the trader only pays at most $p-\varepsilon$ per share, on average they make $\varepsilon$ on each trade, so over time they exploit the market. All we need to do is make the trades continuous, and ensure that the trader does not go below a fixed budget (as in the proof of \Thm{con}).

\begin{proof}[Proof of \Thm{benford}]
  Suppose for a contradiction that $\phis$ is an \ec sequence of $\Theory$-decidable sentences such that for every \pgenable divergent weighting $\fuz$,
  \[
    \lim_{\nn \to\infty} \frac{\sum_{i < \nn} \fuz_{i} \cdot \thmind{\phi_i}}{\sum_{i < n} \fuz_{i}} = p,
  \]
  but nevertheless, for some $\varepsilon > 0$ and  infinitely many $\nn$, $|\pt_\nn(\phi_\nn) - p| > \varepsilon$.  Without loss of generality, assume that for infinite many $\nn$, %
  \[\pt_\nn(\phi_\nn) < p -\varepsilon.\]
  (The argument for the case where 
  $\pt_\nn(\phi_\nn) > p +\varepsilon$ infinitely often will be the same, and one of these two cases must obtain.)

  \tinysection{Definition of the trader $\Trader$}
  \newcommand{\open}{\operatorname{Open}}
  \newcommand{\mopen}{\operatorname{MO}}
  We define $\open:(\Sentences\times\NN)\to \BB$ to be the following (potentially very slow) computable function:
  \[
    \open(\phi,n)=\begin{cases}
      0 &\mbox{if $D_\nn\vdash \phi$ or $D_\nn\vdash \lnot\phi$};\\
      1 &\mbox{otherwise.}
    \end{cases}
  \]
  $\open$ is computable because (remembering that $\vdash$ stands for propositional provability) we can just search through all truth assignments to the prime sentences of all sentences in $\dt_\nn$ that make the sentences in $\dt_\nn$ true, and see if they all yield the same truth value to $\phi$.  We now define a much faster function  $\mopen:(\NN\times\NN)\to \BB$ (mnemonic: ``maybe open'') by
  \[
    \mopen(\phi,n)=\begin{cases}
      0 &\parbox{20em}{if for some $m\le\nn$, $\open(\phi,m)$\\ returns $0$ in $\leq \nn$ steps}\vspace{2ex}\\
      1 &\text{otherwise.}
    \end{cases}
  \]
  Observe that $\mopen(\phi,n)$ runs in $\Oo(\nn^2)$ time, and that for any decidable $\phi$,
  \begin{itemize*}
    \item  $\mopen(\phi,\nn)= 0$ for some sufficiently large $\nn$;
    \item  if $\mopen(\phi,\nn)= 0$ then $\open(\phi,\nn)$ = 0;
    \item  if $\mopen(\phi,\mm)= 0$ and $\nn>\mm$ then $\mopen(\phi,\nn)=0$.
  \end{itemize*}
  (Note that $\mopen$ may assign a value of $1$ when $\open$ does not, hence the mnemonic ``maybe open''.)  

  We will now use $\mopen$ to define a trader $\Trader$ recursively, along with a helper function $\beta$ to ensure that it never holds a total of more than $1$ unit of open (fractional) shares.  We let $\beta_1= 0$ and for $\nn\geq 1$,
  \begin{align*}
    \beta_\nn &:= 1- \sum_{i<\nn} \mopen(\phi_i, \nn) \strat_i[\phi_i];\\
    \strat_\nn[\phi_\nn] &:= \beta_\nn \cdot \ctsind{\varepsilon/2}\mleft(\phi_{\nn}^{*\nn} < p -\varepsilon /2\mright);\\
    \strat_\nn &:= \strat_\nn[\phi_n]\cdot(\phi_\nn - \phi_{\nn}^{*\nn}).
  \end{align*}
  Observe that the expressible feature $\strat_\nn$ can be computed (from scratch) in $\poly(\nn)$ time using $\mopen$, so $\Trader$ is an \ec trader. Notice also that $\beta_\nn$ and all the $\strat_\nn(\phi)$ are always in $[0,1]$.

  \tinysection{A divergent weighting}
  For the rest of the proof, we leave the constant argument $\MP$ implicit to reduce clutter, writing, e.g., $\phi_{i}^{*i}$ for $\phi_{i}^{*i}(\MP) = \pt_i(\phi_i)$, $\strat_\nn[\phi]$ for $\strat_\nn[\phi](\MP)$, and so on.

  We will show that the sequence of trade coefficients $\fuz_\nn = \strat_\nn[\phi_\nn]$ made by  $\Trader$ against the market $\MP$ form a \pgenable divergent weighting. Our trader $\Trader$ is efficiently computable and $\strat_\nn[\phi_\nn]\in [0,1]$ for all $n$, so it remains to show that, on input $\pt_{\leq\nn}$,
  \[\sum_{\nn \in \NN^+} \strat_\nn[\phi_\nn] =\infty.\]
  %
  Suppose this were not the case, so that for some sufficiently large $m$, 
  \begin{equation}\label{eq:mbound}
    \sum_{m<j} \strat_j[\phi_j] < 1/2.
  \end{equation}
  By the definition of $\mopen$, there exists some large $\mm'$ such that for all $i<\mm$, $\mopen(\phi_i,\mm')=0$.  At that point, for any $n>\mm'$, we have 
  \begin{align*}
    \beta_{n} :=&\,  1- \sum_{i<n} \strat_i[\phi_i]\cdot\mopen(\phi_i, n) \\
    =&\, 1 - \sum_{m<i<n} \strat_i[\phi_i]\cdot\mopen(\phi_i, n) \\
    \geq&\, 1 - \sum_{m<i} \strat_i[\phi_i]\\
    \intertext{which, by \Eqn{mbound}, means that}
    \beta_\nn\geq&\, 1/2.
  \end{align*}
  Then, by the earlier supposition on $\MP$, for some $n>m'$ we have $\pt_\nn(\phi_\nn) < p -\varepsilon$, at which point
  \[\strat_\nn[\phi_\nn] = \beta_\nn \cdot \ctsind{\varepsilon/2}\mleft(\phi_{\nn}^{*\nn} < p -\varepsilon /2\mright) \geq \beta_\nn\cdot 1 \geq 1/2\]
  which contradicts the $1/2$ bound in \Eqn{mbound}.  Hence, the sum $\sum_i \strat_i[\phi_\nn]$ must instead be bounded.  This means $(\strat_\nn[\phi_\nn])_{\nn\in\NN^+}$ is a \pgenable divergent weighting.

  \tinysection{Proof of exploitation} Now, by definition of $\phis$ being pseudorandom with frequency $p$ over the class of \pgenable divergent weightings, we have that 
  \[
    \lim_{\nn \to\infty} \frac{\sum_{i\le\nn} \strat_i[\phi_i] \cdot \thmind(\phi_i)}{\sum_{i\le n} \strat_i[\phi_i] } = p.
  \]
  Thus, for all sufficiently large $n$, 
  \[
    {\sum_{i\le\nn} \strat_i[\phi_i] \cdot \thmind(\phi_i)}\geq (p-\varepsilon/4)\cdot{\sum_{i\le n} \strat_i[\phi_i] } .
  \]

  \newcommand{\minval}{\operatorname{minval}}

  \noindent Now, since our construction makes $\beta_\nn\in[0,1]$ for all $\nn$, we have
  \[\sum_{i\le\nn} \trade_i[\phi_i]\cdot\mopen(\phi_i,n) \leq 1.\]
  Also, 
  \begin{align*}
    \World(\phi_i) &\geq\thmind(\phi_i)-\mopen(\phi_i,n).\\
    \intertext{Multiplying this by $\strat_i[\phi_i]$ and summing over $i$ gives}
    \sum_{i\le\nn} \trade_i[\phi_i]\cdot\World(\phi_i) 
    &\ge \mleft(\sum_{i\le\nn} \trade_i[\phi_i]\cdot\thmind(\phi_i)\mright) - \mleft(\sum_{i\le\nn} \trade_i[\phi_i]\cdot\mopen(\phi_i,n)\mright)\\
    &\ge \mleft(\sum_{i\le\nn} \trade_i[\phi_i]\cdot\thmind(\phi_i)\mright) - 1 \\
    &\ge  - 1 + (p-\varepsilon/4)\sum_{i\le\nn} \trade_i[\phi_i].\\
    \intertext{By the definition of $\Trader$, and since $\phi_{i}^{*i}\le(p-\varepsilon/2)$ whenever $\strat_i[\phi_i]\neq 0$,}
    -\sum_{i\le\nn} \strat_i[\phi_i]\cdot\phi_{i}^{*i} &\ge - (p-\varepsilon/2)\sum_{i\le\nn} \strat_i[\phi_i].\\
    \intertext{Adding the above two inequalities gives}
    \World\mleft(\sum_{i\le\nn} \strat_i\mright) &\ge -1 + (\varepsilon/4)\sum_{i\le\nn} \strat_i[\phi_i]\\
    &\to\infty \text{ as } \nn\to\infty
  \end{align*}
  because $\strat_i[\phi_i]$ is a divergent weighting (as shown above).  Hence, $\Trader$ exploits the market $\MP$, contradicting the \lic{}.  Therefore, for $\MP$ to satisfy the \lic{}, we must have
  \[\lim_{n\to\infty}\pt_\nn(\phi_\nn) = p.\]
\end{proof}


\subsection{Provability Induction}\label{sec:provind}
Recall \Thm{provind}:
\restatepi*

\begin{proof}[Proof of \Thm{provind}]
  Suppose $\phis$ is  an \ec sequence of sentences with $\Theory \vdash \phi_\nn$ for all $n$.  Notice that for every $i$, the indicator $\thmind(\phi_i)$ evaluates to 1. Therefore we immediately have that for any divergent weighting $\fuz$ at all, 
  \[\lim_{\nn \to\infty} \frac{\sum_{i < \nn} \fuz_{i} \cdot \thmind{\phi_i}}{\sum_{i < n} \fuz_{i}}= 1.\]
  \noindent That is, the sequence $\phis$ is pseudorandom (over any class of weightings) with frequency 1.  Hence, by Learning Pseudorandom Frequencies (\Thm{benford}), 
  \[  \pt_\nn(\phi_\nn) \eqsim_\nn 1,\]
  as desired. The proof that $\pt_\nn(\psi_\nn) \eqsim_\nn 0$ proceeds analogously.
\end{proof}

\noindent Examining the proof of \Theorem{benford} in the special case of provability induction yields some intuition. In this case, the trader defined in that proof essentially buys $\phi_\nn$-shares every round that $\pt_\nn(\phi_\nn) < 1-\varepsilon$. To avoid overspending, it tracks which $\phi_\nn$ have been proven so far, and never has more than 1 total share outstanding. Since eventually each $\phi_\nn$ is guaranteed to be valued at 1 in every plausible world, the value of the trader is increased by at least $\varepsilon$ (times the number of $\phi_\nn$-shares it purchased) infinitely often. In this way, the trader makes profits for so long as $\MP$ fails to recognize the pattern $\phis$ of provable sentences.

%
%
%
%
%
%
%
%
%
%
%
%
%
%
%
%
%
\section{Discussion}\label{sec:discussion}

We have proposed the \emph{\lic} as a criterion on the beliefs of deductively limited reasoners, and we have shown that reasoners who satisfy this criterion (\emph{logical inductors}) possess many desirable properties when it comes to developing beliefs about logical statements (including statements about mathematical facts, long-running computations, and the reasoner themself). We have also given a computable algorithm \LIA for constructing a logical inductor. We will now discuss applications of logical induction (\Sec{applications}) and speculate about how and why we think this framework works (\Sec{analysis}). We then discuss a few variations on our framework (\Sec{variations}) before concluding with a discussion of a few open questions (\Sec{openquestions}).

%
%
%
%
%
%
%
%
%
%
%
%
%
%
%
%
%
\subsection{Applications}\label{sec:applications}

Logical inductors are not intended for practical use. The algorithm to compare with logical induction is not Belief Propagation \mbp{an efficient method for approximate inference in Bayesian networks \citep{Pearl:1988}} but Solomonoff's theory of inductive inference \mbp{an uncomputable method for making ideal predictions about empirical facts \citep{Solomonoff:1964}}. Just as Solomonoff's sequence predictor assigns probabilities to all possible observations and learns to predict any computable environment, logical inductors assign probabilities to all possible sentences of logic and learns to recognize any efficiently computable pattern between logical claims.

Solomonoff's theory involves a predictor that considers all computable hypotheses about their observations, weighted by simplicity, and uses Bayesian inference to zero in on the best computable hypothesis. This (uncomputable) algorithm is impractical, but has nevertheless been of theoretical use: its basic idiom---consult a series of experts, reward accurate predictions, and penalize complexity---is commonplace in statistics, predictive analytics, and machine learning. These ``ensemble methods'' often perform quite well in practice. Refer to \citet{opitz1999popular,dietterich2000ensemble} for reviews of popular and successful ensemble methods.

One of the key applications of logical induction, we believe, is the development of an analogous idiom for scenarios where reasoners are uncertain about logical facts. Logical inductors use a framework similar to standard ensemble methods, with a few crucial differences that help them manipulate logical uncertainty. The experts consulted by logical inductors don't make predictions about what is going to happen next; instead, they observe the aggregated advice of all the experts (including themselves) and attempt to exploit inefficiencies in that aggregate model. A trader doesn't need to have an opinion about whether or not $\phi$ is true; they can exploit the fact that $\phi$ and $\lnot\lnot\phi$ have different probabilities without having any idea what $\phi$ says or what that's supposed to mean. This idea and others yield an idiom for building models that integrate logical patterns and obey logical constraints. 

In a different vein, we expect that logical inductors can already serve as a drop-in replacement for formal models of reasoners that assume logical omniscience and/or perfect Bayesianism, such as in game theory, economics, or theoretical models of artificial reasoners.


The authors are particularly interested in tools that help AI scientists attain novel statistical guarantees in settings where robustness and reliability guarantees are currently difficult to come by. For example, consider the task of designing an AI system that reasons about the behavior of computer programs, or that reasons about its own beliefs and its own effects on the world. While practical algorithms for achieving these feats are sure to make use of heuristics and approximations, we believe scientists will have an easier time designing robust and reliable systems if they have some way to relate those approximations to theoretical algorithms that are known to behave well in principle \mbp{in the same way that Auto-Encoding Variational Bayes can be related to Bayesian inference \citep{kingma2013auto}}. Modern models of rational behavior are not up to this task: formal logic is inadequate when it comes to modeling self-reference, and probability theory is inadequate when it comes to modeling logical uncertainty. We see logical induction as a first step towards models of rational behavior that work in settings where agents must reason about themselves, while deductively limited.

When it comes to the field of meta-mathematics, we expect logical inductors to open new avenues of research on questions about what sorts of reasoning systems can achieve which forms of self-trust. The specific type of self-trust that logical inductors achieve (via, e.g., \Thm{st}) is a subtle subject, and worthy of a full paper in its own right. As such, we will not go into depth here.

\subsection{Analysis}\label{sec:analysis}

Mathematicians, scientists, and philosophers have taken many different approaches towards the problem of unifying logic with probability theory. (For a sample, refer to \Sec{relatedwork}.) In this subsection, we will speculate about what makes the logical induction framework tick, and why it is that logical inductors achieve a variety of desiderata. The authors currently believe that the following three points are some of the interesting takeaways from the logical induction framework:

\tinysection{Following Solomonoff and Gaifman} One key idea behind our framework is our paradigm of making predictions by combining advice from an ensemble of experts in order to assign probabilities to all possible logical claims. This merges the framework of \citet{Solomonoff:1964} with that of \citet{Gaifman:1964}, and it is perhaps remarkable that this can be made to work. Say we fix an enumeration of all prime sentences of first-order logic, and then hook \LIA (\Alg{li}) up to a theorem prover that enumerates theorems of $\PA$ (written using that enumeration). Then all \LIA ever ``sees'' (from the deductive process) is a sequence of sets like
\[
  \text{\{\#92305 or \#19666 is true; \#50105 and \#68386 are true; \#8517 is false\}.}
\]
From this and this alone, \LIA develops accurate beliefs about all possible arithmetical claims. \LIA does this in a manner that outpaces the underlying deductive process and satisfies the desiderata listed above. If instead we hook \LIA up to a \ZFC-prover, it  develops accurate beliefs about all possible set-theoretic claims. This is very reminiscent of Solomonoff's framework, where all the predictor sees is a sequence of 1s and 0s, and they start figuring out precisely which environment they're interacting with.

This is only one of many possible approaches to the problem of logical uncertainty. For example, Adams' probability logic \citeyearpar{adams1996primer} works in the other direction, using logical axioms to put constraints on an unknown probability distribution and then using deduction to infer properties of that distribution. Markov logic networks \citep{Richardson:2006} construct a belief network that contains a variable for every possible way of grounding out each logical formula, which makes them quite ill-suited to the problem of reasoning about the behavior of complex Turing machines.\footnote{Reasoning about the behavior of a Turing machine using a Markov logic network would require having one node in the graph for every intermediate state of the Turing machine for every input, so doing inference using that graph is not much easier than simply running the Turing machine. Thus, Markov logic networks are ill-suited for answering questions about how a reasoner should predict the behavior of computations that they cannot run.} In fact, there is no consensus about what form an algorithm for ``good reasoning'' under logical uncertainty should take. Empiricists such as \citet{Hintikka:1962:knowledge,Fagin:1995:knowledge} speak of a set of modal operators that help differentiate between different types of knowledge; AI scientists such as \citet{russell1991principles,Hay:2012:Selecting,lin2015metareasoning} speak of algorithms that are reasoning about complicated facts while also making decisions about what to reason about next; mathematicians such as \citep{briol2015probabilistic,briol2015frank,hennig2015probabilistic} speak of numerical algorithms that give probabilistic answers to particular questions where precise answers are difficult to generate.

Our approach achieves some success by building an approximately-coherent distribution over all logical claims. Of course, logical induction does not solve all the problems of reasoning under deductive limitation---far from it! They do not engage in meta-cognition \mbp{in the sense of \citet{russell1991principles}} to decide which facts to reason about next, and they do not give an immediate practical tool \mbp{as in the case of probabilistic integration \citep{briol2015probabilistic}}, and they have abysmal runtime and uncomputable convergence bounds. It is our hope that the methods logical inductors use to aggregate expert advice will eventually yield algorithms that are useful for various applications, in the same way that useful ensemble methods can be derived from Solomonoff's theory of inductive inference.

\tinysection{Keep the experts small} One of the key differences between our framework and Solomonoff-inspired ensemble methods is that our ``experts'' are not themselves predicting the world. In standard ensemble methods, the prediction algorithm weighs advice from a number of experts, where the experts themselves are also making predictions. The ``master algorithm'' rewards the experts for accuracy and penalizes them for complexity, and uses a weighted mixture of the experts to make their own prediction. In our framework, the master algorithm is still making predictions (about logical facts), but the experts themselves are not necessarily predictors. Instead, the experts are ``traders'', who get to see the current model (constructed by aggregating information from a broad class of traders) and attempt to exploit inefficiencies in that aggregate model. This allows traders to identify (and eliminate) inconsistencies in the model even if they don't know what's actually happening in the world. For example, if a trader sees that $\pt(\phi) + \pt(\lnot \phi) \ll 1$, they can buy shares of both $\phi$ and $\lnot \phi$ and make a profit, even if they have no idea whether $\phi$ is true or what $\phi$ is about. In other words, letting the experts buy and sell shares (instead of just making predictions), and letting them see the aggregate model, allows them to contribute knowledge to the model, even if they have no idea what's going on in the real world.

We can imagine each trader as contributing a small piece of logical knowledge to a model---each trader gets to say ``look, I don't know what you're trying to predict over there, but I do know that this piece of your model is inconsistent''. By aggregating all these pieces of knowledge, our algorithm builds a model that can satisfy many different complicated relationships, even if every individual expert is only tracking a single simple pattern.

\tinysection{Make the trading functions continuous} As stated above, our framework gets significant mileage from showing each trader the aggregate model created by input from all traders, and letting them profit from identifying inconsistencies in that model. Showing traders the current market prices is not trivial, because the market prices on day $\nn$ depend on which trades are made on day $\nn$, creating a circular dependency. Our framework breaks this cycle by requiring that the traders use continuous betting strategies, guaranteeing that stable beliefs can be found.

In fact, it's fairly easy to show that something like continuity is strictly necessary, if the market is to have accurate beliefs about itself. Consider again the paradoxical sentence $\chi := \quot{\enc{\pt}_{\enc{\nn}}(\enc{\chi}) < 0.5}$ which is true iff its price in $\MP$ is less than 50\textcent{} on day $\nn$. If, on day $\nn$, traders were allowed to buy when $\chi < 0.5$ and sell otherwise, then there is no equilibrium price. Continuity guarantees that the equilibrium price will always exist.

This guarantee protects logical inductors from the classic paradoxes of self-reference---as we have seen, it allows $\MP$ to develop accurate beliefs about its current beliefs, and to trust its future beliefs in most cases. We attribute the success of logical inductors in the face of paradox to the continuity conditions, and we suspect that it is a general-purpose method that deductively limited reasoners can use to avoid the classic paradoxes.

%
%
%
%
%
%
%
%
%
%
%
%
%
%
%
%
%
\subsection{Variations}\label{sec:variations}

One notable feature of the logical induction framework is its generality. The framework is not tied to a polynomial-time notion of efficiency, nor to any specific model of computation. All the framework requires is a method of enumerating possible patterns of logic (the ``traders'') on the one hand, and a method of enumerating provable sentences of logic (the ``deductive process'') on the other. Our algorithm then gives a method for aggregating those patterns into a combined model that respects the logical patterns that actually hold.

The framework would work just as well if we used the set of linear-time traders in place of the set of poly-time traders. Of course, the market built out of linear-time traders would not satisfy all the same desirable properties---but the \emph{method of induction}, which consists of aggregating knowledge from a collection of traders and letting them all see the combined model and attempt to exploit it, would remain unchanged.

There is also quite a bit of flexibility in the definition of a trader. Above, traders are defined to output continuous piecewise-rational functions of the market prices. We could restrict this definition (e.g., by having traders output continuous piecewise-linear functions of the market prices), or broaden it (by replacing piecewise-rational with a larger class), or change the encoding scheme entirely. For instance, we could have the traders output not functions but upper-hemicontinuous relations specifying which trades they are willing to purchase; or we could give them oracle access to the market prices and have them output trades (instead of trading strategies). Alternatively, we could refrain from giving traders access to the market prices altogether, and instead let them sample truth values for sentences according to that sentence's probability, and then consider markets that are almost surely not exploited by any of these traders.

In fact, our framework is not even specific to the domain of logic. Strictly speaking, all that is necessary is a set of atomic events that can be ``true'' or ``false'', a language for talking about Boolean combinations of those atoms, and a deductive process that asserts things about those atoms (such as $\quot{a \land \lnot b}$) over time. We have mainly explored the case where the atoms are prime sentences of first order logic, but the atoms could just as easily be bits in a webcam image, in which case the inductor would learn to predict patterns in the webcam feed. In fact, some atoms could be reserved for the webcam and others for prime sentences, yielding an inductor that does empirical and logical induction simultaneously.

For the sake of brevity, we leave the development of this idea to future works.

%
%
%
%
%
%
%
%
%
%
%
%
%
%
%
%
%
\subsection{Open Questions}\label{sec:openquestions}

With \Def{lic}, we have presented a simple criterion on deductively limited reasoners, such that any reasoner who meets the criterion satisfies a large number of desiderata, and any reasoner that fails to meet the criterion can have their beliefs exploited by an efficient trader. With $\seq\LIA$ we have shown that this criterion can be met in practice by computable reasoners.

The logical induction criterion bears a strong resemblance to the ``no Dutch book'' criteria used by \citet{Ramsey:1931,DeFinetti:1937:foresight,teller1973conditionalization,lewis1999papers} to support Bayesian probability theory. This fact, and the fact that a wide variety of desirable properties follow directly from a single simple criterion, imply that logical induction captures a portion of what it means to do good reasoning under deductive limitations. That said, logical induction leaves a number of problems wide open. Here we discuss four, recalling desiderata from \Sec{desiderata}:

\renewcommand{\rparenthetical}[1]{}
\desdes*

\noindent In the case of logical inductors, we can interpret this desideratum as saying that it should be possible to tell a logical inductor to reason about one sentence in particular, and have it efficiently allocate resources towards that task. For example, we might be curious about Goldbach's conjecture, and wish to tell a logical inductor to develop its beliefs about that particular question, i.e.\ by devoting its computing resources in particular to sentences that relate to Goldbach's conjecture (such as sentences that might imply or falsify it).

Our algorithm for logical induction does not do anything of this sort, and there is no obvious mechanism for steering its deliberations. In the terminology of \citet{Hay:2012:Selecting}, $\seq\LIA$ does not do metalevel reasoning, i.e., it does nothing akin to ``thinking about  what to think about''. That said, it is plausible that logical induction could play a role in models of bounded decision-making agents. For example, when designing an artificial intelligence (AI) algorithm that \emph{does} try to reason about Goldbach's conjecture, it would be quite useful for that algorithm to have access to a logical inductor that tells it which other mathematical facts are likely related (and how). We can imagine a resource-constrained algorithm directing computing resources while consulting a partially-trained logical inductor, occasionally deciding that the best use of resources is to train the logical inductor further. At the moment, these ideas are purely speculative; significant work remains to be done to see how logical induction bears on the problem of allocation of scarce computing resources when reasoning about mathematical facts.

\descounterp*

\noindent In the year 1993, if you asked a mathematician about what we would know about mathematics if Fermat's last theorem was false, they would talk about how that would imply the existence of non-modular elliptic curves. In the year 1994, Fermat's last theorem was proven true, so by the principle of explosion, we now know that if Fermat's last theorem were false, then 1=2 and $\sqrt{2}$ is rational, because from a contradiction, anything follows. The first sort of answer seems more reasonable, and indeed, reasoning about counterpossibilities (i.e., proving a conjecture false by thinking about what would follow if it were true) is a practice that mathematicians engage in regularly. A satisfactory treatment of counterpossibilities has proven elusive; see \citep{Cohen:1990,vander2004counterpossibles,brogaard2007counterpossibles,krakauer2012counterpossibles,bjerring2014counterpossibles} for some discussion and ideas. One might hope that a good treatment of logical uncertainty would naturally result in a good treatment of counterpossibilities.

There are intuitive reasons to expect that a logical inductor has reasonable beliefs about counterpossibilities. In the days before $\DP$ has (propositionally) ruled out worlds inconsistent with Fermat's last theorem, $\MP$ has to have beliefs that allow for Fermat's last theorem to be false, and if the proof is a long time in coming, those beliefs are likely reasonable. However, we do not currently have any guarantees of this form---$\pt_\infty$ still assigns probability~0 to Fermat's last theorem being false, and so the conditional probabilities are not guaranteed to be reasonable, so we haven't yet found anything satisfactory to say with confidence about $\MP$'s counterpossible beliefs.

While the discussion of counterpossibilities may seem mainly academic, \citet{Soares:2015:toward} have argued that counterpossibilities are central to the problem of designing robust decision-making algorithms. Imagine a deterministic agent \texttt{agent} evaluating three different ``possible scenarios'' corresponding to three different actions the agent could take. Intuitively, we want the $n$th scenario (modeled inside the agent) to represent what would happen if the agent took the $n$th action, and this requires reasoning about what would happen if \texttt{agent(observation)} had the output \texttt{a} vs \texttt{b} vs \texttt{c}. Thus, a better understanding of counterpossible reasoning could yield better decision algorithms. Significant work remains to be done to understand and improve the way that logical inductors answer counterpossible questions.

\desoldev*

\noindent The canonical example of the problem of old evidence is Einstein's development of the theory of general relativity and its retrodiction of the precession in Mercury's orbit. For hundreds of years before Einstein, astronomers knew that Newton's equations failed to model this precession, and Einstein's retrodiction counted as a large boost for his theory. This runs contrary to Bayes' theorem, which says that a reasoner should wring every drip of information out of every observation the moment that the evidence appears. A Bayesian reasoner keeps tabs on all possible hypotheses at all times, and so they never find a new hypothesis in a burst of insight, and reward it for retrodictions. Humans work differently---scientists spent centuries without having even one good theory for the precession of Mercury, and the difficult scientific labor of Einstein went into \emph{inventing the theory}.

There is a weak sense in which logical inductors solve the problem of old evidence---as time goes on, they get better and better at recognizing patterns in the data that they have already seen, and integrating those old patterns into their new models. That said, a strong solution to the problem of old evidence isn't just about finding new ways to use old data every so often; it's about giving a satisfactory account of how to algorithmically \emph{generate new scientific theories}. In that domain, logical induction has much less to say: they ``invent'' their ``theories'' by sheer brute force, iterating over all possible polynomial-time methods for detecting patterns in data.

There is some hope that logical inductors will shed light on the question of how to build accurate models of the world in practice, just as ensemble methods yield models that are better than any individual expert in practice. However, the task of using logical inductors to build practical models in some limited domain is wide open.

\desefficient*

\noindent Logical inductors are far from efficient, but they do raise an interesting empirical question.
While the theoretically ideal ensemble method \mbp{the universal semimeasure \citep{Li:1993}} is uncomputable, practical ensemble methods often make very good predictions about their environments. It is therefore plausible that practical logical induction-inspired approaches could manage logical uncertainty well in practice. Imagine we pick some limited domain of reasoning, and a collection of constant- and linear-time traders. Imagine we use standard approximation methods (such as gradient descent) to find approximately-stable market prices that aggregate knowledge from those traders. Given sufficient insight and tweaking, would the resulting algorithm be good at learning to respect logical patterns in practice?  This is an empirical question, and it remains to be tested.

\subsection{Acknowledgements}

We acknowledge Abram Demski, Alex Appel, Benya Fallenstein, Daniel Filan, Eliezer Yudkowsky, Jan Leike, J\'anos Kram\'ar, Nisan Stiennon, Patrick LaVictoire, Paul Christiano, Sam Eisenstat, Scott Aaronson, and Vadim Kosoy, for valuable comments and discussions. We also acknowledge contributions from attendees of the MIRI summer fellows program, the MIRIxDiscord group, the MIRIxLA group, and the MIRI$\chi$ group.

This research was supported as part of the Future of Life Institute (futureoflife.org) FLI-RFP-AI1 program, grant~\#2015-144576.

\printbibliography[heading=bibintoc,title={References}]
\newpage
\appendix
\renewcommand{\MainTextOnly}[1]{}

\section{Preliminaries}\label{app:preliminaries}

\subsection{Organization of the Appendix}

The appendix is organized differently from the paper. Here we describe the broad dependency structure of the proofs and mention the theorems that are proven by constructing explicit traders (rather than as corollaries). 
Note that theorems that were proven in Section 6 are also proven here, but differently (generally much more concisely, as a corollary of some other theorem). 

\textbf{\ref{app:preliminaries}. Preliminaries.} \App{expressiblefeatures} describes expressible features in full detail. \App{preliminaries}.3  defines some notions for combinations, and defines when a sequence of traders can be ``efficiently emulated'', which will be useful in \ref{app:convergenceproofs}, \ref{app:recunbiasedaff}, and \ref{app:nondogproofs}.

\textbf{\ref{app:convergenceproofs}. Convergence.}  \App{roi} introduces a tool for constructing traders (\Lem{type3},  Return on Investment) that is used in  \ref{app:convergenceproofs} and \ref{app:recunbiasedaff}. Appendices \ref{app:affpolymax} (Affine Preemptive Learning)   and \ref{app:peraffkno} (Persistence of Affine Knowledge) prove those  theorems using \Lem{type3}, and the remainder of \ref{app:convergenceproofs} derives some corollaries (convergence and non-affine special cases). 

\textbf{\ref{app:coherenceproofs}. Coherence.} \App{affcoh} proves Affine Coherence, giving (Affine) Provability Induction as corollaries. The remainder of \ref{app:coherenceproofs} derives corollaries of Provability Induction (consistency and halting) and of Affine Provability Induction (coherence and exclusive-exhaustive relationships).

\textbf{\ref{app:statisticalproofs}. Statistics.} \App{recunbiasedaff} proves Affine Recurring Unbiasedness using \Lem{type3}, giving Simple Calibration (\ref{app:simcal}) as a corollary. Appendices \ref{app:wubaff}  (Affine Unbiasedness From Feedback) and \ref{app:prandaff} (Learning Pseudorandom Affine Sequences) prove those theorems by constructing traders, and the remainder of \App{statisticalproofs} derives corollaries (varied and non-affine cases). 

\textbf{\ref{app:expectationsproofs}. Expectations.} \App{mesh} proves the Mesh Independence Lemma by constructing a trader, and \ref{app:conluvapprox} and \ref{app:limexpapprox} prove two other lemmas on expectations; basic properties  of expectations such as convergence and linearity are  also proved. These proofs rely on theorems proven in \ref{app:convergenceproofs} and \ref{app:coherenceproofs}. The remainder of \ref{app:expectationsproofs}  proves analogs for expectations of the convergence, coherence, and statistical theorems   by applying their affine versions to $\features$-combinations expressing expectations. 

\textbf{\ref{app:introspectiontrust}. Introspection and Self-Trust.} The first part of \App{introspectiontrust} proves introspection properties using Affine Provability Induction and Expectation Provability Induction. The remainder derives the self-trust properties as applications of theorems proven in \App{expectationsproofs}.

\textbf{\ref{app:nondogproofs}. Non-Dogmatism and Closure.} \App{nondogproofs} is mostly self-contained. \App{parametric} proves a simple analog of the return on investment lemma with stronger hypotheses; this is applied to constructing traders in \ref{app:obu} (Uniform Non-Dogmatism), \ref{app:ob} (Occam Bounds), and \ref{app:dus} (Domination of the Universal Semimeasure), with non-dogmatism and strict domination as corollaries. \App{scon} (Conditionals on Theories) uses uniform non-dogmatism, preemptive learning, and  \ref{app:ifp} (Closure under Finite Perturbations).

\subsection{Expressible Features}\label{app:expressiblefeatures}
\emph{This section can be safely skipped and referred back to as desired.}

Recall that a trading strategy for day $\nn$ is given by an affine combination of sentences with expressible feature coefficients.  As such, a machine that implements a trader must use some notation for writing down those features.  Here, to be fully rigorous, we will make an explicit choice of notation for expressible features.  Recall their definition:

\expressiblefeatures*

A (multi-line) string representing an expressible feature will be called a \emph{well-formed feature expression}, and will be built from smaller expressions  involving variables (mainly to save space when a particular expression would otherwise need to be repeated many times).

We define the set of \emph{variable feature expressions} $\Xi$ inductively to include:
\begin{itemize}
\item Past and present market prices: for all $i \leq n$ and for all $\psi\in\Sentences$, there is a symbol $\psi^{* i}\in\Xi$.
\item Rationals: $\QQ \subset \Xi$. 
\item Variables: $V \subset \Xi$. 
\end{itemize}
Further, if $\xi\in\Xi$ and $\zeta\in\Xi$, then the following operations on them are as well:
\begin{itemize}
\item 
Addition: $\xi + \zeta\in\Xi$.
\item 
Multiplication: $\xi \cdot \zeta\in\Xi$.
\item 
Maximum: $\max(\xi, \zeta)\in\Xi$.
\item
Safe reciprocation: $1/\max(1,\xi)\in\Xi$.
\end{itemize}
These operations are sufficient to generate all the expressible features we will need. For example, 
\begin{align*}
-\xi &:= (-1)\cdot \xi;\\
\min(\xi,\zeta) &:= -\max(-\xi,-\zeta);\\
|\xi|&:= \max(\xi,-\xi);\\
\intertext{and when $\zeta \geq \varepsilon$ for some constant $\varepsilon > 0$, we can define}
\xi/\zeta &:= (1/\varepsilon) \cdot \xi/\max(1,  (1/\varepsilon) \cdot \zeta ).
\end{align*}

\noindent We now define a \emph{well-formed feature expression} to be a (multi-line) string of the following form:
\begin{align*}
&v_1 := (\text{feature expression with no variables});\\
&v_2 := (\text{feature expression involving $v_1$});\\
&\cdots\\
&v_k := (\text{feature expression involving $v_1,\ldots,v_{k-1}$});\\
&\returntrade (\text{feature expression involving $v_1,\ldots,v_k$}),
\end{align*}
\noindent where the final expression after ``$ \returntrade$'' is the expression evaluated to actually compute the expressible feature defined by this code block.
\subsubsection{Examples}
The following well-formed feature expression defines a rank 7 expressible feature:
\begin{align*}
&v_1 := \phi_1^{* 7} + \phi_2^{* 4}\\
&v_2 := v_1 - 1\\
&\returntrade 3\cdot \max(v_1,v_2).
\end{align*}
If the market at time 7 has $\pt_7(\phi_1)= 0.8$ and the market at time $4$ had $\pt_4(\phi_2)=0$, then this expressible feature evaluates to 
\[3\cdot \max(v_1,v_2) = 3\cdot \max(0.8,-0.2) = 2.4.\]
An $\nn$-strategy can now be written down in a very similar format, sharing variable definitions used in the various coefficients to save space.  For example, the following code defines a $7$-strategy:
\begin{align*}
&v_1 := \phi_1^{* 7} + \phi_2^{* 4}\\
&v_2 := v_1 \cdot v_1\\
&\strat[\phi_1] := 3\cdot \max(v_1,v_2)\\
&\strat[\phi_2] := 6 \cdot \max(v_1,v_2).\\
&\strat := \sum_{i=1}^2 T[\phi_i]\cdot(\phi_i-\phi_i^{*\nn})\\
&\returntrade T
\end{align*}

Notice that the function $\phi_1^{*7}$ returning the current market price of $\phi_1$ affects (via $v_1$) how many shares of $\phi_1$ this trader buys.  This is permitted, and indeed is crucial for allowing traders to base their trades on the current market prices. 

\subsubsection{Dynamic programming for traders}\label{app:dynamicprogramming}

We will often define traders that make use of indexed variables that are defined recursively in terms of previous indices, as in e.g.\ the proof of \Theorem{con} in Section~\ref{sec:convergence}. In particular, we often have traders refer to their own past trades, e.g.\ using  expressible features of the form  $\trade_i[\phi]$ for $i<\nn$ to define their trade at time $\nn$. This can be written down in polynomial time using the expression language for features, via dynamic programming. For example, to use previous trades, a trader can recapitulate all the variables used in all its previous trading strategies. As long as the trading strategies are efficiently computable given previous trades as variables, they are still efficiently computable without them (possibly with a higher-degree polynomial).

%
%
%
%
%
%
%
%
%
%
%
%
%
%
%
%
%

\subsection{Definitions}

\subsubsection{Price of a Combination}

\begin{definition}[Price of a Combination]
\noindent Given any affine combination 
  \[
    \aff =
    \affconst
      + \exf_1\phi_1
      + \cdots
      + \exf_k\phi_k
  \]
of rank $\leq n$, observe that the map $\seq\Valuation \mapsto \Valuation_\nn(\aff)$ is an expressible feature, called the \textbf{price} of $\aff$ on day $\nn$, and is given by the expression 
  \[
    \pf{\aff} :=
    \affconst
      + \exf_1\pf{\phi_1}
      + \cdots
      + \exf_k\pf{\phi_k}.
  \]
\noindent For any valuation sequence $\seq\UU$, observe by linearity and associativity that
\[(\Valuation(\aff))(\seq \UU) = \Valuation(\aff(\seq \UU)) = \affconst(\seq\UU) + \sum_\phi \exf_\phi(\seq\UU)\Valuation(\phi) . \]
\end{definition}

\subsubsection{Buying a Combination}

\begin{definition}[Buying a Combination]  Given any $\exfeatures$-combination  $\aff^\gens$ of $\rank \leq \nn$, we define a corresponding $n$-strategy called \textbf{buying~$\bm{\aff^\gens}$ on day $\bm \nn$} to equal
  \[
     \aff^\gens - \aff^{\gens * n}.
  \]
Observe that buying $\aff$ on day $\nn$ is indeed an $\nn$-strategy.  
\end{definition}

\subsubsection{$\features$-Combinations Corresponding to $\features$-LUV Combinations}\label{app:featureluvcombos}

\begin{definition}[$\alta$]
  Let $\affluv := \affconst + \exf_1 X_1 + \cdots + \exf_k X_k$ be an $\features$-LUV combination.
  Define 
  \[ \alta_m(\aff) = \affconst + \exf_1 \sum_{i=0}^{m-1} \frac{1}{m} (\quot{\enc{X_1} > \enc{i}/\enc{m}}) + \cdots + \exf_k \sum_{i=0}^{m-1} \frac{1}{m} (\quot{\enc{X_k} > \enc{i}/\enc{m}}) \]
  to be a $\features$-affine combination corresponding to $\affluv$.  Note that $\Valuation(\alta_m(\affluv)) = \EE_m^\Valuation(\affluv)$. Also note that if $(\affluv_\nn)_n$ is bounded, then  $(\alta_n(\affluv_\nn))_n$ is bounded; we will use this fact freely in what follows.
\end{definition}

\subsubsection{Efficiently Emulatable Sequence of Traders}\label{app:emulatable}

In Appendices \ref{app:convergenceproofs}, \ref{app:recunbiasedaff}, and \ref{app:nondogproofs},
we will construct traders that allocate their money across multiple strategies for exploiting the market. 
In order to speak unambiguously about multiple overlapping long-term strategies for making trades, 
we define the notion of a sequence of traders that can be efficiently emulated by one trader.
\begin{definition}[Efficiently Emulatable Sequence of Traders]\label{def:emulatabletraders}
  We say that a sequence of traders  $( \Trader^k)_k$ is \emph{efficiently emulatable} if
\begin{itemize}
\item the sequence of programs that compute the $\Trader^k$ can be efficiently generated; 
\item those programs for $\Trader^k$ have uniformly bounded runtime, i.e., 
  there exists a constant $c$ such that for all $k$ and all times $n$, the program that computes $\Trader^k$ runs in time $\Oo(n^c)$; and
\item for all $k$ and all $n<k$, we have that $\Trader^k_n$ is the zero trade.
\end{itemize}
\end{definition}
\noindent Efficiently emulatable sequences are so named because a single trader $\Trader$ can emulate the entire sequence of traders 
$(\Trader^k)_k$. That is, on time $n$, $\Trader$ can directly compute all the trading strategies $\trade^k_n$ for $k\le n$ by listing the appropriate programs
and running them on input $n$. This can be done in polynomial time by definition of an efficiently emulatable sequence. 
We require that $\Trader^k$ does not make non-zero trades before time $k$ so that the emulator $\Trader$ need not truncate any trades made by the $\Trader^k$.

\section{Convergence Proofs}\label{app:convergenceproofs}

\subsection{Return on Investment}\label{app:roi}

This section provides a useful tool for constructing traders, which will be applied in \App{convergenceproofs} and \App{recunbiasedaff}. The reader may wish to first begin with the proof in Appendix~\ref{app:affpolymax} of Theorem~\ref{thm:affpolymax} as motivation of the return on investment lemma.

\paragraph{Statement of the $\bm{\varepsilon}$-ROI lemma.}

If we have a \li{} $\MP$, we know that $\MP$ cannot be exploited by any trader. 
It will often be easy to show that if $\MP$ fails to satisfy some property, then there is a trader $\Trader$ that takes advantage of a specific, one-shot
opportunity to trade against the market in a way that is guaranteed to eventually be significantly higher value than the size of the original investment; 
and that such opportunities arise infinitely often.
In order to use such a situation to ensure that the market $\MP$ satisfies the property, we will now show that \li{}s are not susceptible to 
repeatable methods for making a guaranteed, substantial profit.

To define a notion of return on investment, we first define the ``magnitude'' of a trade made by a trader, so that we can talk about traders that are
profitable in proportion to the size of their trades:
\[ \magnit{\trade(\MP)} := \sum_{\phi\in\Sentences} |\trade[\phi](\MP)|. \]
This number will be called the \textbf{magnitude} of the trade.  It
is just the total number of shares traded by $\trade$ against the market $\MP$, whether the shares are
bought or sold. Note that the magnitude is \emph{not} the same as the $\magn{\any}$-norm of $\trade(\MP)$; the magnitude omits the constant term $\trade[1](\MP)$.  

The magnitude is a simple bound on the value of the holdings $\trade_n(\MP)$: for any world $\World$ (plausible or not),
\[
\left| \sum_{\phi\in\Sentences} \trade_n[\phi](\MP) \cdot \left(\World(\phi) - \pt_n(\phi)\right)\right|
\le \sum_{\phi\in\Sentences} \left| \trade_n[\phi](\MP)\right| \cdot 1  = \magnit{\trade_n(\MP)},  
\]
since $\World(\phi)\in\{0,1\}$ and $\pt_n(\phi)\in [0,1]$. 
Now we define the total magnitude of a trader over time.
\begin{definition}[Magnitude of a Trader]\label{def:tradermag}
  The \textbf{magnitude} $\magnit{\Trader(\MP)}$ of a trader $\Trader$ against the market $\MP$ is
\[ \magnit{\Trader(\MP)} := \sum_{n\in {\NN^+}}\magnit{\trade_n(\MP)} \equiv \sum_{n\in {\NN^+}} \sum_{\phi\in\Sentences}|\trade_n[\phi](\MP)|. \]
\end{definition}
The magnitude of $\Trader$ is the total number of shares it trades (buys or sells) over all time.



Now we define what it means for a trader to increase its net value by a substantial fraction of its investment, i.e., its magnitude. 
\begin{restatable}[$\varepsilon$ Return on Investment]{definition}{roi} 
\label{def:roi}
For $\varepsilon>0$, we say that a trader $\Trader$ trading against $\MP$ has \emph{$\varepsilon$ return on investment} or \emph{$\varepsilon$-ROI} if, for all $\World \in \cworlds(\Theory)$,
\[ \lim_{n \to \infty} \World \left(\sum_{i \leq n} T_i \ftn \right) \geq \varepsilon \magnit{\Trader(\MP)}. \]
\end{restatable}
In words, a trader $\Trader$ has $\varepsilon$-ROI if, in the limit of time and deduction, the value of its holdings is, in every world,  at least $\varepsilon$
times its total  investment $\magnit{\Trader(\MP)}$. Note that this does not merely say that $\Trader$ recoups at least an $\varepsilon$ fraction of its original
cost; rather, the net value is guaranteed in all worlds consistent with $\Theory$
to have increased by an $\varepsilon$ fraction of the magnitude  $\magnit{\Trader(\MP)}$ of $\Trader$'s trades.  



Recall from Definition~\ref{def:ece} that a sequence $\seq\alpha$ of rationals is \pgenable if there is some \ec $\exfeatures$-progression $\seq{\gen \alpha}$ such that $\gen{\alpha_\nn} \ftn = \alpha_\nn$ for all~$\nn$.

\begin{restatable}[No Repeatable $\varepsilon$-ROI ]{lemma}{roilemma} 
\label{lem:type3}
Let $\MP$ be a \li{} with respect to some deductive process $\DP$, and let $(\Trader^k)_{k\in\NN^+}$ be an efficiently emulatable sequence of traders (Definition~\ref{def:emulatabletraders}).
  Suppose that for some fixed $\varepsilon>0$, each trader $\Trader^k$ has $\varepsilon$-ROI\@. Suppose further that there is some
  \pgenable sequence $\seq\alpha$ 
  such that for all $k$,
  \[ \magnit{\Trader^k \ftn } = \alpha_k. \]
  Then \[ \lim_{k \rightarrow \infty} \alpha_k = 0. \]
\end{restatable}
In words, this says roughly that there is no efficient, repeatable method for producing a substantial guaranteed return on an investment. 
The condition that $\seq\alpha$ is \pgenable will help with the budgeting done by the trader that emulates 
the sequence $(\Trader^k)_k$.


\paragraph{Proof strategy.}
We will construct a trader $\Trader$ that emulates the sequence $(\Trader^k)$ in a manner such that if the traders $\Trader^k$ did not make trades of
vanishing limiting value, then our trader $\Trader$ would accrue unbounded profit by repeatedly making investments that are guaranteed to pay out by a substantial amount. 
Very roughly, on time $n$, $\Trader$ will sum together the trades $\trade^k_n$ made by all the $\Trader^k$ with $k\le n$. In this way, $\Trader$ will accrue all the profits made by each of the $\Trader^k$.

The main problem we have to deal with is that $\Trader$ risks going deeper and deeper into debt to finance its investments, as discussed before the proof of \Theorem{con} in Section~\ref{sec:convergence}. That is, it may be that each
$\Trader^k$ makes an investment that takes a very long time for all the worlds
$\World\in\pcworlds(\dt_\nn)$ plausible at time $n$ to value highly. In the meanwhile, $\Trader$ continues to spend money buying shares and taking on risk 
from selling shares that might plausibly demand a payout. 
In this way, despite the fact that each of its investments will eventually become profitable, $\Trader$ may have holdings with unboundedly negative plausible
value. 

To remedy this, we will have  our trader $\Trader$ keep track of its 
``debt'' and of which investments have already paid off, and then scale down new traders $\Trader^k$ so that $\Trader$ maintains a lower bound on the plausible value of
its holdings.  Roughly speaking, $\Trader$ at time $n$ checks whether the current holdings of $\Trader^k$ are guaranteed to have a positive value in all
plausible worlds, for each $k\le n$. Then $\Trader$ sums up the total magnitudes $\alpha_k$ of all the trades ever made by those $\Trader^k$ whose trades are not yet guaranteed to be profitable. This sum is used to scale down all trades made by $\Trader^n$, so that the total magnitude of the unsettled investments made by $\Trader$ will remain bounded.

\paragraph{Proof of \Lem{type3}.}

\newcommand{\opentrades}{\textrm{open}}
\begin{proof}
  We now prove \Lem{type3}. 
  
  We can assume without loss of generality that each $\alpha_n \leq 1$ by dividing $\Trader^k$'s trades by $\max(1, \alpha_k)$.

\paragraph{Checking profitability of investments.}

At time $n$, our trader $\Trader$ runs a (possibly very slow) search process to enumerate traders $\Trader^k$ 
from the sequence $(\Trader^k)_k$ that have made trades that are already guaranteed to be profitable, as judged by what is plausible according to the deductive process $\DP$ with respect to which $\MP$ is a \li{}. 
That is, $\Trader$ runs a search for 
pairs of numbers  $k, m\in\NN^+$ such that:
\begin{align*}
\sum_{i \leq m} \magnit{\trade^k_i(\MP)} &\geq (1 - \varepsilon/3) \alpha_k  &\textrm{ (few future trades), and } \\
\inf_{\World\in\pcworlds(\dt_\mm)}\World \left( {\textstyle \sum_{i \leq m}  \trade^k_i(\MP)
 }\right) &\geq (2\varepsilon /3) \alpha_k & \textrm{ (guaranteed profit). }
\end{align*}
If the trader $\Trader^k$ has few future
trades and guaranteed profit at time $m$ then  we say that the trader's holdings have matured. We denote the least such $m$ by $m(k)$. 

The first condition (few future trades)  says that $\Trader^k$ has made trades of total magnitude at least $(1 - \varepsilon/3) \alpha_k$ after time $k$ up until time $m$. By the assumption that $ \magnit{\Trader^k \ftn}= \alpha_k$, for each $k$ there is some time step $m$ such that this condition holds. By that same assumption, $\Trader^k$ will make trades of total magnitude at most $(\varepsilon/3)\alpha_k$ in all time steps after $m$. 

The second condition (guaranteed profit) says that the minimum value plausible at time $m$ of all trades made by $\Trader^k$ up until time $m$ is at
least $(2\varepsilon /3) \alpha_k$. By the assumption that $\Trader^k$ has $\varepsilon$-ROI, i.e., that the minimum 
value of $\Trader^k$ is eventually at least $\varepsilon \magnit{\trade^k_k}$, the condition of guaranteed profit will hold at some $m$. 

The idea is that, since $\Trader^k$ will  trade at most $(\varepsilon/3) \alpha_k$ shares after
$m$ and the holdings of $\Trader^k$ from trades up until the current time have minimum plausible value at least $(2\varepsilon /3) \alpha_k$,
it is guaranteed that  the holdings of 
$\Trader^k$ at any time after $m$ will have minimum plausible value at least $(\varepsilon/3) \alpha_k$. This will allow our trader $\Trader$ to ``free up''
funds allocated to emulating $\Trader^k$, in the sense that there is no longer ever any plausible net downside to the holdings from trades made by $\Trader^k$.

\paragraph{Definition of the trader $\bm{\Trader}$.}

For fixed $k$ and $m$, these two conditions
refer to specific individual computations (namely  $\dt_m$,  the $\trade^k_i(\MP)$ for $i \leq m$, and $\alpha_k$).
On time step $n$, for all $k,j \leq n$,
our trader  $\Trader$ sets Boolean variables $\opentrades(k,j) :=0$ if it is verified in $j$ steps of computation that  the holdings of $\Trader^k$ have matured; and $\opentrades(k,j) :=1$ if $\Trader^k$ has open investments. 
Since the holdings of each $\Trader^k$ will eventually
mature, for all $k$ there is some $n$ such that $\opentrades(k,n) = 0$. 

Let $\seq{\gen\alpha}$ be an \ec $\exfeatures$ progression such that for each $n$ we have $\gen{\alpha_n} \ftn = \alpha_n$.
Then $\Trader$ outputs the trading strategy 
\[ \trade_n := \sum_{k\leq n} \gen{\beta_k} \cdot \trade^k_n,  \]
where the $\gen{\beta_k}$ are defined recursively by 
\[ \gen{\beta_k} := 1- \sum_{i < k} \opentrades(i,k) \gen{\beta_i} \gen{\alpha_i} .\]
That is, the machine computing $\Trader$ outputs the definitions of the budget variables
$\gen{\beta_k}$  
for each $k \leq n$, and then lists the trades
\[ \returntrade \phi := \sum_{k\leq n} \gen{\beta_k} \trade^k_n[\phi]\]
for each $\phi$ listed by any of the trades $\trade^k_n$ for $k \leq n$. 
As shorthand, we write $\beta_k := \gen{\beta_k} \ftn$.  
Notice that since $(\Trader^k)_k$ is efficiently emulatable, we have $\forall k: \forall i<k: \trade^k_i \equiv 0$, and therefore 
\[ \forall n: \trade_n \ftn = \sum_{k\in {\NN^+}} \beta_k \trade^k_n \ftn.\]
Note that each $\opentrades(i,k)$ is pre-computed 
by the machine that outputs our trader $\Trader$ and then is encoded as a constant in the expressible feature $\gen{\beta_k}$. 
The trade coordinate $\trade_n[\phi]$ is an expressible feature because the $\gen{\beta_k}$ and $\trade^k_n[\phi]$ are expressible features.

\paragraph{Budgeting the traders $\bm{\Trader^k}$.}

Since we assumed each $\alpha_k \leq 1$, it follows from the definition of the budget variable $\gen{\beta_k}$ that 
\[ \beta_k \alpha_k  \leq  1- \sum_{i < k} \opentrades(i,k) \beta_i \alpha_i,\]
and then $\beta_k$ is used as the constant scaling factor for $\Trader^k$ in the sum defining $\Trader$'s trades. 
In this way, we maintain the invariant that for any $n$, 
\[\sum_{k\leq n} \opentrades(k,n) \beta_k \alpha_k \leq 1.\]
Indeed, by induction on $k$, using the fact that $\opentrades(i,n)$ implies $\opentrades(i,m)$ for $m\ge n$, we have $\beta_k \geq 0$ and the above invariant holds.

In words, this says that out of all the traders $\Trader^k$ with investments still open at time $n$, the sum of the magnitudes  $\beta_k \alpha_k$ of
their total investments (as budgeted by the $\beta_k$) is bounded by 1. 

\paragraph{Analyzing the value of $\bm{\Trader}$'s holdings.}

Now we lower bound the value of the holdings of $\Trader$ from trades against the market $\MP$. Fix any time step $n$ and world $\World\in\pcworlds(\dt_\nn)$ plausible
at time $n$. Then we have that the value of $\Trader$'s holdings at time $n$ is
\begin{align*}
  \World( {\textstyle \sum_{i \leq n} \trade_i(\MP) }) 
&= \sum_{k\leq n}   \World( {\textstyle \sum_{i \leq n} \beta_k \trade^k_i(\MP) })\\
\shortintertext{ by linearity and by definition of our trader $\Trader$;}
&= \sum_{\substack{k\leq n\\  \opentrades(k,n)}}   \World( {\textstyle \sum_{i \leq n} \beta_k \trade^k_i(\MP) }) 
+  \sum_{\substack{k\leq n\\  \lnot \opentrades(k,n)}}   \World( {\textstyle \sum_{i \leq n} \beta_k \trade^k_i(\MP) }) 
\end{align*}
again by linearity. We analyze the first term, the value of the holdings that have not yet matured, as follows:
\begin{align*}
 \sum_{\substack{k\leq n\\  \opentrades(k,n)}}   \World( {\textstyle \sum_{i \leq n} \beta_k \trade^k_i(\MP) }) 
 &\geq -\sum_{\substack{k\leq n\\  \opentrades(k,n)}}   \beta_k \sum_{i \leq n} \magnit{\trade^k_i \ftn}
 \\
 &\geq -\sum_{\substack{k\leq n\\  \opentrades(k,n)}}   \beta_k \sum_{i \in \NN^+} \magnit{\trade^k_i \ftn}
 \\
 &= -\sum_{k\leq n  }  \opentrades(k,n) \beta_k \alpha_k \\ 
 &\geq -1 ,
\end{align*}
by the previous discussion of the $\beta_k$. In short, the $\beta_k$ were chosen so that the total magnitude of all of $\Trader$'s holdings  from trades made
by any $\Trader^k$ that haven't yet matured stays at most 1, so that its plausible value stays at least $-1$. 

Now we analyze the second term in the value of $\Trader$'s holdings, representing the value of the holdings that have already matured, as follows:
\begin{align*}
  & \sum_{\substack{k\leq n\\  \lnot\opentrades(k,n)}}   \World( {\textstyle \sum_{i \leq n} \beta_k \trade^k_i(\MP) }) 
  \\
 =& \sum_{\substack{k\leq n\\  \lnot\opentrades(k,n)}}   
  \left( \World( {\textstyle \sum_{i \leq m(k)} \beta_k \trade^k_i(\MP) }) +
  \World( {\textstyle \sum_{m(k) < i \leq n} \beta_k \trade^k_i(\MP) }) \right)
\shortintertext{where $m(k)$ is minimal such that $\Trader^k$ has guaranteed profit and makes few future trades at time $m(k)$, as defined
above;} 
&\geq
  \sum_{\substack{k\leq n\\ \lnot\opentrades(k,n)}} 
  \left(\beta_k (2\varepsilon/3) \alpha_k - \sum_{ i>  m(k) } \beta_k \magnit{ \trade^k_i \ftn } \right)\\
\shortintertext{since  by definition of $m(k)$ and the guaranteed profit condition, the value of the holdings of $\Trader^k$ from its trades 
  up until time $m(k)$ is at least 
$(2\varepsilon/3)\alpha_k$ in any world in $\dt_\nn$;}
&\geq
  \sum_{\substack{k\leq n\\ \lnot\opentrades(k,n)}} 
\left(\beta_k (2\varepsilon/3) \alpha_k -\beta_k (\varepsilon/3) \alpha_k \right)\\
\shortintertext{since $\Trader^k$ is guaranteed to make trades of magnitude at most $(\varepsilon /3) \alpha_k$ after time $m(k)$;}
& = \sum_{\substack{k\leq n\\ \lnot\opentrades(k,n)}}  \beta_k (\varepsilon/3)\alpha_k.
\end{align*}

Completing our analysis, we have a 
lower bound on the value in $\World$ of the holdings of $\Trader$ at time $n$:
\[  \World( {\textstyle \sum_{i \leq n} \trade_i(\MP) })   \geq -1 +
\sum_{\substack{k\leq n\\ \lnot\opentrades(k,n)}}  \beta_k (\varepsilon/3) \alpha_k. \]

\paragraph{$\bm{\Trader}$ exploits $\bm{\MP}$ unless $\bm{\seq\alpha}$ vanishes.}

Since $\Trader$ is an efficient trader and $\MP$ is a \li{}, $\Trader$ does not exploit $\MP$. That is, the set 
\[    \left\{ \World\mleft({\textstyle \sum_{i \leq n} \trade_i} (\MP) \mright) \,\middle|\, \nn\in\NN^+, \World\in\pcworlds(\dt_\nn)
\right\} \]
is bounded above, since it is bounded below by $-1$ by the above analysis. In words, the plausible 
value of $\Trader$'s holdings is always at least $-1$, so by the \lic{} it
cannot go to infinity. Therefore, again by the above analysis,  we must have  
\[ \lim_{n \to\infty} \sum_{\substack{k\leq n\\ \lnot\opentrades(k,n)}}  \beta_k (\varepsilon/3) \alpha_k <\infty. \]
As shown above,
for any $k$ the conditions for $\lnot\opentrades(k,n)$ will eventually be met by all sufficiently large $n$. Thus 
\[ \lim_{n \to\infty} \sum_{\substack{k\leq n\\ \lnot\opentrades(k,n)}}  \beta_k (\varepsilon/3) \alpha_k = 
\sum_{k}  (\varepsilon/3) \beta_k \alpha_k <\infty.\]
Now we show that $\lim_{k \to\infty} \alpha_k = 0$. 
Suppose by way of contradiction that for some $\delta \in (0, 1)$, $\alpha_k > \delta$ for infinitely many $k$, but nevertheless for some sufficiently large time
step $n$, we have
\[ \sum_{i>n} \beta_i \alpha_i < 1/2. \]
Recall that  for each $i\leq n$, at some time $n(i)$, $\opentrades(i,n(i)) =0$ verifies that the holdings of $\Trader^i$ have matured.
 Let $N$ be any number greater than $n(i)$ for all $i \leq n$. Then 
\begin{align*}
\sum_{i < N} \opentrades(i,N) \beta_i \alpha_i
& = \sum_{i \leq n}0 \cdot \beta_i \alpha_i + \sum_{n<i < N} \opentrades(i,N) \beta_i \alpha_i\\
& \leq 0 + \sum_{n<i < N} \beta_i \alpha_i\\
& \leq 1/2. 
\end{align*}
So for infinitely many sufficiently large $k$ we have 
\begin{align*}
  \alpha_k \beta_k 
  &= \alpha_k \left(1- \sum_{i < k} \opentrades(i,k) \beta_i \alpha_i \right) \\
  &\geq \alpha_k (1 - 1/2) \\
  &\geq \delta/2. 
\end{align*}
Thus
\[ \sum_k (\varepsilon/3) \beta_k \alpha_k =\infty, \]
contradicting that this sum is bounded. 
Therefore in fact $\alpha_k \eqsim_k 0$, as desired.
\end{proof}


\subsection{Affine Preemptive Learning}\label{app:affpolymax}

\affpolymax*

\paragraph{Proof strategy:  buying combinations that will appreciate, and ROI.}

The inequality
\[ \liminf_{\nn\to\infty}\pt_{\nn}(\aff_{\nn}) \ge \liminf_{\nn\rightarrow\infty}\sup_{m \ge n}\pt_{m}(\aff_{n}) \]
 states roughly that $\pt_n$ cannot infinitely often underprice   the $\RR$-combination
   $\aff_{n}$ by a substantial amount in comparison to any price $\pt_{m}(\aff_{n})$ assigned to $\aff_{n}$ by $\pt_m$ at any future time $m\ge
   n$.

Intuitively, if the market $\MP$ did not satisfy this inequality, then $\MP$ would be exploitable by a trader that buys the $\RR$-combination $\aff_{n}$
when 
  its price is low, and then sells it back when, inevitably, the price is substantially higher. 
  If we have sold back all our shares in some sentence $\phi$, then there is no
  contribution, positive or negative, to our
  net value from our $\phi$-shares (as opposed to their prices); for every share we owe, there is a matching share that we hold.
  So if we buy low and sell high, we have made a profit off of the price
  differential, and once the inter-temporal arbitrage is complete we have not taken on any net risk from our stock holdings.  

The fact that we can accrue stock 
holdings that we are guaranteed to eventually sell back for more than their purchase price is not sufficient to exploit the market. It
may be the case that at every time $n$ we spend $\$ 1$ on some $\RR$-combination that we eventually sell back at $\$ 2$, but not 
 until time
$4n$. (That is, until time $4n$, the price remains low.) 
Then at every time $n$ we owe $-\$ n$ in cash, but only have around $\$ 2(n/4)$ worth of cash from shares we have sold off,
for a net value of around $-n/2$. Thus we have net value unbounded below and hence do not exploit the market,
despite the fact that each individual investment we make is eventually guaranteed to be  profitable.

To avoid this obstacle, we will apply 
  the  $\varepsilon$-return on investment lemma (\Lem{type3}) to the sequence of traders $(\Trader^k)_k$ that enforce the inequality at time
 $k$ as  described above. That is, $\Trader^k$ myopically ``keeps $\MP$ sensible about
$\seq{\aff}$'' at time $k$ by buying the $\RR$-combination $\aff_k$
described above if that $\RR$-combination is under-priced at time $k$, and otherwise $\Trader^k$ does nothing.
The ROI Lemma guarantees that the inequality cannot infinitely often fail substantially, or else this sequence would have $\delta$-ROI for some $\delta$. 

  The main technical difficulty is that 
we have to buy the $\RR$-combination $\aff_{n}$ at time $n$ (if it is underpriced), wait for the
  price of the combination to increase substantially, and then sell it off, possibly over multiple time steps. 
  The traders $\Trader^k$ will therefore have to track what fraction of their initial investment they have sold off at any given time. 

\paragraph{Proof.}
\begin{proof}
  We show the first equality; the second equality follows from the first by considering the negated sequence $(-\aff_n)_n$. 

  Since for all $n$ we have 
   $ \sup_{m \ge n}\pt_{m}(\aff_{n}) \ge \pt_{\nn}(\aff_{\nn}) $, the corresponding inequality in the limit infimum is immediate. 

Suppose for contradiction that the other inequality doesn't hold, so that 
  \[
    \liminf_{\nn\to\infty}\pt_{\nn}(\aff_{\nn}) < \liminf_{\nn\rightarrow\infty}\sup_{m \ge n}\pt_{m}(\aff_{n})\ .
  \]
  Then there are rational numbers $\varepsilon>0$ and $\localportfoliovalue$  such that
we have 
\[ \liminf_{\nn\to\infty} \pt_{\nn}(\aff_{\nn})<  \localportfoliovalue - \varepsilon<  \localportfoliovalue + \varepsilon<\liminf_{\nn\to\infty} 
\sup_{m \ge n}\pt_{m}(\aff_{n})\ .\]
Therefore we
can fix some sufficiently large $s_\varepsilon$ such that:
\begin{itemize}
\item for all $n>s_\varepsilon$, we have
$\sup_{m \ge n}\pt_{m}(\aff_{n})
> \localportfoliovalue + \varepsilon$, and 
\item for infinitely many $n> s_\varepsilon$, we have $\pt_{\nn}(\aff_{\nn})<\localportfoliovalue - \varepsilon$.
\end{itemize}

We will assume without loss of generality that each $\magnit{\aff_n} \leq 1$; they are assumed to be bounded, so they can be scaled down appropriately.

\paragraph{An efficiently emulatable sequence of traders.}
Let $\seq{\gen\aff}$ be an $\exfeatures$-combination progression such that $\gen{\aff_n}(\MP) = \aff_n$ for all $n$.
We now define our sequence of traders $(\Trader^k)_k$. 
For $k \leq s_\varepsilon$, define
$\Trader^k$ to be the zero trading strategy at all times $n$.

For $k > s_\varepsilon$, define $\trade^k_n$ to be the
zero trading strategy
for $n < k$, and define $\trade^k_k$ to be the trading strategy
\[ \trade^k_k :=
\underpriced_k
 \cdot \left( \aff^\gens_{k} - \aff^{\gens * k}_{k}\right)\ ,
\]
where
\[ \underpriced_k :=
  \ctsind{\varepsilon/2}\left( \aff^{\gens * k}_{k} < \localportfoliovalue - \varepsilon/2
\right) \ .
\]
This is a buy order for the $\RR$-combination $\aff_{k}$,
scaled down by the continuous indicator function $\underpriced_k$ for the event that $\pt_k$ has underpriced that
$\RR$-combination at time $k$.
Then, for times $n>k$, we define $\Trader^k$ to submit the trading strategy 
\[ \trade^k_n :=  
-F_n  \cdot 
\left(\underpriced_k \cdot \left( \aff^\gens_{k}  - \aff^{\gens * k}_{k}  \right)\right) \ ,
\]
where we define $F_n \geq 0$ recursively in 
the previous fractions $F_i$:
\begin{align*}
 F_n &:= 
\overpriced_n^k \cdot
\left( 1- \sum_{k<i<n} F_i \right) ,
\end{align*}
using the continuous indicator $\overpriced_n^k := \ctsind{\varepsilon/2}\left( \aff^{\gens * n}_{k}
> \localportfoliovalue + \varepsilon/2 \right)$ of the $\RR$-combination being overpriced at time $n$.

In words, 
$\trade^k_n$ is a sell order for the $\RR$-combination 
$\aff_{k}$, scaled down by the fraction $\underpriced_k$ of this $\RR$-combination that $\Trader^k$ purchased at time $k$, 
 and also scaled down by the fraction $F_n$ of the original purchase $\trade^k_k$ 
 that will be sold on this time step. 
 That is, 
  $\sum_{k<i<n} F_i$  the total fraction of the original purchase $\underpriced_k \cdot \aff_k$
  that has already been sold off on all previous rounds since time $k$.
Then $\trade^k_n$ sells off the remaining fraction $1- \sum_{k<i<n} F_i$ of the  $\RR$-combination 
$\underpriced_k \cdot \aff_k$, scaled down by the extent $\overpriced_n^k$ to which $\aff_k$ is overpriced at time $n$. 

Notice that since $\overpriced_i^k\in [0,1]$ for all $i$, by induction on $n$ we have that $\sum_{k<i\leq n} F_i \leq 1$ and $F_n \ge 0$. This justifies
thinking of the $F_i$ as portions of the original purchase being sold off. 

By assumption, the $\exfeatures$-combination progression $\seq{\aff^\gens}$ is \ec Also, each trader $\Trader^k$ does not trade before time $k$.   
Therefore the sequence of traders $(\Trader^k)_k$ is efficiently emulatable (see \ref{app:dynamicprogramming} on dynamic programming).   (The constant $s_\varepsilon$ before which the $\Trader^{k\le s_\varepsilon}$
  make no trades can be hard-coded in the efficient
  enumeration.)

\paragraph{$\bm{(\varepsilon/2)}$ return on investment for $\bm{\Trader^k}$.}

Now we show that each $\Trader^k$ 
has $(\varepsilon/2)$-ROI;  i.e., for all $\World \in \cworlds(\Theory)$, 
 \[ \lim_{n \to\infty}\World\left( {\textstyle \sum_{i \leq n} \trade^k_i(\MP)}\right)
\geq
(\varepsilon/2)
\magnit{\Trader^k(\MP)} . \]
In words, this says that the trades made by $\Trader^k$ across all time are valued positively in any $\World\in\cworlds(\Theory)$, by a fixed fraction 
$(\varepsilon/2)$ of the magnitude of $\Trader^k$. For $k \leq s_\varepsilon$, this is immediate since $\magnit{\Trader^k(\MP)}=0$ by definition. 

For each $k> s_\varepsilon$,  by definition $\Trader^k$ makes a trade of magnitude $\magnit{\Trader^k_k \ftn} = \underpriced_k(\MP) \cdot \magnit{ \aff_k }$,
followed by trades of magnitude 
\[ \sum_{n>k} F_n \magnit{\Trader^k_k(\MP)} \leq    \magnit{\Trader^k_k(\MP)}\ , \]
by the earlier comment that the $F_n$ are non-negative and sum to at most 1. Furthermore, by assumption, there is some $m>k$ such that 
$\pt_{m}(\aff_{k})  > \localportfoliovalue + \varepsilon$. At that point, $\overpriced_m^k \ftn = 1$, so that $F_m\ftn = 
\left(1 - \sum_{k<i<m}F_i\ftn\right) $; intuitively this implies that at time $m$, $\Trader^k$ will sell off the last of its stock holdings from trades 
in $\aff_{k}$. Formally we have 
\begin{align*}
\sum_{k< i \leq m} \magnit{\Trader^k_i(\MP)} 
&= \left(F_m\ftn +  \sum_{k<i <m} F_i\ftn \right) \cdot
 \underpriced_k\ftn  \cdot \magnit{\aff_k }\\
 &= \magnit{\Trader^k_k(\MP)}\ . 
\end{align*}
Furthermore, for all times $M>m$ we have $F_M \ftn = 0$,  so that $\trade^k_M \ftn \equiv 0$. Therefore $ \magnit{\Trader^k(\MP)}= \sum_{k\leq
i \leq n} \magnit{\Trader^k_i(\MP)}= 2\magnit{\Trader^k_k(\MP)}$. 

Now fix any world  $\World\in\cworlds(\Theory)$. Then the limiting value of $\Trader^k$ in $\World$ is:
\begin{align*}
 \lim_{n \to\infty}\World\left( {\textstyle \sum_{i \leq n} \trade^k_i(\MP)}\right)
 &= \World\left( {\textstyle \sum_{k\leq i \leq m} \trade^k_i(\MP)}\right)\\
\shortintertext{since by the above analysis, $\trade^k_i$ is the zero trade for $i<k$ and for $i>m$;} 
 &= \World\left( {\textstyle \trade^k_k(\MP) + \sum_{k< i \leq m} \trade^k_i(\MP)}\right)\\
 &= \;\; \underpriced_k(\MP) \cdot
 \World\left( {\textstyle \; \aff_{k}-  \pt_k(\aff_{k})\; } \right) \\
&\;\;\; + \underpriced_k(\MP) \cdot
\World\left( {\textstyle 
\sum_{k< i \leq m} (-F_i\ftn) \cdot \left(\; \aff_{k}-  \pt_i(\aff_{k})\; \right) }\right)\\
\shortintertext{by linearity, by the definition of the trader $\Trader^k$, and since by definition $ \aff^{\gens * k}_{k}
(\MP) =  \pt_k(\aff^\gens_{k} (\MP)) = \pt_k(\aff_{k}) $. Note that the prices $\pt_i(\aff_{k})$ of $\aff_k$ in the summation
change with the time step $i$. Then }
 &= \;\; \underpriced_k(\MP) \cdot
 \World\left( {\textstyle \; \aff_{k}-  \aff_{k}\; } \right) \\
&\;\;\; + \underpriced_k(\MP) \cdot
\World\left( {\textstyle -\pt_k(\aff_{k}) + 
\sum_{k< i \leq m} F_i \cdot  \pt_i(\aff_{k}) }\right)\\
 \lim_{n \to\infty}\World\left( {\textstyle \sum_{i \leq n} \trade^k_i(\MP)}\right) &=  \underpriced_k(\MP) \cdot
\left( {\textstyle -\pt_k(\aff_{k}) + 
\sum_{k< i \leq m} F_i \cdot  \pt_i(\aff_{k}) }\right)\ ,
\end{align*}
using linearity to rearrange the terms in the first and second lines, and using that $\sum_{k< i \leq m} F_i = 1$ as shown above. Note that this last quantity does not contain any stock holdings whose value depends on the world; intuitively this is because $\Trader^k$ sold off exactly all of its initial purchase. The remaining quantity is the difference between the price at
which the $\RR$-combination $\aff_k$ was bought and the prices at which it was sold over time, scaled down
by the fraction $\underpriced_k \ftn $ of $\aff_k$ that $\Trader^k$ purchased at time $k$.  

If at time step $k$ the $\RR$-combination $\aff_k$ was not underpriced, i.e., $\underpriced_k (\MP)= 0$, then 
\[ \lim_{n \to\infty}\World\left( {\textstyle \sum_{i \leq n} \trade^k_i(\MP)}\right) = 0 = (\varepsilon/2)\magnit{\Trader^k(\MP)}\ ,  \]
as desired. On the other hand, suppose that $\underpriced_k(\MP)  > 0$. That is, 
\[ \ctsind{\varepsilon/2}\left( \pt_k(\aff_k) < \localportfoliovalue - \varepsilon/2
\right) >0 \ , \]
i.e., $\aff_k$ was actually underpriced at time $k$. Therefore
\begin{align*}
 \lim_{n \to\infty}\World\left( {\textstyle \sum_{i \leq n} \trade^k_i(\MP)}\right) 
 &\geq  \underpriced_k(\MP) \cdot
\left( {\textstyle -(\localportfoliovalue - \varepsilon/2)  + 
\sum_{k< i \leq m} F_i \ftn \cdot ( \localportfoliovalue + \varepsilon/2)
}\right)\\ 
\shortintertext{since when $F_i \ftn >0$ we have $\overpriced_n^k(\MP)>0$ and hence $ \pt_i(\aff_{k})  \geq \localportfoliovalue + \varepsilon/2$,
and using the fact that $\underpriced_k(\MP)$ is nonnegative; }
&= \underpriced_k(\MP) \cdot \varepsilon\\
&\geq (\varepsilon/2)\magnit{\Trader^k(\MP)}\ ,
\end{align*}
since 
\[ \magnit{\Trader^k(\MP)} = 2\magnit{\Trader^k_k(\MP)} = 2\cdot\underpriced_k(\MP) \cdot \magnit{\aff_k} \leq 2\cdot \underpriced_k(\MP) \]
Thus $\Trader^k$ has $(\varepsilon/2)$-ROI\@.

\paragraph{Deriving a contradiction.}

We have shown that the sequence of traders $(\Trader^k)_k$ is bounded, efficiently emulatable, and has $(\varepsilon/2)$-return on investment. 
The remaining condition to \Lem{type3}
states that for all $k$, the magnitude $\magnit{\Trader^k(\MP)}$ of all trades made by $\Trader^k$ must equal $\alpha_k$ for some
\pgenable $\seq{\alpha_k}$. This condition is satisfied for $\alpha_k := 2\magnit{\Trader^k_k(\MP)}$, since as
shown above, $\magnit{\Trader^k(\MP)} = 2\magnit{\Trader^k_k(\MP)}$.

Therefore we can apply \Lem{type3} (the ROI lemma) to the sequence of traders $(\Trader^k)_k$. We conclude that  $\alpha_k \eqsim_k 0$.
Recall that we supposed by way of contradiction that 
the $\RR$-combinations in $\seq{\aff_k}$ are underpriced infinitely often. That is, for infinitely many days $k$,
$\pt_{k}(\aff_{k})< \localportfoliovalue - \varepsilon$. 
But for any such $k>s_\varepsilon$, $\Trader^k_k$ purchases a full  $\RR$-combination  $\aff_k$, 
and then sells off the resulting stock holdings for at least $\localportfoliovalue + \varepsilon/2$, at which point
$\Trader^k$ has profited by at least $\varepsilon$. More precisely, for these $k$ we have 
\[ \underpriced_k(\MP)
= \ctsind{\varepsilon/2}\left( \pt_{k}(\aff_{k})  < \localportfoliovalue - \varepsilon/2\right) = 1 .\]
Since 
\[ \alpha_k = 2 \magnit{\Trader^k_k \ftn} = 2 \cdot \underpriced_k \ftn \cdot \magnit{\aff_k} = 2 \magnit{\aff_k}
\]
 and $\magnit{\aff_k} \geq \varepsilon/2$ (since $\pt_m(\aff_k) - \pt_k(\aff_k) \geq  \varepsilon$ for some $m$),
we have $\alpha_k \geq \varepsilon$ for infinitely many $k$,
which contradicts $\alpha_k \eqsim_k 0$.
\end{proof}


\subsection{Preemptive Learning}\label{app:tbo}

\restatetbo*

\begin{proof}
  This is a special case of \Theorem{affpolymax}, using the combination $\aff_n := \phi_n$.
\end{proof}

\subsection{Convergence}\label{app:con}

\convergence*

\begin{proof}
  By \Theorem{tbo},
\begin{align*}
 \liminf_{n \to \infty} \pt_n(\phi) 
 &= \liminf_{n \to \infty} \sup_{m \geq n} \pt_m(\phi) \\
  &= \lim_{k \to \infty} \inf_{n \geq k} \sup_{m \geq n} \pt_m(\phi) \\
 &= \limsup_{n \to \infty} \pt_n(\phi).
\end{align*}
Since the $\liminf$ and $\limsup$ of $\pt_n(\phi)$ are equal, the limit exists.
\end{proof}


\subsection{Persistence of Affine Knowledge}\label{app:peraffkno}

Let $\seq{\aff}\in\BCS$. Then
\[
\liminf_{\nn\rightarrow\infty}\inf_{\mm\geq \nn}\pt_\mm(\aff_\nn)= \liminf_{\nn\to\infty}\pt_\infty(\aff_\nn)
\]
and
\[
\limsup_{\nn\rightarrow\infty}\sup_{\mm\geq \nn}\pt_\mm(\aff_\nn)=\limsup_{\nn\to\infty}\pt_\infty(\aff_\nn).
\]
\paragraph{Proof strategy: keeping $\bm{\pt_m}$ reasonable on all $\bm{\afffeat_{n \le m}}$.}

In the same vein as the proof in Appendix~\ref{app:affpolymax} of Theorem~\ref{thm:affpolymax}, the  inequality  
\[
\liminf_{\nn\rightarrow\infty}\inf_{\mm\geq \nn}\pt_\mm(\aff_\nn) \ge  \liminf_{\nn\to\infty}\pt_\infty(\aff_\nn)
\]
 says roughly that $\pt_m$ cannot underprice
   the $\RR$-combination $\aff_{n}$ by a substantial amount infinitely often, where the $\RR$-combination is
   ``underpriced'' in comparison to the value of the $\RR$-combination as judged by the limiting belief state $\pt_\infty$.  

As the proof of the present theorem is quite 
   similar to the proof of Theorem~\ref{thm:affpolymax}, we will highlight the differences in this proof, and otherwise give
  a relatively terse proof. 

Intuitively, if the market $\MP$ did not satisfy the present inequality then $\MP$ would be exploitable by a trader that buys $\aff_{n}$ at
any time $m$ such that its price $\pt_{m}(\aff_n)$ is lower than its eventual price, and then sells back the $\RR$-combination when the price rises.

We would like to apply the  
return on investment lemma (\Lem{type3}) as in Theorem~\ref{thm:affpolymax}. One natural attempt is to have, for each $n$,
a trader for that watches
the price of $\aff_n$ at all times $m \ge n$, buying low and selling high. This proof strategy may be feasible, but does not follow 
straightforwardly from the ROI lemma: those traders may be required to make multiple purchases
of $\aff_n$ in order to guard against their prices \emph{ever} dipping too low. This pattern of trading may violate the condition for applying
the ROI lemma that requires traders to have a total trading volume  that is predictable by a \pgenable $\exfeatures$-progression (in order to enable verifiable budgeting). 

Thus we find it easier to index our traders by the time rather than by $\aff_n$. That is, we will define  a sequence of traders $(\Trader^k)_k$, where
the trader $\Trader^k$
ensures that $\pt_k$ does not assign too low a price to any $\aff_n$ for $n \le k$. Specifically, $\Trader^k$ at time $k$
buys any $\RR$-combination $\aff_n$  for
$n \le k$ with $\pt_{k}(\aff_n)$ sufficiently low,
and then sells back each such purchase as the price
$\pt_{m}(\aff_n)$ rises. In this way, if $\RR$-combinations are ever underpriced at any time above the main
diagonal, there is a trader ready to buy that $\RR$-combination in full.

\paragraph{Proof.}
\begin{proof}

  We show the first equality; the second equality follows from the first by considering the negated progression $(-\aff_n)_n$. 

For every $n$, since  $\pt_\infty$ is the limit of the $\pt_m$ and since $\Valuation(\aff_n)$ is continuous as a function of the valuation $\Valuation$, we have that $\inf_{m \ge n}\pt_{m}(\aff_n) 
\le \pt_{\infty}(\aff_{n})$. Therefore 
the corresponding inequality in the limit infimum is immediate. 

Suppose for contradiction that the other inequality doesn't hold, so that 
  \[
    \liminf_{\nn\to\infty}\inf_{m \ge n}\pt_{m}(\aff_n) < \liminf_{\nn\rightarrow\infty}\pt_{\infty}(\aff_{n}).
  \]
  Then there are rational numbers $\varepsilon>0$ and $\localportfoliovalue$  such that
we have 
\[  \liminf_{\nn\to\infty}\inf_{m \ge n}\pt_{m}(\aff_n)
<  \localportfoliovalue - \varepsilon<\localportfoliovalue + \varepsilon<  \liminf_{\nn\rightarrow\infty}\pt_{\infty}(\aff_{n})\ , \]
  and therefore we
can fix some sufficiently large $s_\varepsilon$ such that 

\begin{itemize}
\item for all $n>s_\varepsilon$, we have  $\pt_{\infty}(\aff_{n})
> \localportfoliovalue + \varepsilon$, and 
\item for infinitely many $n>s_\varepsilon$, we have   $\inf_{m\ge n} \pt_{m}(\aff_n)<\localportfoliovalue - \varepsilon$.
\end{itemize}

We will assume without loss of generality that each $\magnit{\aff_n} \leq 1$; they are assumed to be bounded, so they can be scaled down appropriately.

\paragraph{An efficiently emulatable  sequence of traders.}

Now we define our sequence of traders $(\Trader^k)_k$. 
Let $\seq{\gen\aff}$ be an $\exfeatures$-combination progression such that $\gen\aff_n(\MP) = \aff_n$ for all $n$.
For $n < k$, define $\trade^k_n$ to be the
zero trading strategy. 
Define $\trade^k_k$ to be the trading strategy
\[ \trade^k_k := 
\sum_{s_\varepsilon < n\le k} \left( \underpriced_k^n \cdot \left( \aff^\gens_n - \aff_{n}^{\gens * k} \right)   \right)  \  ,
\]
where
\[ \underpriced_k^n:= \ctsind{\varepsilon/2}\left( \aff_{n}^{\gens * k} < \localportfoliovalue - \varepsilon/2
\right)\ . \]
This is a buy order for each $\RR$-combination $\aff_{n}$ for $s_\varepsilon <
n \le k$, scaled down by the continuous indicator function $\underpriced_k^n$ for the event that $\aff_n$ is underpriced at time $k$ by $\MP$.
Then, for time steps $m>k$, we define $\trade^k_m$ to be the trading strategy 
\[ \trade^k_m := F_m    \cdot 
\sum_{s_\varepsilon < n\le k} 
  \left( -\underpriced_k^n \cdot \left( \aff^\gens_n - \aff_{n}^{\gens* m} \right)\right)  \  ,
\]
where 
\[ F_m := 
\left( 1- \sum_{k<i<m} F_i \right)  \cdot
\prod_{s_\varepsilon < n\le k} 
  \overpriced_m^n
\]
and
\[ \overpriced_m^n:= \ctsind{\varepsilon/2}\left( \aff_{n}^{\gens* m} > \localportfoliovalue + \varepsilon/2 \right) \ . \]
This trade is a sell order for the entire $\RR$-combination comprising the sum of the scaled $\RR$-combinations
$\underpriced_k^n \ftn \cdot \aff_{n}$ for $s_\varepsilon < n \leq k$ purchased at time $k$ by $\trade^k_k$, scaled down by 
the fraction $F_m \ftn$.  We define 
$F_m$ so that it  represents the 
fraction of the original purchase made by the  trader $\Trader^k$ that has not yet been sold off by time $m$, scaled down by the  continuous indicator 
$ \prod_{s_\varepsilon < n\le k} 
  \overpriced_m^n$ for the event that
all of those $\RR$-combinations $\aff_{n}$ for $s_\varepsilon < n \leq k$ are overpriced at time $m$.

By assumption, the $\exfeatures$-combination progression $\seq{\aff^\gens}$ is 
\ec[,] and  each trader $\Trader^k$ does not trade before time $k$. 
Therefore the sequence of traders $(\Trader^k)_k$ is efficiently emulatable (see \ref{app:dynamicprogramming} on dynamic programming).

\paragraph{$\bm{(\varepsilon/2)}$ return on investment for $\bm{\Trader^k}$.}

Now we show that each $\Trader^k$ 
has $(\varepsilon/2)$-ROI, i.e., for all $\World \in \cworlds(\Theory)$: 
 \[ \lim_{n \to\infty}\World\left( {\textstyle \sum_{i \leq n} \trade^k_i(\MP)}\right)
\geq
(\varepsilon/2)
\magnit{\Trader^k \ftn}\ .\]

Roughly speaking, $\Trader^k$ gets 
$(\varepsilon/2)$-ROI for the same reason as the traders in the proof of Theorem~\ref{thm:affpolymax}: the stock holdings from 
each $\aff_{n}$ that $\Trader^k$ purchased will be sold
off for at least $(\varepsilon/2)$-ROI, so the sum of the $\RR$-combinations is sold off for $(\varepsilon/2)$-ROI\@. 

Since $\overpriced_i^n \ftn \in [0,1]$ for all $n$ and $i$, 
by induction on $m$ we have that $\sum_{k<i\leq m} F_i\ftn  \leq 1$ and $F_m\ftn  \ge 0$. 
Therefore  for each $k> s_\varepsilon$,  by definition $\Trader^k$ makes a trade of magnitude $\magnit{\Trader^k_k \ftn } =  
\sum_{s_\varepsilon < n\le k}\underpriced_k^n\ftn  \cdot \magnit{\aff_n} $, 
followed by trades of magnitude 
\[ \sum_{n>k} F_n\ftn  \magnit{\Trader^k_k \ftn } \leq   \magnit{\Trader^k_k \ftn }\ .  \]
By assumption, there is some time $m$  such that 
$\pt_{m}(\aff_n)  > \localportfoliovalue + \varepsilon/2$ for all $s_\varepsilon < n \le k$. At that point, $\overpriced_m^n \ftn = 1$ for each such $n$, so that $F_n \ftn = 
\left(1 - \sum_{k<i<n}F_i \ftn \right) $. Then at time $m$, $\Trader^k$ will sell of the last of its stock holdings from trades in $\aff_{n}$, so that $\sum_{k< i \leq m} \magnit{\Trader^k_i \ftn}$ is equal to
\[  \left(F_m \ftn +  \sum_{k<i <m} F_i \ftn\right) \cdot 
 \sum_{s_\varepsilon < n\le k}\underpriced_k^n \ftn \cdot \magnit{\aff_n}  = \magnit{ \Trader^k_k\ftn} \ . \]
Furthermore, for all times $M>m$ we have $F_M\ftn = 0$,  so that $\trade^k_M\ftn \equiv 0$. Therefore $ \magnit{\Trader^k \ftn}= \sum_{k\leq
i \leq n} \magnit{\Trader^k_i\ftn}= 2\magnit{\Trader^k_k\ftn}$. 

From this point, the proof of return on investment is essentially identical to the analogous proof of Theorem~\ref{thm:affpolymax}. The only difference is that here the
trader $\Trader^k$ holds a combination of $\RR$-combinations. Therefore we will not belabor the details; inserting a summation  $\sum_{s_\varepsilon < n\le k}$ in
front of the trades made by the traders in the proof of Theorem~\ref{thm:affpolymax} will produce the precise derivation. 

In short, since  $\Trader^k$ will eventually hold no net  shares, the value of its holdings is determined by the prices of the shares
it trades, regardless of plausible worlds. By definition, $\Trader^k$ purchases a mixture of $\RR$-combinations
\[ \sum_{s_\varepsilon < n\le k} \left( \underpriced_k^n \ftn \cdot \aff_n \right)    , \]
where each $\aff_n$ with $\underpriced_k^n \ftn >0$
has price $\pt_k(\aff_n)$ at most $\localportfoliovalue -\varepsilon/2$ at time $k$.
Then $\Trader^k$ sells off that mixture, at times for which each $\RR$-combination has price
at least $\localportfoliovalue +\varepsilon/2$. Thus $\Trader^k$ eventually has holdings with value at least 
\begin{align*}
&\;\;\;\sum_{s_\varepsilon < n\le k} \left( \underpriced_k^n(\MP) \cdot \left(  \localportfoliovalue +\varepsilon/2 - (\localportfoliovalue -\varepsilon/2) \right) \right) \\ 
&= \sum_{s_\varepsilon < n\le k} \left( \underpriced_k^n(\MP)\cdot  \varepsilon \right)  \\
&\geq \varepsilon \sum_{s_\varepsilon < n\le k} \left( \underpriced_k^n(\MP) \cdot |\aff_n| \right)    \\
&\geq \varepsilon \magnit{\trade^k_k\ftn} \\ 
&= (\varepsilon/2 ) \magnit{\Trader^k\ftn} \ . 
\end{align*}
Thus, $\Trader^k$ 
 has $(\varepsilon/2)$-ROI\@.

\paragraph{Deriving a contradiction.}

We have shown that the sequence of traders $(\Trader^k)_k$ is bounded, efficiently emulatable, and has $(\varepsilon/2)$-return on investment. 
The remaining condition to \Lem{type3}
states that for all $k$, the magnitude $\magnit{\Trader^k(\MP)}$ of all trades made by $\Trader^k$ must equal $\alpha_k$ for some \pgenable $\seq{\alpha_k}$. This condition is satisfied for $\alpha_k := 2\magnit{\Trader^k_k(\MP)}$, since as
shown above, $\magnit{\Trader^k(\MP)} = 2\magnit{\Trader^k_k(\MP)}$.

Therefore we can apply \Lem{type3} (the ROI lemma) to the sequence of traders $(\Trader^k)_k$. We conclude that $\alpha_k \eqsim_k 0$.
Recall that we supposed by way of contradiction that infinitely often, 
some $\aff_n$ is underpriced. That is, for infinitely many times $k$ and indices
$s_\varepsilon < n\le k$,
$\pt_{k}(\aff_n)< \localportfoliovalue - \varepsilon$.

But for any such $k $ and $n$, $\trade^k_k$ will purchase the full  $\RR$-combination $\aff_{n}$, as
\[ \underpriced_k^n (\MP)
= \ctsind{\varepsilon/2}\left( \pt_{k}(\aff_n)  < \localportfoliovalue - \varepsilon/2\right) = 1 \ . \]
Now $\alpha_k = 2 \magnit{\Trader^k_k} \geq 2 \cdot \underpriced_k^n \magnit{\aff_n} =2 \magnit{\aff_n}$, and $\magnit{\aff_n} \geq \varepsilon/2$
(since $\pt_m(\aff_n) - \pt_k(\aff_n) \geq \varepsilon$ for some $m$).
So $\alpha_k \geq \varepsilon$ infinitely often, contradicting $\alpha_k \eqsim_k 0$.
\end{proof}


\subsection{Persistence of Knowledge}\label{app:perkno}

\perkno*

\begin{proof}
  The second and third statements are a special case of \Theorem{peraffkno}, using the combination $\aff_n := \phi_n$; the first statement follows from the second and third.
\end{proof}

\section{Coherence Proofs}\label{app:coherenceproofs}


\subsection{Affine Coherence}\label{app:affcoh}

\affcoh*

\begin{proof}

  We show the first series of inequalities; the second series follows from the first by considering the negated progression $(-\aff_n)_n$. Let $\seq{\gen\aff}$ be an $\exfeatures$-combination progression such that $\gen\aff_n(\MP) = \aff_n$ for all $n$.

\paragraph{Connecting $\bm{\cworlds(\Theory)}$ to $\bm{\pt_\infty}$.}

First we show that 
\[ \liminf_{n \to \infty} \inf_{\World\in \cworlds(\Theory)} \World(\aff_{\nn}) \le \liminf_{n \to \infty}  \pt_{\infty}(\aff_{n}). \]
It suffices to show the stronger statement that for any $n\in\NN^+$, 
\[ \inf_{\World\in\cworlds(\Theory)} \World(\aff_{\nn}) \le \pt_{\infty}(\aff_{n}). \]
 This is a generalization 
of coherence in the limit to affine relationships; its proof will follow a strategy essentially identical to the one used in the proof of Theorem~\ref{sec:limitprops} (coherence)
to show the particular coherence relationships that are sufficient to imply ordinary probabilistic coherence. That is,
we will construct a trader that waits for the coherence relationship to approximately hold (so to speak) and for the price of the corresponding $\RR$-combination to
approximately converge, and then buys the combination repeatedly if it is underpriced. 

 Suppose by way of contradiction that the inequality does not hold, so for some fixed $n$ 
  there are  rational numbers $\varepsilon>0$ and $\localportfoliovalue$ and  a time step $s_\varepsilon$  such that for all $m> s_\varepsilon$ we have 
  \[ \pt_{m}(\aff_{n})   <  \localportfoliovalue  -\varepsilon   <  \localportfoliovalue +\varepsilon
  < \inf_{\World\in\pcworlds(\dt_\mm)} \World(\aff_{\nn}) \ . \]
 Therefore we can define a trader $\Trader$ that waits
  until time $s_\varepsilon$, and thereafter buys a full 
$\RR$-combination $\aff_{n}$ on every time step. That is, we take $\trade_m$ to be the zero trading
  strategy for $m\le s_\varepsilon$, and we define $\trade_m$ for $m > s_\varepsilon$ to be 
  \[ \trade_m := \aff^\gens_n - \aff_n^{\gens* m} .  \]
  Intuitively, since the infimum over plausible worlds of the value of the stocks in this $\RR$-combination
  is already substantially higher than its price, the value of the  total holdings of our trader $\Trader$ immediately increases  by at least $2\varepsilon$. More formally, we have that for any time $m$ and any $\World\in\pcworlds(\dt_\mm)$, 
  \begin{align*}
    \World\left(\sum_{i\le m}\trade_i(\MP)\right) &=     \World\left(\sum_{s_\varepsilon < i\le m}\trade_i(\MP)\right)\\ 
   \shortintertext{since $\trade_m \equiv 0$ for $m \le s_\varepsilon$;}
    &= \sum_{s_\varepsilon < i\le m} \World(\aff_{\nn}) - \pt_{i}(\aff_{n}) \\ 
   \shortintertext{by linearity, by definition of $\trade_i$, and since $\aff^\gens_{\nn}(\MP) \equiv \aff_{\nn}$ and
   $\aff^{\gens * m}_{\nn}(\MP) \equiv \pt_m(\aff_{\nn})$; } 
    &\geq \sum_{s_\varepsilon < i\le m}  \localportfoliovalue  +\varepsilon -(\localportfoliovalue  -\varepsilon  )  \\ 
  &=  2\varepsilon (m-s_\varepsilon).
  \end{align*}
  This is bounded below by 0 and  unbounded above  as $m$ goes to $\infty$. Thus $\Trader$ exploits the market $\MP$, contradicting that $\MP$ is a \li{}. Therefore in fact we must have 
  \[ \liminf_{n \to \infty} \inf_{\World\in\cworlds(\Theory)} \World(\aff_{\nn}(\MP)) \le \liminf_{n \to \infty}  \pt_{\infty}(\aff_{n}(\MP)) , \]
as desired. 

\paragraph{Connecting  $\bm{\pt_n}$ to $\bm{\pt_\infty}$ and to fast diagonals.}

Now we show that 
\[ \liminf_{n \to\infty} \pt_{\infty}(\aff_{n}) \le \liminf_{n \to\infty} \pt_{\nn}(\aff_{\nn}). \]
This says, roughly speaking, that affine
relationships that hold in the limiting belief state $\pt_\infty$ also hold  along the main diagonal.  We show this inequality in two steps. 
First, by Theorem~\ref{thm:peraffkno} (Persistence of affine knowledege), we have 
\[ \liminf_{n \to\infty} \pt_{\infty}(\aff_{n}) \le  \liminf_{n \to\infty}\inf_{m \ge n} \pt_{m}(\aff_{n}). \]
This says roughly that if the limiting beliefs end up satisfying some sequence of affine relationships, then eventually all belief states above the main
diagonal satisfy that relationship to at least the same extent. Second, it is immediate that 
\[ \liminf_{n \to\infty}\inf_{m \ge n} \pt_{m}(\aff_{n}) \le \liminf_{n \to\infty}  \pt_{n}(\aff_{n}), \]
since for all $n$, 
$\inf_{m \ge n} \pt_{m}(\aff_{m}) \le \pt_{\nn}(\aff_{\nn})$. Thus we have the desired inequality. 
\end{proof}

\subsection{Affine Provability Induction}\label{app:affprovind}

\affprovind*

\begin{proof}
  We prove the statement in the case of $\geq$; the case of $\leq$ is analogous, and the case of $=$ follows from the conjunction of the other two cases.  By \Theorem{affcoh},

  \[ \liminf_{n \to \infty} \pt_n(\aff_n) \geq \liminf_{n \to \infty} \inf_{\World \in \cworlds(\Theory)} \World(\aff_n) \geq b. \]
  We will usually apply this theorem using the $=$ case.
\end{proof}

\subsection{Provability Induction}\label{app:provind}

\restatepi*

\begin{proof}
  Since $\seq\phi$ is a sequence of theorems, for all $n$ and $\World \in \cworlds(\Theory)$, $\World(\phi_n) = 1$.  So by \Theorem{affprovind},
  \[ \pt_n(\phi_n) \eqsim_n 1. \]
  Similarly, since $\seq\psi$ is a sequence of disprovable sentences, for all $n$ and $\World \in \cworlds(\Theory)$, $\World(\psi_n) = 0$.  So by \Theorem{affprovind},
  \[ \pt_n(\psi_n) \eqsim_n 0. \]
\end{proof}

\subsection{Belief in Finitistic Consistency}\label{app:pac}

\restatepac*

\begin{proof}
  Since each statement $\consen(\Theory)(\quot{\enc{f}(\enc{\nn})})$ is computable and true, and $\Theory$ \representscomputations, each of these statements is provable in $\Theory$.  Now apply \Theorem{provind} to get the desired property.
\end{proof}

\subsection{Belief in the Consistency of a Stronger Theory}\label{app:pazfc}

\restatepazfc*

\begin{proof}
  Since each statement $\consen(\Theory')(\quot{\enc{f}(\enc{\nn})})$ is computable and true, and $\Theory$ \representscomputations, each of these statements is provable in $\Theory$.  Now apply \Theorem{provind} to get the desired property.
\end{proof}

\subsection{Disbelief in Inconsistent Theories}\label{app:incons}

\restateincons*

\begin{proof}
  Since each statement 
  $\quot{\enc{\Theory^\prime_\nn}\textnormal{\ is inconsistent}}$ is provable in $\PA$, and $\Theory$ \representscomputations, each of these statements is provable in $\Theory$.  Now apply \Theorem{provind} to get the first desired property.

  Similarly, since each statement $\quot{\enc{\Theory^\prime_\nn}\textnormal{\ is consistent}}$ is disprovable in $\PA$, and $\Theory$ \representscomputations, each of these statements is disprovable in $\Theory$.  Now apply \Theorem{provind} to get the second desired property.
\end{proof}

\subsection{Learning of Halting Patterns}\label{app:halts}

\restatehalts*

\begin{proof}
  Since each statement $\quot{\text{$\enc{m_\nn}$ halts on input $\enc{x_\nn}$}}$ is computable and true, and $\Theory$ \representscomputations, each of these statements is provable in $\Theory$.  Now apply \Theorem{provind} to get the desired property.
\end{proof}

\subsection{Learning of Provable Non-Halting Patterns}\label{app:loops}

\restateloops*

\begin{proof}
  Each statement $\quot{\text{$\enc{q_\nn}$ halts on input $\enc{y_\nn}$}}$ is disprovable in $\Theory$.  Now apply \Theorem{provind} to get the desired property.
\end{proof}

\subsection{Learning not to Anticipate Halting}\label{app:dontwait}

\restatedontwait*

\begin{proof}
  Since each statement $\quot{\text{$\enc{q_\nn}$ halts on input $\enc{y_\nn}$ within $\enc{f}(\enc{\nn})$ steps}}$ is computable and false, and $\Theory$ \representscomputations, each of these statements is disprovable in $\Theory$.  Now apply \Theorem{provind} to get the desired property.
\end{proof}

\subsection{Limit Coherence}\label{app:lc}

\limitcoherence*

\begin{proof}
The limit $\pt_\infty(\phi)$ exists by \Theorem{con}, so $\pt_\infty$ is well-defined. \citet{Gaifman:1964} shows that $\pt_\infty$ defines a probability measure over $\cworlds(\Theory)$ so long as the following three implications hold for all sentences $\phi$ and $\psi$:
  \begin{itemize}
    \item If $\Theory \vdash \phi$, then $\pt_\infty(\phi) = 1$,
    \item If $\Theory \vdash \lnot \phi$, then $\pt_\infty(\phi) = 0$,
    \item If $\Theory \vdash \lnot(\phi \land \psi)$, then $\pt_\infty(\phi \lor \psi) = \pt_\infty(\phi) + \pt_\infty(\psi)$.
  \end{itemize}
Let us demonstrate each of these three properties.
\begin{itemize}
  \item Assume that $\Theory \vdash \phi$.  By \Theorem{provind}, $\pt_\infty(\phi) = 1$.
  \item Assume that $\Theory \vdash \lnot \phi$.  By \Theorem{provind}, $\pt_\infty(\phi) = 0$.
  \item Assume that $\Theory \vdash \lnot(\phi \land \psi)$.  For all $\World \in \cworlds(\Theory)$, $\World(\phi \vee \psi) = \World(\phi) + \World(\psi)$.  So by \Theorem{affprovind}, $\pt_\infty(\phi \vee \psi) = \pt_\infty(\phi) + \pt_\infty(\psi)$.
\end{itemize}

\end{proof}

\subsection{Learning Exclusive-Exhaustive Relationships}\label{app:lex}

\restatelex*

\begin{proof}
  Define $\aff_n := \phi_n^1 + \cdots + \phi_n^k$.  Note that for all $\World \in \cworlds(\Theory)$, $\World(\aff_n) = 1$.

  So by \Theorem{affprovind} and linearity,
  \[  \pt_n(\phi_n^1) + \cdots + \pt_n(\phi_n^k ) = \pt_n(\phi_n^1 + \cdots + \phi_n^k ) =     \pt_n(\aff_n) \eqsim_n 1. \]

\end{proof}

\section{Statistical Proofs}\label{app:statisticalproofs}

\subsection{Affine Recurring Unbiasedness}\label{app:recunbiasedaff}

\recunbiasedaff*

\begin{proof}

  Define 
  \[ \mathrm{Bias}_n := 
      \frac
        {\sum_{i\leq \nn}\fuz_i  \cdot(\pt_i(\aff_i)-\thmval(\aff_i))}
        {\sum_{i\leq \nn}\fuz_i} .
      \]

  Our proof consists of three steps:
  \begin{enumerate}
    \item Proving $\limsup_{n \to \infty} \mathrm{Bias}_n \geq 0$.
    \item Noting that the first argument can be applied to the sequence $(-\aff_n)_n$ to prove $\liminf_{n \to \infty} \mathrm{Bias}_n \leq 0$.
    \item Proving that, given these facts, $(\mathrm{Bias}_n)_n$ has 0 as a limit point.
  \end{enumerate}
  The first step will be deferred.  The second step is trivial.  We will now show that the third step works given that the previous two do:

  Let $a := \liminf_{n \to \infty} \mathrm{Bias}_\nn \leq 0$ and $b := \limsup_\nn \mathrm{Bias}_n \geq 0$.  If $a = 0$ or $b = 0$ then of course
  0 is a limit point.  Otherwise, let $a < 0 < b$.  If 0 is not a limit point of $(\mathrm{Bias}_n)_n$, then there are $\varepsilon > 0$ and 
  $N \in\NN$ such that $\forall n> N : \mathrm{Bias}_\nn \notin (-\varepsilon,\varepsilon) \subseteq (a,b)$.
  Choose $M>N$ such that $\mathrm{Bias}_M \in (\varepsilon, b]$ and for all $n>M$, $\mathrm{Bias}_n-\mathrm{Bias}_{n+1}<\varepsilon$; sufficiently late adjacent terms are close because $\sum_{i \leq \nn} \fuz_i$ goes to $\infty$ and the absolute difference between successive numerators is at most 1.
  Then $(\mathrm{Bias}_n)_{n>M}$ must remain positive (it cannot cross the $2\varepsilon$-wide gap), contradicting that $a$ is also a limit point and $a < 0$.

  At this point we have shown that the second and third steps follow from the first step, so we need only show the first step: $\limsup_{n \to \infty} \mathrm{Bias}_n \geq 0$.  Suppose this is not the case.  
  Then there is some natural $N$ and rational $\varepsilon \in (0, 1)$ such that for all $\nn \geq N$,
  
      \[ 
      \frac
      {\sum_{i\leq \nn}\fuz_i\cdot (\pt_i(\aff_i) - \thmval(\aff_i))}
        {\sum_{i\leq \nn}\fuz_i}
        <- 2\varepsilon
      \]
        or equivalently,
        \[ 
        \sum_{i\leq \nn}\fuz_i\cdot (\thmval(\aff_i) - \pt_i(\aff_i))
        > 2\varepsilon\sum_{i \leq \nn} \fuz_i .
      \]

      \paragraph{An efficiently emulatable sequence of traders.}
      We will consider an infinite sequence of traders, each of which will buy a ``run'' of $\RR$-combinations, and which will have $\varepsilon$-ROI\@.  Then we will apply \Lem{type3} to derive a contradiction.

      Without loss of generality, assume each $\magnit{\aff_n} \leq 1$; they are uniformly bounded and can be scaled down without changing the theorem statement.
      Let $\seq{\gen \fuz}$ be an \ec $\exfeatures$ progression such that $\gen{\fuz_n} \ftn = \fuz_n$.  Let $\seq{\gen \aff}$ be an \ec $\exfeatures$-combination progression such that $\gen{\aff_n} \ftn = \aff_n$.  Let $\seq{\gen\aff}$ be equal to
      \[
        \gen{\aff_n} = c + \xi_1 \phi_1 + \cdots + \xi_{l(n)} \phi_{l(n)},
      \]
      and define the expressible feature $\magnit{\gen{\aff_n}} := \sum_{i=1}^{l(n)} |\xi_i|$.

      For $k < N$, trader $\Trader^k$ will be identically zero.  For $k \geq N$, trader $\Trader^k$ will buy some number of copies of $\gen{\aff_n}$ on day $n$; formally,

      \[ \trade^k_n := \gen{\gamma_{k,n}} \cdot (\gen{\aff_n} - \aff_n^{\gens * n}) , \]
      with $\gen{\gamma_{k,n}}$ to be defined later.  To define $\gen{\gamma_{k,n}}$, we will first define a scaling factor on the trader's purchases:
      \[ \delta_k := \frac{\varepsilon}{1 + k} . \]
      Now we recursively define
      \[
        \gen{\gamma_{k, \nn}} := [\nn \geq k] \min\left(\delta_k \cdot \gen{\fuz_\nn}, 1 - \sum_{i\le \nn-1} \gen{\gamma_{k, i}} \magnit{ \gen{\aff_n} } \right),
      \]
      where $[\nn \geq k] $ is Iverson bracket applied to $n \geq k$, i.e. the 0-1 indicator of that condition.
      This sequence of traders is efficiently emulatable, because $\seq{\gen\aff}$ and $\seq{\gen\fuz}$ are \ec[], and $\Trader^k$ makes no trades before day $k$.

      \paragraph{Analyzing the trader.}
      On day $n$, trader $\Trader^k$ attempts to buy $\delta_k \fuz_n$ copies of $\aff_n$, but caps
      its total budget at 1 dollar; the $\min$ in the definition of $\gen{\gamma_{k, n}}$ ensures that $\magnit{\Trader^k \ftn } \leq 1$.

      Observe that $\sum_{i\le n} \magnit{ \Trader^k_i \ftn } = \max\{1, \sum_{i=k}^n \delta_k \fuz_n \magnit{\aff_n} \}$.  We can use this to show that $\magnit{\Trader^k \ftn} = 1$.  For all $n \geq N$,
        \[ 
2 \varepsilon \sum_{i \leq \nn} \fuz_i
<
\sum_{i\leq \nn}\fuz_i \cdot (\thmval(\aff_i) - \pt_i(\aff_i))
\leq 2 \sum_{i \leq n} \fuz_i \magnit{\aff_i}.
        \]
        Since the left hand side goes to $\infty$ as $n \to \infty$, so does the right hand side.  So indeed, $\sum_{n=k}^\infty \delta_k \fuz_n \magnit{\aff_n} = \infty$, and $\magnit{\Trader^k \ftn} = 1$.

        \paragraph{$\bm{\Trader^k}$ has $\bm{\varepsilon}$ return on investment.}
        We will now show that each trader $\Trader^k$ has $\varepsilon$ return on investment.  Trivially, for $k < N$, $\Trader^k$ has $\varepsilon$ return on investment, because $\magnit{\Trader^k \ftn} = 0$.  So consider $k \geq N$.

        As shorthand, let $\gamma_{k, n} := \gen{\gamma_{k, n}} \ftn$.  Let $m_k$ be the first $m \geq k$ for which $\gamma_{k,m} < \delta_k \fuz_m$.  We have $\gamma_{k, n} = \delta_k \fuz_n$ for $k \leq n < m_k$, and $\gamma_{k, n} = 0$ for $n > m_k$.  So for all $\World \in \cworlds(\Theory)$,

        \begin{align*}
          &\World \left(\sum_{n=1}^\infty \trade^k_n \ftn\right)
          \displaybreak[0]\\
          =&
        \sum_{n=1}^\infty \World(\trade^k_n \ftn)
          \displaybreak[0]\\
        =&
        \sum_{k \leq n \leq m_k} \World(\trade^k_n \ftn)
          \displaybreak[0]\\
        =&
        \sum_{k \leq n \leq m_k} \delta_k \fuz_n (\World(\aff_n) - \pt_n(\aff_n)) - (\delta_{k} \fuz_{m_k} - \gamma_{k, m_k}) (\World(\aff_{m_k}) - \pt_n(\aff_{m_k}))
          \displaybreak[0]\\
        \geq&
        \sum_{k \leq n \leq m_k} \delta_k \fuz_n (\World(\aff_n) - \pt_n(\aff_n)) - \delta_k
          \displaybreak[0]\\
        =&
        \sum_{n \leq m_k} \delta_k \fuz_n (\World(\aff_n) - \pt_n(\aff_n)) - \sum_{n < k} \delta_k \fuz_n (\World(\aff_n) - \pt_n(\aff_n)) - \delta_k
          \displaybreak[0]\\
        \geq&
          \sum_{n \leq m_k} \delta_k \fuz_n (\World(\aff_n) - \pt_n(\aff_n)) - (k \delta_k + \delta_k)
          \displaybreak[0]\\
        =&
        \sum_{n \leq m_k} \delta_k \fuz_n (\thmval{\aff_n} - \pt_n(\aff_n)) - \varepsilon
          \displaybreak[0]\\
        \geq&
        \ 2\varepsilon  - \varepsilon
          \displaybreak[0]\\
        =&
        \ \varepsilon.
        \end{align*}
        So each trader $\Trader^k$ with $k \geq N$ makes at least $\varepsilon$ profit with trades of total magnitude 1, ensuring that it has $\varepsilon$ return on investment.

        \paragraph{Deriving a contradiction.}
        Note that the magnitudes of the traders are \pgenable (the first $N - 1$ have magnitude 0 and the rest have magnitude 1).  By \Lem{type3}, $\magnit{\Trader^k} \eqsim_k 0$.  $\magnit{\Trader^k} = 1$ for all $k \geq N$ (by the above analysis), so this is a contradiction.

\end{proof}

\subsection{Recurring Unbiasedness}\label{app:recurringunbiasedness}

\recurringunbiasedness*

\begin{proof}
  This is a special case of \Theorem{recunbiasedaff}.
\end{proof}

\subsection{Simple Calibration}\label{app:simcal}

\simcal*
\begin{proof}
Define $\fuz_i := \ctsind{\delta_i}(a < \pf[i]{\phi_i} < b) $.  By \Theorem{recurringunbiasedness}, the sequence
 \[
   \left( \frac
      {\sum_{i\leq \nn}\fuz_i\cdot (\pt_i(\phi_i) - \thmind(\phi_i))}
        {\sum_{i\leq \nn}\fuz_i}
      \right)_{n \in \NN^+}
  \]
  has 0 as a limit point.  Let $n_1, n_2, \ldots$ be a the indices of a subsequence of this sequence that converges to zero.  We also know that for all $n$ high enough,
 \[
      a \leq \frac {\sum_{i\leq \nn}\fuz_i \pt_i(\phi_i)}{\sum_{i\leq \nn}\fuz_i} \leq b
  \]
  because $\fuz_i = 0$ whenever $\pt_i(\phi_i) \not\in [a, b]$. Now consider the sequence
  \begin{align*}
    & \left(\frac
      {\sum_{i\leq \nn_k}\fuz_i\cdot \thmind(\phi_i)}
        {\sum_{i\leq \nn_k}\fuz_i}
      \right)_{k \in \NN^+}
      \\
      = &
        \left(
      \frac {\sum_{i\leq \nn_k}\fuz_i \pt_i(\phi_i)}{\sum_{i\leq \nn}\fuz_i}
      -
      \frac{\sum_{i\leq \nn_k}\fuz_i\cdot (\pt_i(\phi_i) - \thmind(\phi_i))}
        {\sum_{i\leq \nn_k}\fuz_i}
      \right)_{k \in \NN^+}
    \end{align*}
      The first term is bounded between $a$ and $b$, and the second term goes to zero, so the sequence has a
      $\liminf$ at least $a$ and a $\limsup$ no more than $b$.  By the Bolzano-Weierstrass theorem,
      this sequence has a convergent subsequence, whose limit must be between $a$ and $b$.
      This subsequence is also a subsequence of
      \[
      \left(\frac
        {\sum_{i\leq \nn}\fuz_i\cdot \thmind(\phi_i)}
          {\sum_{i\leq \nn}\fuz_i}
        \right)_{n \in \NN^+}
      \]
      which proves that this sequence has a limit point in $[a, b]$, as desired.
\end{proof}

\subsection{Affine Unbiasedness From Feedback}\label{app:wubaff}

\wubaff*

\begin{proof}
  Without loss of generality, assume that each $\|\aff_n\|_1 \leq 1$.  
  Define
  \[
    \mathrm{Bias}_k :=
    \frac
    {\sum_{i\leq k }\fuz_{\deff(i)}  \cdot(\pt_{\deff(i)}(\aff_{\deff(i)})-\thmind(\aff_{\deff(i)}))}
    {\sum_{i\leq k} \fuz_{\deff(i)}}
  \]
  and observe that $\mathrm{Bias}_k \eqsim_k 0$ implies the theorem statement, since we need only consider the sum over $n$ in the support of $\deff$.
  We wish to show that $\mathrm{Bias}_k \eqsim_k 0$.  We will show a trader that exploits $\MP$ under the assumption that $\lim\inf_{k \to \infty} \mathrm{Bias}_k < 0$, proving $\mathrm{Bias}_k \gtrsim_k 0$.  We can apply the same argument to the sequence $(-\aff_n)_n$ to get $\mathrm{Bias}_k \lesssim_k 0$.

Suppose $\lim\inf_{k \to \infty} \mathrm{Bias}_k < 0$.  Under this supposition, infinitely often, $\mathrm{Bias}_k < - 3 \varepsilon$ for some rational $0 < \varepsilon < 1/6$.

\paragraph{Defining the trader.}

  Let $\seq{\gen \fuz}$ be an \ec $\exfeatures$ progression such that $\gen{\fuz_n} \ftn = \fuz_n$.  Let $\seq{\gen \aff}$ be an \ec $\exfeatures$-combination progression such that $\gen{\aff_n} \ftn = \aff_n$.  Recursively define

\begin{align*}
  \gen{\beta_i} &:= \varepsilon \cdot \gen{\mathrm{Wealth}_i} \cdot \gen{\fuz_i }
  \\
  \gen{\mathrm{Wealth}_i} &:= 1 + \sum_{j\le i-1} \gen{\beta_j} \cdot \left(\aff_{\deff(j)}^{\gens * \deff(j+1)} - \aff_{\deff(j)}^{\gens * \deff(j)}\right)
\end{align*}
in order to define the trader
\[
  T_\nn := \begin{cases}
    \gen{\beta_i} \cdot \left(\gen{\aff_{\deff(i)}} - \aff_{\deff(i)}^{\gens * n}\right)  - [i > 1] \cdot \gen{\beta_{i-1}} \cdot \left(\gen{\aff_{\deff(i-1)}} - \aff_{\deff(i-1)}^{\gens * n} \right)
    & \text{if $\exists i: \nn = \deff(i)$}
    \\
    0 & \text{otherwise.}
  \end{cases}
\]
Note that $\gen{\beta_i}$ and $\gen{\mathrm{Wealth}_i}$ have rank at most $\deff(i)$, and so $T_n$ has rank at most $n$.

\paragraph{Analyzing the trader.}
As shorthand, let $\beta_i := \gen{\beta_i} \ftn$, and $\mathrm{Wealth}_i := \gen{\mathrm{Wealth}_i} \ftn$.

Intuitively, on day $\deff(i)$, $\Trader$ buys $\aff_{\deff(i)}$ according to a fraction of its ``wealth'' $\mathrm{Wealth}_i$ (how much money $\Trader$ would have if it started with one dollar), and then sells $\aff_{\deff(i)}$ at a later time $\deff(i+1)$.  Thus, $\Trader$ makes money if the price of $\aff_{\deff(i)}$ is greater at time $\deff(i+1)$ than at time $\deff(i)$.  

Betting according to a fraction of wealth resembles the Kelley betting strategy and ensures that the trader never loses more than \$1.
$\mathrm{Wealth}_i - 1$ is a lower bound on $\Trader$'s worth in any world $\World\in\pcworlds(\dt_{\deff(i)})$, based on trades $\Trader$ makes on  $\RR$-combinations $\aff_{\deff(1)}$ through $\aff_{\deff(i-1)}$.  Thus, since the number of copies of $\aff_n$ that $\Trader$ buys is no more than $\varepsilon$ times its current wealth, and $\|\aff_n\| \leq 1$, $T$'s minimum worth is bounded below by $-1$.

Now it will be useful to write $\mathrm{Wealth}_i$ in log space.  Intuitively, this should be enlightening because $T$ always bets a fraction of its wealth (similar to a Kelley bettor), so its winnings multiply over time rather than adding. By induction,
\[
\log \mathrm{Wealth}_i 
= \sum_{j\le i-1} \log \left(1 + \varepsilon \fuz_j (\pt_{\deff(j+1)}(\aff_{\deff(j)}) - \pt_{\deff(j)}(\aff_{\deff(j)}))\right)
\]
This statement is trivial when $i=1$.  For the inductive step, we have
\begin{align*}
  \log \mathrm{Wealth}_{i+1} 
  &=
  \log \left(\mathrm{Wealth}_i + \beta_i (\pt_{\deff(i+1)}(\aff_{\deff(i)}) - \pt_{\deff(i)}(\aff_{\deff(i)})) \right)
  \\
  &=
  \log \left(\mathrm{Wealth}_i + \varepsilon \cdot \mathrm{Wealth}_i \cdot \fuz_i (\pt_{\deff(i+1)}(\aff_{\deff(i)}) - \pt_{\deff(i)}(\aff_{\deff(i)})) \right)
  \\
  &= 
  \log \mathrm{Wealth}_i + \log\left(1 + \varepsilon \fuz_j (\pt_{\deff(j+1)}(\aff_{\deff(j)}) - \pt_{\deff(j)}(\aff_{\deff(j)}))\right)
  \\
  &=
  \sum_{j\le i} \log \left(1 + \varepsilon \fuz_j (\pt_{\deff(j+1)}(\aff_{\deff(j)}) - \pt_{\deff(j)}(\aff_{\deff(j)}))\right)
\end{align*}

For $|x| \leq 1/2$, we have $\log (1 + x) \geq x - x^2$.  Therefore, since $\varepsilon < 1/4$, $\fuz_j \leq 1$, and $|\pt_{\deff(j+1)}(\aff_{\deff(j)}) - \pt_{\deff(i)}(\aff_{\deff(j)})| \leq 1$,

\begin{align*}
  \log \mathrm{Wealth}_i &\geq 
  \sum_{j\le i-1} \Big(
    \varepsilon \fuz_j (\pt_{\deff(j+1)}(\aff_{\deff(j)}) - \pt_{\deff(j)}(\aff_{\deff(j)})) \\[-1em]
    & \phantom{xxxxxxxx} - \varepsilon^2 \fuz_j^2 (\pt_{\deff(j+1)}(\aff_{\deff(j)}) - \pt_{\deff(j)}(\aff_{\deff(j)}))^2 
  \Big)
  \\
  &\geq
  \sum_{j\le i-1} \left( \varepsilon \fuz_j (\pt_{\deff(j+1)}(\aff_{\deff(j)}) - \pt_{\deff(j)}(\aff_{\deff(j)})) - \varepsilon^2\fuz_j \right)
  \\
  &=
  \sum_{j\le i-1} \varepsilon \fuz_j\left( (\pt_{\deff(j+1)}(\aff_{\deff(j)}) - \pt_{\deff(j)}(\aff_{\deff(j)})) - \varepsilon \right)
\end{align*}

At this point, it will be useful to show a relation between $\pt_{\deff(j+1)}(\aff_{\deff(j)})$ and $\thmval(\aff_{\deff(j)})$.  Consider the sequence
\[ \aff'_\nn := \begin{cases}
  \aff_{\deff(i)} - \thmval(\aff_{\deff(i)})& \text{if $\exists i: \deff(i+1) = \nn$}
  \\
  0 & \text{otherwise}
\end{cases}
\]
 which is in $\BCS$ because $\thmval(\aff_{\deff(j)})$ is computable in time polynomial in $\deff(j+1)$.
Of course, each $\aff'_\nn$ has value 0 in any world $\World \in \cworlds(\Theory)$.
So by \Theorem{affprovind},
\[ \pt_\nn(\aff'_\nn) \eqsim_\nn 0 \]
so for all sufficiently high $j$, $\pt_{\deff(j+1)}(\aff_{\deff(j)}) \geq \thmval(\aff_{\deff(j)})-\varepsilon$.  Thus, for some constant~$C$,
\begin{align*}
  \log \mathrm{Wealth}_i 
  &\geq 
  \sum_{j\le i-1} \varepsilon \fuz_j \left( (\thmval(\aff_{\deff(j)}) - \pt_{\deff(j)}(\aff_{\deff(j)})) - 2 \varepsilon \right) - C
  \\
  &=
  \varepsilon \left( \sum_{j\le i-1} \fuz_j \right) (-\mathrm{Bias}_{i-1} - 2 \varepsilon) - C
\end{align*}

Now, for infinitely many $i$ it is the case that $\mathrm{Bias}_{i-1} < - 3 \varepsilon$.  So it is infinitely often the case that $\log \mathrm{Wealth}_i \geq \varepsilon^2 \sum_{j\le i - 1} \fuz_i - C$.  Since $\seq\fuz$ is divergent, $T$'s eventual wealth (and therefore max profit) can be arbitrarily high.  Thus, $T$ exploits $\MP$.
\end{proof}

\subsection{Unbiasedness From Feedback}\label{app:wub}

\restatewub*

\begin{proof}
  This is a special case of \Theorem{wubaff}.
\end{proof}

\subsection{Learning Pseudorandom Affine Sequences}\label{app:prandaff}

\prandaff*

\begin{proof}

  We will prove the statement in the case of $\gtrsim_n$; the case of $\lesssim_n$ follows by negating the $\RR$-combination sequence $\seq\aff$, and the case of $\eqsim_n$ follows from the conjunction of the other two cases.
Suppose it is not the case that $\pt_\nn(\aff_\nn) \gtrsim_\nn 0$.
Then there is a rational $\varepsilon > 0$ such that $\pt_\nn(\aff_\nn) < - 2 \varepsilon$ infinitely often.

\newcommand{\DefinitelySettled}{\mathrm{DefinitelySettled}}

\paragraph{Defining the trader.}
  Let $\seq{\gen\aff}$ be an \ec $\exfeatures$-combination 
   progression such that $\gen{\aff_n} \ftn = \aff_n$.
Let an affine combination $\aff$ be considered \emph{settled} by day $m$ if $\World(\aff) = \thmval(\aff)$ for each $\World \in \pcworlds(\dt_m)$.
We may write $\mathrm{Settled}(\nn, m)$ to be the proposition that $\aff_\nn$ is settled by day $\mm$.
$\mathrm{Settled}(\nn, m)$ is decidable; let $\texttt{settled}$ be a Turing machine deciding $\mathrm{Settled}(\nn, m)$ given $(\nn, m)$.
Now we define a lower-approximation to $\mathrm{Settled}$:

\[ \DefinitelySettled(\nn, m) :\leftrightarrow \exists i \leq m: \texttt{settled}(\nn, i) \text{~returns true within $m$ steps}. \]

Note that
\begin{itemize}
\item $\DefinitelySettled(\nn, m)$ can be decided in time polynomial in $m$ when $\nn \leq m$,
\item $\DefinitelySettled(\nn, m) \rightarrow \mathrm{Settled}(\nn, m)$, and
\item If $\mathrm{Settled}(\nn, m)$, then $\DefinitelySettled(\nn, M)$ for some $M \geq m$.
\end{itemize}

To define the trader, first we will define $\seq{\gen{\alpha}}$ recursively by

\begin{align*}
  \gen{\alpha_\nn} &:= (1 - \gen{C_\nn} ) \ctsind{\varepsilon}(\aff_n^{\gens * n}  < - \varepsilon)
 \\
 \gen{C_\nn} &:= \sum_{i < \nn} [\neg\DefinitelySettled(i, \nn) \vee \deff(i) > \nn] \cdot \gen{\alpha_i}.
\end{align*}

The trader itself buys $\gen{\alpha_\nn} \ftn$ copies of the combination $\gen\aff_n \ftn$ on day $n$:

\[ T_\nn := \gen{\alpha_\nn} \cdot (\gen{\aff_\nn} - \aff_n^{\dagger * n}). \]

Intuitively, $\gen{C_\nn}$ is the total number of copies of $\RR$-combinations that the trader has bought that are either possibly-unsettled (according to $\DefinitelySettled$), or whose deferral time $\deff(i)$ is past the current time $\nn$. 

\paragraph{Analyzing the trader.}
As shorthand, define $\alpha_\nn := \gen{\alpha_\nn} \ftn$ and $C_n := \gen{C_n} \ftn$.
Some important properties of $\seq T$ are:
\begin{itemize}
\item Each $C_\nn \leq 1$.
\item $\alpha_\nn = 1 - C_\nn $ when $\pt_\nn(\aff_\nn) < - 2 \varepsilon$.
\item Whenever $\alpha_\nn > 0$, $\pt_\nn(\aff_\nn) < - \varepsilon$.

\item $\sum_{\nn\in {\NN^+}} \alpha_\nn =\infty$.  
  Suppose this sum were finite.  Then there is some time $N$ for which $\sum_{\nn \geq N} \alpha_\nn < 1/2$.  
  For some future time $N'$, and each $\nn < N$, we have $\DefinitelySettled(\nn, N') \wedge \deff(\nn) \leq N'$.  
  This implies that $C_\nn < 1/2$ for each $\nn \geq N'$.  
  However, consider the first $\nn \geq N'$ for which $\pt_\nn(\aff_\nn) < - 2 \varepsilon$.  
  Since $C_\nn < 1/2$, $\alpha_n \geq 1/2$.  But this contradicts $\sum_{n \geq N} \alpha_n < 1/2$.
\end{itemize}

Let $b \in \QQ$ be such that each $\magnit{\aff_n} < b$.  Consider this trader's profit at time $m$ in world $\World\in\pcworlds(\dt_\mm)$:

\begin{align*}
& \sum_{\nn \leq m} \alpha_\nn (\World(\aff_\nn) - \pt_\nn(\aff_\nn))
\\
\geq &
\sum_{\nn \leq m} \alpha_\nn  (\World(\aff_\nn) + \varepsilon)
\\
\geq &
\sum_{\nn \leq m} \alpha_\nn  (\thmval(\aff_\nn) + \varepsilon) - 2b
\end{align*}

where the last inequality follows because $\sum_{n \leq m} [ \neg \mathrm{Settled}(n, m) ] \alpha_n \leq 1$, and an unsettled copy of $\aff_n$ can only differ by $2b$ between worlds, while settled copies of $\aff_n$ must have the same value in all worlds in $\pcworlds(\dt_m)$.

$T$ holds no more than 1 copy of $\RR$-combinations unsettled by day $M$, and an unsettled combination's value can only differ by $2b$ between different worlds while settled affine combinations must have the same value in different worlds in $\pcworlds(\dt_m)$.  We now show that this quantity goes to $\infty$, using the fact that $\affs$ is pseudorandomly positive.

Observe that $\alphas$ is a divergent weighting.  It is also $\deff$-patient, since $\sum_{n \leq m} [\deff(n) \geq m] \alpha_n \leq 1$.   So by assumption,

\[ \lim\inf_{m\rightarrow\infty}
\frac{\sum_{\nn \leq m}\alpha_\nn \cdot \thmval(\aff_\nn)}
{\sum_{\nn\leq m}\alpha_\nn } \geq 0. \]

At this point, note that

\[ \sum_{\nn \leq m} \alpha_n (\thmval(\aff_n) + \varepsilon) = 
  \left(\sum_{n \leq m} \alpha_n\right) 
  \left( \frac{\sum_{\nn \leq m}\alpha_\nn \cdot \thmval(\aff_\nn)}{\sum_{\nn\leq m}\alpha_\nn }
    + \varepsilon
  \right). \]

  For all sufficiently high $m$, $\frac{\sum_{\nn \leq m}\alpha_\nn \cdot \thmval(\aff_\nn)}{\sum_{\nn\leq m}\alpha_\nn } \geq -\varepsilon/2$, and $\sum_{n \in \NN^+} \alpha_n = \infty$, so

\[ \lim\inf_{m\rightarrow\infty} \sum_{\nn \leq m} \alpha_\nn (\thmval(\aff_\nn) + \varepsilon) =\infty.\]

If we define $g(m)$ to be the minimum plausible worth at time $m$ over plausible worlds $\World\in\pcworlds(\dt_m)$, we see that $g(m)$ limits to infinity, implying that the trader's maximum worth goes to infinity.  The fact that $g(m)$ limits to infinity also implies that $g(m)$ is bounded from below, so the trader's minimum worth is bounded from below.  Thus, this trader exploits the market $\MP$.

\end{proof}

\subsection{Learning Varied Pseudorandom Frequencies}\label{app:seqprand}\label{app:prand}

\seqprand*

\prand*

\begin{proof}

  We will prove the statement in the case of pseudorandom above; the case of pseudorandom below is analogous, and the case of pseudorandom follows from the other cases.  

  Define $\aff_n := \phi_n - p_n$ and note that $\seq\aff \in \BCS$.  Observe that, because $\seq \aff$ is varied pseudorandomness above $\seq p$, for any $\deff$-patient divergent weighting $\fuz$,
  \[ 
    \frac{\sum_{i \leq n} \fuz_i \thmval{\aff_n}}{\sum_{i \leq n} \fuz_i} \gtrsim_n 0.
  \]
  Now apply \Theorem{prandaff} to get

  \[ \pt_n(\aff_n) = \pt_n(\phi_n) - p_n \gtrsim_n 0. \]

\end{proof}

\subsection{Learning Pseudorandom Frequencies}\label{app:pseudorandom}\label{app:benford}

\pseudorandom*

\benford*

\begin{proof}
  Let $q$ be any rational number less than $p$.  Note that $\phis$ is varied pseudorandom above $q$, so by \Theorem{prand},
  \[ \pt_n(\phi_n) \gtrsim_n q. \]
  But we could have chosen any rational $q < p$, so $\pt_n(\phi_n) \gtrsim_n p$.  An analogous argument shows $\pt_n(\phi_n) \lesssim_n p$.
\end{proof}

\section{Expectations Proofs}\label{app:expectationsproofs}

\subsection{Consistent World LUV Approximation Lemma}\label{app:conluvapprox}

\begin{lemma}\label{lem:conluvapprox}
  Let $\seq\affluv \in \BLCS $ be a $\RR$-LUV combination bounded by some rational number $b$.
  For all natural numbers $n$ and all $\World \in \cworlds(\Theory)$, we have
  \[ | \EE_n^\World(\affluv) - \World(\affluv)| \leq b/n. \]
\end{lemma}
\begin{proof}

  Let $\World \in \cworlds(\Theory)$.  For any $[0, 1]$-LUV X, by \Def{e},
  \[
    |\EE_n^\World(X) - \World(X)| = \left|\sum_{i=0}^{n-1}\frac{1}{n} \World\mleft(\quot{\enc{X} > \enc{i}/\enc{n}} \mright) - \World(X) \right|
  \]
  Since $\Theory$ \representscomputations, the number of $i$ values in $\{0, \ldots, n-1\}$ for which $\World(\quot{\enc{X} > \enc{i}/\enc{n}}) = 1$ is at least $\lfloor n \World(X) \rfloor \geq n\World(X) - 1$, so
  \[
    \sum_{i=0}^{n-1}\frac{1}{n} \World\mleft(\quot{\enc{X} > \enc{i}/\enc{n}} \mright) \geq \World(X) - 1/n.
  \]
  Similarly, the number of $i$ values in $\{0, \ldots, n-1\}$ for which $\World(\quot{\enc{X} > \enc{i}/\enc{n}})$ is no more than $\lceil n\World(X)  \rceil \leq n\World(X) + 1$, so
  \[
    \sum_{i=0}^{n-1}\frac{1}{n} \World\mleft(\quot{\enc{X} > \enc{i}/\enc{n}} \mright) \leq \World(X) + 1/n.
  \]
  We now have
  \begin{align*}
    |\EE_n^\World(\affluv) - \World(\affluv)| &= |c_1 (\EE_n^\World(X_1) - \World(X_1)) + \cdots + c_k (\EE_n^\World(X_k) - \World(X_k))| 
    \\
    &\leq c_1 |\EE_n^\World(X_1) - \World(X_1)| + \cdots + c_k |\EE_n^\World(X_k) - \World(X_k)|
    \\
    &\leq c_1 /n + \cdots + c_k/n 
    \\
    &\leq b/n.
  \end{align*}
\end{proof}

\subsection{Mesh Independence Lemma}\label{app:mesh}

\begin{lemma}\label{lem:mesh}
  Let $\seq\affluv \in \BLCS$.  Then
  \[ \lim_{n \to \infty} \sup_{m \geq n} |\EE_n^{\pt_m}(\affluv_n) - \EE_m(\affluv_n)| = 0. \]
\end{lemma}
\begin{proof}

  We will  prove the claim that
\[ \limsup_{m \to \infty} \max_{n \leq m}\left(  |\EE_n^{\pt_m}(\affluv_n) - \EE_m(\affluv_n)| -(2/n) \right) \leq 0 .  \]
  This claim implies that, for any $\varepsilon > 0$, there are only finitely many $(n, m)$ with $n \leq m$ such that
  $|\EE_n^{\pt_m}(\affluv_n) - \EE_m(\affluv_n)| > 2/n + \varepsilon$, which in turn implies that, for any $\varepsilon' > 0$,
  there are only finitely many $(n, m)$ with $n \leq m$ such that $|\EE_n^{\pt_m}(\affluv_n) - \EE_m(\affluv_n)| > \varepsilon'$.
  This is sufficient to show the statement of the theorem.

  We will now prove
\[ \limsup_{m \to \infty} \max_{n \leq m}\left(  \EE_n^{\pt_m}(\affluv_n) - \EE_m(\affluv_n) -(2/n) \right) \leq 0 .  \]
The proof with $\EE_m(\affluv_n) - \EE_n^{\pt_m}(\affluv_n) $ instead is analogous, and together these inequalities prove the claim. 
  
  Suppose this inequality does not hold.
  Then there is some rational $\varepsilon > 0$ such that for infinitely many $m$,
  \[ \max_{n \leq m} (\EE_n^{\pt_m}(\affluv_n) - \EE_m(\affluv_n) - (2/n)) >  \varepsilon. \]

  Let $\seq{\gen\affluv}$ be an \ec $\exfeatures$-combination progression such that $\gen{\affluv_n} \ftn = \affluv_n$.
  Assume without loss of generality that each $\|\affluv_n\|_1 \leq 1$ (they are assumed to be bounded and can be scaled down appropriately).
  Define $\exfeatures$-combinations
  \[
    \gen{A_{n,m}} := \alta_n(\gen{\affluv_n}) - \alta_m(\gen{\affluv_n}) - 2/n, 
  \]
  using the $\features$-combinations $\alta$ defined in \ref{app:featureluvcombos}.
  As shorthand, we write $A_{n,m} := \gen{A_{n,m}} \ftn$.  By \Lem{conluvapprox}, for all $\World \in \cworlds(\Theory)$, $\World(A_{n, m}) \leq 0$.  
  We aim to show $\pt_m(A_{n, m}) < \varepsilon$ for all sufficiently high $m$ and $n \leq m$, but we cannot immediately
  derive this using \Theorem{affprovind}, since $A$ has two indices.  We get around this difficulty
  by taking a ``softmax'' over possible values of $n$ given a fixed value of $m$.  Specifically, for $n \leq m$, define
  expressible features (of rank $m$)
  \[
    \gen{\alpha_{n,m}} := \ctsind{\varepsilon/2}\left(A_{n,m}^{\gens * m} > \varepsilon/2\right) \cdot \left(1 - \sum_{i < n} \gen{\alpha_{i,m}}\right).
  \]
  As shorthand, we write $\alpha_{n, m} := \gen{\alpha_{n,m}} \ftn$.  Intuitively, $\alpha_{n,m}$ will distribute weight among $n$ values for which $A_{m,n}$ is overpriced at time $m$.
  Now we define the $\exfeatures$-combination progression
  \[
    \gen{G_m} := \sum_{n \leq m} \alpha_{n,m} \cdot \gen{A_{n,m}}.
  \]
  As shorthand, we write $G_{m} := \gen{G_m} \ftn$.
  Fix $m$ and suppose that $\pt_m(A_{n,m}) \geq \varepsilon$ for some $n \leq m$.  Then $\sum_{n \leq m} \alpha_{n, m} = 1$.
  Therefore, 
  \[ \pt_m(G_m) = \sum_{n \leq m} \alpha_{n, m} \pt_m(A_{n,m}) \geq \sum_{n \leq m} \alpha_{n,m} \cdot \varepsilon / 2 = \varepsilon / 2. \]

  So if we can show $\pt_m(G_m) \lesssim_m 0$, that will be sufficient to show that $\max_{n \leq m} \pt_m(A_{n,m}) < \varepsilon$ for
  all sufficiently high $m$.  We now show this.
  Let $\World \in \cworlds(\Theory)$.
  Since each $\alpha_{n,m} \geq 0$, and $\World(A_{n,m}) \leq 0$, we have $\World(G_m) \leq 0$.  So by \Theorem{affprovind}, $\pt_m(G_m) \lesssim_m 0$; here we use that $(G_m)_m$ is bounded, since the $A_{n,m}$ are bounded and since for each $m$, $\sum_{n \leq m} \alpha_{n, m} \leq  1$ by construction.

  So for all sufficiently high $m$ we have $\max_{n \leq m} \pt_m(A_{n,m}) < \varepsilon$ (or equivalently, $\max_{n \leq m} (\EE_n^{\pt_m}(\affluv_n) - \EE_m(\affluv_n)) < 2/n + \varepsilon$).  
  But this contradicts our assumption that for infinitely many $m$,

\[ \max_{n \leq m} (\EE_n^{\pt_m}(\affluv_n) - \EE_m(\affluv_n)-(2/n)) >  \varepsilon. \]

\end{proof}

\subsection{Expectation Preemptive Learning}\label{app:exppolymax}

\exppolymax*
\begin{proof}
  We prove only the first statement; the proof of the second statement is analogous.
  Apply \Theorem{affpolymax} to the bounded sequence $(\alta_n(\affluv_\nn))_n$ to get
  \[
    \liminf_{\nn\to\infty} \EE_n^{\pt_\nn}(\affluv_\nn) = \liminf_{\nn\rightarrow\infty}\sup_{\mm\geq \nn} \EE_n^{\pt_m}(\affluv_n),
  \]
  using that by definition $\pt_m (\alta_n(\affluv_\nn)) = \EE_n^{\pt_m}(\affluv_n)$.
  By \Lem{mesh},
  \[
    \lim_{\nn\rightarrow\infty}\sup_{\mm\geq \nn} |\EE_n^{\pt_m}(\affluv_n) - \EE_m(\affluv_n)| = 0
  \]
  so
  \[
\liminf_{\nn\rightarrow\infty}\sup_{\mm\geq \nn} \EE_n^{\pt_m}(\affluv_n) =
\liminf_{\nn\rightarrow\infty}\sup_{\mm\geq \nn} \EE_m(\affluv_n).
  \]
\end{proof}

\subsection{Expectations Converge}\label{sec:exc}\label{app:ec}

\restateec*
\begin{proof}
  By applying \Theorem{exppolymax} to the constant sequence $X, X, \ldots$, we have

  \[ \liminf_{n \to \infty} \EE_n(X) = \liminf_{n \to \infty} \sup_{m \geq n} \EE_m(X) = \limsup_{n \to \infty} \EE_n(X) . \]
\end{proof}

\subsection{Limiting Expectation Approximation Lemma}\label{app:limexpapprox}

\begin{lemma}\label{lem:limexpapprox}
  For any $\seq\affluv \in \BLCS$,
  \[ | \EE_n^{\pt_\infty}(\affluv_n) - \EE_\infty(\affluv_n)| \eqsim_n 0.\]
\end{lemma}
\begin{proof}
  By \Lem{mesh} and by continuity of $\Valuation \mapsto \EE_n^{\Valuation}(\affluv_n)$,
  \begin{align*}
   \lim_{n \to \infty} | \EE_n^{\pt_\infty}(\affluv_n) - \EE_\infty(\affluv_n)| 
    &= \lim_{n \to \infty} \lim_{m \to \infty} | \EE_n^{\pt_m}(\affluv_n) - \EE_m(\affluv_n)| \\
        &\leq \lim_{n \to \infty} \sup_{m \geq n} | \EE_n^{\pt_m}(\affluv_n) - \EE_m(\affluv_n)| \\
  &= 0 .
\end{align*}
\end{proof}

\subsection{Persistence of Expectation Knowledge}\label{app:perexpkno}

\perexpkno*

\begin{proof}
  We prove only the first statement; the proof of the second statement is analogous.  Apply \Theorem{peraffkno} to $(\alta_n(\affluv_\nn))_n$ to get
  \[
    \liminf_{\nn\rightarrow\infty}\inf_{\mm\geq \nn}\EE_n^{\pt_\mm}(\affluv_\nn )= \liminf_{\nn\to\infty}\EE_n^{\pt_\infty}(\affluv_\nn ).
  \]
  We now show equalities on these two terms:
  \begin{enumerate}
    \item
  By \Lem{mesh},
  \[
    \lim_{\nn\rightarrow\infty}\sup_{\mm\geq \nn}|\EE_n^{\pt_\mm}(\affluv_\nn ) - \EE_m(\affluv_n)| = 0
  \]
  so
  \[ 
    \liminf_{\nn\rightarrow\infty}\inf_{\mm\geq \nn}\EE_m(\affluv_\nn )= 
    \liminf_{\nn\rightarrow\infty}\inf_{\mm\geq \nn}\EE_n^{\pt_\mm}(\affluv_\nn ).
  \]
\item
  By \Lem{limexpapprox},
  \[
    \liminf_{\nn\to\infty}\EE_n^{\pt_\infty}(\affluv_\nn ) = \liminf_{n \to \infty} \EE_\infty(\affluv_n).
  \]
  \end{enumerate}
  Together, these three equalities prove the theorem statement.
\end{proof}

\subsection{Expectation Coherence}\label{app:expcoh}

\expcoh*

\begin{proof}
  We prove only the first statement; the proof of the second statement is analogous.
  Apply \Theorem{affcoh} to $(\alta_n(\affluv_\nn))_n$ to get
  \[
    \liminf_{\nn\rightarrow\infty}\inf_{\World\in\cworlds(\Theory)}
      \EE_n^\World(\affluv_\nn )
    \le 
    \liminf_{\nn\rightarrow\infty}
    \EE_n^{\pt_\infty}(\affluv_\nn )
    \le \liminf_{\nn\to\infty}
    \EE_n(\affluv_\nn ),
  \]
  We now show equalities on the first two terms:
  \begin{enumerate}
    \item
      Let $b$ be the bound on $\seq\affluv$.  By \Lem{conluvapprox},
      \[
      \liminf_{\nn\rightarrow\infty}\inf_{\World\in\cworlds(\Theory)} |\EE_n^\World(\affluv_\nn ) - \World(\affluv_n)|
      \leq \liminf_{n \to \infty} \inf_{\World \in \cworlds(\Theory)} b/n = 0
      \]
      so
      \[
      \liminf_{\nn\rightarrow\infty}\inf_{\World\in\cworlds(\Theory)} \EE_n^\World(\affluv_\nn ) 
      =
      \liminf_{\nn\rightarrow\infty}\inf_{\World\in\cworlds(\Theory)} \World(\affluv_\nn ).
    \]
    \item
      By \Lem{limexpapprox},
      \[
      \liminf_{\nn\rightarrow\infty}
      |\EE_n^{\pt_\infty}(\affluv_\nn ) - \EE_\infty(\affluv_n)| = 0
      \]
      so
      \[
      \liminf_{\nn\rightarrow\infty}\EE_n^{\pt_\infty}(\affluv_\nn )
      =
      \liminf_{\nn\rightarrow\infty}\EE_\infty(\affluv_\nn ).
      \]
  \end{enumerate}
  Together, these three equalities prove the theorem statement.

\end{proof}

\subsection{Expectation Provability Induction}\label{app:expprovind}

\expprovind*

\begin{proof}

  We prove the statement in the case of $\geq$; the case of $\leq$ is analogous, and the case of $=$ follows from the conjunction of the other two cases.  By \Theorem{expcoh},

  \[ \liminf_{n \to \infty} \EE_n(\affluv_n) \geq \liminf_{n \to \infty} \inf_{\World \in \cworlds(\Theory)} \World(\affluv_n) \geq b. \]
    We will usually apply this theorem using the $=$ case.

\end{proof}

\subsection{Linearity of Expectation}\label{app:loe}

\restateloe*

\begin{proof}
  Observe that $\World(a_n X_n + b_n Y_n - Z_n) = 0$ for all $n$ and $\World \in \pcworlds(\Theory)$.  So by \Theorem{expprovind}, $\EE_n(a_n X_n + b_n Y_n - Z_n) \eqsim_n 0$; the theorem statement immediately follows from the definition of $\EE_n$ applied to a LUV-combination (where $a_n X_n + b_n Y_n - Z_n$ is interpreted as a LUV-combination, not another LUV).
\end{proof}

\subsection{Expectations of Indicators}\label{app:ei}
\restateei*
\begin{proof}
  Observe that $\World(\alta_n(\sentind{\phi_n})) = \World(\phi_n)$ for all $\World \in \cworlds(\Theory)$; either 
  \begin{itemize}
    \item $\World(\phi) = 0$ and $\World(\alta_n(\sentind{\phi_n})) = \sum_{i=0}^{n-1} \frac{1}{n} \World(\quot{\enc{\sentind{\phi_n}} > \enc{i}/\enc{n}}) = \sum_{i=0}^{n-1} 0 = 0$, or
    \item $\World(\phi) = 1$ and $\World(\alta_n(\sentind{\phi_n})) = \sum_{i=0}^{n-1} \frac{1}{n} \World(\quot{\enc{\sentind{\phi_n}} > \enc{i}/\enc{n}}) = \sum_{i=0}^{n-1} \frac{1}{n} = 1$.
  \end{itemize}
  So by \Theorem{affprovind},
  \[ \EE_n(\sentind{\phi_n}) \eqsim_n \pt_n(\phi_n). \]
\end{proof}

\subsection{Expectation Recurring Unbiasedness}\label{app:recurringunbiasednessexp}

\recurringunbiasednessexp*

\begin{proof}

  Let $\World \in \cworlds(\Gamma)$.  Apply \Theorem{recunbiasedaff} to $(\alta_n(\affluv_n))_n$ and $\seq\fuz$ to get that

  \[ \left( \frac{\sum_{i \leq n} \fuz_i (\EE_i(\affluv_i) - \EE_i^{\World}(\affluv_i))}{\sum_{i \leq n}\fuz_i} \right)_{n \in \NN^+}  \]

  has 0 as a limit point.  Furthermore, by \Lem{conluvapprox},
  $|\EE_i^\World(\affluv_i) - \thmval(\affluv_i)| \leq b/i$
  where $b$ is a bound on $\seq\affluv$.  As a result, for any subsequence of
  \[ \left( \frac{\sum_{i \leq n} \fuz_i (\EE_i(\affluv_i) - \EE_i^{\World}(\affluv_i))}{\sum_{i \leq n}\fuz_i} \right)_{n \in \NN^+} \]
  that limits to zero, the corresponding subsequence of
  \[ \left( \frac{\sum_{i \leq n} \fuz_i (\EE_i(\affluv_i) - \thmval(\affluv_i))}{\sum_{i \leq n}\fuz_i} \right)_{n \in \NN^+} \]
  also limits to zero, as desired.

\end{proof}

\subsection{Expectation Unbiasedness From Feedback}\label{app:wubexp}

\wubexp*

\begin{proof}
  Let $\World \in \cworlds(\Theory)$.  Note that if $\thmval{\affluv_n}$ can be computed in time polynomial in $g(n + 1)$, then so can $\thmval{\alta_k(\affluv_n)}$.  Apply \Theorem{wubaff} to $(\alta_n(\affluv_n))_n$ to get
  \[
    \frac
        {\sum_{i\leq \nn}\fuz_i  \cdot(\EE_i(\affluv_i)-\EE_i^\World(\affluv_i))}
        {\sum_{i\leq \nn}\fuz_i}
    \eqsim_\nn 0.
  \]
  Furthermore, by \Lem{conluvapprox},
  $|\EE_i^\World(\affluv_n) - \thmval(\affluv_n)| \leq b/n$
  where $b$ is a bound on $\seq\affluv$.  As a result,
  \[
    \frac
        {\sum_{i\leq \nn}\fuz_i  \cdot(\EE_i(\affluv_i)-\thmval(\affluv_i))}
        {\sum_{i\leq \nn}\fuz_i}
    \eqsim_\nn 0.
  \]
  as desired.
\end{proof}

\subsection{Learning Pseudorandom LUV Sequences}\label{app:seqprandexp}\label{app:prandexp}

\prandexp*

\begin{proof}
  We will prove the statement in the case of $\gtrsim$; the case of $\lesssim$ is analogous, and the case of $\eqsim$ follows from
  the other cases.  
  
  Let $b$ be the bound of $\seq\affluv$.  Let $\World \in \cworlds(\Theory)$.
  First, note that by \Lem{conluvapprox},
  $|\EE_i^\World(\affluv_i) - \thmval(\affluv_i)| \leq b/i$.
  Therefore,
  \[
    \frac{\sum_{i \leq n} \fuz_i \cdot \EE_i^\World(\affluv_i)}{\sum_{i \leq n} \fuz_i} \eqsim_n \frac{ \sum_{i \leq n} \fuz_i \cdot \thmval(\affluv_i)}{\sum_{i \leq n} \fuz_i} \gtrsim_n 0.
  \]
  So we may apply \Theorem{prandaff} to $(\alta_n(\affluv_n))_n$ to get
  \[
    \EE_n(\affluv_\nn) \gtrsim_n 0.
  \]
\end{proof}

\section{Introspection and Self-Trust Proofs}\label{app:introspectiontrust}

\subsection{Introspection}\label{app:ref}
\restateref*
\begin{proof}
  Define $\psi_n := \quot{\enc{a_\nn} < \enc{\pt}_\enc{\nn}(\enc{\phi_\nn}) < \enc{b_\nn}}$.

  \paragraph{Proof of the first statement.}

  Observe that for all $n$, and all $\World\in\cworlds(\Theory)$, 

  \[ \ctsind{\delta_n}(a_n < \pt_n(\phi_n) < b_n) \cdot (1 - \World(\psi_n))=0, \]
since regardless of $\pt_n(\phi_n)$, one of the two factors is 0.
Thus, applying \Theorem{affprovind} gives 
\[ \ctsind{\delta_n}(a_n < \pt_n(\phi_n) < b_n) \cdot (1-\pt_n(\psi_n)) \eqsim_n 0. \]
Define 
\[ \varepsilon_n := \ctsind{\delta_n}(a_n < \pt_n(\phi_n) < b_n) \cdot (1-\pt_n(\psi_n)) + 1/n \] 
and note that $\varepsilon_n > 0$ and $\varepsilon_n \eqsim_n 0$.  For any $n$ for which $\pt_n(\phi_n) \in (a_n+\delta_n,b_n-\delta_n)$, the first factor is 1, so $\pt_n(\psi_n) = 1 - \varepsilon_n+1/n > 1 - \varepsilon_n$.

\paragraph{Proof of the second statement.}
  Observe that for all $n$, and all $\World\in\cworlds(\Theory)$, 
\[ \left(\ctsind{\delta_n}(\pt_n(\phi_n)<a_n)+\ctsind{\delta_n}(\pt_n(\phi_n) > b_n)\right)\cdot \World(\psi_n)=0, \]
since regardless of $\pt_n(\phi_n)$, one of the factors is 0.
Thus, applying \Theorem{affprovind} gives 
\[ \left(\ctsind{\delta_n}(\pt_n(\phi_n)<a_n)+\ctsind{\delta_n}(\pt_n(\phi_n) > b_n)\right)\cdot \pt_n(\psi_n) \eqsim_n 0. \]
Define 
\[\varepsilon_n := \left(\ctsind{\delta_n}(\pt_n(\phi_n)<a_n)+\ctsind{\delta_n}(\pt_n(\phi_n) > b_n)\right)\cdot \pt_n(\psi_n) + 1/n\] 
and note that $\varepsilon_n > 0$ and $\varepsilon_n \eqsim_n 0$.  For any $n$ for which $\pt_n(\phi_n)\notin (a_n-\delta_n,b_n+\delta_n)$, the first factor is 1, so $\pt_n(\psi_n) < \varepsilon_n$.

\end{proof}

\subsection{Paradox Resistance}\label{app:lp}

\restatelp*

\begin{proof}

  We prove $\pt_n(\phi_n) \gtrsim_n p$ and $\pt_n(\phi_n) \lesssim_n p$ individually.
  \begin{enumerate}
    \item
      Suppose it is not the case that $\pt_n(\phi_n) \gtrsim_n p$, so $\pt_n(\phi_n) < p - \varepsilon$ infinitely often for some $\varepsilon > 0$.
      Observe that for all $n$, and all $\World \in \cworlds(\Theory)$,
      \[ \ctsind{1/n}(\pt_n(\phi_n) < p) \cdot (1 - \World(\phi_n)) = 0, \]
      since regardless of $\pt_n(\phi_n)$, one of the factors is 0.  Thus, applying \Theorem{affprovind} yields
      \begin{equation}\label{eq:lp1}
        \ctsind{1/n}(\pt_n(\phi_n) < p) \cdot (1 - \pt_n(\phi_n)) \eqsim_n 0.
      \end{equation}
      But infinitely often,
      \[ \ctsind{1/n}(\pt_n(\phi_n) < p) \cdot (1 - \pt_n(\phi_n)) \geq 1 \cdot (1 - (p - \varepsilon)) \geq \varepsilon \]
      which contradicts \Eqn{lp1}.
    \item
      Suppose it is not the case that $\pt_n(\phi_n) \lesssim_n p$, so $\pt_n(\phi_n) > p + \varepsilon$ infinitely often for some $\varepsilon > 0$.
      Observe that for all $n$, and all $\World \in \cworlds(\Theory)$,
      \[ \ctsind{1/n}(\pt_n(\phi_n) > p) \cdot \World(\phi_n) = 0, \]
      since regardless of $\pt_n(\phi_n)$, one of the factors is 0.  Thus, applying \Theorem{affprovind} yields
      \begin{equation}\label{eq:lp2}
        \ctsind{1/n}(\pt_n(\phi_n) > p) \cdot \pt_n(\phi_n) \eqsim_n 0.
      \end{equation}
      But infinitely often,
      \[ \ctsind{1/n}(\pt_n(\phi_n) > p) \cdot \pt_n(\phi_n) \geq 1 \cdot (p + \varepsilon) \geq \varepsilon \]
      which contradicts \Eqn{lp2}.
  \end{enumerate}

\end{proof}

\subsection{Expectations of Probabilities}\label{app:epr}

\restateepr*

\begin{proof}

  Observe that for all $n$, and for all $\World \in \cworlds(\Theory)$, $\World(\pt_n(\phi_n) - \quot{\enc{\pt}_\enc{n}(\enc{\phi_n})}) = 0$ (where $\pt_n(\phi_n)$ is a number and $ \quot{\enc{\pt}_\enc{n}(\enc{\phi_n})})$ is a LUV).  Thus, by \Theorem{expprovind},

  \[ \pt_n(\phi_n) - \EE_n(\quot{\enc{\pt}_\enc{n}(\enc{\phi_n})}) \eqsim_n 0. \]
\end{proof}

\subsection{Iterated Expectations}\label{app:er}

\restateer*

\begin{proof}
  Observe that for all $n$, and for all $\World \in \cworlds(\Theory)$, $\World(\EE_n(X_n)- \quot{\enc{\EE}_\enc{n}(\enc{X_n})}) = 0$ (where $\EE_n(X_n)$ is a number and $ \quot{\enc{\EE}_\enc{n}(\enc{X_n})})$ is a LUV).  Thus, by \Theorem{expprovind},

  \[ \EE_n(X_n) - \EE_n(\quot{\enc{\EE}_\enc{n}(\enc{X_n})}) \eqsim_n 0.\]

\end{proof}

\subsection{Expected Future Expectations}\label{app:cee}

\restatecee*

\begin{proof}

  Let $Y_m := X_n$ if $m = f(n)$ for some $n$, and $Y_m := \quot{0}$ otherwise.  Observe that $(Y_m)_m$ is \ec[].  By \Theorem{er},

  \[ \EE_{f(n)}(X_n) \eqsim_n \EE_{f(n)}( \quot{ \enc{\EE}_{\enc{f(n)} } (\enc{X_n})}  ). \]

  We now manipulate the encodings $\enc{f(n)}$ (for the number $f(n)$) and $\enc{f}(\enc{n})$ (for the program computing $f$ and its input $n$). Observe than for all $\World \in \cworlds(\Theory)$, 

  \[ \World ( \quot{ \enc{\EE}_{\enc{f(n)} } (\enc{X_n})}  ) = 
  \World ( \quot{ \enc{\EE}_{\enc{f}(\enc{n}) } (\enc{X_n})}  ). \]

  So by \Theorem{expprovind},

  \[ \EE_{f(n)}(X_n) \eqsim_n \EE_{f(n)}( \quot{ \enc{\EE}_{\enc{f}(\enc{n}) } (\enc{X_n})}). \]

  By \Theorem{exppolymax},

  \[ \EE_n(X_n) \eqsim_n \EE_n(\quot{ \enc{\EE}_{\enc{f}(\enc{n}) } (\enc{X_n}) }).  \]
\end{proof}

\subsection{No Expected Net Update}\label{app:ceu}

\restateceu*

\begin{proof}

  Let $\psi_m := \phi_n$ if $m = f(n)$ for some $n$, and $\psi_m := \bot$ otherwise.  Observe that $(\psi_m)_m$ is \ec[].  By \Theorem{epr},

  \[ \pt_{f(n)}(\phi_n) \eqsim_n \EE_{f(n)}( \quot{ \enc{\pt}_{ \enc{f(n)} } (\enc{\phi_n})}  ). \]

We now manipulate the encodings $\enc{f(n)}$  and $\enc{f}(\enc{n})$. 
  Observe that for all $\World \in \cworlds(\Theory)$, 

  \[ \World ( \quot{ \enc{\pt}_{\enc{f(n)} } (\enc{\phi_n})}  ) = 
  \World ( \quot{ \enc{\pt}_{\enc{f}(\enc{n}) } (\enc{\phi_n})}  ). \]

  So by \Theorem{expprovind},

  \[ \pt_{f(n)}(\phi_n) \eqsim_n \EE_{f(n)}( \quot{ \enc{\pt}_{\enc{f}(\enc{n}) } (\enc{\phi_n})}  ). \]

  By \Theorem{ei},

  \[ \EE_{f(n)} (\sentind{\phi_n}) \eqsim_n \EE_{f(n)}( \quot{ \enc{\pt}_{\enc{f}(\enc{n}) } (\enc{\phi_n})}  ). \]

  By \Theorem{exppolymax},

  \[ \EE_n(\sentind{\phi_n})  \eqsim_n \EE_n( \quot{ \enc{\pt}_{\enc{f}(\enc{n}) } (\enc{\phi_n})}  ). \]

  By \Theorem{ei},

  \[ \pt_n(\phi_n) \eqsim_n \EE_n( \quot{ \enc{\pt}_{\enc{f}(\enc{n}) } (\enc{\phi_n})}  ). \]

\end{proof}

\subsection{No Expected Net Update under Conditionals}\label{app:ccee}

\restateccee*

\begin{proof}

  By \Theorem{er} and \Theorem{expprovind},
  \[ \EE_{f(n)}(X_n) \eqsim_n \EE_{f(n)}( \quot{ \enc{\EE}_{\enc{f(n) }} (\enc{X_n})} ) \eqsim_n \EE_{f(n)}( \quot{ \enc{\EE}_{\enc{f}(\enc{n}) } (\enc{X_n})}  ) \]
  and thus
  \[ \EE_{f(n)}(X_n) \cdot \fuz_{f(n)} \eqsim_n \EE_{f(n)}( \quot{ \enc{\EE}_{\enc{f}(\enc{n}) } (\enc{X_n})}  ) \cdot \fuz_{f(n)}. \]
  Observe that for all $n$, and for all $\World \in \cworlds(\Theory)$, 
  \[
    \World(X_n) \cdot \fuz_{f(n)} =  \World(\quot{ \enc{X_n} \cdot \enc{\fuz}_{ \enc{f}(\enc{n} )} }).
  \]
  So by \Theorem{expprovind},
  \[
  \EE_{f(n)}(X_n) \cdot \fuz_{f(n)}   
  \eqsim_n
  \EE_{f(n)}(\quot{ \enc{X_n} \cdot \enc{\fuz}_{ \enc{f}(\enc{n} )} }).
  \]
  Similarly, for all $n$ and all $\World \in \cworlds(\Theory)$,
  \[
    \World( \quot{ \enc{\EE}_{\enc{f}(\enc{n}) } (\enc{X_n})}) \cdot \fuz_{f(n)} \eqsim_n
    \World( \quot{ \enc{\EE}_{\enc{f}(\enc{n}) } (\enc{X_n}) \cdot \enc{\fuz}_{\enc{f}(\enc{n})} } ).
  \]
  So by \Theorem{expprovind},
  \[
    \EE_{f(n)}( \quot{ \enc{\EE}_{\enc{f}(\enc{n}) } (\enc{X_n})}) \cdot \fuz_{f(n)} \eqsim_n
    \EE_{f(n)}( \quot{ \enc{\EE}_{\enc{f}(\enc{n}) } (\enc{X_n}) \cdot \enc{\fuz}_{\enc{f}(\enc{n})} } ).
  \]
  Combining these,
  \[
    \EE_{f(n)}( \quot{ \enc{\EE}_{\enc{f}(\enc{n}) } (\enc{X_n}) \cdot \enc{\fuz}_{\enc{f}(\enc{n})}}  )
    \eqsim_n
    \EE_{f(n)}(\quot{ \enc{X_n} \cdot \enc{\fuz}_{ \enc{f}(\enc{n} )} }).
  \]
  So by \Theorem{exppolymax},
  \[
    \EE_{n}( \quot{ \enc{\EE}_{\enc{f}(\enc{n}) } (\enc{X_n}) \cdot \enc{\fuz}_{\enc{f}(\enc{n})} } )
    \eqsim_n
    \EE_{n}(\quot{ \enc{X_n} \cdot \enc{\fuz}_{ \enc{f}(\enc{n} )} }).
  \]
\end{proof}

\subsection{Self-Trust}\label{app:st}

\restatest*

\begin{proof}
  Define  $\alpha_n := \ctsind{\delta_n}(P_{f(n)}(\phi_n) > p_n)$.
  By \Theorem{epr},
  \[ \pt_{f(n)}(\phi_n) 
     \eqsim_n 
     \EE_{f(n)}(\quot{ \enc{\pt}_{\enc{f}(\enc{n}) } (\enc{\phi_n})} )
   \]
   and so
  \[ \pt_{f(n)}(\phi_n) \cdot \alpha_n
     \eqsim_n 
     \EE_{f(n)}(\quot{ \enc{\pt}_{\enc{f}(\enc{n}) } (\enc{\phi_n})} ) \cdot \alpha_n.
   \]
   Observe that for all $\World \in \cworlds(\Theory)$,
   \[
   \World(\sentind{\phi_n}) \cdot \alpha_n
   =
  \World(\quot{ \enc{\sentind{\phi_n}} \cdot \enc{\alpha}_{ \enc{n}} }).
  \]
   So by \Theorem{ei} and \Theorem{expprovind},
  \[
    \pt_{f(n)}(\phi_n) \cdot \alpha_n \eqsim_n \EE_{f(n)}(\sentind{\phi_n}) \cdot \alpha_n
  \eqsim_n
  \EE_{f(n)}(\quot{ \enc{\sentind{\phi_n}} \cdot \enc{\alpha}_{ \enc{n}} }).
  \]
  By two more similar applications of \Theorem{expprovind},
  \[ 
     \EE_{f(n)}(\quot{ \enc{\pt}_{\enc{f}(\enc{n}) } (\enc{\phi_n})} ) \cdot \alpha_n
     \eqsim_n
     \EE_{f(n)}(\quot{ \enc{\pt}_{\enc{f}(\enc{n}) } (\enc{\phi_n}) \cdot \enc{\alpha}_{\enc{n}}} )
     \gtrsim_n
     p_n \cdot \EE_{f(n)}(\quot{\enc{\alpha}_{\enc{n}}}).
   \]
   Combining these,
   \[
     \EE_{f(n)}(\quot{ \enc{\sentind{\phi_n}} \cdot \enc{\alpha}_{ \enc{n}} })
    \gtrsim_n
     p_n \cdot \EE_{f(n)}(\quot{\enc{\alpha}_{\enc{n}}}).
   \]
   By \Theorem{exppolymax},
   \[
     \EE_n(\quot{ \enc{\sentind{\phi_n}} \cdot \enc{\alpha}_{ \enc{n}} })
    \gtrsim_n
     p_n \cdot \EE_n(\quot{\enc{\alpha}_{\enc{n}}}).
   \]

\end{proof}

\section{Non-Dogmatism and Closure Proofs}\label{app:nondogproofs}

\subsection{Parametric Traders}\label{app:parametric}


Now we show that  there is no uniform strategy (i.e., efficiently emulatable sequence of traders) for taking on increasing finite amounts of possible reward with uniformly bounded possible losses.


\begin{lemma}[Parametric Traders]\label{lem:type2}
  Let $\MP$ be a logical inductor over $\DP$. Then there does not exist an efficiently emulatable sequence of traders $(\Trader^k)_k$ such that for all $k$, the set
  \[
    \left\{ \World\mleft({\textstyle \sum_{i \leq n} \trade^k_i\mleft(\MP\mright)}\mright) \,\middle|\, \nn \in \NN^+, \World \in
    \pcworlds(\dt_{\nn}) \right\}
  \]
of plausible values of $\Trader^k$'s holdings is bounded below by $-1$ and has supremum at least $k$. 
\end{lemma}
In words, this lemma states that if $\MP$ is a \li{} then there is no 
efficiently emulatable sequence of traders $(\Trader^k)_k$ such that each $\Trader^k$ never risks more than $\$ 1$, and exposes $\MP$ to at least $\$ k$ of
plausible risk. To show this lemma, roughly speaking we sum together scaled versions of some of the $\Trader^k$ so that the sum of their risks converges but
the set of their plausible profits diverges.  In this proof only  we will use the abbreviation $\itwoi{j} := j2^j$ for $j \in \NN^+$.  

\begin{proof}
  Suppose for contradiction that such a sequence $(\Trader^k)_k$  exists. Define a trader $\Trader$ by the formula
\[ \trade_n := \sum_{j: \itwoi{j} \leq n} \frac{\trade_n^{\itwoi{j}}}{2^j}. \]
This is well-defined as it is a finite sum of trading strategies, and it is efficiently computable in $n$ because $(\Trader^k)_k$  is efficiently emulatable. 
Then for any time $n$ and any world $\World \in \pcworlds(\dt_{n})$, 
\begin{align*}
\World\left({\textstyle \sum_{i \leq n} \trade_i(\MP)}\right) 
&= 
\World\left(\sum_{i \leq n} \sum_{j: \itwoi{j} \leq n} \frac{\trade_i^{\itwoi{j}} (\MP)}{2^j}\right) \\
\shortintertext{by definition of $\Trader$ and  since $\trade_n^{\itwoi{j}}\equiv 0 $ if $\itwoi{j} > n$;}
&= 
\sum_{j: \itwoi{j} \leq n}\frac{1}{2^j}\World\left(\textstyle{ \sum_{i \leq n}  \trade_i^{\itwoi{j}} (\MP)} \right) \\
\shortintertext{by linearity;}
&\geq \sum_{j\in \NN^+}\frac{1}{2^j} \cdot (-1) \\
&\geq -1 ,
\end{align*}
by the assumption that the plausible values $\World\left(\textstyle{ \sum_{i \leq n}  \trade_i^{\itwoi{j}} (\MP)} \right)$ are bounded below by $-1$.
Furthermore, for any $k \in \NN^+$, consider the trader $\Trader^{\itwoi{k}}$.  By assumption, for some time $n$ and some 
world $\World \in \pcworlds(\dt_{n})$, we have 
\[\World\left(\textstyle{ \sum_{i \leq n}  \trade_i^{\itwoi{k}} (\MP)} \right) \geq \itwoi{k} \equiv k 2^{k}. \]
Then, by the above analysis, we have
\begin{align*}
\World\left({\textstyle \sum_{i \leq n} \trade_i(\MP)}\right) 
&\geq  \frac{1}{2^{k}}\cdot \World\left(\textstyle{ \sum_{i \leq n}  \trade_i^{\itwoi{k}} (\MP)} \right)  + \sum_{j \in \NN^+}\frac{1}{2^j} \cdot (-1)\\
&\geq  \frac{k 2^{k}}{2^{k}} -1\\
&=k-1.
\end{align*}
Thus we have shown that the plausible values  
$\World\left({\textstyle \sum_{i \leq n} \trade_i(\MP)}\right) $ of our trader $\Trader$ are bounded below by $-1$ but unbounded above, i.e. 
$\Trader$ exploits the market $\MP$. This  contradicts that $\MP$ is a \li{}, showing that this sequence $(\Trader^k)_k$ cannot exist. 

\end{proof}


\subsection{Uniform Non-Dogmatism}\label{app:obu}
Recall Theorem~\ref{thm:obu}:

\restateunondog*

Roughly speaking, to show this, 
we will construct a parametric trader using Lemma~\ref{lem:type2}  by defining  an efficiently emulatable sequence $(\Trader^k)_k$ of traders.
Each trader $\Trader^k$ will attempt to ``defend'' the probabilities of the $\phi_i$ from dropping too far by buying no more than $k+1$ total shares in various
$\phi_i$ when they are priced below $1/(k+1)$.  If the property doesn't hold of $\MP$, then each $\Trader^k$ will buy a full $(k+1)$-many shares, at a total price of at most
$-1$. But since the $\phi_i$ are all collectively consistent, there is always a plausible world that values the holdings of $\Trader^k$ at no less than
$k+1-1=k$. Then the parametric trader that emulates  $(\Trader^k)_k$ will exploit $\MP$, contradicting that $\MP$ is a \li{}.

\begin{proof}
We can assume without loss of generality that for each $\phi_i$ that appears in the sequence of sentences $\phis$, that same sentence $\phi_i$ appears in $\phis$
infinitely often, by transforming the machine that enumerates $\phis$ into a machine that enumerates $(\phi_1,\phi_1,\phi_2,
\phi_1,\phi_2,\phi_3,\phi_1,\cdots)$. 
Futhermore,  we can assume that $\phis$ is efficiently computable by, if necessary, padding $\phis$ with copies of $\top$
while waiting for the enumeration to list the next $\phi_i$.

\paragraph{Constructing the traders.}
We now define our sequence  $(\Trader^k)_k$ of traders. For $n<k$,  let $\trade^k_n$ be the zero trading strategy. For $n \ge k$, define $\trade^k_n$ 
to be the trading strategy 
\[ \trade^k_n :=  (k + 1 - \amountbought^k_n)  \cdot \lowprice^k_n\cdot  
(\phi_n - \phi_n^{* n}), 
\]
where 
\[\lowprice^k_n:=  \ctsind{1/(2k + 2)} \left(\phi_n^{* n} < \frac{1}{k+1}\right) \]
and
\[ \amountbought^k_n := \sum_{i\leq n-1} \magnit{\trade^k_i} .\]
We will make use of the convention for writing the coefficients of a trade:
\[ \trade^k_n[\phi_n] \equiv  (k + 1 - \amountbought^k_n)  \cdot \lowprice^k_n .
\]
In words, $\trade^k_n$ is a buy order for $(k + 1 - \amountbought^k_n)$-many shares of $\phi_n$, scaled down by the extent
$\lowprice^k_n$  to which  $\phi_n$ is priced below  $1/(k+1)$ at time $n$.  The quantity  $\amountbought^k_n$ measures the total number of shares that
$\Trader^k$ has purchased before the current time step $n$.  

For some fixed polynomial independent of $k$,
the $\Trader^k$ are uniformly computable with runtime bounded by that polynomial as a function of $n$,
using the discussion in \ref{app:dynamicprogramming} on dynamic programming. 
Furthermore, $\trade^k_n \equiv 0$ for $n < k$ by definition. Hence $(\Trader^k)_k$ is an
efficiently emulatable sequence of traders as defined in~\ref{def:emulatabletraders}. 

Note that by the definition of $\trade^k_n$, the magnitude $\magnit{\trade^k_n(\MP)}$ of the trade is bounded by $k+1- \amountbought^k_n(\MP)$. 
By definition of   
$\amountbought^k_n$ and by  induction on $n$, we have that 
\begin{align*}
\amountbought^k_1 (\MP)& = 0 \leq k+1\\
\shortintertext{ as $ \amountbought^k_1 $ is an empty sum, and}
\amountbought^k_{n +1 }(\MP) &= \sum_{i\leq  n}  \magnit{\trade^k_i (\MP)} \\
  &= \amountbought^k_n (\MP)+  \magnit{\trade^k_n (\MP)}\\ 
&\leq  \amountbought^k_n (\MP)+ k+1- \amountbought^k_n(\MP)\\
&= k+1.
\end{align*}
 In words,  $\Trader^k$ never trades more than
$k+1$ shares in total. Furthermore, since by definition $\lowprice$ is always
non-negative, we have that 
\[\magnit{\trade^k_n(\MP)} = |\trade^k_n[\phi_n](\MP) | = |(k + 1 - \amountbought^k_n(\MP) )  \cdot \lowprice^k_n(\MP)  | \geq 0. \]

\paragraph{Analyzing the value of $\bm{\Trader^k}$.} 
Fix $\Trader^k$ and a time step $n$. 
For any plausible world $\World \in \pcworlds(\dt_{n})$, the value in $\World$ of holdings from trades made by $\Trader^k$ up to time $n$ is
\begin{align*}
\World\left({\textstyle \sum_{i \leq n} \trade^k_i(\MP) }\right) 
&= 
\World\left({\textstyle \sum_{i \leq n} \trade^k_i[\phi_i](\MP) \cdot (\phi_i - \phi_i^{* i}(\MP)) }\right) \\
&= \;\;\;
\sum_{i \leq n} \trade^k_i[\phi_i](\MP)  \cdot \World(\phi_i)\\
&\;\;\;
+ \sum_{i \leq n}\trade^k_i[\phi_i](\MP)  \cdot (-\pt_i(\phi_i)),
\end{align*}
by linearity and by the definition $ \phi_i^{* i }(\MP) \equiv   \pt_i (\phi_i)$. 
We analyze the second term first, which represents the contribution to the value of $\Trader^k$ from the prices of the $\phi_i$-shares that it has purchased
up to time $n$. 
We have that 
\begin{align*}
&\;\;\;\;  
\sum_{i \leq n}\trade^k_i[\phi_i](\MP)  \cdot (-\pt_i(\phi_i))\\
&\geq   \sum_{i \leq n} 
\magnit{\trade^k_i} \cdot \left( - \frac{1}{k+1} \right) \\
\shortintertext{since $\magnit{\trade^k_i(\MP)} = \trade^k_i[\phi_i](\MP) \ge 0$ and since   $ \pt_i (\phi_i) \leq 1/(k+1)$ whenever $\lowprice^k_i(\MP)$ is non-zero;}
&\geq   - \frac{k+1}{k+1} \\
&= -1,
\end{align*}
since $\Trader^k$ never purchases more than $k+ 1$ shares.  Now consider the value 
\[\sum_{i \leq n} \trade^k_i[\phi_i](\MP)  \cdot \World(\phi_i)\]
in $\World$ of the stock holdings
from trades made by $\Trader^k$ up to time $n$. Since both $\World(\phi_i)$  and $\trade^k_i[\phi_i](\MP) $ are non-negative, this value is
non-negative. Hence we have shown that 
\[\World\left({\textstyle \sum_{i \leq n} \trade^k_i(\MP) }\right)  \geq  -1 + 0 = -1 ,\]
i.e. the total value of $\Trader^k$ is bounded below by $-1$. 

Furthermore, since $\Theory \cup \phis$ is consistent, there is always a plausible world 
$\World \in \pcworlds(\dt_{n})$ such that $\forall i\leq n: \World(\phi_i) = 1$, and therefore  
\[\World\left({\textstyle \sum_{i \leq n} \trade^k_i(\MP) }\right)  \geq  -1 + 
\sum_{i \leq n} \trade^k_n[\phi](\MP).\]

\paragraph{Exploitation by the parametric trader.}

Now suppose by way of contradiction that the market $\MP$ does not satisfy 
the uniform non-dogmatism property.  Then for every $k$, in particular the property does not hold for $\varepsilon = 1/(2k+2)$, so 
there is some $\phi_i$ in the sequence $\phis$ such that
$\pt_\infty(\phi_i) <  1/(2k+2)$. Since by assumption $\phi_i$ appears infinitely often in $\phis$, for some sufficiently large
$n$ we have $\pt_n(\phi_n) \equiv \pt_n(\phi_i)  <  1/(2k+2)$,  at which point
\[\lowprice^k_n(\MP) = \ctsind{1/(2k + 2)} \left(\pt_n(\phi_n) < \frac{1}{k+1}\right)= 1. \]
Therefore 
\[\trade^k_n[\phi_n] = (k + 1 - \amountbought^k_n),  \]
so that 
\[\sum_{i\leq  n} \magnit{ \trade^k_i (\MP)} =  \amountbought^k_{n} (\MP)+ k+1- \amountbought^k_{n}(\MP) = k+1. 
\]
Thus
\[\World\left({\textstyle \sum_{i \leq n} \trade^k_i(\MP) }\right)  \geq  -1 + k+1 = k. \]
In words, once the price of some $\phi_i$ dips below $1/(2k+2)$, the trader $\Trader^k$ will purchase the remaining $k+1-\amountbought^k_n(\MP)$ shares it
will ever buy. Then in a world $\World$ that witnenesses that $\phis$ is consistent with $\Theory$, all the shares held by $\Trader^k$ are valued at $\$1$
each, so $\Trader^k$ has stock holdings valued at $k+1$, and cash holdings valued at no less than $-1$. 

Therefore each $\Trader^k$ has plausible value bounded below by $-1$ and at least $k$ in some plausible world at some time step, and therefore 
Lemma~\ref{lem:type2} applies, contradicting that $\MP$ is a \li{}. Therefore in fact 
$\MP$ does satisfy 
the uniform non-dogmatism property. 

\end{proof}

\subsection{Occam Bounds}\label{app:ob}
Recall Theorem~\ref{thm:ob}:

\restateoccam*

We show the result for $\phi$ such that $\Theory \not\vdash \lnot \phi$; the result for $\Theory \not\vdash \phi$ follows by
considering $\lnot\lnot \phi$ and using the coherence of $\pt_\infty$. 

Roughly speaking, we will construct an efficiently emulatable sequence of traders $(\Trader^k)_k$ where $\Trader^k$ attempts to ensure that $\pt_n(\phi)$
does not drop below 
$2^{-\kappa(\phi)}/(k+1)$ for any $\phi$. We do this by having $\Trader^k$ purchase shares in any $\phi$ that are underpriced in this way, as
judged by a computable  approximation from below of $2^{-\kappa(\phi)}$. The trader $\Trader^k$ will purchase at most $k+1$ shares in each $\phi$, and hence 
spend at most $\$ 2^{-\kappa(\phi)}$ for each $\phi$ and at most $\$1$ in total.
On the other hand, if the market $\MP$ does not satisfy the Occam property with constant $C = 1/(k+1)$, then for some $\phi$ with
$\Theory \not\vdash \lnot \phi$, we will have that
$\Trader^k$ purchases a full $k+1$ shares in $\phi$. Since there is always a plausible world that values $\phi$ at $\$1$, $\Trader^k$ will have a plausible value
of at least $\$ k$,  taking into account the $\$ 1$ maximum total prices paid. This contradicts Lemma~\ref{lem:type2}, so in fact $\MP$ satisfies the Occam property. 

To implement this strategy, we will use tools similar to those used in the proof of Theorem~\ref{thm:obu}, and the proof is similar 
in spirit, so we will elide some details.

\begin{proof}
Observe  that $2^{-\kappa(\phi)}$ is approximable from below uniformly in $\phi$, since we can (slowly) enumerate all prefixes on which our fixed UTM halts
and outputs $\phi$. 
Let $\seq\phi$ be an efficiently computable enumeration of all sentences. Let $M$ be a Turing machine that takes an index $i$ into our enumeration and takes
a time $n$, and outputs a non-negative rational number. We further specify that $M$ runs in   
time polynomial in $i+n$, satisfies
$\forall n,i: M(n,i) \leq 2^{-\kappa(\phi_i)}$, and satisfies $\lim_{n\rightarrow\infty}M(n,i)=2^{-\kappa(\phi_i)}$. 
Note that since we are using prefix complexity,  we have $\sum_{\phi\in\Sentences}2^{-\kappa(\phi)}\leq 1$.  
(We can assume without loss of generality that $M(n,i) >0$ for all $i \le n$, by padding the sequence $\phis$ with $\phi_1$ while waiting until $M(n,i)>0$ to
enumerate $\phi_i$,  using the fact that our UTM outputs each $\phi$ for some prefix.) 

We define  a sequence of traders $(\Trader^k)_k$. For $n<k$, define $\trade^k_n$ to be the zero trading strategy. For $n\geq k$, 
define $\trade^k_n$ to be the trading strategy given by 
\[\trade^k_\nn :=  \sum_{i \leq n} (k+1 - \amountbought^k_n(i)) \cdot \lowprice^k_n(i) \cdot (\phi_i - \phi_i^{* n})   ,\]
where 
\[ \lowprice^k_n(i) :=  \ctsind{M(n,i)/(2k + 2)}\left(\pf{\phi}<\frac{M(n,i)}{k + 1}\right)\] 
and
\[\amountbought^k_n(i) := \sum_{j \leq n-1} \trade^k_j[\phi_i] .\]
This is similar to the parametric trader used in the proof of Theorem~\ref{thm:obu}, except that here on time $n$,  $\Trader^k$ buys any $\phi_i$ when
$\pt_\nn(\phi_i)$ is too low, 
and  $\amountbought^k_n(i)$ tracks the number of shares bought by $\Trader^k$ in each $\phi_i$ individually up to time $n$. 
By the discussion of dynamic programming in \ref{app:dynamicprogramming}, $(\Trader^k)_k$ is an efficiently emulatable sequence of traders. (We use that $M(n,i)$ is pre-computed by the machine that computes
$\Trader^k$, and hence appears in
the feature expression for $\trade^k_n$ as a constant which is strictly positive by assumption.)

Observe that for each $i$, by induction on $n$ we have $\amountbought^k_n(i) \leq k+1$, so that $\Trader^k$ only buys positive shares in the various $\phi_i$, 
and $\Trader^k$ only buys up to $k + 1$ shares of $\phi_i$. 
Further,  $\Trader^k$ 
only buys $\phi_i$-shares at time $n$ if the price $\pt_\nn(\phi_i)$ is below $1/(k+1)$-th of the approximation $M(n,i)$ to $2^{-\kappa(\phi)}$, i.e. 
\[\pt_\nn(\phi_i) < \frac{M(n,i)}{k + 1} \leq \frac{2^{-\kappa(\phi_i)}}{k+1} .\]
Therefore $\Trader^k$
spends at most $\$ 2^{-\kappa(\phi_i)}$ on $\phi_i$-shares, and hence spends at most $\$ 1$ in total. 

Suppose for contradiction that $\MP$ does not satisfy the Occam property.  
Then for every $k$ there exists some $\phi_i$ such that  the limiting price of $\phi_i$ is 
\[\pt_\infty(\phi_i) < \frac{2^{-\kappa(\phi_i)}  }{(2k+2)}, \]  

but nevertheless $\Theory\nvdash \neg\phi_i $. 
Then for some time step $n$ we will have that $ \pt_n(\phi_i) < M(n,i)/(2k+2) $,  and hence $\lowprice^k_n(i)= 1 $. At that point 
$\Trader^k$ will purchase  $k+1 - \amountbought^k_n(i)$ shares in
$\phi_i$, bringing   $\amountbought^k_{n+1} (i)$ to $k+1$; that is, $\Trader^k$ will have bought $k + 1$ shares of $\phi_i$.  
Since $\phi$ is consistent with $\Theory$, 
there is some  plausible world $\World \in \pcworlds(\dt_{n})$ that values those shares at $\$ 1$ each, so that the total value of all of the holdings from trades made
by $\Trader^k$ is at least $k$.
By Lemma~\ref{lem:type2} this contradicts that $\MP$ is a  \li{}, so in fact $\MP$ must satisfy the Occam property. 
\end{proof}

\subsection{Non-Dogmatism}\label{app:nondog}

\restatenondog*

\begin{proof}
  This is a special case of~\ref{thm:ob}, since $\kappa(\phi) >0$ for any $\phi$.
\end{proof}

\subsection{Domination of the Universal Semimeasure}\label{app:dus}

\restatedus*

\begin{proof}

\newcommand{\Purchased}{\mathrm{Purchased}}
\newcommand{\MeanPayout}{\mathrm{MeanPayout}}
\newcommand{\MaxPayout}{\mathrm{MaxPayout}}

Let $(\seq{\sigma}^i)_i$ be an \ec enumeration of all finite strings.  Let $l(\seq{\sigma})$ be the length of $\seq{\sigma}$.  Define

\[
  \phi_i := \quot{
    (b_1 \iff \enc{\sigma^i_1}=1) \land
    (b_2 \iff \enc{\sigma^i_2}=1) \land
    \ldots \land
    (b_{\enc{l(\seq{\sigma}^i)}} \iff \enc{\sigma^i_{l(\seq{\sigma}^i)}}=1)}
\]

  to be the sentence saying that the string $(b_1, b_2, \ldots)$ starts with $\sigma^i$.  Suppose the theorem is not true; we will construct
a sequence of parametric traders to derive a contradiction through \Lem{type2}.

\paragraph{Defining a sequence of parametric traders.}
To begin, let $A(\seq{\sigma}, \nn)$ be a lower-approximation of $M(\seq{\sigma})$ that can be computed in time polynomial in $n$ and the length of $\seq{\sigma}$.  Specifically, $A$ must satisfy
\begin{itemize}
  \item
   $A(\seq{\sigma}, \nn) \leq M(\seq{\sigma})$, and 
   \item
   $\lim_{\nn \to \infty} A(\seq{\sigma}, \nn) = M(\seq{\sigma})$.
\end{itemize}
  Now, recursively define
\begin{align*}
  \gen{\alpha_{k,n,i}} &:= \begin{cases}
  \ctsind{\frac{1}{4(k+1)}} \left( \frac{ \phi_i^{* n}}{ A(\seq{\sigma}^i, n) } < \frac{1}{2(k+1)} \right) & \text{if } A(\seq{\sigma}^i, n) > 0 \wedge n \geq k \wedge i \leq n \\
    0 & \text{otherwise}
  \end{cases}
    \\
  \gen{\beta_{k,n,i}} &:= \min(\gen{\alpha_{k,n,i}}, 1 - \gen{\gamma_{k,n,i}})
  \\
  \gen{\gamma_{k,n,i}} &:= \sum_{m < n, j \leq m} \gen{\beta_{k, m, j}} {\phi}_j^{* m} + \sum_{j < i} \gen{\beta_{k, n, j}} \phi_j^{* n}
\end{align*}
in order to define a parametric trader
\[
  T^k_n := \sum_{i \leq n} \beta_{k,n,i} \cdot (\phi_i - \phi_i^{* n})
\]

As shorthand, we write $\alpha_{k,n,i} := \gen{\alpha_{k,n,i}} \ftn$, $\beta_{k,n,i} := \gen{\beta_{k,n,i}} \ftn$, $\gamma_{k,n,i} := \gen{\gamma_{k,n,i}} \ftn$. Intuitively,
\begin{itemize}
  \item $\alpha_{k,n,i}$ is the number of copies $\phi_i$ that  $\Trader^k$ would buy on day $n$ if it were not for budgetary constraints.  It is high if $\pt_n$ obviously underprices
    $\phi_i$ relative to $M$ (which can be checked by using the lower-approximation $A$).
  \item $\beta_{k,n,i}$ is the actual number of copies of $\phi_i$ that $\Trader^k$ buys, which is capped by budgetary constraints.
  \item $\gamma_{k, n, i}$ is the amount of money $T^k$ has spent on propositions $\phi_1, \ldots, \phi_{n-1}$ ``before considering'' buying $\phi_i$ on day $n$.  We imagine that,
        on each day $n$, the trader goes through propositions in the order $\phi_1, \ldots, \phi_n$.
\end{itemize}

\paragraph{Analyzing the sequence of parametric traders.}

Observe that $T^k$ spends at most $\$1$ in total, since $\beta_{k,n,i} \leq 1 - \gamma_{k,n,i}$.  Now we will analyze the trader's maximum payout.  Assume by contradiction
that $\pt_\infty$ does not dominate $M$.  Define
\[ \Purchased_{k,n,i} := \sum_{m\le n} \beta_{k,m,i} \]
to be the number of shares of $\seq{\sigma}^i$ that $T$ has bought by time $\nn$, and
\[ \MeanPayout_{k,n} := \sum_{i \in \NN^+} M(\seq{\sigma}^i) \Purchased_{k,n,i}. \]
to be the ``mean'' value of stocks purchased by time $\nn$ according to the semimeasure $M$.  Both of these quantities are nondecreasing in $n$.  Now we show that there is some $N$ such that $\MeanPayout_{k,n} \geq k+1$ for all $n \geq N$:
\begin{itemize}
  \item Every purchase costing $c$ corresponds to $\MeanPayout_{k,n}$ increasing by at least $c\cdot 2(k+1)$.  This is because the trader only buys $\phi_i$ when $\frac{\pt_\nn(\phi_i)}{A(\seq{\sigma}_i, \nn)} < \frac{1}{2(k+1)}$, and $A(\seq{\sigma}_i, \nn) \leq M(\seq{\sigma}_i)$.
  \item For some $N$, $\MeanPayout_{k, N} \geq k+1$.  Suppose this is not the case.
Since we are supposing that $\pt_\infty$ does not dominate the universal semimeasure, there is some $i$ such that
$\pt_\infty(\phi_{i}) < \frac{M(\seq{\sigma}^{i})}{8(k + 1)}$.  So we will have
$\pt_n(\phi_{i}) < \frac{M(\seq{\sigma}^{i})}{8(k + 1)}$ for infinitely many $n$; let $\mathcal{N}$ be the set of such $n$.

For all sufficiently high $n$ we have $A(\seq{\sigma}^i, n) \geq M(\seq{\sigma}^i)/2$, so for all sufficiently high
$n \in \mathcal{N}$,

\[ \frac{\pt_n(\phi_i)}{A(\seq{\sigma}^i,n)} \leq \frac{\pt_n(\phi_i)}{M(\seq{\sigma}^i)/2} \leq \frac{1}{4(k+1)}  \]

and so there is some infinite subset $\mathcal{N}' \subseteq \mathcal{N}$ for which $\alpha_{k,n,i} = 1$.
By assumption, $\forall n: \MeanPayout_{k, n} < k+1$, so the trader never has spent more than $\$1/2$ (using the previous step), so $\gamma_{k,n,i} \leq 1/2$.
This means $\beta_{k,n,i} \geq 1/2$, which implies an increase in mean payout $\MeanPayout_{k,n} - \MeanPayout_{k,n-1} \geq M(\seq{\sigma}^i) > 0$.  But this increase happens for infinitely many $n$,
so $\lim_{n \to \infty} \MeanPayout_{k, n} = \infty$.  This contradicts the assumption that $\MeanPayout_{k, N} < k+1$ for all $N$.
  \item $\MeanPayout_{k,N}$ is nondecreasing in $N$, so $\MeanPayout_{k,n} \geq k+1$ for all $n \geq N$.
\end{itemize}

Using this lower bound on $\MeanPayout_{k,n}$, we would now like to show that $T^k$'s purchases pay out at least $k+1$ in some $W \in \pcworlds(\dt_\infty)$. To do this, define

\[
\MaxPayout_{k,n} := \sup_{\seq{\sigma} \in \mathbb{B}^{\leq \NN^+}} \sum_{\seq{\sigma'}^i \text{~prefix of~} \seq{\sigma}} \Purchased(\seq{\sigma}^i, \nn)
\]
to be the maximum amount that $T^k$'s purchases pay out over all possible strings (finite and infinite).
Since $M$ is a semimeasure over finite and infinite bitstrings, we have $\MeanPayout(\nn) \leq \MaxPayout(\nn)$.
Since each $\phi_i$ is independent of $\Theory$, $\Trader^k$'s maximum worth is at least
\[ \limsup_{\nn \to \infty} \MaxPayout(\seq{\varepsilon}, \nn) - 1 \geq \limsup_{\nn \to \infty} \MeanPayout(\seq{\varepsilon}, \nn) - 1 \geq k + 1 - 1 = k. \]

This is sufficient to show a contradiction using \Lem{type2}.
\end{proof}

\subsection{Strict Domination of the Universal Semimeasure}\label{app:strict}

Recall \Theorem{strict}:

\restatestrict*

\begin{proof}

Consider the sets of codes for Turing machines 
\begin{align*}
  A_0 &:= \{ M \mid \text{ $M$ halts on input 0 and outputs 0}\}\\
  \shortintertext{and}
  A_1 &:= \{ M \mid \text{ $M$ halts on input 0 and outputs 1}\}.
\end{align*}
Both of these sets are computably enumerable and disjoint, so by \Theorem{obu}, $\pt_\infty$ assigns positive measure to the set $[A]$ of infinite bitstrings that encode a separating set for $A_0$ and $A_1$, i.e., a set $A$ such that $A\cap A_0 = \emptyset$ and $A\supseteq A_1$. 

Thus it suffices to show that the universal semimeasure assigns 0 measure to $[A]$.  This is a known result from computability theory, using the fact that $A_0$ and $A_1$ are recursively inseparable; see for example \cite{kuvcera2011demuth}. Here we give an elementary proof sketch.

Suppose for contradiction that $m$ computes a universal semimeasure and $m([A])= r >0$; we argue that we can compute some separating set $A$. Let $q \in [4r/5,r] \cap \QQ$. There is some fixed $k$ such that the finite binary subtree $A_k$ consisting of finite prefixes of length $k$ of strings in $[A]$  is assigned  $m(A_k) \in [r,6r/5]$. 

On input $n$, we can run $m$ on the set of strings of length up to $n$ until the set of extensions of strings in $A_k$ has measure at least $q$; this will happen because $m([A]) >q$. Then we output 0 if the majority of the measure is on strings with $n$th bit equal to 0, and we output 1 otherwise. If we output 0 but in fact $n \in A_1$, then there is measure at most $6r/5-2r/5 = 4r/5$ on extensions of strings in $A_k$ that are consistent with separating $A_0$ and $A_1$; but this is impossible, as $[A]$ has measure $r$. Likewise if we output 1 then it can't be that $n \in A_0$. Thus we have recursively separated $A_0$ and $A_1$, contradicting that $A_0$ and $A_1$ are recursively inseparable.

\end{proof}

\subsection{Closure under Finite Perturbations}\label{app:ifp}

Recall \Theorem{ifp}:

\ifp*

In short, a trader that exploits $\MP$ also exploits ${\seq{\pt^\prime}}$ since all but finitely many of its trades are identically valued. 
The proof mainly concerns a minor technical issue; we have to make a small adjustment to the trader to ensure that it makes exactly the same
trades against ${\seq{\pt^\prime}}$ as it does against $\MP$.

\begin{proof}
  Assume there is a trader $\Trader$ which exploits $\MP$. We will construct a new trader $\Trader'$ that exploits ${\seq{\pt^\prime}}$. Fix $N$ large enough that $ \pt_\nn= \pt^\prime_\nn$ for all $\nn\geq N$. 

We will define $\Trader'$ so that it makes the same trades against the market ${\seq{\pt^\prime}}$ as the trader $\Trader$ makes against $\MP$.   That is, we want that for
all $n$, 
\[
  \trade'_n({\seq{\pt^\prime}}) =\trade_n(\MP).  
\]
It is insufficient to set the trading strategy $\trade'_n$ 
equal to $\trade_n$ for all $n$.  This is because $\trade_n$ may infinitely often make different trades given the history $\pt^\prime_{\leq n}$ instead of
the history 
$\pt_{\leq n}$. For example, it may be that every day $\Trader$ buys $\Valuation_1(\phi)$-many shares in $\phi$ against $\Valuation$; in this case if
$\pt^\prime_1(\phi) \ne \pt_1(\phi)$, then at each time $n$, 
$\trade_n({\seq{\pt^\prime}})$ will buy a different number of shares from $\trade_n(\MP)$. Roughly speaking, we will patch this problem by copying $\Trader$, but feeding it
``false reports'' about the market prices so that it appears to the $\trade_n$ that they are reacting to $\MP$ rather than ${\seq{\pt^\prime}}$.

More precisely, let $F$ be a computable function from feature expressions to feature expressions, 
in the expression language discussed in \ref{app:expressiblefeatures}. For a feature expression $\alpha$, we define $F(\alpha
)$ to be identical to $\alpha$ but with all occurrences of an expression $\phi^{* i}$ for $i<N$ replaced by a constant $\pt_i(\phi)$.

Note that $F$ is efficiently computable: by the assumption that $ \pt_\nn= \pt^\prime_\nn$ for all $\nn\geq N$, only finitely many constants $\pt_i(\phi)$ are
needed, and can be hard-coded into $F$. Furthermore, $F$ behaves as intended: for any $\alpha$, we have $F(\alpha)({\seq{\pt^\prime}})= \alpha(\MP)$ (using a slight abuse of
notation, treating $\alpha$ as both an expression and as the feature thus expressed). This follows by structural induction
the expression $\alpha$, where every step is trivial except the base cases for symbols $\phi^{* i}$ with $i<N$, which follow from the
definition of $F$. Now we define
\[
  \trade'_n := \sum_{\phi \in \Sentences} F(\trade_n[\phi]) (\phi - \phi^{* n}) 
\]
for any $n$. This is efficiently computable because $\trade_n$ and $F$ are both \ec  Furthermore, for all $n \ge N$, we have that $\trade'_n({\seq{\pt^\prime}}) =
\trade_n(\MP)$. Therefore for any $n$ we have that 
\begin{align*}
&\;\;\;\;\left| 
\World\mleft({\textstyle \sum_{i \leq n} \trade_i\mleft(\MP\mright)}\mright) - 
\World\mleft({\textstyle \sum_{i \leq n} \trade'_i\mleft({\seq{\pt^\prime}}\mright)}\mright) 
\right|\\
&\leq  
\left| 
\World\mleft({\textstyle \sum_{i < N } \trade_i\mleft(\MP\mright)}\mright) - 
\World\mleft({\textstyle \sum_{i < N } \trade'_i\mleft({\seq{\pt^\prime}}\mright)}\mright) 
\right|,
\end{align*}
which is a fixed constant, where we use that all terms for $i\ge N$ cancel with each other. This says that at all times and all plausible worlds, there is a
fixed upper bound  on the difference between the values of $\Trader$ against
$\MP$ and of
$\Trader'$ against ${\seq{\pt^\prime}}$. Thus if 
\[  \left\{ \World\mleft({\textstyle \sum_{i \leq n} \trade_i\mleft(\MP\mright)}\mright) \,\middle|\, \nn\in\NN^+, \World\in\pcworlds(\dt_\nn)
\right\}\]
is bounded below but unbounded above, then so is
\[  \left\{ \World\mleft({\textstyle \sum_{i \leq n} \trade'_i\mleft({\seq{\pt^\prime}}\mright)}\mright) \,\middle|\, \nn\in\NN^+, \World\in\pcworlds(\dt_\nn)
\right\}. \]
Therefore,
if some trader exploits $\MP$,  so that $\MP$ is not a logical inductor, then some trader exploits ${\seq{\pt^\prime}}$, so ${\seq{\pt^\prime}}$ also fails to be a \li{}. Symmetrically,  if ${\seq{\pt^\prime}}$ is not a \li{},
then neither is $\MP$.
\end{proof}

\subsection{Conditionals on Theories}\label{app:scon}

\newcommand{\symb}{\circ}
\newcommand{\altDP}{{\DP^{\symb}}}
\newcommand{\altdt}{{\dt^{\symb}}}
\newcommand{\altpsi}[1]{{\psi^{\symb}_{#1}}}
\newcommand{\altMP}{{\MP^{\symb}}}
\newcommand{\altpt}{{\pt^{\symb}}}

\restatescon*

  Since  $\MP$ is a \li{} over $\Theory$, we can fix some particular $\Theory$-complete deductive process $\DP$ over which $\MP$ is a \li{}, which exists by definition of ``\li{} over $\Theory$''. 
Let $\DP'$ be any other \ec deductive process.
Write 
\[ \altpsi{n} := \bigwedge_{\psi \in \dt'_n} \psi  \]
for the conjunction of all sentences $\psi$ that have appeared in $\DP'$ up until time $n$. (We take the empty conjunction to be the sentence
$\top$.) 
Write $\altMP$ to mean the market 
$\left(  \pt_n(\any \mid  \altpsi{n}) \right)_{n \in \NN^+}$. 

We will show the slightly more general fact that for any \ec $\DP'$, if the theory \[  \Theory \cup \{\psi'
\mid \exists n: \psi' \in \dt'_n \} \] is consistent, then $\altMP$
  is a \li{} over the deductive process $\altDP$ defined for any $n$ by $\altdt_n := \dt_n \cup \dt'_n$, which is complete for that theory. 
  This implies the theorem by specializing to the $\{\psi\}$-complete deductive process
  $(\{\psi\},\{\psi\},\{\psi\},\ldots)$, and to the $\Psi$-complete deductive process
$(\{\psi_1\},\{\psi_1,\psi_2\},\{\psi_1,\psi_2,\psi_3\},\ldots)$ (where we pad with $\top$ to ensure this sequence is efficiently computable).

Roughly speaking, we'll take a supposed trader $\Trader^\symb$ that exploits $\altMP$ and construct a trader $\Trader$ that exploits $\MP$.
We'd like our trader $\Trader$ to mimic $\Trader^\symb$ ``in the worlds where $\altpsi{n}$ is true'', and otherwise remain neutral. A first attempt
would be to have our trader buy the combination
\[
  \phi \wedge \altpsi{n} - \frac{\pt_n(\phi \wedge \altpsi{n})}{\pt_n(\altpsi{n})} \cdot \altpsi{n}
\]
whenever $\Trader^\symb$ buys a share in $\phi$. 
The idea is to make a purchase that behaves like a conditional contract that pays out if $\phi$ is true but only has any effect in worlds where  $\altpsi{n}$
is true.
That is, the hope is that the price of this combination is 0; in worlds where $\altpsi{n}$
 is false, the stock holdings from this trade are valued at 0; and in worlds where $\altpsi{n}$
 is true, the stock holdings have the same value as that of  purchasing a $\phi$-share against $\altMP$. 

There are some technical problems with the above sketch. First, the ratio of probabilities in front of $\altpsi{n}$ in the above trade is not well-defined if 
$\pt_n(\altpsi{n})=0$. We will fix this using a safe reciprocation for the ratio. To avoid having this affect the performance of  $\Trader$ in comparison to  $\Trader^\symb$, we will first correct the market using Lemma~\ref{app:ifp} (closure under finite perturbations) so that, essentially, the safe reciprocation never makes a difference.

Second, if 
$\pt_n(\phi \wedge \altpsi{n})$ is greater than   $\pt_n(\altpsi{n})$, then their ratio $\frac{\pt_n(\phi \wedge \altpsi{n})}{\pt_n(\altpsi{n})}$ is greater than the conditional probability $\pt_n(\phi \mid
\altpsi{n}) = 1$ as
defined in~\ref{def:condp} (which is capped at 1). In this case, our trader  $\Trader$  has stock holdings with a different (and possibly lower) value from those of the original trader  $\Trader^\symb$ exploiting
$\altMP$, and therefore  $\Trader$ possibly has less value overall across time than  $\Trader^\symb$, which breaks the desired implication (that is, maybe the original trader exploits $\altMP$, but our new, less successful trader does not exploit $\MP$). If we simply replace the ratio with the conditional probability 
$\pt_n(\phi \mid \altpsi{n}) $, then when $\pt_n(\phi \wedge \altpsi{n}) > \pt_n(\altpsi{n}) $, the value of the cash holdings for  $\Trader$ may be non-zero (in particular, may be negative). Instead
we will have $\Trader$ cut off its trades when both $\pt_n(\phi \wedge \altpsi{n}) > \pt_n(\altpsi{n}) $ and also $\Trader^\symb$ is buying $\phi$; this is no loss for $\Trader$ relative to $\Trader^\symb$, since in this case $\Trader^\symb$ is buying $\phi$ at the price of 1, and so is not making any profit anyway. 

We now implement this construction strategy. 

\newcommand{\csell}{{\rm Sell}}
\newcommand{\cbuy}{{\rm Buy}}

\begin{proof}

Let $\DP$, $\altDP$, and $\altMP$ be defined as above.

We may assume that the collection of sentences that appear in $\altDP$ is consistent. If not, then no trader exploits $\altMP$: for all sufficiently large $n$ 
the set of plausible worlds $\pcworlds(\altdt_n)$ is empty, so the set of plausible values of any
trader's holdings is a finite set, and hence bounded above. 

We may further assume without loss of generality that there exists a rational $\varepsilon>0$ such that $\pt_\nn(\altpsi{n})>\varepsilon$ for all $\nn$. Indeed,
by Theorem~\ref{thm:obu} (uniform non-dogmatism), since $\altDP$ is consistent, there is some $\varepsilon>0$ such that $\pt_\infty(\altpsi{n})>\varepsilon$ for
all sufficiently large $n$. Hence by Theorem~\ref{thm:tbo} (preemptive learning), we have $\liminf_{n\to \infty} \pt_n (\altpsi{n}) > \varepsilon$. This implies that
there are only finitely many time steps $n$ such that $\pt_n (\altpsi{n}) \leq \varepsilon$. Therefore by Lemma~\ref{app:ifp} (closure under finite perturbations),
the market ${\MP'}$
defined to be identical to $\MP$ except with
$\pt'_n (\altpsi{n})$ with  1 for all such $n$ is still a \li{}, and has the desired property. If we show that ${\altMP}'$ is
a \li{},  then again by Lemma~\ref{app:ifp},  $\altMP$ is also a \li{}.

Now suppose some trader $\Trader^\symb$ exploits $\altMP$.  We will construct a trader $\Trader$ that exploits $\MP$.

Consider the $\EF$-combination
\begin{align*}
\cbuy_\nn(\phi) &:= \phi \wedge \altpsi{n} - \frac{(\phi \wedge \altpsi{n})^{* n} }{ \max(\varepsilon, \altpsi{n}^{* n})} \cdot \altpsi{n}  
\end{align*}
parametrized by a sentence $\phi$. 
We write $(\cbuy_\nn(\phi))^{* n}$ for the expressible feature that computes the price of the $\EF$-combination $\cbuy_n(\phi)$ at time $n$, defined in the natural way by replacing sentences
with their $^{* n}$ duals. Intuitively, this combination is a ``conditional contract'' which is roughly free to buy (and valueless) in worlds where $\altpsi{n}$ is false, but behaves like a $\phi$-share in worlds where $\altpsi{n}$ is true.

Now define the trader $\Trader$ by setting
\begin{align*}
  \trade_n &:= \sum_{\phi} \alpha_n \cdot (\cbuy_n(\phi) - (\cbuy_\nn(\phi))^{* n} ) \\ 
  \alpha_n &:= \min\left(\trade^\symb_n[\phi]_{\symb}, \trade^\symb_n[\phi]_{\symb}\cdot \ctsind{\varepsilon_n}\left(\frac{(\phi \wedge \altpsi{n})^{* n} }{ \max(\varepsilon, \altpsi{n}^{* n})} < 1 + \varepsilon_n\right) \right)\\
  \varepsilon_n &:= \frac{2^{-n}}{\max(1,\magnit{\trade_n^\symb})}
\end{align*}
where $\trade^\symb_n[\phi]_{\symb}$ is defined to be the market feature $\trade^\symb_n[\phi]$ with every occurrence of a sub-expression $\chi^{* i}$ for some sentence $\chi$ replaced with 
\[  \max \left( 1,
\frac{(\chi \wedge \altpsi{i})^{* i} }{ \max(\varepsilon, \altpsi{i}^{* i})} \right) .
\]
That is, $\trade^\symb_n[\phi]_{\symb}$ is defined so that $\trade^\symb_n[\phi]_{\symb}(\MP) = \trade^\symb_n[\phi](\altMP)$, i.e., this market feature $\trade^\symb_n[\phi]_{\symb}$ behaves against
the market $\MP$ just as $\trade^\symb_n[\phi]$ behaves against the conditional market $\altMP$. Note that $\ctsind{\varepsilon_n}$ is a valid expressible feature: $\ctsind{\varepsilon_n}(x<y) := \max(0,\min(1,2^n\max(1,\magnit{\trade_n^\symb})(y-x) ))$.

The idea is that $\Trader$ will roughly implement the conditional contracts as described above, and will thus perform just as well against $\MP$ as $\Trader^\symb$ performs against $\altMP$. The catch is that it may be that $\pt_n(\phi \wedge \altpsi{n}) > \pt_n( \altpsi{n})$, in which case $\cbuy_n(\phi)$ will no longer quite function as a conditional contract, since $\altpt_n(\phi)$ is capped at 1. To prevent $\Trader$ from losing relative to $\Trader^\symb$, we use $\alpha_n$ to quickly stop $\Trader$ from buying once $\pt_n(\phi \wedge \altpsi{n}) > \pt_n( \altpsi{n})$; no profit is lost, as the price of $\phi$ for $\Trader^\symb$ is in that case just 1.

We now formalize this analysis of the value of the trades made by $\Trader$  against $\MP$ according to each term in the above summation and by cases on the traded sentences $\phi$. 

\textit{Case 1.}
First suppose that $\trade^\symb_n[\phi]_{\symb}(\MP)\leq 0$ and/or $\pt_n(\phi \wedge \altpsi{n})/  \pt_n(\altpsi{n})  =  \pt_n (\phi \mid \altpsi{n})$.
Then $\alpha_n = \trade^\symb_n[\phi]_{\symb}(\MP)$. Let $\World$ be any world; using linearity throughout, we have
\begin{align*}
 &\;\;\;\; \World(\; \alpha_n \cdot (\cbuy_n(\phi) - (\cbuy_\nn(\phi))^{* n} )\; )(\MP) \\
  &=
 \trade^\symb_n[\phi]_{\symb}(\MP) \cdot \World(  \cbuy_n(\phi) (\MP) - (\cbuy_\nn(\phi))^{* n} (\MP) ) \\
 &=\;\;\; \trade^\symb_n[\phi]_{\symb}(\MP) \cdot \World\left(  \phi \wedge \altpsi{n} - 
 \frac{\pt_n (\phi \wedge \altpsi{n})}{\pt_n(\altpsi{n}) } 
 \cdot \altpsi{n}
 \right) \\
& \;\;\;\; -  \trade^\symb_n[\phi]_{\symb}(\MP) \cdot \left( 
\pt_n(\phi \wedge \altpsi{n})- 
\frac{\pt_n(\phi \wedge \altpsi{n})}{ \pt_n(\altpsi{n}) } 
\cdot 
\pt_n(\altpsi{n}) \right) \\
\intertext{by the definition of $\cbuy$;}
&=\;\;\; \trade^\symb_n[\phi]_{\symb}(\MP) \cdot \left(\World ( \phi \wedge \altpsi{n}) - 
 \frac{\pt_n (\phi \wedge \altpsi{n})}{\pt_n(\altpsi{n}) } 
 \cdot \World(\altpsi{n})
 \right)\\
\intertext{by distribution, and where the cash term simply cancels;}
 &\geq \;\;\; \trade^\symb_n[\phi]_{\symb}(\MP) \cdot \left( \World\left(  \phi \wedge \altpsi{n}\right) -
 \pt_n (\phi \mid \altpsi{n})  \cdot \World \left(  \altpsi{n}
 \right)
\right),  
\end{align*}
by definition of   $\pt_n (\phi \mid \altpsi{n})$, and by the assumptions on $\pt_n (\phi \wedge \altpsi{n})$, $\pt_n(\altpsi{n})$, and $\trade^\symb_n[\phi]_{\symb}(\MP)$. Note that if $\World \left(  \altpsi{n}
 \right)=0$ then this quantity is 0, and if $\World \left(  \altpsi{n}
 \right)=1$ then this quantity is 
 \[ \trade^\symb_n[\phi](\altMP) \cdot \left( \World\left( \altpsi{n}\right) -
 \pt_n (\phi \mid \altpsi{n})
\right),  
\]
 which is just the value of $\trade^\symb_n$'s holdings in $\phi$ from trading against $\altMP$.

To lower-bound the value of the $- \altpsi{n}$ term by $-
 \pt_n (\phi \mid \altpsi{n})  \cdot \World \left(  \altpsi{n}\right)$, we use the fact that $\trade^\symb_n[\phi]_{\symb}(\MP)\leq 0$, the fact that $\World \left(  \altpsi{n}
 \right) \geq 0$, and the fact that  $ \pt_n(\phi \wedge \altpsi{n})/  \pt_n(\altpsi{n})  \geq 
 \pt_n (\phi \mid \altpsi{n})  $; or just the fact that $\pt_n(\phi \wedge \altpsi{n})/  \pt_n(\altpsi{n})  =  \pt_n (\phi \mid \altpsi{n})$. (Intuitively: $\Trader$ sells (the equivalent of) $\phi$ at the price of $\frac{\pt_n(\phi \wedge \altpsi{n})}{ \pt_n(\altpsi{n}) }$, while $\Trader^{\symb}$ sells $\phi$  at the no greater price of $\pt_n (\phi \mid \altpsi{n})$; or else $\Trader$ buys (the equivalent of) $\phi$ at the same price as $\Trader^{\symb}$; and so $\Trader$ does at least as well as $\Trader^{\symb}$.)

\textit{Case 2.}
Now suppose that $\trade^\symb_n[\phi]_{\symb}(\MP)\geq 0$, and also $\pt_n(\phi \wedge \altpsi{n})/  \pt_n(\altpsi{n})  >  \pt_n (\phi \mid \altpsi{n})$.
Then $\alpha_n = \trade^\symb_n[\phi]_{\symb}\cdot \ctsind{\varepsilon_n}\left(\frac{(\phi \wedge \altpsi{n})^{* n} }{ \max(\varepsilon, \altpsi{n}^{* n})} < 1 + \varepsilon_n\right)$.
Let $\World$ be any world. We have:
\begin{align*}
 &\;\;\;\; \World(\; \alpha_n \cdot (\cbuy_n(\phi) - (\cbuy_\nn(\phi))^{* n} )\; )(\MP) \\
   &=\;\;\; \trade^\symb_n[\phi](\altMP)\cdot \ctsind{\varepsilon_n}\left(\frac{\pt_n (\phi \wedge \altpsi{n})}{\pt_n(\altpsi{n}) }  < 1 + \varepsilon_n\right) \cdot \World\left(  \phi \wedge \altpsi{n} - 
 \frac{\pt_n (\phi \wedge \altpsi{n})}{\pt_n(\altpsi{n}) } 
 \cdot \altpsi{n}
 \right).
  \end{align*}
 If $\World(\altpsi{n}) = 0$ then this quantity is 0. If $\World(\altpsi{n}) = 1$, 
then if we subtract off the value of $\trade^\symb_n$'s holdings in $\phi$ from trading against $\altMP$, we have:
\begin{align*}
  &  \trade^\symb_n[\phi](\altMP)\cdot \ctsind{\varepsilon_n}\left(\frac{\pt_n (\phi \wedge \altpsi{n})}{\pt_n(\altpsi{n}) }  < 1 + \varepsilon_n\right) \cdot \World\left(  \phi \wedge \altpsi{n} - 
 \frac{\pt_n (\phi \wedge \altpsi{n})}{\pt_n(\altpsi{n}) } 
 \cdot \altpsi{n}
 \right)\\
 - & \;\trade^\symb_n[\phi](\altMP)\cdot  \left(\World\left(  \phi  \right) 
 -   \pt_n (\phi \mid \altpsi{n})
 \right)\\
 &=  \trade^\symb_n[\phi](\altMP)\cdot \ctsind{\varepsilon_n}\left(\frac{\pt_n (\phi \wedge \altpsi{n})}{\pt_n(\altpsi{n}) }  < 1 + \varepsilon_n\right) \cdot \left( \World( \phi ) - 
 \frac{\pt_n (\phi \wedge \altpsi{n})}{\pt_n(\altpsi{n}) } 
 \right)\\
 - & \;\trade^\symb_n[\phi](\altMP)\cdot  \left(\World\left(  \phi  \right) 
 -  1
 \right)\\
 \intertext{by the assumption that $\World(\altpsi{n}) = 1$, and since $\pt_n (\phi \mid \altpsi{n})= 1$  by the assumption that $\pt_n(\phi \wedge \altpsi{n})/  \pt_n(\altpsi{n})  >  \pt_n (\phi \mid \altpsi{n})$;}
 =&\; \trade^\symb_n[\phi](\altMP)\cdot \bigg(    
 \left(\ctsind{\varepsilon_n}\left(\frac{\pt_n (\phi \wedge \altpsi{n})}{\pt_n(\altpsi{n}) }  < 1 + \varepsilon_n\right)
 -1 \right) \cdot \World(\phi) \\
 +&\; 1 - \ctsind{\varepsilon_n}\left(\frac{\pt_n (\phi \wedge \altpsi{n})}{\pt_n(\altpsi{n}) }  < 1 + \varepsilon_n\right) \cdot \frac{\pt_n (\phi \wedge \altpsi{n})}{\pt_n(\altpsi{n}) }
 \bigg)\\
 \intertext{by rearranging;}
 \geq  &\; \trade^\symb_n[\phi](\altMP)\cdot \left(    
 \ctsind{\varepsilon_n}\left(\frac{\pt_n (\phi \wedge \altpsi{n})}{\pt_n(\altpsi{n}) }  < 1 + \varepsilon_n\right)
  \left(1- \frac{\pt_n (\phi \wedge \altpsi{n})}{\pt_n(\altpsi{n})}\right)  \right)
\\
 \intertext{since $\ctsind{\varepsilon_n}\leq 1,$ so in the worst case $\World(\phi) =1$;}
 \geq  &\; \trade^\symb_n[\phi](\altMP)\cdot \left(-  \varepsilon_n  \right),
\\
  \end{align*}
by definition of $\ctsind{\varepsilon_n}$.

\textit{Combining the cases.} Now summing over all $\phi$,  for any world $\World$ such that $\World(\altpsi{n})=1$, we have:

\begin{align*}
&\World\left( \trade_n (\altMP)\right) - \World\left( \trade^\symb_n(\altMP) \right)\\
& = \sum_{\phi} \left( \alpha_n(\altMP) \cdot (\cbuy_n(\phi)(\altMP) - (\cbuy_\nn(\phi))^{* n}(\altMP) ) \right) - \trade^\symb_n[\phi](\altMP)\left(\phi -\altMP(\phi) \right) \\
& \geq \sum_{\phi}\trade^\symb_n[\phi](\altMP)\cdot \left(- \varepsilon_n  \right)\\
\intertext{since for each $\phi$ the corresponding inequality holds by the above analyses;}
& \geq -\magnit{\trade_n^\symb} \cdot \frac{2^{-n}}{\max(1,\magnit{\trade_n^\symb})}\\
\intertext{by definition of $\magnit{\trade_n^\symb}$ and of $\varepsilon_n$;}
& \geq -2^{-n}.
\end{align*}

 In particular,  for any world $\World \in \pcworlds(\altdt_n)$ plausible at time $n$ according to  $\altDP$, 
 \[ \World\left( {\textstyle \sum_{i \leq n} \trade_i(\MP) } \right) \geq  \World\left( {\textstyle \sum_{i \leq n} \trade^\symb_i(\altMP) } \right) -1. \]
Since $\Trader^\symb$ exploits $\altMP$ over $\altDP$, by definition the set
\[ \left\{    \World\left( {\textstyle \sum_{i \leq n} \trade^\symb_i(\altMP) } \right)  \mid   n \in \NN^+,   \World \in \pcworlds(\altdt_n) \right\} \]
is bounded below and unbounded above. Therefore the set 
\[ \left\{    \World\left( {\textstyle \sum_{i \leq n} \trade_i(\MP) } \right)  \mid   n \in \NN^+,   \World \in \pcworlds(\dt_n) \right\} \]
is unbounded above, since for all $n$ we have 
$\altdt_\nn \supseteq \dt_\nn$ and hence $ \pcworlds(\altdt_n) \subseteq \pcworlds(\dt_n) $. 

It remains to show that this set is unbounded below. Suppose for contradiction that it is not, so 
there is some infinite sequence $\{(\World_i, n_i)\}$ with $\World_i\in\pcworlds(\dt_{n_i})$ on which the value $\World_i \left( {\textstyle \sum_{j \leq n_i}
\trade_j(\MP) } \right)  $   of $\Trader$ is unbounded below.

We may assume without loss of generality that each $\World_i$ is inconsistent with $\altDP$. Indeed, if there is no subsequence with this property and with the
values of $\Trader$ unbounded below, then  the 
$\World_i$ consistent with $\altDP$ have the corresponding values $\World_i \left( {\textstyle \sum_{j \leq n_i}
\trade_j(\MP) } \right)  \geq \World_i \left( {\textstyle \sum_{j \leq n_i}
\trade^\symb_j(\altMP) } \right)-1 $ unbounded below, contradicting that 
$\Trader^\symb$ exploits $\altMP$ over  $\altDP$. 
Having made this assumption, there is an infinite sequence $m_i$ with $\World_i(\altpsi{m_i-1}) = 1 \wedge \World_i(\altpsi{m_i}) = 0$ for all $i$.

We may further assume without loss of generality that for each $i$, we have $n_i \leq m_i-1$.  Indeed, for any $n \ge m_i$, we have by the above analysis that 
$\World_i \left( { \textstyle  \trade_n (\MP) } \right) \ge 0$; in this case replacing $n_i$ with $m_i-1$  would only decrease  the values
$\World_i ( {\textstyle \sum_{j \leq n_i}
\trade_j(\MP) } ) $, and hence would preserve that this sequence is unbounded below. 

In particular, it is the case that $\altpsi{m_i - 1}$ propositionally
implies $\altpsi{n_i}$. 
Because $\World_i(\altpsi{m_i - 1}) = 1$ and $\World_i\in\pcworlds(\dt_{n_i})$, this implies $\World_i\in\pcworlds(\altdt_{n_i})$, i.e., $\World_i$ was
plausible at time step $n_i$ according to $\altDP$.  
But then we have that the sequence of values $\World_i ( {\textstyle \sum_{j \leq n_i}
\trade_j(\MP) } )  \geq \World_i ( {\textstyle \sum_{j \leq n_i}
\trade^\symb_j(\altMP) } )-1 $ is unbounded below, contradicting that $\Trader^\symb$ exploits $\altMP$ over $\altDP$.

Thus we have shown that, assuming that $\Trader^\symb$ exploits $\altMP$ over $\altDP$, also $\Trader$ exploits $\MP$ over $\DP$.  This contradicts that $\MP$
is a \li{}, so in fact it cannot be that $\Trader^\symb$ exploits $\altMP$; thus $\altMP$ is a \li{} over $\altDP$, as desired.

\end{proof}

\end{document}